%% file: main_thesis.tex
\documentclass[
  widemargins,
  11pt,
  openany,
  doublespacing,
]{ut-thesis}
\usepackage[colorlinks]{hyperref}
\usepackage[backend=bibtex,maxbibnames=99]{biblatex}
\usepackage{graphicx}
\usepackage{booktabs}
\usepackage[hang]{footmisc}
\input{packs}

\input{defs}

\author{Fartash Faghri}
\title{Training Efficiency and Robustness in Deep Learning}
\degree{Doctor of Philosophy}
\department{Computer Science}
\gradyear{2022}
\begin{document}
  \frontmatter
    \maketitle
    \begin{abstract}

\input{abstract}
    \end{abstract}
    \begin{acknowledgements}
\input{ack}
    \end{acknowledgements}
    \tableofcontents
  \mainmatter
    \input{ch-intro}

    \input{ch-background}

    \input{ch-vsepp}

    \input{ch-gvar}

    \input{ch-robust}

    \input{ch-conclusion}

  \appendix
    \chapter{Appendices to the Chapter on Gradient Clustering}
    \input{ch-gvar-tex/appendix}

    \chapter{Appendices to the Chapter on Bridging the Gap Chapter}
    \input{ch-robust-tex/appendix}

  \backmatter
  \printbibliography[heading=bibintoc]
\end{document}

%% file: packs.tex
\usepackage[utf8]{inputenc} %
\usepackage[T1]{fontenc}    %

\usepackage{algorithm}
\usepackage{amsfonts}       %
\usepackage{amsmath}
\usepackage{amssymb}
\usepackage{amsthm}
\usepackage{array}
\usepackage{bm}
\usepackage[capitalise,nameinlink]{cleveref}
\usepackage{caption}
\usepackage{colortbl}
\usepackage{empheq}
\usepackage{enumitem}
\usepackage{etoolbox}
\usepackage{float}
\usepackage{makecell}
\usepackage{mathrsfs}
\usepackage{mdframed}
\usepackage{microtype}      %
\usepackage{multirow}
\usepackage{nccmath}
\usepackage{nicefrac}       %
\usepackage[noend]{algorithmic}
\usepackage{pifont}
\usepackage{ragged2e}
\usepackage{rotating}
\usepackage{subcaption}
\usepackage{tabularx}
\usepackage{tikz}
\usepackage{times}
\usepackage{url}            %
\usepackage{wrapfig}
\usepackage{xcolor, color, colortbl}
\usepackage{xspace}
\usepackage{thm-restate}

\newcommand{\twocolfigwidth}{0.47\linewidth}
\newcommand{\threecolfigwidth}{0.32\textwidth}

\newcommand{\citep}{\parencite}
\newcommand{\citet}{\cite}

\newcounter{rowcntr}[table]
\renewcommand{\therowcntr}{\arabic{chapter}.\the\numexpr\arabic{table}+1.\arabic{rowcntr}}
\newcommand{\numrow}{\refstepcounter{rowcntr}\therowcntr}
\AtBeginEnvironment{tabular}{\setcounter{rowcntr}{0}}
\newcolumntype{H}{>{\setbox0=\hbox\bgroup}c<{\egroup}@{}}

\crefname{prob}{Problem}{Problems}
\creflabelformat{prob}{#2{\upshape(#1)}#3}

%% file: defs.tex
\def\EE{{\mathbb E}\,}

\renewcommand{\gg}{\boldsymbol{g}}

\newcommand{\vvv}[1]{{{\mathcal{\text{vec}}{{\{#1\}}}}}}

\DeclareMathOperator{\sign}{sign}

\DeclareMathOperator*{\argmin}{{arg\,min}}
\DeclareMathOperator*{\argmax}{{arg\,max}}

\theoremstyle{plain}
\newtheorem{thm}{Theorem}[section]
\newtheorem{lem}[thm]{Lemma}
\newtheorem{prop}[thm]{Proposition}
\newtheorem{cor}{Corollary}

\theoremstyle{definition}

\newtheorem{defn}{Definition}[section] 

\theoremstyle{remark}
\newtheorem*{rem}{Remark}

\newcommand{\E}{{\mathbb{E}}}
\newcommand{\V}{{\mathbb{V}}}
\newcommand{\C}{{\mathbb{C}}}
\newcommand{\R}{{\mathbb{R}}}
\newcommand{\N}{{\mathcal{N}}}

\newcommand{\I}{{\mathbb{I}}}
\newcommand{\U}{{\mathbb{U}}}

\newcommand{\D}{{\mathcal{D}}}

\newcommand{\dxy}[2]{{\frac{\partial {#1}}{\partial #2}}}

\newcommand{\comment}[1]{}

\newcommand{\SUM}{\emph{SH}}
\newcommand{\MAX}{\emph{MH}}

\newcommand{\OC}{\emph{1C}}
\newcommand{\TC}{\emph{10C}}
\newcommand{\RC}{\emph{RC}}
\newcommand{\TCV}{\emph{10C+rV}}
\newcommand{\RCV}{\emph{RC+rV}}
\newcommand{\RV}{\emph{rV}}

\newcommand{\order}{\emph{Order}}
\newcommand{\orderz}{\emph{Order0}}
\newcommand{\orderpp}{\emph{Order++}}
\newcommand{\VSE}{\emph{UVS}}
\newcommand{\VSEpp}{\emph{VSE++}}
\newcommand{\VSEppFt}{\emph{VSE++ (FT)}}
\newcommand{\VSEppRes}{\emph{VSE++ (ResNet)}}
\newcommand{\VSEppResFt}{\emph{VSE++ (ResNet, FT)}}
\newcommand{\VSEz}{\emph{VSE0}}
\newcommand{\VSEzFt}{\emph{VSE0 (FT)}}
\newcommand{\VSEzRes}{\emph{VSE0 (ResNet)}}
\newcommand{\VSEzResFt}{\emph{VSE0 (ResNet, FT)}}

\newcommand{\coco}{MS-COCO}
\newcommand{\fthk}{Flickr30K}

\newcommand{\xxi}{{\bm{x}_i}}
\newcommand{\param}{{\bm{\theta}}}
\newcommand{\lparam}{{\bm{\theta}}}

\newcommand{\Nk}{{N_k}}
\newcommand{\gi}{{\bm{g}_i}}

\newcommand{\gpk}{{\bm{g}^{(k)}}}
\newcommand{\gC}{{\bm{\hat{g}}}}
\newcommand{\Ai}{{a_i}}
\newcommand{\Aa}{{\bm{a}}}
\newcommand{\Ck}{{\bm{C}_k}}
\newcommand{\gb}{{\bm{g}_i}}  %

\newcommand{\Di}{{I}}
\newcommand{\Do}{{O}}
\newcommand{\ck}{{\bm{c}_k}}
\newcommand{\dk}{{\bm{d}_k}}
\newcommand{\Ab}{{\bm{A}_i}}  %
\newcommand{\Db}{{\bm{D}_i}}  %
\newcommand{\AAA}{{\bm{A}}}
\newcommand{\DDD}{{\bm{D}}}

\newcommand{\SK}{{\sum_{k=1}^K}}
\newcommand{\SI}{{\sum_{i=1}^N}}
\newcommand{\SJ}{{\sum_{j=1}^\Nk}}
\newcommand{\SB}{{\sum_{i=1}^N}}  %

\newcommand{\dth}[1]{{\frac{\partial {#1}}{\partial\param}}}

\renewcommand{\algorithmiccomment}[1]{\hfill$\triangleright$ #1}
\newcommand{\bA}{{$\mathcal{A}$} }
\newcommand{\bU}{{$\mathcal{U}$} }
\newcommand{\bAU}{{$\mathcal{AU}$} }

\newcommand{\Gluster}{\mbox{GC}\xspace}
\newcommand{\GC}{\mbox{GC}\xspace}
\newcommand{\SG}{\mbox{SG-B}\xspace}
\newcommand{\SGdB}{\mbox{SG-2B}\xspace}

\def\xx{{\boldsymbol x}}

\def\ee{{\boldsymbol e}}
\def\ddelta{{\boldsymbol \delta}}

\def\kk{{\boldsymbol k}}

\def\rr{{\boldsymbol r}}

\def\qq{{\boldsymbol q}}

\def\KK{{\boldsymbol K}}

\def\yy{{\boldsymbol y}}
\def\vv{{\boldsymbol v}}
\def\uu{{\boldsymbol u}}
\def\ww{{\boldsymbol w}}

\def\AA{{\boldsymbol A}}

\def\bF{{\boldsymbol F}}

\def\Proba{\mathbb{P}}

%% file: abstract.tex
Deep Learning has revolutionized machine learning and artificial intelligence, 
achieving superhuman performance in several standard benchmarks.
It is well-known that deep learning models are inefficient to train; they learn 
by processing millions of training data multiple times and require powerful 
computational resources to process large batches of data
in parallel at the same time rather than sequentially.  Deep learning models 
also have unexpected failure modes; they can be fooled into misbehaviour, 
producing unexpectedly incorrect predictions.

In this thesis, we study approaches to improve the training efficiency and 
robustness of deep learning models. In the context of learning visual-semantic 
embeddings, we find that prioritizing learning on more informative training 
data increases convergence speed and improves generalization performance on 
test data. We formalize a simple trick called hard negative mining as 
a modification to the learning objective function with no computational 
overhead.
Next, we seek improvements to optimization speed in general-purpose 
optimization methods in deep learning.
We show that a redundancy-aware modification to the sampling of training data 
improves the training speed and develops an efficient method for detecting the 
diversity of training signal, namely, gradient clustering.
Finally, we study adversarial robustness in deep learning and approaches to 
achieve maximal adversarial robustness without training with additional data.  
For linear models, we prove guaranteed maximal robustness achieved only by 
appropriate choice of the optimizer, regularization, or architecture.

%% file: ack.tex
Throughout the years, David Fleet was the best supervisor I could ever hope to 
have.
He provided a calm environment without typical pressures, encouraged research 
wanderings, and valued my approaches.
His responses to strange ideas were either constructive and kind feedback or 
a surprisingly deep and thought-provoking question.
As an absolute ethical role model, David always promptly responded with 
intelligent resolutions.
His generous support went beyond financial to caring for my physical and mental 
wellness.
It has been an honour to have David as my sage mentor.

I have been blessed to have the support and company of my cheerful, 
considerate, kindhearted, encouraging, and simply amazing wife, Sara Sabour.

I have been fortunate enough to receive lasting influences from numerous 
mentors.
I highly appreciate the guidance, encouragement, and constructive feedback of 
my supervisory committee,
David Duvenaud,
Roger Grosse,
Graham Taylor, and
Chris Maddison.
I am also grateful for the wisdom, unique perspective, and pivotal advice of
Ian Goodfellow,
Nicolas Le Roux,
Fabian Pedregosa,
Jimmy Ba,
Daniel Roy,
Sanja Fidler,
Foteini Agrafioti,
Alexey Kurakin,
Nicholas Carlini,
Mohammad Norouzi,
Mehrdad Farajtabar,
Mohammad Amin Sadeghi,
Ehsan Fazl Ersi,
Aideen NasiriShargh,
Afshin Nikzad,
Kian Mirjalali, and
Vahid Liaghat.

I appreciate the enjoyable, fun and fruitful collaborations with my coauthors
Yanshuai Cao,
Ali Ramezani-Kebrya,
Iman Tabrizian,
Avery Ma,
Justin Gilmer,
Nicolas Papernot,
Jamie Kiros, and
Sven Gowal.

Thanks to my friends, lab-mates and coworkers for making my graduate studies 
a pleasurable journey. Among many others at the University of Toronto and the 
Vector Institute, I would like to thank
Alireza Shekaramiz,
Saeed Seddighin,
Jake Snell,
Nona Naderi,
Kevin Swersky,
Eleni Triantafillou,
Micha Livne,
Ladislav Rampasek,
Alimohammad Rabbani,
Bahareh Mostafazadeh,
Sepideh Mahabadi,
Kaveh Ghasemloo,
Sajad Norouzi,
Aryan Arbabi,
Rahmtin Rotabi,
Farzaneh Derakhshan,
Alireza Makhzani,
Farzaneh Mahdisoltani,
Taylor Killian,
Ehsan Amiri, and
Sobhan Foroughi.

I would like to thank the staff at the University of Toronto, the Department of 
Computer Science, and the Vector Institute for their support.  A special thanks 
to Relu Patrascu for providing technical support throughout the years.

I am grateful to my family
Laleh Kordavani,
Hassan Faghri,
Faraz Faghri,
and Ali Faghri,
for their foundational role in my growth and development.

I am grateful to various organizations for financial support.  These include 
the Department of Computer Science and the University of Toronto, and the 
Ontario Graduate Scholarship (OGS).  Resources used in my research were also 
provided, in part, by the Province of Ontario, the Government of Canada through 
CIFAR, and companies sponsoring the Vector 
Institute.\footnote{\url{www.vectorinstitute.ai/\#partners}}  In addition, part 
of the research in this thesis was initiated during an internship with the 
Brain Team in Google Research.

%% file: ch-intro.tex
\chapter{Introduction}
\label{ch:intro}

Deep Learning has revolutionized machine learning and artificial intelligence, 
achieving superhuman performance in several standard benchmarks, including 
image classification~\citep{he2016deep}, the game of 
Go~\citep{silver2017mastering}, and natural language 
translation~\citep{popel2020transforming}. Deep neural networks have become 
major components of emerging technologies such as autonomous 
driving~\citep{grigorescu2020survey}, robotics~\citep{sunderhauf2018limits}, 
and drug discovery~\citep{chen2018rise}.

It is also well-known that deep learning models are inefficient to 
train~\citep{strubell2019energy}; they learn by processing millions of training 
data multiple 
times~\citep{he2016deep,silver2017mastering,popel2020transforming}.
They require powerful computational resources to process large batches of data
in parallel at the same time rather than 
sequentially~\citep{radford2021learning,ramesh2021zeroshot,brown2020language}.  
Vast amounts of image and text data gathered in the past few decades, and 
technological advancements in processing hardware, allowed breakthroughs in 
deep learning 
applications~\citep{deng2009imagenet,bojar2014findings,lin2014microsoft}.  
However, this recipe
is challenging to scale. For example, significant advances in language 
modelling require vast amounts of data and have yet to achieve superhuman 
performance on some datasets~\citep{brown2020language}.

Deep learning models also have unexpected failure modes; they can be fooled 
into misbehaviour, producing incorrect 
predictions~\citep{szegedy2013intriguing,goodfellow2014explaining}. Adversarial 
perturbations are small modifications to the input that change the output of 
a model.  Adversarial perturbations against deep learning image models can be 
designed to be sufficiently subtle that they become imperceptible to humans.
Adversarial samples are particularly troubling in applications with significant 
security and privacy 
requirements~\citep{kurakin2016adversarial,kindermans2019reliability}.
Moreover, adversarial samples are not only security flaws but an intriguing 
example of our limitation in understanding deep learning 
models~\citep{ghorbani2019interpretation}.  The lack of explanation for 
adversarial samples limits our trust and confidence in the decisions and 
predictions of machine learning models.
Understanding and mitigating the impact of adversarial samples provides ways to 
improve robustness and efficiency of models~\citep{salman2020adversarially}.

This thesis comprises three parts, each of which explores aspects of efficiency
and robustness in deep learning from different perspectives.

The first problem concerns training efficiency in computer vision. We 
hypothesize that training data are not equally important;
the available training signal from some inputs is stronger than others.  
Moreover, we hypothesize that prioritizing learning on more informative 
training data increases convergence speed and improves generalization 
performance on test data. In \cref{ch:vsepp} (originally published in 
\citet{faghri2018vse++}) we illustrate how a simple trick referred to as hard 
negative mining speeds up training and improves test performance.
Although previously used in ad-hoc ways, we formalize hard negative mining as 
a modification to the learning objective function.  We demonstrate these 
improvements in learning joint embeddings that is a standard benchmark for 
representation learning.

Joint embeddings enable a wide range of tasks in image, video and language 
understanding. Examples include shape-image embeddings~\citep{li2015joint} for 
shape inference, bilingual word embeddings~\citep{zou2013bilingual}, human 
pose-image embeddings for 3D pose inference~\citep{li2015maximum}, fine-grained 
recognition~\citep{reed2016learning}, zero-shot 
learning~\citep{frome2013devise}, and modality conversion via 
synthesis~\citep{reed2016generative,reed2016learning}.  Such embeddings entail 
mappings from two (or more) domains into a common vector space in which 
semantically associated inputs (e.g., text and images) are mapped to similar 
locations.  The embedding space thus represents the underlying domain 
structure, where location and often direction are semantically meaningful.

Visual-semantic embeddings have been central to
image-caption retrieval and 
generation~\citep{kiros2014unifying,karpathy2015deep}, visual 
question-answering~\citep{Malinowski15}, and more recently in large-scale 
multi-modal representation learning such as CLIP~\citep{radford2021learning}.  
One approach to visual question-answering, for example, is to first describe an 
image by a set of captions, and then to find the nearest caption in response to 
a question~\citep{agrawal2017vqa,zitnick2016measuring}.  For image synthesis 
from text, one could map from text to the joint embedding space, and then
back to image space~\citep{reed2016generative,reed2016learning}.

It is common to use a contrastive loss function to learn representations with 
a meaningful distance metric that resembles a conceptual contrast in the input 
space~\citep{hadsell2006dimensionality}. For example, we would use 
a contrastive loss to learn representations of images such that the 
representation of a dog is closer to other dogs relative to any cat. In 
particular, triplet ranking losses consist of terms that contrast a positive 
similar pair with a dissimilar negative sample. A non-zero loss is incurred if 
the representations of the similar pair are farther from each other than they 
are individually from the negative sample. A triplet loss requires 
a ground-truth collection of matching pairs but the negative sample often 
requires no annotation as it can be any non-matching input from the database.  
Triplet losses are defined given a distance metric (or more generally 
a similarity score) and differ in how the loss scales with the magnitude of the 
distance violation.

Our main contribution is to incorporate hard negatives in the loss function 
with no computational overhead. A hard negative for any positive pair is 
defined as the training sample incurring the maximum triplet loss. Finding the 
hardest negative is computationally expensive. Instead, we use semi-hard 
negatives that are the hardest contrastive triplets in each mini-batch. This 
simple change has no additional computational cost while significantly 
improving the training performance.

Various works have extended our work and applied semi-hard negative mining in 
visual-semantic embeddings and related applications.  Examples are 
visual-semantic embedding with generalized pooling 
operator~\citep{chen2021learning}, Consensus-aware visual-semantic 
embedding~\citep{wang2020consensus}, and top ranking methods in video retrieval 
competition tasks~\citep{awad2020trecvid} for  ad-hoc video search and 
video-to-text description generation. These examples illustrate the 
effectiveness of our simple modification to the loss function in computer 
vision applications.

The second problem addressed in this thesis is training efficiency as a general 
challenge in training deep neural networks. Instead of focusing on 
a task-specific loss function, we seek improvements to optimization speed in 
gradient-based optimization methods for deep learning.
We revisit our hypothesis from the first problem that training data are not all 
equally important. We hypothesize that training deep learning models on diverse 
and heterogeneous data distributions should be slower than homogeneous data 
distributions with less diversity. In \cref{ch:gvar} and the publication 
\citep{faghri2020study}, we develop a method for detecting the diversity of 
training signal and propose it as a method for improving training speed. Our 
method can generally be used with any gradient-based optimization method. In 
contrast to \cref{ch:vsepp}, very few general-purpose ideas exist that 
consistently increase training efficiency.

We motivate our hypothesis with a simple example for which there exist 
duplicate data points. To decrease the processing time per iteration, it is 
common to uniformly sample mini-batches of data during training.  As 
illustrated in the following example, uniform sampling is not always ideal.  
Suppose there are duplicate data points in a training set.  We can save 
computation time by removing all but one of the duplicates.  To get the same 
gradient update, in expectation, it is sufficient to rescale the gradient of 
the remaining samples in proportion to the number of duplicates. In this 
example, standard optimization methods for deep learning will be inefficient 
because with uniform sampling there is a chance of sampling mini-batches 
containing duplicates of one data type.

In general, redundancy can be extended from exact duplicates to points that are 
very similar to each other according to a similarity metric.  In 
\cref{ch:gvar}, we focus on the similarity in the gradient space and establish 
the connection between gradient redundancy and training speed through 
a connection to gradient variance reduction.  Prior works have theoretically 
established that smaller gradient variance results in faster training.
Through an illustrative example, we show that a redundancy-aware modification 
to the sampling of training data reduces the gradient variance.  We prove that 
the gradient variance is minimized if the elements are sampled from a weighted 
clustering in gradient space. Accordingly, we propose a gradient clustering 
method as a computationally efficient method for clustering in the gradient 
space.

Although outside the scope of this thesis, we also note that we have used our 
observations about the gradient distribution during training in designing 
gradient quantization methods to speed up data-parallel optimization for deep 
learning~\citep{faghri2020adaptive, ramezani2021nuqsgd}.

The third and last problem studied in this thesis concerns adversarial 
robustness in deep learning. We hypothesize that deep learning models are not 
inherently weak; their robustness depends on the implicit and explicit training 
mechanisms. In \cref{ch:robust}, we show that these mechanisms imply maximal 
robustness to the types of perturbations not commonly considered in adversarial 
attacks.  This observation allows us to find alternative approaches to 
adversarial robustness that introduce no additional training overhead compared 
to standard robustness methods. Hence, increasing training efficiency in the 
task of adversarial robustness. Next, we introduce adversarial robustness and 
standard robustness methods in more detail.

Deep neural networks achieve high accuracy on standard test sets.  
Nevertheless, \citet{szegedy2013intriguing} showed that any natural input 
correctly classified by a neural network can be modified with adversarial 
perturbations that fool the network into misclassification. They observed that 
adversarial perturbations exist even when they are constrained to be small 
enough that they do not significantly affect human perception.  Adversarially 
perturbed inputs are also commonly referred to as Adversarial samples or 
adversarial examples.

Adversarial training~\citep{goodfellow2014explaining}
was proposed to improve the robustness of models through
training on adversarial samples
and further generalized
as a saddle-point optimization problem~\citep{madry2017towards}.
In practice, adversarial training refers to methods that solve the saddle-point 
problem approximately by minimizing the loss on adversarial samples.

The state-of-the-art approach to increase adversarial robustness is adversarial 
training, i.e., empirical risk minimization on adversarial samples.  
Adversarial training is another example of inefficient training;
i.e., in addition to training on large datasets, for robustness, we need to 
create and train on artificial adversarial samples to improve models.  
Adversarial training also exhibits a trade-off between standard generalization 
and adversarial robustness.  That is, it achieves improved {\it robust 
accuracy}, on
adversarially perturbed data, at the expense of {\it standard accuracy}, the
probability of correct predictions on natural data 
\citep{tsipras2018robustness}.

In \cref{ch:robust}, we discuss an alternative formulation to adversarial 
robustness, namely, maximally robust classification, with no robustness 
trade-off. By connecting this problem to the literature on the implicit bias of 
optimization methods\citep{gunasekar2018characterizing}, we also find maximally 
robust classifiers at no additional computational cost. This link allows for 
the bidirectional transfer of results between the two literatures on 
adversarial robustness and optimization bias.

Finally, we conclude in \cref{ch:conclusion} with a summary of the main 
results, and a discussion of future directions.

\section{Publications}

The contents of this thesis have largely been taken from the following 
publications and technical reports:
\begin{itemize}
    \item {\bf Faghri, Fartash} and Fleet, David J.\ and Kiros, Jamie R.\ and 
        Fidler, Sanja, {\sl ``VSE++: Improving Visual-Semantic Embeddings with 
        Hard Negatives"}, British Machine Vision Conference (BMVC), 2018.  
        (\cref{ch:vsepp})
    \item {\bf Faghri, Fartash} and Duvenaud, David and Fleet, David J.\ and 
        Ba, Jimmy,
        {\sl ``A Study of Gradient Variance in Deep Learning"}, arXiv, 
        2020. (\cref{ch:gvar})
    \item {\bf Faghri, Fartash} and Duvenaud, David and Fleet, David J.\ and 
        Ba, Jimmy,
        {\sl ``Gluster: Variance Reduced Mini-Batch SGD with Gradient Clustering"}, 
        Conference on Neural Information Processing Systems (NeurIPS), Workshop 
        on Beyond First Order Methods in ML, 2019. (\cref{ch:gvar})
    \item {\bf Faghri, Fartash} and Gowal, Sven, Vasconcelos, Cristina and 
        Fleet, David J.\ and Pedregosa, Fabian and Le Roux, Nicolas, {\sl 
        ``Bridging the Gap Between Adversarial Robustness and Optimization 
        Bias"}, Workshop on Security and Safety in Machine Learning Systems, 
        International Conference on Learning Representations (ICLR), 2021.  
        (\cref{ch:robust})
\end{itemize}

During my PhD studies, beyond the research in this thesis, I have authored and 
contributed to the following papers. We do not discuss the details of the 
following papers in this thesis.
\begin{itemize}
    \item {\bf Faghri, Fartash}$^\ast$~\footnote{$^\ast$ joint first 
        co-authors}, Tabrizian, Iman$^\ast$ and Markov, Ilia and Alistarh, Dan 
        and Roy, Dan M.\ and
        Ramezani-Kebrya, Ali,
        {\sl ``Adaptive Gradient Quantization for Data-Parallel SGD"}, Conference 
        on Neural Information Processing Systems (NeurIPS), 2020.
    \item Ramezani-Kebrya, Ali and {\bf Faghri, Fartash} and
        Markov, Ilya and Aksenov, Vitalii and Alistarh, Dan  and Roy, Dan M.\ 
        et al.,
        {\sl ``NUQSGD: Provably Communication-efficient Data-parallel SGD via 
        Nonuniform Quantization"}, Journal of Machine Learning Research 
        22.114 (2021): 1-43..
    \item Ma, Avery and {\bf Faghri, Fartash} and Papernot, Nicolas and 
        Farahmand, Amir-massoud, {\sl ``SOAR: Second-Order Adversarial 
        Regularization"}, ArXiv, 2020.
    \item Zhang, Qingru and Wu, Yuhuai and {\bf Faghri, Fartash} and Zhang, 
        Tianzong and Ba, Jimmy,
        {\sl ``A Non-asymptotic comparison of SVRG and SGD: tradeoffs between 
        compute and speed"}, ArXiv, 2020.
    \item Gilmer, Justin and Metz, Luke and {\bf Faghri, Fartash} and 
        Schoenholz, Samuel S.\ and Raghu, Maithra and Wattenberg, Martin and 
        Goodfellow, Ian, {\sl ``Adversarial Spheres"}, ICLR Workshop, 2018.
    \item Sabour, Sara and Cao, Yanshuai and {\bf Faghri, Fartash} and Fleet David 
        J.\, {\sl ``Adversarial Manipulation of Deep Representations"}, 
        International Conference on Learning Representations (ICLR), 2016.
\end{itemize}

%% file: ch-background.tex
\chapter{Background}
\label{ch:background}

In this section, we will cover the general background for the thesis and review 
common standard practices for optimization in deep learning.  Further related 
work that is specific to individual chapters is developed in subsequent 
chapters.  We refer the reader to \citet{goodfellow2016deep} for a general 
introduction to deep learning.

\section{Optimization for Deep Learning}
Minimizing an objective is at the core of deep learning.
In this section, we focus on methods and ideas for solving or simplifying 
optimization problems in deep learning.

Many machine learning tasks entail the minimization of the risk,
$\E_{(\xx, y) \sim D}[\ell(\xx,y; \param)]$,
where the input $\xx, y$ are random variables distributed according to the data 
distribution $D$, and $\ell$ is the per-example loss function parametrized by 
$\param$, the parameters of the neural network.  In supervised learning, $y$ 
denotes the ground-truth label. Empirical risk approximates the population risk 
by the risk of i.i.d.\ samples $\{(\xxi, y_i)\}_{i=1}^N$, the training set, as 
$L(\param)=\SI \ell(\xxi, y_i; \param)/N$.
For differentiable loss functions, the gradient of $\xx_i$ is defined as 
$\dth{}\ell(\xx_i, y_i; \param)$, i.e., the gradient of the loss with respect 
to the parameters evaluated at a point $\xx_i$~\citep{vapnik1999overview}.

Optimization algorithms used for deep learning are predominantly 
iterative~\citep{bottou2016optimization}.  An iterative optimizer starts with 
an initial estimate of the parameters and improves the estimate 
iteratively~\citep{nocedal2006numerical}.  Each iteration is often broken into 
two parts, i.e., choosing a direction and choosing the step size to move along 
this direction.

Optimization methods are mainly evaluated based on their rate of convergence to 
a local or global minimum of the loss function. Such evaluation is either by 
theoretically proving convergence rates on specific families of functions 
(e.g., convex functions)~\citep{bottou2016optimization} or empirically, by 
training a model on benchmark datasets (e.g., a convolutional network on the 
MNIST dataset~\citep{lecun1998gradient}). The rate of convergence is either 
measured by the number of iterations or the total wall-clock time in 
non-distributed settings. \citet{choi2019empirical} suggests guidelines for 
empirical evaluation of optimization methods in deep learning.

Generalization, the ability to perform well on unseen data, is one of the key 
elements of machine learning~\citep{bishop2007pattern}.  Empirically, 
generalization is measured by evaluating a trained model on a held out test 
set.  It is known that the choice of optimizer affects generalization 
performance~\citep{gunasekar2018characterizing,vaswani2020each,amari2020does}).

The scope of this section is restricted to optimization methods popular in deep 
learning.  We focus on continuous parameters and assume that the gradient of 
the loss is defined and exists almost everywhere.  Optimizing neural networks 
is often an unconstrained problem but regularization methods are commonly used.  
The functions minimized in deep learning are often non-convex with multiple 
local minima.  Nevertheless, we focus on local optimization methods as they are 
more popular.  Bayesian learning~\citep{neal1995bayesian} and evolution 
strategies~\citep{salimans2017evolution} are examples of non-local methods not 
covered in this chapter.

Within the optimization framework, \cref{ch:vsepp} can be described as 
a modification to an optimization problem to better match the task's objective, 
\cref{ch:gvar} is an improvement to solving common optimization problems in 
deep learning, and \cref{ch:robust} characterizes the implicit bias of 
optimization on adversarial metrics.

\subsection{Stochastic Gradient Descent}
Gradient Descent (GD) or Steepest Descent~\citep{nocedal2006numerical}, is an 
iterative optimization method that assumes access to the first-order derivative 
of the objective or loss function, e.g., $\ell(\xx, y; \param)$ defined above.  
We use GD to optimize a neural network by evaluating the gradient over the full 
training set.  The update rule for GD is
\begin{align}
    \param^t &= \param^{t-1} - \alpha_t \sum_{i\in S} \left.  \dth{}\ell(\xxi, 
    y_i; \param)\right|_{\param=\param^{t-1}},
    \label{eq:gd}
\end{align}
where $\ell$ is the loss function parameterized at the $t$-th iteration of 
optimization by $\param^t$ on the training data $\xx_i$. $\alpha_t$ is the 
optimization step size, or learning rate, at iteration $t$.

GD requires the computation of gradients on the entire training set in each 
step while the alternative, Stochastic Gradient Descent (SGD), uses the 
gradient of a uniformly sampled data point from the training set. The gradient 
from a single sample is an unbiased estimate of the average gradient on the 
training set such that they are equivalent in expectation. Nevertheless, there 
exists a trade-off between the performance of GD versus SGD (discussed in 
\cref{ch:gvar}).  In mini-batch Stochastic Gradient Descent (mini-batch SGD) we 
can control the trade-off between GD and SGD by instead sampling $B$ training 
samples, which comprise a mini-batch in each step.  The mini-batch size is 
a hyper-parameter that provides flexibility in trading per-step computation time 
for potentially fewer total steps.
The update rule for mini-batch SGD is
\begin{align}
    \param^t &= \param^{t-1} - \alpha_t \sum_{i\in B_t} \left.  
    \dth{}\ell(\xxi, y_i; \param)\right|_{\param=\param^{t-1}},
    \label{eq:mbsgd}
\end{align}
where $B_t$ is the mini-batch of training samples at iteration $t$. In GD the 
mini-batch is the entire training set while in SGD it is a single sample.

GD and SGD are guaranteed to converge under specific 
assumptions~\citep{bottou2016optimization,robbins1951stochastic}.
Theoretically the convergence rates for GD/SGD are known under certain 
conditions~\citep{moulines2011non}.
The effect of the mini-batch size in deep learning has been studied 
extensively~\citep{shallue2018measuring,zhang2019algorithmic}.

To use GD/SGD to train deep neural networks, we usually need to compute the 
gradient of the loss function with respect to the network parameters.  
Back-Propagation~\citep{rumelhart1986learning} is one way of numerically 
computing these gradients. In feed-forward networks, layers are composed such 
that the inputs of a layer are the outputs of lower layers or the inputs of the 
model.  Back-propagation (back-prop) is the application of the chain rule to 
this decomposition.

Although GD/SGD are popular optimizers in deep learning, they have known 
limitations.  In \citet{sutton1986two}, two problems are identified, namely, 
the elongated ravine problem and the unlearning problem.  An elongated ravine 
refers to a special shape of the loss function in the space of the model 
parameters.  There are two directions in an elongated ravine, the principal 
direction of the ravine and the perpendicular steep direction of the ravine's 
wall. The optimization step size has to be small enough to prevent divergence 
in the steepest direction while the loss decreases along the principal 
direction.  The convergence rate is determined by the ratio of the two slopes.

The unlearning problem, more recently known as catastrophic forgetting, is when 
the optimizer ruins well-optimized parameters to learn new concepts, even 
though the model has underemployed parameters.  Catastrophic forgetting is 
a challenge for training models in the setting of continual 
learning~\citep{lopez2017gradient}.  In continual learning, there exist 
multiple tasks and we only get the training data of a task after we are done 
with learning with previous tasks. A model should perform well on all tasks 
seen so far.  In this setting, current optimizers quickly lose the performance 
on previous tasks, when they are trained only with the data of a new task.

\subsection{Momentum}
\label{sec:intro:mom}
The most commonly used variant of SGD in deep learning incorporates Polyak's 
momentum~\citep{polyak1964some} (also known as the heavy ball momentum)
\begin{align}
    \vv^t &= \beta\vv^{t-1} +
    \sum_{i\in B_t} \left.  \dth{}\ell(\xxi, y_i; 
    \param)\right|_{\param=\param^{t-1}}\\
    \param^t &= \param^{t-1} - \alpha_t \vv^t,
\end{align}
where $\beta\in[0, 1)$ is the momentum coefficient and $\vv$ is a velocity 
vector.  Empirically, SGD with momentum is often significantly faster than 
SGD~\citet{shallue2018measuring}.  Polyak's momentum has the same convergence 
rate as GD with an improved constant.
\citet{nesterov1983method} proposed an acceleration method with an improved 
convergence rate on convex and continuously differentiable functions with 
a Lipschitz continuous gradient~\citep{bottou2016optimization}.  This method 
was later revitalized as Nesterov's momentum by \citet{sutskever2013importance} 
in a similar form to Polyak's momentum, where the gradient estimate for the 
current step is computed not at the current parameter value, but at the 
anticipated next value. In practice as well as certain convex settings, 
momentum increases flexibility in the choice of the step size and mini-batch 
size~\citep{goh2017why,zhang2019algorithmic}.

Polyak's momentum can also be interpreted as a gradient noise reduction method 
that connects it to the ideas discussed in \cref{ch:gvar}.

\subsection{Second Order Methods}
\label{sec:intro:sec}
Second-order methods address the elongated ravine problem by incorporating the 
curvature information into the optimization.
The Hessian, $H(\param) ={\frac{\partial^2 {}}{\partial\param^2}}\ell(\xxi, 
y_i; \param)$,
represents the information about the curvature.  The generalization of Newton's 
method is to replace $\gg$ in the GD update step (\cref{eq:gd}) with 
$H^{-1}\gg$ where $\gg=\dth{}\ell(\xxi, y_i; \param)$.  Given the Hessian, 
Newton's method normalizes the  gradients along the direction of the walls to 
prevent overshooting and allows consistent improvements in all directions.

Computing the inverse of the Hessian is computationally demanding when the 
parameter space is high dimensional.  As such, approximations have been 
suggested. A diagonal approximation is suggested in 
\citep{becker1988improving}.  Hessian-Free optimization~\citep{martens2010deep} 
is a second-order method that neither computes the Hessian, nor any 
approximation to it.  Instead, Conjugate Gradient iterations are used to choose 
the next optimization direction.

Natural gradient descent, e.g., K-FAC~\citep{martens2015optimizing}, methods 
can also be viewed as second-order methods~\citep{amari1998natural}.
Natural gradient~\citep{amari1998natural} is defined as $F^{-1}\gg$, where
\begin{align}
    F(\param) & =\E_{\xx\sim p(\xx)}\E_{\yy\sim p(y|\xx, \param)}[\dth{\log 
    p(y|\xx, \param)} \dth{\log p(y|\xx, \param)}^T],
\end{align}
is the Fisher information matrix where $\xx$ is sampled from the underlying 
data distribution and $y$ is sampled from the distribution of model predictions 
(See \citet{kunstner2019limitations} for common errors in the definition).
The Fisher information matrix captures the uncertainty due to sampling of the 
training set and the sampling of mini-batches.
Under certain conditions, the Fisher information matrix is equivalent to 
a positive semi-definite approximation to the 
Hessian~\citep{martens2012training}.

Second order methods are still not widely used in deep learning, often because 
of implementation challenges. Recent work have demonstrated wall-clock time 
improvements with scalable implementations~\citep{anil2021scalable}.
Besides, second order methods and more generally preconditioned optimization 
exhibit alternative generalization properties~\citep{amari2020does}. Our 
results in \cref{ch:robust} can be extended to preconditioned optimization 
methods to imply alternative robustness guarantees.

\subsection{Adaptive Gradient Methods}
\label{sec:intro:adaptive}
In settings where gradients have known characteristics such as sparsity, 
specialized methods are preferred over SGD with momentum.
ADAGRAD~\citep{duchi2011adaptive} proposes an update direction that works well 
when gradients are sparse.  In this method, each gradient dimension is divided 
by $\sqrt{\sum_{i=1}^T{g_{i,d}^2}}$, where $g_{i,d}$ is the $d$th gradient 
dimension at the $i$th training iteration.  RMSprop~\citep{tieleman2012lecture} 
has a similar update rule that uses an exponentially-weighted average of the 
gradients over time instead of equally-weighted averaging gradients of all 
iterations. Adam~\citep{kingma2014adam} is a popular adaptive gradient method 
that combines Polyak's momentum with RMSprop's adaptive gradient. The update 
rule for Adam divides the velocity vector from Polyak's momentum by the second 
moment of the gradients.

ADAGRAD, RMSprop, and Adam propose dividing the gradient vector by a function 
of its second moment per dimension.  Under certain conditions, these methods 
can be interpreted as approximate second-order methods that use diagonal 
approximations to the Hessian or the Fisher information 
matrix~\citep{bottou2016optimization}.

Numerous optimization methods have been proposed for deep learning, yet Adam 
and SGD with Nesterov momentum still match or exceed the performance of the 
most recent ones with proper hyper-parameter 
tuning~\citep{nado2021large,schmidt2021descending}. We use Adam in 
\cref{ch:vsepp} as the optimization method. Our method in \cref{ch:gvar} can be 
interpreted as an adaptive gradient method with per-example reweighing.  
Adaptive gradient methods have unique generalization properties with robustness 
implications covered in \cref{ch:robust}.

\subsection{Initialization}
For non-convex functions, poor initialization can cause slow convergence or 
convergence to a local minimum significantly worse than the global minimum.  
Practical initializations in deep learning are often based on analysis of the 
distribution of activations and gradients.  For example, the initialization 
should be such that neither gradients nor activations vanish or explode during 
back-propagation~\citep{hochreiter1998vanishing,he2016deep}.

Common initializations in deep learning are often based on sampling from 
zero-mean uniform or Gaussian 
distributions~\citep{glorot2010understanding,he2015delving}. The variance of 
the distribution is determined by the type of activation function and usually 
depends on the number of inputs and outputs of a unit. Orthogonal 
initialization has also been explored for deep linear and non-linear 
networks~\citep{saxe2013exact,pennington2017resurrecting}.

Initialization for deep non-linear models affects the entire training 
trajectory and results in varying properties of a converged 
model~\citep{moroshko2020implicit}. We expect similar results to 
\cref{ch:robust} would show the choice of initialization implicitly affects the 
robustness.

\subsection{Generalization}
In Machine Learning, we seek accurate models that generalize to unseen data 
that is estimated empirically on held out test data.
For performance on test data to be an accurate estimate of generalization, test 
data should not be reused. Model selection and hyper-parameter tuning require 
an estimate of the test performance that is often estimated on a validation set 
that is held out from the training set.
A generalization gap between training and test performance can be due to model 
misspecification, data limitations, or loss function 
shortcomings~\citep{bishop2007pattern}. The optimizer can also affect the 
generalization gap if the choice of the model and loss function form an 
underspecified optimization problem~\citep{gunasekar2018characterizing}.

Modifications to a model, data, or loss function that improve the 
generalization gap can be problem/domain specific.  For example, we often seek 
to learn computer vision models invariant or equivariant to affine image 
transformations. Using convolutional layers instead of fully-connected layers, 
we achieve translation invariance/equivariance at no additional computational 
cost but with reduced model flexibility. A more flexible but computationally 
expensive approach is data augmentation that trains on transformations of the 
training data.  Alternatively, the loss function can be changed to penalize 
disagreements between the model outputs on transformations of a single input, 
providing both flexibility and potentially lower computational cost. In 
\cref{ch:vsepp}, we propose a similar loss function alternative that results in 
both faster training and better generalization. 

The optimization problems in deep learning are often underspecified, i.e., the 
model is complex enough to fit the training data accurately in more than one 
way. As such, there are generic modifications to the model, loss function, and 
optimization that restrict the model complexity. 
Dropout~\citep{srivastava2014dropout} applied to a layer deactivates units 
randomly and forces the model to learn a compact representation but with some 
redundancy.
Regularization penalizes the model by adding additional penalty terms to the 
loss function.
Norm penalties such as $\ell_0$, $\ell_1$, and $\ell_2$ encourage sparsity or 
parameter shrinkage and can be interpreted as assuming a prior distribution 
over parameters~\citep{murphy2014machine}.
Early stopping addresses overfitting by selecting the best performing model 
during training according to the validation performance.  A model is 
overfitting to the training data if the test performance worsens while the 
training performance improves.

A trade-off exists between model complexity and generalization that can be 
characterized by a bias-variance decomposition of the risk in some 
problems~\citep{bishop2007pattern}.  The relation between complexity and 
generalization can be a double-descent curve where the generalization 
performance of a family of models is worst at a critical model size while 
improves monotonically for more or less complex 
models~\citep{hastie2019surprises}.

In \cref{ch:robust}, we discuss how common modifications for improving 
generalization such as model architecture, regularization, and optimization 
method indirectly affect model robustness.

\subsection{Implicit Bias of Optimization Methods}
\label{sec:background:bias}
Minimizing the empirical risk for
an overparametrized model with more parameters than the training
data has multiple solutions.
\citet{zhang2017understanding} observed that overparametrized
deep models can even fit to randomly labeled training data, yet
given the correct labels they consistently generalize to test data.
This behavior has been explained using the implicit bias of
optimization methods towards particular solutions.
\citet{gunasekar2018characterizing} proved that
minimizing the empirical risk using  steepest descent
and mirror descent have an implicit bias towards minimum norm
solutions in overparametrized linear classification. Characterizing
the implicit bias in linear regression proved to be more challenging
and dependent on the initialization.
\citet{ji2018gradient} proved that
training a deep linear classifier using gradient descent
not only implicitly converges to the minimum norm classifier
in the space of the product of parameters,
each layer is also biased towards rank-$1$ matrices aligned with adjacent layers.
\citet{gunasekar2018implicit} proved the implicit bias of
gradient descent in training linear convolutional classifiers
is towards minimum norm solutions in the Fourier domain that
depends on the number of layers.
\citet{ji2020directional} has established the directional alignment
in the training of deep linear networks using gradient flow (gradient descent 
with infinitesimal step size) as well as the implicit bias of training deep 
2-homogeneous networks.
In the case of gradient flow the implicit bias of training multi-layer linear 
models is towards rank-$1$ layers that satisfy directional alignment
with adjacent layers~\citep[Proposition 4.4]{ji2020directional}.
Recently, \citet{yun2020unifying} has proposed a unified framework
for implicit bias of neural networks using tensor formulation 
that includes fully-connected, diagonal, and convolutional
networks and weakened the convergence assumptions.

The recent theory of generalization in deep learning, in particular
the double descent phenomenon, studies the generalization
properties of minimum norm solutions
for finite and noisy training sets~\citep{hastie2019surprises}.
Characterization of the double descent phenomenon relies on the
implicit bias of optimization methods
while using additional assumptions
about the data distribution.
In contrast, our results in \cref{ch:robust} only rely on the implicit bias of 
optimization and hence are independent of the data distribution.

\section{Contrastive Learning}

Contrastive losses~\citep{hadsell2006dimensionality} have been used widely in 
computer vision and machine learning. Triplet losses are widely used in 
retrieval problems where the task is to learn a mapping such that similar 
inputs are mapped close to each other and farther from dissimilar 
inputs~\citep{weinberger2009distance}.
Examples are image retrieval~\citep{frome2007learning,chechik2010large},
learning to rank problems~\citep{li2014learning}, and max-margin structured 
prediction \cite{chapelle2007large,le2007direct}. Let
$f(\xx; \param)$ denote a mapping of input $\xx\in X$ parametrized by $\param$.  
A triplet loss function $\ell(\qq, \kk_{+}, \kk_{-})$ takes a query input, 
$\qq\in X$, together with a similar and dissimilar key inputs, $\kk_+,\kk_-\in 
X$. The loss function encourages the mappings of the query and the similar key 
to be close to each other while discouraging proximity to the dissimilar key.  
A hinge triplet loss is commonly used in image retrieval where the similarity 
is measured by cosine similarity and the model is penalized if the proximity 
between the negative pair is not more than the positive pair by a preset 
margin~\cite{kiros2014unifying, karpathy2015deep, ZhuICCV15, 
socher2014grounded}. The pairwise hinge loss is an alternative in which 
elements of positive pairs
are encouraged to lie within a hypersphere of radius $\rho_1$,
while negative pairs should be no closer than $\rho_2 > \rho_1$.
A pairwise loss restricts the mapping more than a triplet loss and requires two 
hyper-parameters rather than one. In \cref{ch:vsepp}, we propose a modification 
to the hinge triplet loss function. Focal loss is a similar recent work that 
modifies the cross-entropy loss to handle imbalance in classes and training 
data when there are many easy samples but few hard samples such as when there 
are few foreground samples but numerous background 
samples~\citep{lin2017focal}.

Recently, contrastive losses have been successfully employed for unsupervised 
representation learning in methods such as MoCo~\citep{he2020momentum, 
chen2020improved} and SimCLR~\citep{chen2020simple}.
In contrast to ranking and retrieval applications, new applications of 
contrastive losses are in unsupervised representation learning where the 
positive pair is usually the input itself with some perturbation.
As an alternative to hinge losses and others~\citep{he2020momentum}, the 
InfoNCE~\citep{oord2018representation} loss and its variants are based on the 
cross-entropy loss,
\begin{align}
\ell(\qq,\kk_+,\KK_-) &= -\log
\frac{\exp{(\qq\cdot \kk_+/\tau)}}
{\exp{(\qq\cdot \kk_+/\tau)} + \sum_{\kk_-\in\KK_-}
    \exp{(\qq\cdot \kk_-/\tau)}}
\label{eq:co_loss}
\end{align}
where $\qq$ is a query input, $\kk_+$ a similar key and $\KK_-$ a dissimilar 
set of keys to the query. $\tau$ is the temperature hyper-parameter that 
controls the expected separation between similar and dissimilar samples.

An interpretation of $\ell(\qq,\kk_+,\KK_-)$ is as the cross-entropy loss where 
the set of positive and negative keys are the categories and $\exp(\qq\cdot 
\kk/\tau)$ is the unnormalized probability of the event $\qq=\kk$.  The 
cross-entropy loss is commonly used with a fixed set of categories while with 
in \cref{eq:co_loss} the set of keys is specific to each query.

The equivalence to the standard cross-entropy loss is exact if the normalizing 
term consists of all negative keys.  Suppose we use cross-entropy but use 
a sampled subset of all categories in the denominator that decreases the 
computation cost. The drawback is some queries will influence the training more 
than others because of modified probabilities. A fix is to use importance 
sampling according to the unnormalized probabilities of the negative keys at 
the cost of losing the advantage of subsampling.  In the literature on triplet 
losses such ideas are often referred to as hard negative mining where instead 
of using any negative, we mine hard negatives.  Our proposed method in 
\cref{ch:vsepp} is an example of hard negative mining in contrastive losses.

In recent works on contrastive losses, the query is uniformly selected from the 
training set and $\kk_+$ is generated by perturbing the query using common 
input augmentations.  For images, common perturbations are random resize, 
random crop, random color jitter, and random flip.  The novelty in works such 
as MoCo is to keep an adaptive queue of keys to be used as dissimilar examples.  
This ensures the cost of training is reduced by not back-propagating through 
the network for negative examples while keeping an up-to-date set of negatives 
with enough contrast.  Another novelty in works such as SimCLR and MoCo v2 is 
to use more diverse and aggressive data augmentations to generate a larger set 
of possible similar and dissimilar pairs and hence better contrastive examples.

%% file: ch-vsepp.tex
\chapter{VSE++: Improving Visual-Semantic Embeddings with Hard Negative Mining}
\label{ch:vsepp}

The first problem studied in this thesis concerns training efficiency in 
computer vision.  We hypothesize that training data are not equally important; 
training signal from some inputs is stronger than others. Moreover, we 
hypothesize that prioritizing learning on more informative training data 
increases convergence speed and improves generalization performance on test 
data.

To this end, we present a new technique for learning visual-semantic embeddings 
for cross-modal retrieval.  Inspired by hard negative mining, the use of hard 
negatives in structured prediction, and ranking loss functions, we introduce 
a simple change to common loss functions used for multi-modal embeddings.  
That, combined with fine-tuning and use of augmented data, yields significant 
gains in retrieval performance.  We showcase our approach, \VSEpp, on \coco{} 
and \fthk{} datasets, using ablation studies and comparisons with existing 
methods.  On \coco{} our approach outperforms state-of-the-art methods by 
$8.8\%$ in caption retrieval and $11.3\%$ in image retrieval (at R@$1$).

The content of this chapter have appeared in the following publication:
\begin{itemize}
    \item {\bf Faghri, Fartash} and Fleet, David J.\ and Kiros, Jamie R.\ and 
        Fidler, Sanja, {\sl ``VSE++: Improving Visual-Semantic Embeddings with Hard 
        Negatives"}, British Machine Vision Conference (BMVC), 2018.
\end{itemize}

To ensure reproducibility, our code is publicly 
available~\footnote{\url{https://github.com/fartashf/vsepp}}.

\input{ch-vsepp-tex/body}

%% file: ch-vsepp-tex/body.tex
\def\figdir{ch-vsepp-tex/figures}

\section{Introduction}

Joint embeddings enable a wide range of tasks in image, video and language 
understanding. Examples include shape-image embeddings~\citep{li2015joint} 
for shape inference, bilingual word embeddings~\citep{zou2013bilingual}, 
human pose-image embeddings for 3D pose inference~\citep{li2015maximum}, 
fine-grained recognition~\citep{reed2016learning}, zero-shot 
learning~\citep{frome2013devise}, and modality conversion via 
synthesis~\citep{reed2016generative,reed2016learning}.  Such embeddings entail 
mappings from two (or more) domains into a common vector space in which 
semantically associated inputs (e.g., text and images) are mapped to similar 
locations.  The embedding space thus represents the underlying domain 
structure, where location and often direction are semantically meaningful.

{\em Visual-semantic embeddings}\/ have been central to
image-caption retrieval and 
generation~\citep{kiros2014unifying,karpathy2015deep}, and visual 
question-answering~\citep{Malinowski15}.  One approach to visual 
question-answering, for example, is to first describe an image by a set of 
captions, and then to find the nearest caption in response to 
a question~\citep{agrawal2017vqa,zitnick2016measuring}. For image synthesis 
from text, one could map from text to the joint embedding space, and then
back to image space~\citep{reed2016generative,reed2016learning}.

Here we focus on visual-semantic embeddings for cross-modal retrieval; 
i.e., the retrieval of images given captions, 
or of captions for a query image.  As is common in retrieval, 
we measure performance by R@$K$, i.e., recall at $K$ -- the fraction of 
queries for which the correct item is retrieved in the closest $K$ points 
to the query in the embedding space ($K$ is usually a small integer, 
often $1$).  More generally, retrieval is a natural way to assess the 
quality of joint embeddings for image and language data~\citep{hodosh2013framing}.

The basic problem is one of ranking; the correct target(s) 
should be closer to the query than other items in the corpus, not unlike
{\em learning to rank}\/ problems \citep[e.g.,][]{li2014learning}, and 
max-margin structured prediction~\citep{chapelle2007large,le2007direct}.
The formulation and model architecture in this paper are most closely related 
to those of \citet{kiros2014unifying}, learned with a triplet ranking loss.  In 
contrast to that work, we advocate a novel loss, the use of augmented data, and 
fine-tuning, which, together, produce a significant increase in caption 
retrieval performance over the baseline ranking loss on well-known benchmark 
data.  We outperform the best reported result on \coco{} by almost $9\%$.  We 
also show that the benefit of a more powerful image encoder, with fine-tuning, 
is amplified with the use of our stronger loss function.
We refer to our model as \VSEpp{}.

Our main contribution is to incorporate hard negatives in the loss 
function.  This was inspired by the use of hard negative mining in 
classification tasks~\citep{dalal2005histograms, felzenszwalb2010object, 
malisiewicz2011ensemble}, and by the use of
hard negatives for improving image embeddings for face 
recognition~\citep{schroff2015facenet, wu2017sampling}. Minimizing a loss 
function using hard negative mining is equivalent to minimizing a modified 
non-transparent loss function with uniform sampling.  We extend the idea with 
the explicit introduction of hard negatives in the loss for multi-modal 
embeddings, without any additional cost of mining.

We also note that our formulation complements other recent articles that 
propose new architectures or similarity functions for this problem. To this 
end, we demonstrate improvements to \citet{vendrov2015order}.
Among other methods that could be improved with a modified loss, 
\citet{wang2017learning} propose an embedding network to fully replace the 
similarity function used for the ranking loss.  An attention mechanism on both 
images and captions is used by \citet{nam2016dual}, where the authors 
sequentially and selectively focus on a subset of words and image regions to 
compute the similarity.  In \citet{huang2016instance}, the authors use 
a multi-modal context-modulated attention mechanism to compute the similarity 
between images and captions. Our proposed loss function and triplet sampling 
could be extended and applied to other such problems.

\section{Learning Visual-Semantic Embeddings}

For image-caption retrieval the query is a caption and the task is to 
retrieve the most relevant image(s) from a database. Alternatively, the 
query may 
be an image, and the task is to retrieves relevant captions.  The goal is to 
maximize recall at $K$ (R@$K$), i.e., the fraction of queries for which the 
most relevant item is ranked among the top $K$ items returned.

Let $S=\{(i_n, c_n)\}^N_{n=1}$ be a training set of image-caption pairs.  
We refer to $(i_n, c_n)$ as {\em positive pairs}\/ and $(i_n, c_{m\neq n})$ as 
{\em negative pairs}\/; i.e., the most relevant caption to the image $i_n$ is 
$c_n$ and for caption $c_n$, it is the image $i_n$.  We define a similarity 
function $s(i, c)\in \R$ that should, ideally, give higher similarity 
scores to positive pairs than negatives.  In caption retrieval, the query 
is an image and we rank a database of captions based on the similarity 
function; i.e., R@$K$ is the percentage of queries for which the positive 
caption is ranked among the top $K$ captions using $s(i, c)$.  Likewise 
for image retrieval.  In what follows the similarity function is defined 
on the joint embedding space.  This differs from other formulations, such as 
\citet{wang2017learning}, which use a similarity network to directly classify 
an image-caption pair as matching or non-matching.

\subsection{Visual-Semantic Embedding}

Let $\phi(i; \theta_\phi)\in \R^{D_\phi}$ be a feature-based representation 
computed from image $i\,$ (e.g., the representation before logits in 
VGG19~\citep{simonyan2014very} or ResNet152~\citep{he2016deep}).  Similarly, 
let $\psi(c; \theta_\psi) \in \R^{D_\psi}$ be a representation
of caption $c$ in a caption embedding space (e.g., a GRU-based text 
encoder).  Here, $\theta_\phi$ and $\theta_\psi$ denote model 
parameters for the respective mappings to these initial 
image and caption representations.

Then, let the mappings into the {\em joint embedding space}\/ 
be defined by linear projections:
\begin{eqnarray}
    f(i; W_f, \theta_\phi) & = & W_f^T \phi(i; \theta_\phi)\\
    g(c; W_g, \theta_\psi) & = & W_g^T \psi(c; \theta_\psi) %
\end{eqnarray}
where $W_f\in\R^{D_\phi\times D}$ and $W_g\in\R^{D_\psi\times D}$.
We further normalize $f(i; W_f, \theta_\phi)$, and $g(c; W_g, \theta_\psi)$, 
to lie on the unit hypersphere.  Finally, we define the similarity 
function in the joint embedding space to be the usual inner product:
\begin{align}
s(i,c) = f(i; W_f, \theta_\phi)\cdot g(c; W_g, \theta_\psi)\,.
\end{align}
Let $\theta=\{W_f,W_g,\theta_\psi\}$ be the model parameters.  If we also
fine-tune the image encoder, then we would also include $\theta_\phi$ in 
$\theta$.

Training entails the minimization of empirical loss with respect to $\theta$, 
i.e., the cumulative loss over training data ${S=\{(i_n, c_n)\}^N_{n=1}}$:
\begin{eqnarray}
e(\theta, S)=\tfrac{1}{N}\sum^N_{n=1} \ell(i_n, c_n)
\end{eqnarray}
where $\ell(i_n, c_n)$ is a suitable loss function for a single training 
exemplar.
Inspired by the use of a triplet loss for image 
retrieval~\citep[e.g.,][]{frome2007learning,chechik2010large}, recent 
approaches
to joint visual-semantic embeddings have used a hinge-based triplet ranking 
loss~\citep{kiros2014unifying, karpathy2015deep, ZhuICCV15, 
socher2014grounded}:
\begin{equation}
    \ell_{SH}(i, c) ~=~
    \sum_{\hat{c}} [\alpha - s(i,c) + s(i,\hat{c})]_+ 
 \, +\, \sum_{\hat{i}} [\alpha - s(i,c) + s(\hat{i},c)]_+\,,
    \label{eq:contrastive}
\end{equation}
where $\alpha$ serves as a margin parameter, and $[x]_+ \equiv \max(x, 0)$.  
This hinge loss comprises two symmetric terms.  The first sum is taken over all 
negative captions $\hat{c}$, given query $i$.  The second is taken over all 
negative images $\hat{i}$, given caption $c$.  Each term is proportional to the 
expected loss (or {\em violation}\/) over sets of negative samples.  If $i$ and 
$c$ are closer to one another in the joint embedding space than to any 
negative, by the margin $\alpha$, the hinge loss is zero.  In practice, for 
computational efficiency, rather than summing over all  negatives in the 
training set, it is common to only sum over (or randomly sample) the negatives 
in a mini-batch of stochastic gradient descent~\citep{kiros2014unifying, 
socher2014grounded, karpathy2015deep}.  The runtime complexity of computing 
this loss approximation is quadratic in the number of image-caption pairs in 
a mini-batch.

Of course there are other loss functions that one might consider.
One  is a pairwise hinge loss in which elements of positive pairs
are encouraged to lie within a hypersphere of radius $\rho_1$ in 
the joint embedding space,
while negative pairs should be no closer than $\rho_2 > \rho_1$.
This is problematic as it constrains the structure of the latent space more 
than does the ranking loss, and it entails the use of two hyper-parameters 
which can be very difficult to set.
Another possible approach is to use Canonical Correlation Analysis to learn 
$W_f$ and $W_g$, thereby trying to preserve correlation between the text and 
images in the joint 
embedding~\citep[e.g.,][]{klein2015associating,eisenschtat2016linking}.  By 
comparison, when measuring performance as R@$K$, for small $K$, 
a correlation-based loss will not give sufficient influence to the embedding of 
negative items in the local vicinity of positive pairs, which is critical for 
R@$K$.  

\subsection{Emphasis on Hard Negatives}
\label{sec:hard_neg}

\begin{figure*}[t!]

    \hspace*{-1cm}
    \centering
    \def\xx{5}
    \def\rr{1.5}
    \def\nn{.5}
    \hspace*{2cm}
    \begin{subfigure}[b]{.18\textwidth}
        \hspace*{-.6cm}
        \begin{tikzpicture}[scale=0.9]
            \draw[dashed] (0, 0) circle (\rr);
            \fill (0,0) circle (1.5pt) node[left] {$i$};
            \fill (\rr,0) circle (1.5pt) node[right] {$c$};
            \draw[dashed] (0, 0) circle (.5);
            \draw (+\nn,0) circle (1.5pt) node[right] {$c^\prime$};

            \draw (-.1,\rr-.1) circle (1.5pt) node[above] {$\hat{c}_1$};
            \draw (+-\rr+.2,+.1) circle (1.5pt) node[left] {$\hat{c}_2$};
            \draw (0.5,-\rr+.2) circle (1.5pt) node[below] {$\hat{c}_3$};
        \end{tikzpicture}
        \caption{}\label{fig:hardneg}
    \end{subfigure}
    \hspace*{3cm}
    \begin{subfigure}[b]{.18\textwidth}
        \hspace*{-.6cm}
        \begin{tikzpicture}[scale=0.9]
            \draw[dashed] (0, 0) circle (\rr);
            \fill (0,0) circle (1.5pt) node[left] {$i$};
            \fill (\rr,0) circle (1.5pt) node[right] {$c$};
            \draw[dashed] (0, 0) circle (.5);
            \draw (\rr-.5,0) circle (1.5pt) node[right] {$c^\prime$};

            \draw (-.1,\rr-.1-.1) circle (1.5pt) node[above] {$\hat{c}_1$};
            \draw (-\rr+.25,+.2) circle (1.5pt) node[left] {$\hat{c}_2$};
            \draw (0.5,-\rr+.2) circle (1.5pt) node[below] {$\hat{c}_3$};
            \draw (-1.,\rr-.5) circle (1.5pt) node[above] {$\hat{c}_4$};
            \draw (-\rr+.25,-.2) circle (1.5pt) node[left] {$\hat{c}_5$};
            \draw (0.1,-\rr+.2) circle (1.5pt) node[below] {$\hat{c}_6$};
        \end{tikzpicture}
        \captionsetup{margin=.5cm}
        \caption{}\label{fig:softneg}
    \end{subfigure}
    \caption{An illustration of typical positive pairs and the nearest negative 
    samples. Here assume similarity score is the negative distance. Filled 
    circles show a positive pair $(i,c)$, while empty circles are negative 
    samples for the query $i$.  The dashed circles on the two sides are drawn 
    at the same radii.  Notice that the hardest negative sample $c^\prime$ is 
    closer to $i$ in \subref{fig:hardneg}. Assuming a zero margin, 
    \subref{fig:softneg} has a higher loss with the \SUM{} loss compared to 
    \subref{fig:hardneg}.  The \MAX{} loss assigns a higher loss to 
    \subref{fig:hardneg}.}
    \label{fig:examples}
\end{figure*}
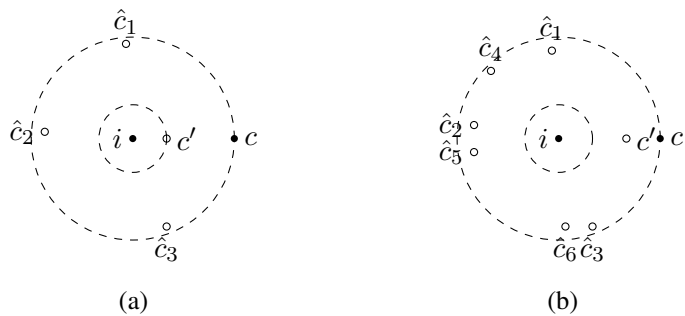%
Inspired by common loss functions used in structured 
prediction~\citep{tsochantaridis2005large, yu2009learning, 
felzenszwalb2010object}, we focus on hard negatives for training, i.e., the 
negatives closest to each training query.  This is particularly relevant for 
retrieval since it is the hardest negative that determines success or failure 
as measured by R@$1$.

Given a positive pair $(i, c)$, the hardest negatives are given by 
$i^\prime=\argmax_{j\neq i} s(j,c)$ and $c^\prime=\argmax_{d\neq c} s(i,d)$.
To emphasize hard negatives we define our loss as
\begin{equation}
    \ell_{MH}(i, c)
    ~=~ \max_{c^\prime} \left[\alpha - s(i,c) + s(i,c^\prime)\right]_+
    \, +\, \max_{i^\prime} \left[\alpha -s(i,c) + s(i^\prime,c)\right]_+ ~ .
    \label{eq:rank_loss}
\end{equation}
Like Eq.~\ref{eq:contrastive}, this loss comprises two terms, 
one with $i$ and one with $c$ as queries.  Unlike Eq.~\ref{eq:contrastive}, 
this loss is specified in terms of the hardest negatives,  $c^\prime$ and 
$i^\prime$. We refer to the loss in Eq.~\ref{eq:rank_loss} 
as {\em Max of Hinges (\MAX{})}\/ loss, and the loss in 
Eq.~\ref{eq:contrastive} as {\em Sum of Hinges (\SUM{})}\/ loss. There is 
a spectrum of loss functions from the \SUM{} loss to the \MAX{} loss. In the 
\MAX{} loss, the winner takes all the gradients, where instead we use 
re-weighted gradients of all the triplets. We only discuss the \MAX{} loss as 
it was empirically found to perform the best.

One case in which the \MAX{} loss is superior to \SUM{} is when multiple 
negatives with small violations combine to dominate the \SUM{} loss.  
For example, Fig.~\ref{fig:examples} depicts a positive pair
together with two sets of negatives.  In Fig.~\ref{fig:hardneg}, a 
single negative is too close to the query, which may require
a significant change to the mapping.  
However, any training step that pushes the hard negative 
away, might cause a number of small violating negatives, as 
in Fig.~\ref{fig:softneg}. Using the \SUM{} loss, these `new' negatives 
may dominate the loss, so the model is pushed back to the first example in 
Fig.~\ref{fig:hardneg}.  As a result, the optimization may oscillate between 
two states with no reduction in the \SUM{} loss while in a similar scenario the 
\MAX{} loss would decrease monotonically because it focuses on the hardest 
negative.

For computational efficiency, instead of finding the hardest negatives in the 
entire training set, we find them within each mini-batch.  
This has the same quadratic 
complexity as the complexity of the \SUM{} loss.  With random 
sampling of the mini-batches, this approximation yields other advantages.  One is 
that there is a high probability of getting hard negatives that are harder than 
at least $90\%$ of the entire training set (see \cref{sec:prob_hard_neg} for 
the explanation).  Moreover, the loss is potentially robust to label errors in 
the training data because the probability of sampling the hardest negative over 
the entire training set is somewhat low.

\subsection{Probability of Sampling the Hardest Negative}
\label{sec:prob_hard_neg}

Let $S=\{(i_n, c_n)\}^N_{n=1}$ denote a training set of image-caption pairs,
and let $C=\{c_n\}$ denote the set of captions.  Suppose we draw $M$ samples 
in a mini-batch, $Q=\{(i_m, c_m)\}^M_{m=1}$, from $S$.  Let the permutation, 
$\pi_m$, on $C$ refer to the rankings of captions according to the similarity 
function $s(i_m,c_n)$ for $c_n\in S\setminus\{c_m\}$. We can assume 
permutations, $\pi_m$, are uncorrelated.

Given a query image, $i_m$, we are interested in the probability of getting 
no captions from the $90$th percentile of $\pi_m$ in the mini-batch.  Assuming 
i.i.d.\ samples, this probability is simply $.9^{(M-1)}$, the probability that 
no sample in the mini-batch is from the $90$th percentile.  This probability 
tends to zero exponentially fast, falling below $1\%$ for $M\geq 44$. Hence, 
for large enough mini-batches, with high probability we sample negative 
captions that are harder than $90\%$ of the entire training set.
The probability for the $99.9$th percentile of $\pi_m$ tends to zero 
more slowly; it falls below $1\%$ for $M\geq 6905$, which is a relatively 
large mini-batch. 

While we get strong signals  by randomly 
sampling negatives within mini-batches, such sampling also provides 
some robustness to outliers, such as negative captions that 
better describe an image compared to the ground-truth caption.
Mini-batches as small as $128$ can provide strong 
enough training signal and robustness to label errors. Of course by 
increasing the mini-batch size, we get harder negative examples and possibly 
a stronger training signal.  However, by increasing the mini-batch size, we lose 
the benefit of SGD in finding good optima and exploiting the gradient noise.  
This can lead to getting stuck in local optima or as observed by 
\citet{schroff2015facenet}, extremely long training time.

\section{Experiments}

\begin{table*}[t!]
\resizebox{\linewidth}{!}{
    \centering
     \begin{tabular}{c|c|c|HccccH|ccccH}
         \# &
         {\bf Model} & {\bf Trainset} & {\bf R Sum}
         & \multicolumn{5}{c|}{Caption Retrieval}
         & \multicolumn{5}{c}{Image Retrieval}\\
         & & & &
         {\bf R@1} & {\bf R@5} & {\bf R@10} & {\bf Med r} & {\bf Mean r} &
         {\bf R@1} & {\bf R@5} & {\bf R@10} & {\bf Med r} & {\bf Mean r}\\
         \hline
         & & \multicolumn{12}{c}{{\cellcolor[gray]{0.8}\bf 1K Test Images}}\\
         \hline
         \numrow{}\label{coco:VSE}&
         \VSE{}~(\citet{kiros2014unifying}, GitHub) & \OC{} (1 fold) &
         $382.5$ &
         $43.4$ & $75.7$ & $85.8$ & $2$ & - &
         $31.0$ & $66.7$ & $79.9$ & $3$ & - \\
         \numrow{} &
         \order{}~\citep{vendrov2015order}& \TCV{} &
         - &
         $46.7$ & - & $88.9$ & $2.0$ & $5.7$ &
         $37.9$ & - & $85.9$ & $2.0$ & $8.1$\\
         \numrow{} &
         Embedding Net~\citep{wang2017learning} & \TCV{} &
         - &
         $50.4$ & $79.3$ & $69.4$ & - & - &
         $39.8$ & $75.3$ & $86.6$ & - & -
         \\
         \numrow{}\label{coco:smLSTM} &
         sm-LSTM~\citep{huang2016instance} &? &
         $431.8$ &
         $53.2$ & $83.1$ & $91.5$ & $\mathbf{1}$ & - &
         $40.7$ & $75.8$ & $87.4$ & $2$ & -
         \\
         \numrow{}\label{coco:twoway} &
         2WayNet~\citep{eisenschtat2016linking} & \TCV{} &
         - &
         $55.8$ & $75.2$ & - & - & - &
         $39.7$ & $63.3$ & - & - & -
         \\
         \hline
         \numrow{}\label{coco:VSEppOC} &
         \VSEpp{} & \OC{} (1 fold) &
         $386.5$ &
         $43.6$ & $74.8$ & $84.6$ & $2.0$ & $8.8$ &
         $33.7$ & $68.8$ & $81.0$ & $3.0$ & $12.9$\\

         \numrow{}\label{coco:VSEppRC} &
         \VSEpp{} & \RC &
         $410.3$ &
         $49.0$ & $79.8$ & $88.4$ & $1.8$ & $6.5$ &
         $37.1$ & $72.2$ & $83.8$ & $2.0$ & $10.8$\\

         \numrow{}\label{coco:VSEppRCV} &
         \VSEpp{} & \RCV{} &
         $423.1$ &
         $51.9$ & $81.5$ & $90.4$ & $\mathbf{1.0}$ & $5.8$ &
         $39.5$ & $74.1$ & $85.6$ & $2.0$ & $10.0$\\

         \numrow{}\label{coco:VSEppFt} &
         \VSEppFt{} & \RCV{} &
         $450.9$ &
         $57.2$ & $86.0$ & $93.3$ & $\mathbf{1.0}$ & $4.2$ &
         $45.9$ & $79.4$ & $89.1$ & $2.0$ & $8.5$\\

         \numrow{}\label{coco:VSEppRes} &
         \VSEppRes{} & \RCV{} &
         $446.8$ &
         $58.3$ & $86.1$ & $93.3$ & $\mathbf{1.0}$ & $4.2$ &
         $43.6$ & $77.6$ & $87.8$ & $2.0$ & $7.8$
         \\

         \numrow{}\label{coco:VSEppResFt} &
         \VSEppResFt{} & \RCV{} &
         $\mathbf{478.6}$ &
         $\mathbf{64.6}$ & $\mathbf{90.0}$ & $\mathbf{95.7}$ & $\mathbf{1.0}$ 
         & $\mathbf{3.4}$ &
         $\mathbf{52.0}$ & $\mathbf{84.3}$ & $\mathbf{92.0}$ & $\mathbf{1.0}$ 
         & $\mathbf{6.1}$
         \\

        \hline
         & & \multicolumn{12}{c}{{\cellcolor[gray]{0.8}\bf 5K Test Images}}\\
         \hline
         \numrow{} &
         \order{}~\citep{vendrov2015order}& \TCV{} &
         - &
         $23.3$ & - & $65.0$ & $5.0$ & $24.4$ &
         $18.0$ & - & $57.6$ & $7.0$ & $35.9$\\

         \numrow{}\label{coco:VSEppFt5} &
         \VSEppFt{} & \RCV{} &
         $312.4$ &
         $32.9$ & $61.7$ & $74.7$ & $3.0$ & $16.9$ &
         $24.1$ & $52.8$ & $66.2$ & $5.0$ & $38.1$\\

         \numrow{}\label{coco:VSEppResFt5} &
         \VSEppResFt{} & \RCV{} &
         $\mathbf{355.8}$ &
         $\mathbf{41.3}$ & $\mathbf{71.1}$ & $\mathbf{81.2}$ & $\mathbf{2.0}$ 
         & $\mathbf{12.4}$ &
         $\mathbf{30.3}$ & $\mathbf{59.4}$ & $\mathbf{72.4}$ & $\mathbf{4.0}$ 
         & $\mathbf{26.0}$
        \\[-2mm]
    \end{tabular}
    }
    \caption{Results of experiments on \coco{}.}
    \label{tb:coco}
\end{table*}

Below we perform experiments with our approach, \VSEpp{}, comparing it to 
a baseline formulation with $\SUM{}$ loss, denoted \VSEz{}, and other 
state-of-the-art approaches. Essentially, the baseline formulation, \VSEz{}, is 
similar to that in \citet{kiros2014unifying}, denoted  \VSE{}.

We experiment with two image encoders: VGG19 by~\citet{simonyan2014very} and 
ResNet152 by~\citet{he2016deep}.  In what follows, we use VGG19 unless 
specified otherwise.  As in previous work we extract image features directly 
from FC$7$, the penultimate fully connected layer.
The dimensionality of the image embedding, $D_\phi$, 
is $4096$ for VGG19 and $2048$ for ResNet152.

In more detail, we first resize the image to $256\times256$, and then 
use either a single crop of size $224\times224$ or the mean of feature vectors 
for multiple crops of similar size. We refer to training with one center crop 
as \OC{}, and training with $10$ crops at fixed locations as \TC{}. These image 
features can be pre-computed once and reused. We also experiment with using 
a single random crop, denoted by \RC{}.  For \RC{}, image features are computed 
on the fly. Recent works have mostly used \RC{}/\TC{}. In our preliminary 
experiments, we did not observe significant differences between \RC{}/\TC{}. As 
such, we perform  most experiments with \RC{}.

For the caption encoder, we use a GRU similar to the one used in 
\citet{kiros2014unifying}.  We set the dimensionality of the GRU, $D_\psi$, and 
the joint embedding space, $D$, to $1024$. The dimensionality of the word 
embeddings that are input to the GRU is set to $300$.

We further note that in \citet{kiros2014unifying}, the caption embedding is 
normalized, while the image embedding is not. Normalization of both vectors 
means that the similarity function is cosine similarity.  In \VSEpp{} we 
normalize both vectors. Not normalizing the image embedding changes the 
importance of samples. In our experiments, not normalizing the image embedding  
helped the baseline, \VSEz{}, to find a better solution.  However, \VSEpp{} is 
not significantly affected by this normalization.

\subsection{Datasets}

We evaluate our method on the Microsoft COCO dataset~\citep{lin2014microsoft} 
and the \fthk{} dataset~\citep{young2014image}. \fthk{} has a standard $30,000$ 
images for training. Following \citet{karpathy2015deep}, we use $1000$ images 
for validation and $1000$  images for testing. We also use the splits of 
\citet{karpathy2015deep} for \coco{}. In this split, the training set contains 
$82,783$ images, $5000$ validation and $5000$ test images.  However, there are 
also $30,504$ images that were originally in the validation set of \coco{} but 
have been left out in this split. We refer to this set as \RV{}. Some papers 
use \RV{} for training ($113,287$ training images in total) to further improve 
accuracy.  We report results using both training sets.  Each image comes with 
$5$ captions.  The results are reported by either averaging over $5$ folds of 
$1K$ test images or testing on the full $5K$ test images.

\subsection{Details of Training}
\label{sec:train_detail}
We use the Adam optimizer~\citep{kingma2014adam}.  Models are trained for at 
most $30$ epochs.  Except for fine-tuned models,  we start training with 
learning rate $0.0002$ for $15$ epochs, and then lower the learning rate to 
$0.00002$ for another $15$ epochs.  The fine-tuned models are trained by taking 
a model trained for $30$ epochs with a fixed image encoder, and then training 
it for $15$ epochs with a learning rate of $0.00002$.  We set the margin to 
$0.2$ for most experiments.  We use a mini-batch size of $128$ in all  
experiments.  Notice that since the size of the training set for different 
models is different, the actual number of iterations in each epoch can vary.  
For evaluation on the test set, we tackle over-fitting by choosing the snapshot 
of the model that performs best on the validation set.  The best snapshot is 
selected based on the sum of the recalls on the validation set.

\subsection{Results on \coco{}}

\begin{table*}[t!]
 \resizebox{\linewidth}{!}{
\centering
     \begin{tabular}{c|c|c|HccccH|ccccH}
         \# &
         {\bf Model} & {\bf Trainset} & {\bf R Sum}
         & \multicolumn{5}{c|}{Caption Retrieval}
         & \multicolumn{5}{c}{Image Retrieval}\\
         & & & &
         {\bf R@1} & {\bf R@5} & {\bf R@10} & {\bf Med r} & {\bf Mean r} &
         {\bf R@1} & {\bf R@5} & {\bf R@10} & {\bf Med r} & {\bf Mean r}\\
         \hline
         \numrow{}\label{coco:VSEzOC} &
         \VSEz{} & \OC{} (1 fold) &
         $383.2$ &
         $43.2$ & $73.9$ & $85.0$ & $2.0$ & $7.6$ &
         $33.0$ & $67.4$ & $80.7$ & $3.0$ & $11.3$\\

         {}\ref{coco:VSEppOC} &
         \VSEpp{} & \OC{} (1 fold) &
         $386.5$ &
         $43.6$ & $74.8$ & $84.6$ & $2.0$ & $8.8$ &
         $33.7$ & $68.8$ & $81.0$ & $3.0$ & $12.9$\\

         \numrow{}\label{coco:VSEzRC} &
         \VSEz{} & \RC &
         $390.0$ &
         $43.1$ & $77.0$ & $87.1$ & $2.0$ & $6.5$ &
         $32.5$ & $68.3$ & $82.1$ & $3.0$ & $9.5$\\

         {}\ref{coco:VSEppRC} &
         \VSEpp{} & \RC &
         $410.3$ &
         $49.0$ & $79.8$ & $88.4$ & $1.8$ & $6.5$ &
         $37.1$ & $72.2$ & $83.8$ & $2.0$ & $10.8$\\

         \numrow{}\label{coco:VSEzRCV} &
         \VSEz{} & \RCV{} &
$402.7$ &
$46.8$ & $78.8$ & $89.0$ & $1.8$ & $6.1$ &
$34.2$ & $70.4$ & $83.6$ & $2.6$ & $8.6$
         \\
         {}\ref{coco:VSEppRCV} &
         \VSEpp{} & \RCV{} &
         $423.1$ &
         $51.9$ & $81.5$ & $90.4$ & $\mathbf{1.0}$ & $5.8$ &
         $39.5$ & $74.1$ & $85.6$ & $2.0$ & $10.0$\\

         \numrow{}\label{coco:VSEzFt} &
         \VSEzFt{} & \RCV{} &
$424.3$ &
$50.1$ & $81.5$ & $90.5$ & $1.6$ & $5.5$ &
$39.7$ & $75.4$ & $87.2$ & $2.0$ & $7.1$
         \\
         {}\ref{coco:VSEppFt} &
         \VSEppFt{} & \RCV{} &
         $450.9$ &
         $57.2$ & $86.0$ & $93.3$ & $\mathbf{1.0}$ & $4.2$ &
         $45.9$ & $79.4$ & $89.1$ & $2.0$ & $8.5$\\

         \numrow{}\label{coco:VSEzRes} &
         \VSEzRes{} & \RCV{} &
$455.1$ &
$52.7$ & $83.0$ & $91.8$ & $1.0$ & $4.5$&
$36.0$ & $72.6$ & $85.5$ & $2.2$ & $7.7$
         \\
         {}\ref{coco:VSEppRes} &
         \VSEppRes{} & \RCV{} &
         $446.8$ &
         $58.3$ & $86.1$ & $93.3$ & $\mathbf{1.0}$ & $4.2$ &
         $43.6$ & $77.6$ & $87.8$ & $2.0$ & $7.8$
         \\

        \numrow{}\label{coco:VSEzResFt} &
         \VSEzResFt{} & \RCV{} &
$470.6$ &
$56.0$ & $85.8$ & $93.5$ & $1.0$ & $4.0$&
$43.7$ & $79.4$ & $89.7$ & $2.0$ & $6.2$
         \\
          {}\ref{coco:VSEppResFt} &
         \VSEppResFt{} & \RCV{} &
         $\mathbf{478.6}$ &
         $\mathbf{64.6}$ & $\mathbf{90.0}$ & $\mathbf{95.7}$ & $\mathbf{1.0}$ 
         & $\mathbf{3.4}$ &
         $\mathbf{52.0}$ & $\mathbf{84.3}$ & $\mathbf{92.0}$ & $\mathbf{1.0}$ 
         & $\mathbf{6.1}$
         \\[-1mm]

    \end{tabular}
    }
    \caption{The effect of data augmentation and fine-tuning. We copy the 
    relevant results for \VSEpp{} from Table~\ref{tb:coco} to enable an easier 
    comparison.  Notice that after applying all the modifications, \VSEz{} 
    model reaches $56.0\%$ for $R@1$, while \VSEpp{} achieves $64.6\%$.}
    \label{tb:vse0}
\end{table*}

The results on the \coco{} dataset are presented in Table~\ref{tb:coco}.  To 
understand the effect of training and algorithmic variations we report ablation 
studies for the baseline \VSEz{} (see Table~\ref{tb:vse0}).  Our best result 
with \VSEpp{} is achieved by using ResNet152 and fine-tuning the image encoder 
(row~\ref{coco:VSEppResFt}), where we see $21.2\%$ improvement in R@1 for 
caption retrieval and $21\%$ improvement in R@1 for image retrieval compared to 
\VSE{} (rows~\ref{coco:VSE} and~\ref{coco:VSEppResFt}).

\emph{Effect of \MAX{} loss}.  Using ResNet152 and fine-tuning can only lead to 
$12.6\%$ improvement using the \VSEz{} formulation (rows~\ref{coco:VSEzResFt} 
and~\ref{coco:VSE}), while our \MAX{} loss function brings a significant 
additional gain of $8.6\%$ (rows~\ref{coco:VSEppResFt} 
and~\ref{coco:VSEzResFt}).

\emph{Effect of the training set}. We compare \VSEz{} and \VSEpp{} by 
incrementally improving the training data.  Comparing the models trained on 
\OC{} (rows~\ref{coco:VSE} and~\ref{coco:VSEppOC}), we only see $2.7\%$ 
improvement in R@1 for image retrieval but no improvement in caption retrieval 
performance. However, when we train using \RC{} (rows~\ref{coco:VSEppRC} 
and~\ref{coco:VSEzRC}) or \RCV{} (rows~\ref{coco:VSEppRCV} 
and~\ref{coco:VSEzRCV}), we see that \VSEpp{} gains an improvement of $5.9\%$ 
and $5.1\%$, respectively,  in R@1 for caption retrieval compared to \VSEz{}.  
This shows that \VSEpp{} can better exploit the additional data.

\emph{Effect of a better image encoding}. We also investigate the effect of 
a better image encoder on the models.  Row~\ref{coco:VSEppFt} and 
row~\ref{coco:VSEzFt} show the effect of fine-tuning the VGG19 image encoder. 
We see that the gap between \VSEz{} and \VSEpp{} increases to $6.1\%$. If we 
use ResNet152 instead of VGG19 (row~\ref{coco:VSEppRes} and 
row~\ref{coco:VSEzRes}), the gap is $5.6\%$. As for our best result, if we use 
ResNet152 and also fine-tune the image encoder (row~\ref{coco:VSEppResFt} and 
row~\ref{coco:VSEzResFt}) the gap becomes $8.6\%$. The increase in the 
performance gap shows that the improved loss of \VSEpp{} can better guide the 
optimization when a more powerful image encoder is used.

\emph{Comparison with state-of-the-art}. Comparing \VSEppResFt{} to the 
state-of-the-art on \coco{} at the time of publication, {\em 2WayNet}\/ 
(row~\ref{coco:VSEppResFt} and row~\ref{coco:twoway}), we see $8.8\%$ 
improvement in R@1 for caption retrieval and compared to {\em sm-LSTM}\/ 
(row~\ref{coco:VSEppResFt} and row~\ref{coco:smLSTM}), $11.3\%$ improvement in 
image retrieval.
We also report results on the full $5K$ test set of \coco{} in 
rows~\ref{coco:VSEppFt5} and~\ref{coco:VSEppResFt5}.

\subsection{Results on \fthk{}}

\begin{table*}[t!]
    \resizebox{\linewidth}{!}{
        \centering
    \begin{tabular}{c|c|c|HccccH|ccccH}
        \# &
        {\bf Model} & {\bf Trainset} & {\bf R Sum}
        & \multicolumn{5}{c|}{Caption Retrieval}
        & \multicolumn{5}{c}{Image Retrieval}\\
        & & & &
        {\bf R@1} & {\bf R@5} & {\bf R@10} & {\bf Med r} & {\bf Mean r} &
        {\bf R@1} & {\bf R@5} & {\bf R@10} & {\bf Med r} & {\bf Mean r}\\
        \hline
        \numrow{}\label{f30k:VSE} &
        \VSE{}~\citep{kiros2014unifying} & \OC{}  &
        $251.9$ &
        $23.0$ & $50.7$ & $62.9$ & $5$ & - &
        $16.8$ & $42.0$ & $56.5$ & $8$ & -
        \\
        \numrow{}\label{f30k:VSEgit} &
        \VSE{} (GitHub) & \OC{}  &
        $287.9$ &
        $29.8$ & $58.4$ & $70.5$ & $4$ & - &
        $22.0$ & $47.9$ & $59.3$ & $6$ & -
        \\
        \numrow{} &
        Embedding Net~\citep{wang2017learning} & \TC{} &
        - &
        $40.7$ & $69.7$ & $79.2$ & - &? &
        $29.2$ & $59.6$ & $71.7$ & - & -
        \\
        \numrow{} &
        DAN~\citep{nam2016dual} &? &
        - &
        $41.4$ & $73.5$ & $82.5$ & $2$ & - &
        $31.8$ & $61.7$ & $72.5$ & $3$ & -
        \\
        \numrow{} &
        sm-LSTM~\citep{huang2016instance} &? &
        $358.7$ &
        $42.5$ & $71.9$ & $81.5$ & $2$ & - &
        $30.2$ & $60.4$ & $72.3$ & $3$ & -
        \\
        \numrow{} &
        2WayNet~\citep{eisenschtat2016linking} & \TC{} &
        - &
        $49.8$ & $67.5$ & - & - & - &
        $36.0$ & $55.6$ & - & - & -
        \\
        \numrow{} &
        DAN (ResNet)~\citep{nam2016dual} &? &
        - &
        $\mathbf{55.0}$ & $\mathbf{81.8}$ & $\mathbf{89.0}$ & $\mathbf{1}$ 
        & - &
        $\mathbf{39.4}$ & $\mathbf{69.2}$ & $\mathbf{79.1}$ & $\mathbf{2}$ & -
        \\
        \hline
        \numrow{} &
        \VSEz{} & \OC{}  &
$294.3$ &
$29.8$ & $59.8$ & $71.9$ & $3.0$ & $24.7$ &
$23.0$ & $48.8$ & $61.0$ & $6.0$ & $35.3$
        \\

        \numrow{}\label{f30k:VSEzRC}&
        \VSEz{} & \RC &
$298.7$ &
$31.6$ & $59.3$ & $71.7$ & $4.0$ & $25.3$&
$21.6$ & $50.7$ & $63.8$ & $5.0$ & $30.0$
        \\

        \numrow{} &
        \VSEpp{} & \OC{}  &
$291.3$ &
$31.9$ & $58.4$ & $68.0$ & $4.0$ & $30.9$ &
$23.1$ & $49.2$ & $60.7$ & $6.0$ & $35.6$
        \\

        \numrow{}\label{f30k:VSEppRC}&
        \VSEpp{} & \RC &
        $326.3$ &
        $38.6$ & $64.6$ & $74.6$ & $2.0$ & $21.3$ &
        $26.8$ & $54.9$ & $66.8$ & $4.0$ & $30.0$\\

\numrow{}&
\VSEzFt{} & \RC &
$333.9$ & 
$37.4$ & $65.4$ & $77.2$ & $3.0$ & $16.2$&
$26.8$ & $57.6$ & $69.5$ & $4.0$ & $22.9$
\\

        \numrow{} &
        \VSEppFt{} & \RC &
$350.9$ &
$41.3$ & $69.1$ & $77.9$ & $2.0$ & $18.2$ &
$31.4$ & $60.0$ & $71.2$ & $3.0$ & $24.9$
        \\

\numrow{}&
\VSEzRes{} & \RC &
$324.2$ & 
$36.6$ & $67.3$ & $78.4$ & $3.0$ & $12.6$&
$23.3$ & $52.6$ & $66.0$ & $5.0$ & $25.4$
\\
        \numrow{} &
        \VSEppRes{} & \RC &
$363.1$ &
$43.7$ & $71.9$ & $82.1$ & $2.0$ & $12.5$&
$32.3$ & $60.9$ & $72.1$ & $3.0$ & $22.2$
        \\

\numrow{}&
\VSEzResFt{} & \RC &
$367.7$ & 
$42.1$ & $73.2$ & $84.0$ & $2.0$ & $11.1$&
$31.8$ & $62.6$ & $74.1$ & $3.0$ & $18.6$
\\
        \numrow{}\label{f30k:VSEppResFt} &
        \VSEppResFt{} & \RC &
$409.7$ &
$52.9$ & $80.5$ & $87.2$ & $1.0$ & $9.5$&
$\mathbf{39.6}$ & $\mathbf{70.1}$ & $\mathbf{79.5}$ & $\mathbf{2.0}$ 
& $\mathbf{16.6}$
        \\[-1mm]

   \end{tabular}
    }
    \caption{Results on the \fthk{} dataset.}
    \label{tb:f30k}

\end{table*}

Tables~\ref{tb:f30k} summarizes the performance on \fthk{}.  We obtain $23.1\%$ 
improvement in R@$1$ for caption retrieval and $17.6\%$ improvement in R@$1$ 
for image retrieval (rows~\ref{f30k:VSE} and~\ref{f30k:VSEppResFt}). We 
observed that \VSEpp{} over-fits when trained with the pre-computed features of 
\OC{}. The reason is potentially the limited size of the \fthk{} training set.  
As explained in Sec.~\ref{sec:train_detail}, we select a snapshot of the model 
before over-fitting occurs, based on performance with the validation set.  
Over-fitting does not occur when the model is trained using the \RC{} training 
data.  Our results show the improvements incurred by our \MAX{} loss persist 
across datasets, as well as across models.

\subsection{Improving Order Embeddings}
\label{sec:order}

Given the simplicity of our approach, our proposed loss function can complement 
the recent approaches that use more sophisticated model architectures or 
similarity functions. Here we demonstrate the benefits of the \MAX{} loss by 
applying it to another approach to joint embeddings called 
order-embeddings~\citep{vendrov2015order}.  The main difference with the 
formulation above is the use of an asymmetric similarity function, i.e., 
$s(i,c)=-\|\max(0, g(c; W_g, \theta_\psi) - f(i; W_f, \theta_\phi)) \|^2$.  
Again, we simply replace their use of the \SUM{} loss by our \MAX{} loss.

Like their experimental setting, we use the training set \TCV{}.  For our 
\orderpp{}, we use the same learning schedule and margin as our other 
experiments.  However, we use their training settings to train \orderz{}. 
We start training with a learning rate of $0.001$ for $15$ epochs 
and lower the learning rate to $0.0001$ for another $15$ epochs.  Like 
\citet{vendrov2015order} we use a margin of $0.05$.  Additionally, 
\citet{vendrov2015order} takes the absolute value of embeddings before 
computing the similarity function which we replicate only for \orderz{}.

Table~\ref{tb:order} reports the results when the \SUM{} loss 
is replaced by the \MAX{} loss.  We replicate their results using 
our \orderz{} formulation and get slightly better results 
(row~\ref{order:order} and row~\ref{order:orderz}). We observe $4.5\%$ 
improvement from \orderz{} to \orderpp{} in R@$1$ for caption retrieval 
(row~\ref{order:orderz} and row~\ref{order:orderpp}).  Compared to the 
improvement from \VSEz{} to \VSEpp{}, where the improvement on the \TCV{} 
training set is $1.8\%$, we gain an even higher improvement here.
This shows  that the \MAX{} loss can potentially improve numerous similar 
loss functions used in retrieval and ranking tasks.

\begin{table}[h]
   \resizebox{\linewidth}{!}{\centering
     \begin{tabular}{c|c|HHccccH|ccccH}
         \# &
         {\bf Model} & {\bf Trainset} & {\bf R Sum}
         & \multicolumn{5}{c|}{Caption Retrieval}
         & \multicolumn{5}{c}{Image Retrieval}\\
         & & & &
         {\bf R@1} & {\bf R@5} & {\bf R@10} & {\bf Med r} & {\bf Mean r} &
         {\bf R@1} & {\bf R@5} & {\bf R@10} & {\bf Med r} & {\bf Mean r}\\
         \hline
         & & \multicolumn{12}{c}{{\cellcolor[gray]{0.8}\bf 1K Test Images}}\\
         \hline

         \numrow{}\label{order:order} &
         \order{}~\citep{vendrov2015order}& \TCV{} &
         - &
         $46.7$ & - & $88.9$ & $2.0$ & $5.7$ &
         $37.9$ & - & $85.9$ & $2.0$ & $8.1$\\
         \hline

         \numrow{}\label{order:vsez} &
         \VSEz{} & \TCV{} &
         $417.0$ &
         $49.5$ & $81.0$ & $90.0$ & $1.8$ & $5.1$ &
         $38.1$ & $73.3$ & $85.1$ & $2.0$ & $8.4$\\

         \numrow{}\label{order:orderz} &
         \orderz{} & \TCV{} &
         $421.3$ &
         $48.5$ & $80.9$ & $90.3$ & $1.8$ & $5.2$ &
         $39.6$ & $75.3$ & $86.7$ & $2.0$ & $7.4$
         \\

         \numrow{}\label{order:vsepp}&
         \VSEpp{} & \TCV{} &
         $425.9$ &
         $51.3$ & $82.2$ & $91.0$ & $1.2$ & $5.0$ &
         $40.1$ & $75.3$ & $86.1$ & $2.0$ & $10.5$
         \\

        \numrow{}\label{order:orderpp}&
         \orderpp{} & \TCV{} &
         $436.0$ &
         ${\bf 53.0}$ & $83.4$ & ${\bf 91.9}$ & ${\bf 1.0}$ & $4.6$ &
         ${\bf 42.3}$ & $77.4$ & ${\bf 88.1}$ & $2.0$ & $8.2$
         \\[-1mm]
    \end{tabular}
    }
        \captionof{table}{Comparison on \coco{}. Training set for all the rows 
        is \TCV{}.}
        \label{tb:order}
\end{table}
\subsection{Behavior of Loss Functions}

We  observe that the \MAX{} loss can take a few epochs to `warm-up' 
during training.  Fig.~\ref{fig:sum_vs_max} depicts such behavior on the 
\fthk{} dataset using \RC{}.  Notice that the \SUM{} loss starts off 
faster, but after approximately $5$ epochs \MAX{} loss surpasses \SUM{} loss.  
To explain this, the \MAX{} loss depends on a smaller set of triplets 
compared to the \SUM{} loss.  Early in training the gradient of the \MAX{} 
loss is  influenced by a relatively small set of triples.  As such, it can 
take more iterations to train a model with the \MAX{} loss. We explored a simple 
form of curriculum learning~\citep{bengio2009curriculum} to speed-up the  
training. We start training with the \SUM{} loss for a few epochs, then 
switch to the \MAX{} loss for the rest of the training. However, it did 
not perform much better than training solely with the \MAX{} loss.

\begin{figure}[t]
    \centering
        \includegraphics[width=.5\linewidth]{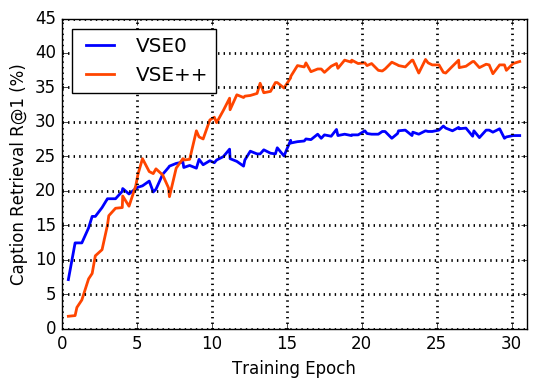}
        \caption{Analysis of the behavior of the \MAX{} loss on the \fthk{} 
        dataset training with \RC{}. This figure compares the \SUM{} loss to 
        the \MAX{} loss (Table~\ref{tb:f30k}, row~\ref{f30k:VSEzRC} and 
        row~\ref{f30k:VSEppRC}). Notice that, in the first $5$ epochs the 
        \SUM{} loss achieves a better performance, however, from there-on the 
        \MAX{} loss leads to much higher recall rates.}
        \label{fig:sum_vs_max}
\end{figure}

In \citet{schroff2015facenet}, it is reported that with a mini-batch size of 
$1800$,  training  is extremely slow. We experienced similar behavior with 
large mini-batches up to $512$. However, mini-batches of size $128$ or $256$ 
exceeded the performance of the \SUM{} loss within the same training time.

\subsection{Examples of Hard Negatives}

Fig.~\ref{fig:sample_hard_negatives} shows the hard negatives in a random 
mini-batch. These examples illustrate that hard negatives from a mini-batch can 
provide useful gradient information.

\begin{figure*}[t]
\centering
\scriptsize
\begin{tabular}[h]{>{\centering\arraybackslash}m{0.22\linewidth}
>{\centering\arraybackslash}m{0.22\linewidth}
>{\centering\arraybackslash}m{0.22\linewidth}
>{\centering\arraybackslash}m{0.22\linewidth}}
\makecell[{{p{\linewidth}}}]{\includegraphics[width=\linewidth, height=\linewidth]{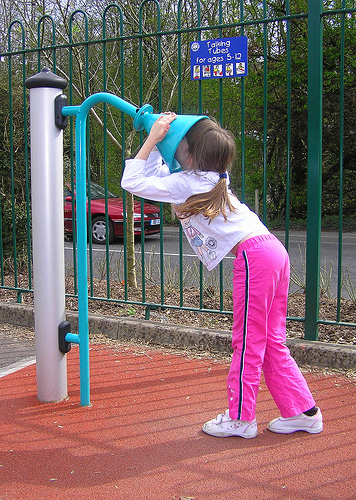}\\[1mm]
{\bf GT}: A little girl wearing pink pants, pink and white tennis shoes and a white shirt with a little girl on it puts her face in a blue Talking Tube. \\[1mm]
{\bf HN}: [0.26] Blond boy jumping onto deck. \\[9mm]
}
&
\makecell[{{p{\linewidth}}}]{\includegraphics[width=\linewidth, height=\linewidth]{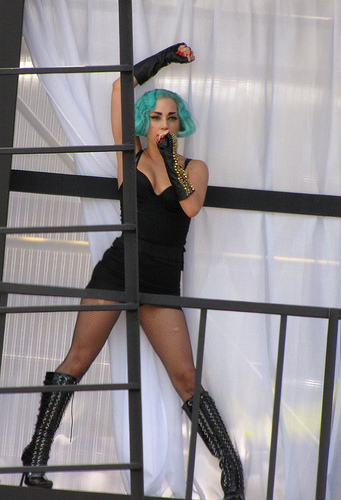}\\[1mm]
{\bf GT}: A teal-haired woman in a very short black dress, pantyhose, and boots standing with right arm raised and left hand obstructing her mouth in microphone-singing fashion is standing. \\[1mm]
{\bf HN}: [0.08] Two dancers in azure appear to be performing in an alleyway. \\[1mm]
}
&
\makecell[{{p{\linewidth}}}]{\includegraphics[width=\linewidth, height=\linewidth]{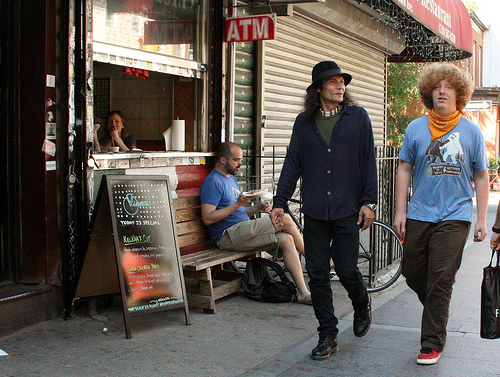}\\[1mm]
{\bf GT}: Two men, one in a dark blue button-down and the other in a light blue tee, are chatting as they walk by a small restaurant. \\[1mm]
{\bf HN}: [0.41] Two men with guitars strapped to their back stand on the street corner with two other people behind them. \\[1mm]
}
&
\makecell[{{p{\linewidth}}}]{\includegraphics[width=\linewidth, height=\linewidth]{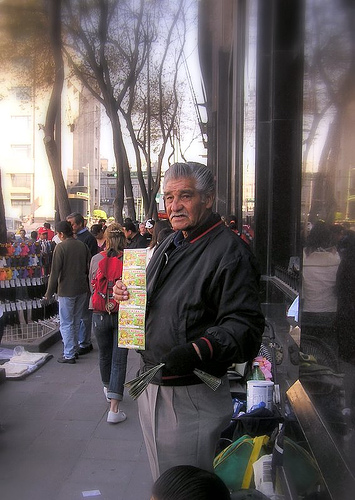}\\[1mm]
{\bf GT}: A man wearing a black jacket and gray slacks, stands on the sidewalk holding a sheet with something printed on it in his hand. \\[1mm]
{\bf HN}: [0.26] Two men with guitars strapped to their back stand on the street corner with two other people behind them. \\[1mm]
}
\\
\makecell[{{p{\linewidth}}}]{\includegraphics[width=\linewidth, height=\linewidth]{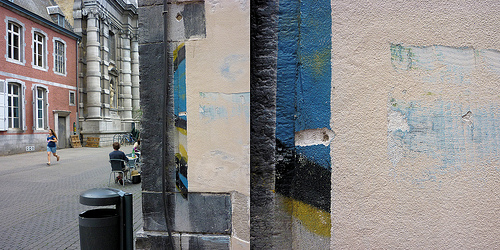}\\[1mm]
{\bf GT}: There is a wall of a building with several different colors painted on it and in the distance one person sitting down and another walking. \\[1mm]
{\bf HN}: [0.06] A woman with luggage walks along a street in front of a large advertisement. \\
}
&
\makecell[{{p{\linewidth}}}]{\includegraphics[width=\linewidth, height=\linewidth]{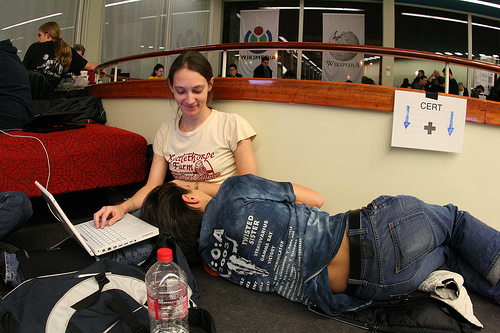}\\[1mm]
{\bf GT}: A man is laying on a girl's lap, she is looking at him, she also has her hand on her notebook computer. \\[1mm]
{\bf HN}: [0.18] A woman sits on a carpeted floor with a baby. \\[11mm]
}
&
\makecell[{{p{\linewidth}}}]{\includegraphics[width=\linewidth, height=\linewidth]{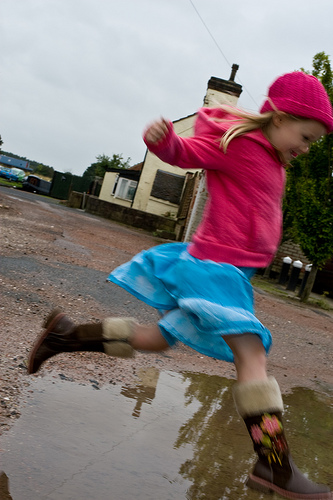}\\[1mm]
{\bf GT}: A young blond girl in a pink sweater, blue skirt, and brown boots is jumping over a puddle on a cloudy day. \\[1mm]
{\bf HN}: [0.51] An Indian woman is sitting on the ground, amongst drawings, rocks and shrubbery. \\[5mm]
}
&
\makecell[{{p{\linewidth}}}]{\includegraphics[width=\linewidth, height=\linewidth]{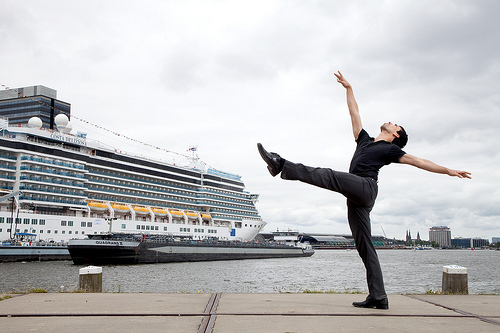}\\[1mm]
{\bf GT}: One man dressed in black is stretching his leg up in the air, behind him is a massive cruise ship in the water. \\[1mm]
{\bf HN}: [0.24] A topless man straps surfboards on top of his blue car. \\[5mm]
}
\\
\end{tabular}
\caption{Examples from the \fthk{} training set along with their hard negatives 
    in a random mini-batch according to the loss of a trained \VSEpp{} model.  
    The value in brackets is the cost of the hard negative and is in the range 
    $[0, 2]$ in our implementation.  HN is the hardest negative in a random 
    sample of size $128$.  GT is the positive caption used to compute the cost 
    of NG.}\label{fig:sample_hard_negatives}
\end{figure*}

Fig.~\ref{fig:sample_outputs} provides additional examples comparing the 
outputs of \VSEpp{} and \VSEz{}.

\begin{figure*}[t]
\centering
\scriptsize
\begin{tabular}[h]{>{\centering\arraybackslash}m{0.22\linewidth}
>{\centering\arraybackslash}m{0.22\linewidth}
>{\centering\arraybackslash}m{0.22\linewidth}
>{\centering\arraybackslash}m{0.22\linewidth}}
\makecell[{{p{\linewidth}}}]{\includegraphics[width=\linewidth, height=\linewidth]{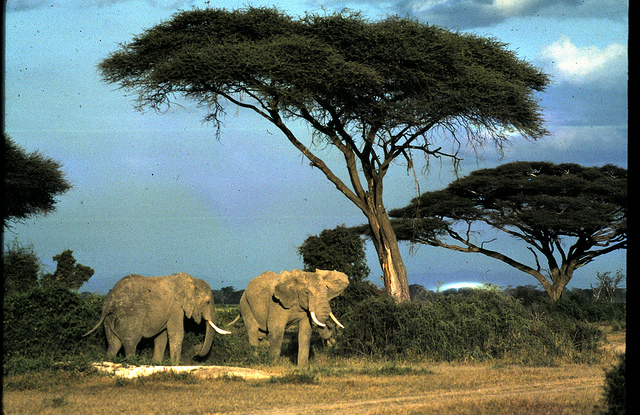}\\[1mm]
{\bf GT}: Two elephants are standing by the trees in the wild.  \\[1mm]
{\bf \VSEz{}}: [9] Three elephants kick up dust as they walk through the flat by the bushes. \\[1mm]
{\bf \VSEpp{}}: [1] A couple elephants walking by a tree after sunset.\\[2mm]
}
&
\makecell[{{p{\linewidth}}}]{\includegraphics[width=\linewidth, height=\linewidth]{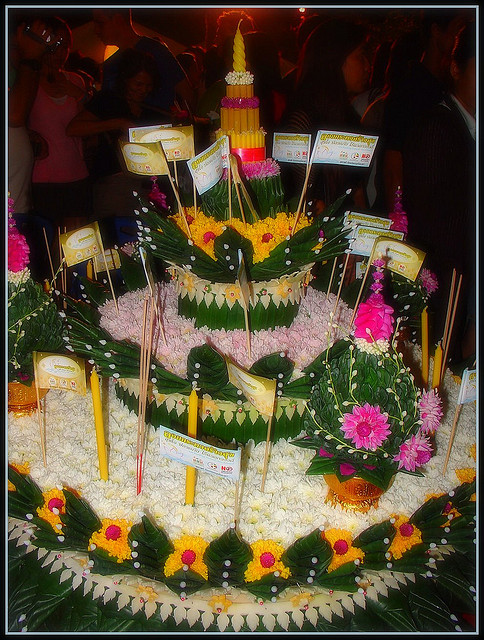}\\[1mm]
{\bf GT}: A large multi layered cake with candles sticking out of it.  \\[1mm]
{\bf \VSEz{}}: [1] A party decoration containing flowers, flags, and candles. \\[1mm]
{\bf \VSEpp{}}: [1] A party decoration containing flowers, flags, and candles. \\[4mm]
}
&
\makecell[{{p{\linewidth}}}]{\includegraphics[width=\linewidth, height=\linewidth]{\figdir/COCO_val2014_000000408449.jpg}\\[1mm]
{\bf GT}: The man is walking down the street with no shirt on. \\[1mm]
{\bf \VSEz{}}: [24] A person standing on a skate board in an alley. \\[1mm]
{\bf \VSEpp{}}: [10] Two young men are skateboarding on the street. \\[4mm]
}
&
\makecell[{{p{\linewidth}}}]{\includegraphics[width=\linewidth, height=\linewidth]{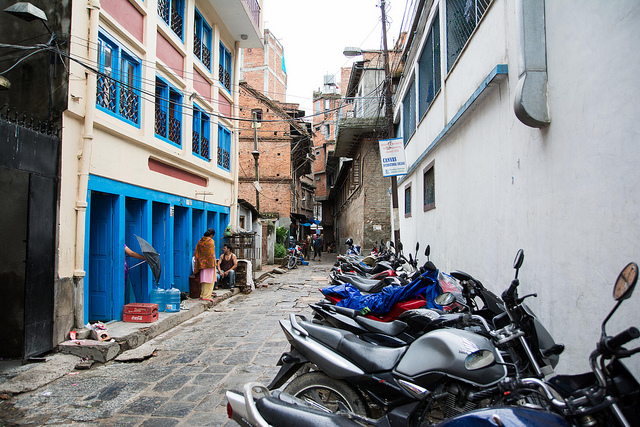}\\[1mm]
{\bf GT}: A row of motorcycles parked in front of a building. \\[1mm]
{\bf \VSEz{}}: [2] a parking area for motorcycles and bicycles along a street \\[1mm]
{\bf \VSEpp{}}: [1] A number of motorbikes parked on an alley \\[4mm]
}
\\
\makecell[{{p{\linewidth}}}]{\includegraphics[width=\linewidth, height=\linewidth]{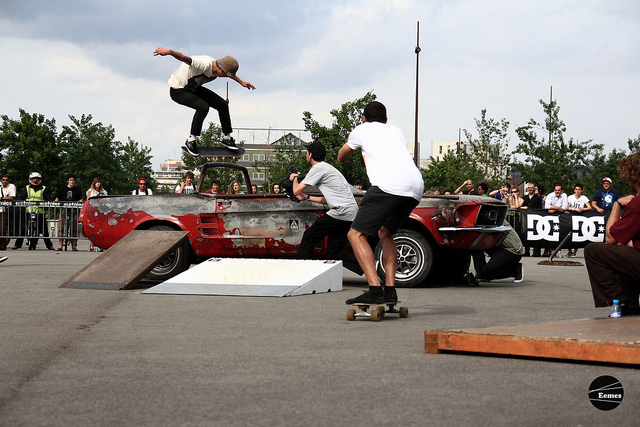}\\[1mm]
{\bf GT}: some skateboarders doing tricks and people watching them \\[1mm]
{\bf \VSEz{}}: [39] Young skateboarder displaying skills on sidewalk near field. \\[1mm]
{\bf \VSEpp{}}: [3] Two young men are outside skateboarding together. \\
}
&
\makecell[{{p{\linewidth}}}]{\includegraphics[width=\linewidth, height=\linewidth]{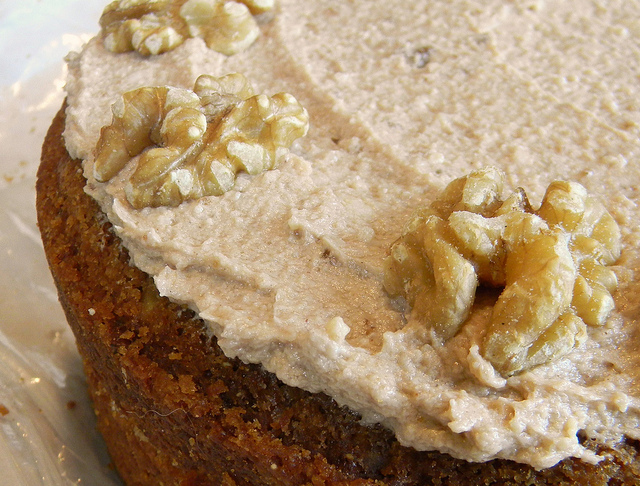}\\[1mm]
{\bf GT}: a brown cake with white icing and some walnut toppings \\[1mm]
{\bf \VSEz{}}: [6] A large slice of angel food cake sitting on top of a plate. \\[1mm]
{\bf \VSEpp{}}: [16] A baked loaf of bread is shown still in the pan. \\
}
&
\makecell[{{p{\linewidth}}}]{\includegraphics[width=\linewidth, height=\linewidth]{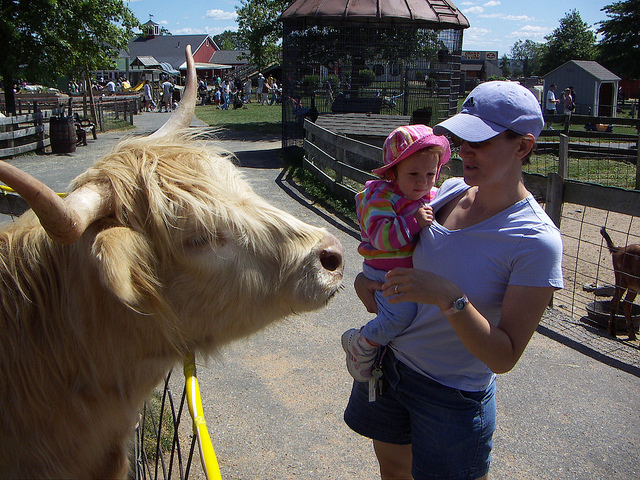}\\[1mm]
{\bf GT}: A woman holding a child and standing near a bull. \\[1mm]
{\bf \VSEz{}}: [1] A woman holding a child and standing near a bull. \\[1mm]
{\bf \VSEpp{}}: [1] A woman holding a child looking at a cow. \\
}
&
\makecell[{{p{\linewidth}}}]{\includegraphics[width=\linewidth, height=\linewidth]{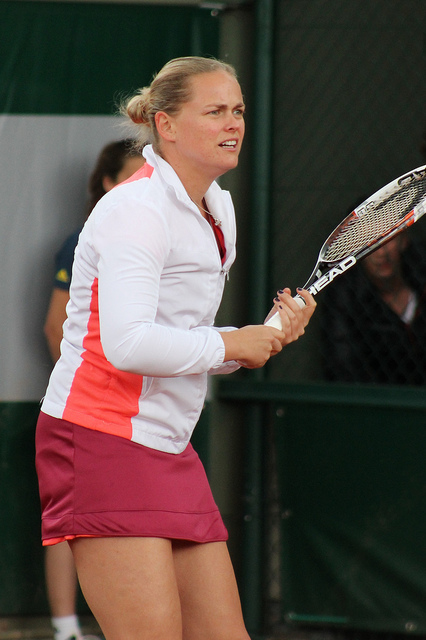}\\[1mm]
{\bf GT}: A woman in a short pink skirt holding a tennis racquet. \\[1mm]
{\bf \VSEz{}}: [6] A man playing tennis and holding back his racket to hit the ball. \\[1mm]
{\bf \VSEpp{}}: [1] A woman is standing while holding a tennis racket. \\
}
\end{tabular}
    \vspace{.3cm}
\caption{Examples of \coco{} test images and the top 1 retrieved captions for \VSEz{} and \VSEpp{} 
(ResNet)-finetune. The value in brackets is the rank of the highest ranked 
ground-truth caption. GT is a sample from the ground-truth captions.}
    \label{fig:sample_outputs}
\end{figure*}%

\subsection{Distribution of distances}

Fig.~\ref{fig:dists_f30k} compares the distribution of distances
for hard negatives using \MAX{} loss versus the \SUM{} loss after $10$ and $30$ 
epochs.
For the \MAX{} loss, the distribution of distances is negatively skewed after 
$10$ epochs, while the distribution using the \SUM{} loss is approximately 
symmetric. The reason is, the \MAX{} loss disproportionately focuses on the 
hardest negatives that results in a sharp drop in the number of hard negatives 
after cosine similarity $0.4$ while a large stack of hard negatives is created 
before $0.4$. As the training progresses to epoch $30$, the stack of hard 
negatives is dispersed but the drop to zero probability is still sharper than 
the \SUM{} loss.  At Epoch $30$, this margin is reduced to $0.3$.

\begin{figure}[t]
    \centering
    \includegraphics[width=.5\linewidth]{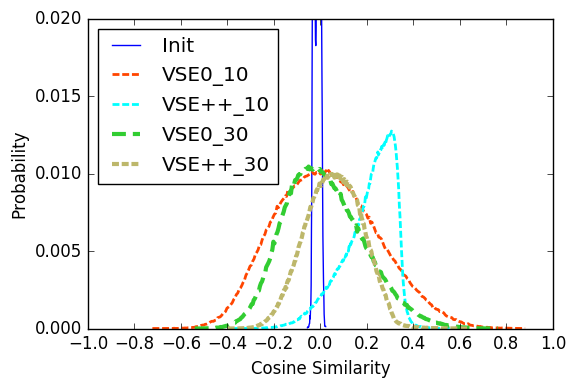}
        \caption{Distribution of distances for \fthk{} training set averaged 
        over $10$ fixed examples over the course of training. The distribution 
        for the \MAX{} loss and the \SUM{} loss is shown for snapshots of the 
        model at epoch $10$ and $30$. For comparison, the distribution for 
        a randomly initialized model is also depicted.}
        \label{fig:dists_f30k}
\end{figure}

\subsection{Effect of Negative Set Size on \MAX{} Loss}

\begin{figure*}[t]
    \centering
    \includegraphics[width=0.5\linewidth]{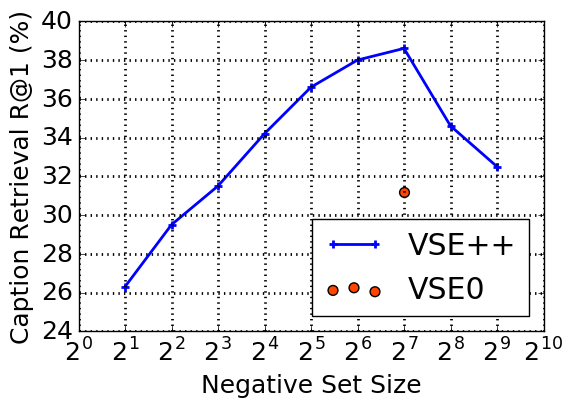}
    \caption{The effect of the negative set size on the R@$1$ performance. The 
    optimal R@$1$ is achieved near mini-batch size $128$.}
    \label{fig:neg_size}
\end{figure*}

Our proposed hard negatives are found in the mini-batch of triplets in each 
training step that is a uniformly sampled subset of the training set. Instead, 
one could find the hardest negatives within the training set which is 
computationally expensive but potentially performs better. We study the 
effective sample size over which we searched for negatives (while keeping the 
mini-batch size fixed at $128$). In the extreme case, when the negative set is 
the training set, we get the hardest negatives in the entire training set.  As 
discussed in Sec.~\ref{sec:hard_neg}, sampling a negative set smaller than the 
training set can potentially be more robust to label errors.

In \cref{fig:neg_size}, we show the effect of the negative sample set size on 
the performance of \MAX{} loss.  We compare the caption retrieval performance 
for different negative set sizes varied from $2$ to $512$.  In practice, for 
negative set sizes smaller than the mini-batch size, $128$, we randomly sample 
the negative set from the mini-batch, and when the mini-batch size is smaller 
than the negative set size, we randomly sample the mini-batch from the negative 
set.  We observe that on this dataset, the optimal negative set size is around 
$128$.  Interestingly, for negative sets as small as $2$, R@$1$ is slightly 
below \VSEz{}. We observe a drop in performance as the negative set size is 
increased to $512$. We hypothesize that this number is dataset dependant and 
for a small dataset like \fthk{}, the probability of sampling a noisy example 
increases significantly. Even though the performance drops with larger 
mini-batch sizes, it still performs better than the \SUM{} loss.

\section{Conclusion}
This chapter focused on learning visual-semantic embeddings for
cross-modal, image-caption retrieval.  Inspired by structured prediction, 
we proposed a new loss based on violations incurred by relatively hard 
negatives compared to current methods that used expected 
errors~\citep{kiros2014unifying,vendrov2015order}. We performed experiments  on 
the \coco{} and \fthk{} datasets and showed that our proposed loss 
significantly  improves performance on these datasets. We observed that the 
improved loss can better guide a more powerful image encoder, ResNet152, and 
also guide better when fine-tuning an image encoder.  At the time of 
publication, with all modifications, our \VSEpp{} model achieved 
state-of-the-art performance on the \coco{} dataset, and was slightly below the 
best model at the time on the \fthk{} dataset.  Our proposed loss function can 
be used to train more sophisticated models that have been using a similar 
ranking loss for training.

%% file: ch-gvar.tex
\chapter{Gradient Clustering and A Study of Gradient Variance in Deep Learning}
\label{ch:gvar}

The second problem addressed in this thesis is training efficiency as a general 
challenge in training deep neural networks. Instead of focusing on a particular 
task-specific loss function, we seek improvements to optimization speed in 
gradient-based optimization methods for deep learning.
We revisit our hypothesis from the first problem that training data are not 
equally important. We hypothesize that training deep learning models on diverse 
and heterogeneous data distributions is slower than homogeneous data 
distributions with less diversity.

In this context, we study the distribution of gradients during training.  We 
introduce a method, Gradient Clustering, to minimize the variance of average 
mini-batch gradient with stratified sampling.  We prove that the variance of 
average mini-batch gradient is minimized if the elements are sampled from 
a weighted clustering in the gradient space.  We measure the gradient variance 
on common deep learning benchmarks and observe that, contrary to common 
assumptions, gradient variance increases at the beginning of the training, and 
smaller learning rates coincide with higher variance.
In addition, we introduce normalized gradient variance as a statistic that 
better correlates with the speed of convergence compared to gradient variance.

The content of this chapter have appeared in the following publication:

\begin{itemize}
    \item {\bf Faghri, Fartash} and Duvenaud, David and Fleet, David J.\ and Ba, 
        Jimmy,
        {\sl ``Gluster: Variance Reduced Mini-Batch SGD with Gradient Clustering"}, 
        Conference on Neural Information Processing Systems (NeurIPS), Workshop on 
        Beyond First Order Methods in ML, 2019.
\end{itemize}

Additionally, the results have been further developed in the following 
publications:

\begin{itemize}
    \item {\bf Faghri, Fartash}, Tabrizian, Iman and Markov, Ilia and Alistarh, Dan 
        and Roy, Dan M.\ and Ramezani-Kebrya, Ali,
        {\sl ``Adaptive Gradient Quantization for Data-Parallel SGD"}, Conference 
        on Neural Information Processing Systems (NeurIPS), 2020.
    \item Ramezani-Kebrya, Ali and {\bf Faghri, Fartash} and
        Markov, Ilya and Aksenov, Vitalii and Alistarh, Dan  and Roy, Dan M.\ 
        et al.,
        {\sl ``NUQSGD: Provably Communication-efficient Data-parallel SGD via 
        Nonuniform Quantization"}, Journal of Machine Learning Research 
        22.114 (2021): 1-43.
\end{itemize}

To ensure reproducibility, our code is publicly 
available~\footnote{\url{https://github.com/fartashf/gvar_code}}~\footnote{\url{https://github.com/fartashf/foptim}}~\footnote{\url{https://github.com/fartashf/nuqsgd}}.

\input{ch-gvar-tex/body}

%% file: ch-gvar-tex/body.tex
\def\figdir{ch-gvar-tex/figures}

\section{Introduction}

Many machine learning tasks entail the minimization of the risk,
$\E_{\bm{x}}[\ell(\bm{x}; \param)]$,
where $\bm{x}$ is an i.i.d.\ sample from a data distribution, and $\ell$ is the 
per-example loss parametrized by $\param$.  In supervised learning, inputs and 
ground-truth labels comprise $\bm{x}$, and $\param$ is a vector of model 
parameters.
Empirical risk approximates the population risk by the risk of a sample set 
$\{\xxi\}_{i=1}^N$, the training set, as $L(\param)=\SI \ell(\xxi; \param)/N$.
Empirical risk is often minimized using gradient-based optimization 
(first-order methods).
For differentiable loss functions, the gradient of $\bm{x}$ is defined as 
$\dth{}\ell(\bm{x}; \param)$, i.e., the gradient of the loss with respect to 
the parameters evaluated at a point $\bm{x}$.
Popular in deep learning, Mini-batch Stochastic Gradient Descent (mini-batch 
SGD) iteratively takes small steps in the opposite direction of the average 
gradient of $B$ training samples.
The mini-batch size is a hyper-parameter that provides flexibility in trading 
per-step computation time for potentially fewer total steps. In GD the 
mini-batch is the entire training set while in SGD it is a single sample.

In general, using any unbiased stochastic estimate of the gradient and 
sufficiently small step sizes, SGD is guaranteed to converge to a minimum for 
various function classes~\citep{robbins1951stochastic}. Common convergence 
bounds in stochastic optimization improve with smaller gradient 
variance~\citep{bottou2018optimization}. Mini-batch SGD is said to converge 
faster because the variance of the gradient estimates is reduced by a rate 
linear in the mini-batch size. In practice however, we observe {\bf diminishing 
returns}\/ in speeding up the training of almost any deep model on deep 
learning benchmarks~\citep{shallue2018measuring}.
The transition point to diminishing returns is known to depend on the choice of 
data, model and optimization method.
\citet{zhang2019algorithmic} observed that the limitation of acceleration in 
large batches is reduced when momentum or preconditioning is used.  Other works 
suggest that very small mini-batch sizes can still converge fast enough using 
a collection of tricks~\citep{golmant2018, masters2018revisiting, lin2018}. One 
hypothesis is that the stochasticity due to small mini-batches improves 
generalization by finding ``flat minima'' and avoiding ``sharp 
minima''~\citep{goodfellow2015qualitatively,keskar2017large}.  But this 
hypothesis does not explain why diminishing returns also happens in the 
training loss.

Motivated by the diminishing returns phenomena, we study and model the 
distribution of the gradients. Given a data distribution characterized by its 
probability density function, $\Proba(\bm{x})$, the gradient distribution is 
defined as a transformation of the data distribution by the gradient of the 
loss function, $\dth{}\ell(\bm{x}; \param)$. The transformation is a function 
of the model parameters that can be deterministic, e.g., with a linear model or 
stochastic, for example, when using random data augmentation or dropout.  As 
such, the gradient distribution varies across model architectures and evolves 
during the training.

An unbiased gradient estimator is an estimate of the mean of the gradient 
distribution that is commonly evaluated by its variance.
In the noisy gradient view, the average mini-batch gradient (or the mini-batch 
gradient) is an unbiased estimator of the expected gradient where increasing 
the mini-batch size reduces the variance of this estimator.
We propose a distributional view and argue that knowledge of the gradient 
distribution can be exploited to analyze and improve optimization speed as well 
as generalization to test data. A mean-aware optimization method is at best as 
strong as a distributional-aware optimization method.
In our distributional view, the mini-batch gradient is only an estimate of the 
mean of the gradient distribution.

{\bf Questions:} We identify the following questions about the gradient 
distribution.
\begin{itemize}
    \item[] {\bf Structure of gradient distribution.}\/ Is there structure in 
        the distribution over gradients of standard learning problems?
    \item[] {\bf Impact of gradient distribution on optimization.}\/ What 
        characteristics of the gradient distribution correlate with the 
        convergence speed and the minimum training/test loss reached?
    \item[] {\bf Impact of optimization on gradient distribution.}\/ To what 
        extent do the following factors affect the gradient distribution:  data 
        distribution, learning rate, model architecture, mini-batch size, 
        optimization method, and the distance to local optima?
\end{itemize}
As we review in \cref{sec:gvar:related}, recent work have begun to investigate 
aforementioned questions but we are far from a comprehensive understanding.

\begin{figure}[t]
    \centering
    \begin{subfigure}[b]{\threecolfigwidth}
        \includegraphics[width=\textwidth]{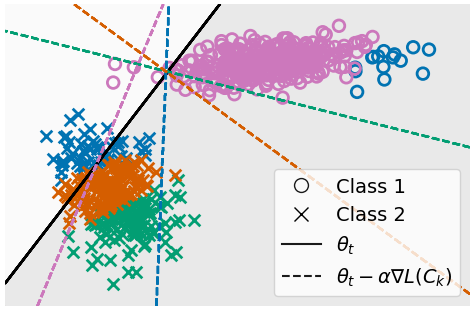}
        \caption{GD step 0}
    \end{subfigure}
    \begin{subfigure}[b]{\threecolfigwidth}
        \includegraphics[width=\textwidth]{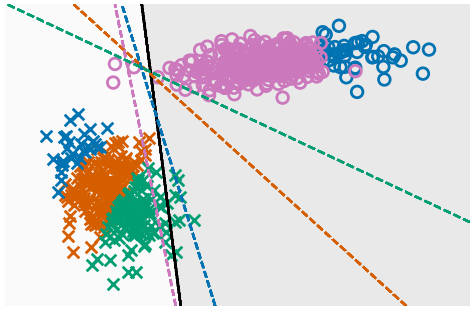}
        \caption{GD step 1}
    \end{subfigure}
    \begin{subfigure}[b]{\threecolfigwidth}
        \includegraphics[width=\textwidth]{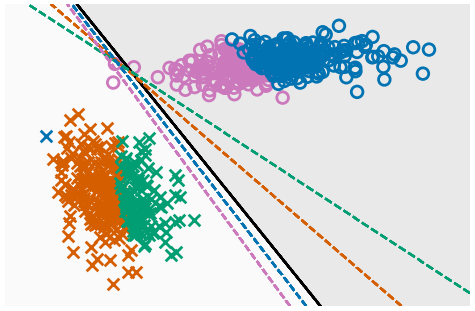}
        \caption{GD step 2}
    \end{subfigure}
    \caption{{\bf Example of clusters found using Gradient Clustering.}\/ 
    A linear classifier visualized during training with gradient descent on 
    2 linearly separable classes (o, x).  Gradients are assigned to $4$ 
    clusters (different colors) using Gradient Clustering (\Gluster).  Black 
    line depicts current decision boundary.  Colored dashed lines depict 
    decision boundaries predicted from current boundary and each of the $4$ 
    individual clusters.  Here, blue points belong to both classes; they have 
    similar gradients, but are far apart in input space. By exploiting the 
    knowledge of \Gluster we can get low variance average mini-batch gradients.
    }
    \label{fig:2dvis_single}
\end{figure}

{\bf Contributions:}
\begin{itemize}
    \item[] {\bf Exploiting clustered distributions.}\/ We consider gradient 
        distributions with distinct modes, i.e., the gradients can be 
        clustered.
        We prove that the variance of average mini-batch gradient is minimized 
        if the elements are sampled from a weighted clustering in gradient 
        space (\cref{sec:gvar}).
    \item[] {\bf Efficient clustering to minimize variance.}\/ We propose 
        Gradient Clustering (\Gluster) as a computationally efficient method 
        for clustering in the gradient space (\cref{sec:gluster}).
        \cref{fig:2dvis_single} shows an example of clusters found by \Gluster.
    \item[] {\bf Relation between gradient variance and optimization.}\/ We 
        study the gradient variance on common deep learning benchmarks (MNIST, 
        CIFAR-10, and ImageNet) as well as Random Features models recently 
        studied in deep learning theory (\cref{sec:exp}). We observe that 
        gradient variance increases for most of the training, and smaller 
        learning rates coincide with higher variance.
    \item[] {\bf An alternative statistic.}\/ We introduce normalized gradient 
        variance as a statistic that better correlates with the speed of 
        convergence compared to gradient variance (\cref{sec:exp}).
\end{itemize}

We emphasize that some of our contributions are primarily empirical yet 
unexpected. We believe our results provide an opportunity for future 
theoretical and empirical work.

\section{Related Work}
\label{sec:gvar:related}

{\bf Modeling gradient distribution.}\/
Despite various assumptions on the mini-batch gradient variance, only recently 
these assumptions have been scrutinized for deep learning models.
It is common to assume bounded variance in convergence 
analyses~\citep{bottou2018optimization}. Works on variance reduction propose 
alternative estimates of the gradient mean with low 
variance~\citep{leroux2012stochastic, johnson2013accelerating} but they do not 
plot the variance which is the actual quantity they seek to reduce.  Their 
ineffectiveness in deep learning has been observed but still requires 
explanation~\citep{defazio2019ineffectiveness}.
There are works that present gradient variance plots~\citep{mohamed2020monte, 
wen2019interplay} but they are usually for a single gradient coordinate and 
synthetic problems.
The Central limit theorem is also used to argue that the distribution of the 
mini-batch gradient is a Gaussian~\citep{zhu2018anisotropic}, which has been 
challenged only recently~\citep{simsekli2019tail,wu2020noisy}. The observation 
that the gradient noise is heavy tailed has been used to justify the 
superiority of the Adam optimizer in training attention 
models~\citep{zhang2019adam}. There also exists a link between the 
Fisher~\citep{amari1998natural}, Neural Tangent Kernel~\citep{jacot2018neural}, 
and the gradient covariance matrix~\citep{martens2014new, 
kunstner2019limitations, thomas2020interplay}. As such, any analysis of 
one~\citep[e.g.,][]{karakida2019pathological} could potentially be used to 
understand others.

{\bf Variance reduction as part of optimization methods for deep learning.}
Variance reduction is an important technique in gradient estimation at the core 
of many machine learning problems~\citep{mohamed2020monte}. A variance reduced 
gradient estimator can then be used with a gradient-based optimization method.  
In such approaches, the performance of the optimizer depends on the gradient 
estimator but there is no feedback from the optimizer to improve the gradient 
estimator.  Our work is related to variance reduction methods that directly 
modify the optimization method.
\citet{le2011improving} considered the difference between the covariance matrix 
of the gradients and the Fisher matrix and proposed incorporating the 
covariance matrix as a measure of model uncertainty in optimization. It has 
also been suggested that the division by the second moments of the gradient in 
Adam can be interpreted as variance adaptation~\citep{kunstner2019limitations}.
Although we do not use Gradient Clustering for optimization, the formulation 
can be interpreted as a unifying approach that defines variance reduction as an 
objective.

{\bf Importance Sampling for Optimization.}
Our work is also closely related to importance sampling for stochastic 
optimization where data points are sampled according to a measure of importance 
such as the loss or the norm of the gradient~\citep{zhao2015stochastic, 
katharopoulos2017biased, csiba2018importance, johnson2018training, 
alain2015variance, needell2014stochastic}.
There are also myriad papers on ad-hoc sampling and re-weighting methods for 
reducing dataset imbalance and increasing data 
diversity~\citep{bengio2008adaptive, jiang2017mentornet, vodrahalli2018, 
jiang2019accelerating}. Based on empirical results, \citet{wu2020curricula} 
suggests that several sampling and ordering methods have only marginal benefits 
on standard datasets when training time is long enough. Our method can be 
viewed as an importance sampling method where the relative size of the cluster 
denotes the importance of its data points. Compared with using only gradient 
norm or the loss for importance sampling, we exploit the entire gradient vector 
for each data point that makes the method significantly more powerful.

{\bf Clustering gradients.}\/
Methods related to gradient clustering have been proposed in low-variance 
gradient estimation~\citep{hofmann2015variance, zhao14accelerating, 
chen2019fast} supported by promising theory.  However, these methods have 
either limited their experiments to linear models or treated a deep model as 
a linear one.  Our proposed \GC method performs efficient clustering in the 
gradient space with very few assumptions.
\GC is also related to works on model visualization where the entire training 
set is used to understand the behaviour of a model~\citep{raghu2017svcca}.

\section{Mini-batch Gradient with Stratified Sampling}\label{sec:gvar}

An important factor affecting optimization trade-offs is the diversity of 
training data. SGD entails a sampling process, often uniformly sampling from 
the training set.  However, as illustrated in the following example, uniform 
sampling is not always ideal. Suppose there are duplicate data points in 
a training set. We can save computation time by removing all but one of the 
duplicates.  To get the same gradient mean in expectation, it is sufficient to 
rescale the gradient of the remaining sample in proportion to the number of 
duplicates. In this example, mini-batch SGD will be inefficient because 
duplicates increase the variance of the average gradient mean.

Suppose we are given i.i.d.\ training data, $\{\xxi\}_{i=1}^N$, and 
a partitioning of their gradients, $\gi=\dth{}\ell(\xxi; \param)$, into $K$ 
clusters, where $\Nk$ is the size of the $k$-th cluster. We can estimate the 
gradient mean on the training set, $\bm{g}=\dth{}L(\param)$, by averaging $K$ 
gradients, one from each of $K$ clusters, uniformly sampled:
\begin{align}
    \gC(\Aa) &= \frac{1}{N} \SK \Nk\, \gpk\,,\quad \gpk \sim \U(S_{k})\,,
    \label{eq:strat_sampl}
\end{align}
where $\gpk$ is a uniformly sampled gradient from the $k$-th cluster 
$S_{k}=\{\gi|\Ai=k\}$, and $\Aa\in\{1,\ldots,K\}^N$ where $\Ai$ is the index of 
the cluster to which $i$-th data point is assigned, so $\Nk=\SI \I(\Ai=k)$.
Each sample is treated as a representative of its cluster and weighted by the 
size of that cluster.  In the limit of $K=N$, we recover the batch gradient 
mean used in GD and for $K=1$ we recover the single-sample stochastic gradient 
in SGD.

\begin{prop}\label{thm:gvar}
    {\bf (Bias/Variance of Mini-batch Gradient with Stratified Sampling).}\/
    For any partitioning of data, the estimator of the gradient mean using 
    stratified sampling (\cref{eq:strat_sampl}) is unbiased ($\E[\gC]=g$) and
    $\V[\gC] = N^{-2} \SK \Nk^2 \V[\gpk]$,
    where $\V[\cdot]$ is defined as the trace of the covariance matrix.
    (Proof in \cref{sec:gvar_proof})
\end{prop}

\begin{rem}
    Under a stratified sampling scheme, in a dataset with duplicate samples, 
    the gradients of duplicates do not contribute to the variance if assigned 
    to the same partition with no other data points.
\end{rem}

\subsection{Weighted Gradient Clustering}
\label{sec:cluster}

Suppose, for a given number of clusters, $K$, we want to find the 
optimal partitioning, i.e., one that minimizes the variance of the gradient 
mean estimator, $\gC$.  For $d$-dimensional gradient vectors, minimizing the 
variance in \cref{thm:gvar}, is equivalent to finding a weighted clustering of 
the gradients of data points,
\begin{align}
    \min_\Aa \V[\gC(\Ai)]
    &= \min_\Aa \SK \Nk^2 \V[\gpk]
    = \min_{\bm{C}, \Aa} \SK \SI \Nk\, \|\Ck-\gi\|^2 \,\I(\Ai=k)\,,
    \label{eq:obj}
\end{align}
where a cluster center, $\Ck\in\R^d$, is the average of the gradients in the 
$k$-th cluster, and ${\V[\gpk]=\frac{1}{\Nk} \SI \|\Ck-\gi\|^2 \,\I(\Ai=k)}$.  
If we did not have the factor $\Nk$, this objective would be equivalent to the 
K-Means objective.  The additional $\Nk$ factors encourage larger clusters to 
have lower variance, with smaller clusters comprising scattered data points.  

If we could store the gradients for the entire training set, the clustering 
could be performed iteratively as a form of block coordinate descent, 
alternating between the following {\em Assignment}\ and {\em Update}\ steps,
i.e., computing the cluster assignments and then the cluster centers:

\noindent\begin{tabular}{>{\centering\arraybackslash} m{.45\textwidth}
    >{\centering\arraybackslash} m{.45\textwidth}}
    \setlength{\tabcolsep}{10pt}
    \centering
    \begin{equation}
    {\text{\bA:}}\quad \Ai = \argmin_k \Nk\, \|\Ck-\gi\|^2
    \label{Estep}
    \end{equation}
    &
    \begin{equation}
    {\text{\bU:}}\quad \Ck = \frac{1}{\Nk} \SI \gi\, \I(\Ai=k) %
    \label{Mstep}
    \end{equation}
\end{tabular}

The \bA step is still too complex given the $\Nk$ multiplier. As such, we first 
solve it for fixed cluster sizes then update $\Nk$ before another \bU step.
These updates are similar to Lloyd's algorithm for K-Means, but 
with the $\Nk$ multipliers, and to Expectation-Maximization for Gaussian
Mixture Models, but here we use hard assignments.  In contrast, the additional 
$\Nk$ multiplier makes the objective more complex in that performing \bAU 
updates does not always guarantee a decrease in the clustering objective.

\subsection{Efficient Gradient Clustering (\Gluster)}\label{sec:gluster}

Performing exact \bAU updates (\cref{Estep,Mstep}) is computationally expensive 
as they require the gradient of every data point. Deep learning libraries 
usually provide efficient methods that compute average mini-batch gradients 
without ever computing full individual gradients.
We introduce Gradient Clustering (\Gluster) for performing efficient \bAU 
updates by breaking them into per-layer operations and introducing a low-rank 
approximation to cluster centers.

For any feed-forward network, we can decompose terms in \bAU updates into 
independent per-layer operations as shown in \cref{alg:full}. The main 
operations are computing $\|\bm{C}_{kl} - \bm{g}_{il}\|^2$ and cluster updates 
$\bm{C}_{\Ai,l} \mathrel{+}= \bm{g}_{il} / N_{\Ai}$ per layer $l$; henceforth, 
we drop the layer index for simplicity.

\begin{figure}[t]
\begin{minipage}[t]{0.33\textwidth}
    \centering
    \begin{algorithm}[H]
        \centering
        \caption{\bA step using \cref{eq:fc:Estep}}
        \begin{algorithmic}
            \FOR{$i=1$ {\bf to} $N$}
              \FOR{$k=1$ {\bf to} $K$} %
                \FOR{$l=1$ {\bf to} $L$} %
                  \STATE $D_{kl} = \| \bm{C}_{kl} - \bm{g}_{il}\|^2$
                \ENDFOR
              \ENDFOR
              \STATE $\bm{S} = \sum_l D_{\cdot l}$
              \STATE $\Ai = \argmin_k \Nk \bm{S}$
            \ENDFOR
        \end{algorithmic}
        \label{alg:A_step}
    \end{algorithm}
\end{minipage}
\hfill
\begin{minipage}[t]{0.30\textwidth}
    \centering
    \begin{algorithm}[H]
        \centering
        \caption{$\Nk$ update}
        \begin{algorithmic}
            \STATE $\Nk = 0,\quad \forall k={1,\cdots,K}$
            \FOR{$i=1$ {\bf to} $N$}
              \STATE $N_{\Ai} \mathrel{+}= 1$
            \ENDFOR
        \end{algorithmic}
        \label{alg:A_step}
    \end{algorithm}
\end{minipage}
\hfill
\begin{minipage}[t]{0.33\textwidth}
    \centering
    \begin{algorithm}[H]
        \centering
        \caption{\bU step using \cref{eq:fc:Mstep}}
        \begin{algorithmic}
            \STATE $\bm{C}_k = 0,\quad \forall k={1,\cdots,K}$
            \FOR{$i=1$ {\bf to} $N$}
              \FOR{$l=1$ {\bf to} $L$} %
                \STATE $\bm{C}_{\Ai,l} \mathrel{+}= \bm{g}_{il} / N_{\Ai}$
              \ENDFOR
            \ENDFOR
        \end{algorithmic}
        \label{alg:U_step}
    \end{algorithm}
\end{minipage}
    \caption{\textbf{Gradient Clustering Algorithm steps} applied on the 
    gradients of a deep neural network with $L$ layers.  $\bm{g}_{il}$ denotes 
    the gradients of the $i$-th example w.r.t.\ the parameters of the $l$-th 
    layer, $\bm{C}_{kl}$ denotes $k$-th cluster center for the $l$-th layer, 
    $a_i$ denotes the assignment index for the $i$-th data point and $N_k$ 
    denotes the size of the $k$-th cluster.}
    \label{alg:full}
\end{figure}

For a single fully-connected layer, we denote the layer weights by $\lparam\in 
\R^{\Di\times \Do}$, where $\Di$ and $\Do$ denote the input and output 
dimensions for the layer.  We denote the gradient with respect to $\lparam$ for 
the training set by $\bm{g}=\AAA \DDD^\top$, where $\AAA\in \R^{\Di\times N}$ 
comprises the input activations to the layer, and $\DDD\in \R^{\Do\times N}$ 
represents the gradients with respect to the layer outputs.
The coordinates of cluster centers corresponding to this layer are denoted by 
$C\in \R^{K\times \Di\times \Do}$.  We index the clusters using $k$ and the 
data by $i$. The $k$-th cluster center is approximated as
${\Ck=\ck \dk^\top}$, using vectors ${\ck\in \R^\Di}$ and ${\dk\in\R^\Do}$.

In the \bA step we need to compute ${\|\Ck-\gb\|_F^2}$ as part of the 
assignment cost, where $\|\cdot\|_F$ is the Frobenius-norm. We expand this term 
into three inner-products, and compute them separately. In particular, the term 
$\vvv{\Ck}\odot\vvv{\gb}$ can be written as,
\begin{align}
    \vvv{\Ck}\odot \vvv{\Ab \Db^\top}
    = (\Ab \odot \ck)(\Db\odot \dk)\,,
    \label{eq:fc:Estep}
\end{align}
where $\odot$ denotes inner product, and the RHS is the product of two scalars.  
Similarly, we compute the other two terms in the expansion of the assignment 
cost, i.e., $\vvv{\Ck}\odot\vvv{\Ck}$ and $\vvv{\gb}\odot\vvv{\gb}$ 
(\citet{goodfellow2015efficient} proposed a similar idea to compute the 
gradient norm).

The \bU step in \cref{Mstep} is written as,
$\ck \dk^\top
= \Nk^{-1} \SB \Ab \Db^\top \I(\Ai=k)$.
This equation might have no exact solution for $\ck$ and $\dk$ because the sum 
of rank-$1$ matrices is not necessarily rank-$1$.  One approximation is the 
min-Frobenius-norm solution to $\ck, \dk$ using truncated SVD, where we use 
left and right singular-vectors corresponding to the largest singular-value of 
the RHS.
However, the following updates are exact if
activations and gradients of the outputs are uncorrelated, i.e., $\E_i[\Ab 
\Db]=\E_i[\Ab]E_i[\Db]$ (similar to assumptions in 
K-FAC~\citep{martens2015kfac}),
\begin{align}
    \ck = \frac{1}{\Nk} \SB \Ab\I(\Ai=k)
    \qquad\qquad
    \dk = \frac{1}{\Nk} \SB \Db\I(\Ai=k)\,.
    \label{eq:fc:Mstep}
\end{align}

In \cref{sec:conv}, we describe similar update rules for convolutional layers 
and in \cref{sec:comp}, we provide complexity analysis of \Gluster.
We can make the cost of \Gluster negligible by making sparse incremental 
updates to cluster centers using mini-batch updates. The assignment step can 
also be made more efficient by processing only a portion of data as is common 
for training on large datasets. The rank-$1$ approximation can be extended to 
higher rank approximations with multiple independent cluster centers though 
with challenges in the implementation.

\section{Experiments}\label{sec:exp}

In this section, we evaluate the accuracy of estimators of the gradient mean.  
This is a surrogate task for evaluating the performance of a model of the 
gradient distribution. We encourage the reader to predict the behaviour of 
gradient estimators before advancing in this section. In particular, does the 
reader expect the gradient variance to increase or decrease during the 
training? Surprisingly, we find that the gradient variance often increases in 
the majority of training.

We compare our proposed \GC estimator to average mini-batch Stochastic Gradient 
(\SG), and \SG with double the mini-batch size (\SGdB).
\SGdB is an important baseline for two reasons. First, it is a competitive 
baseline that always reduces the variance by a factor of $2$ and requires at 
most twice the memory size and twice the run-time per 
mini-batch~\citep{shallue2018measuring}.
Second, the extra overhead of \Gluster is approximately the same as keeping an 
extra mini-batch in the memory when the number of clusters is equal to the 
mini-batch size.
We also include Stochastic Variance Reduced Gradient 
(SVRG)~\citep{johnson2013accelerating} as a method with the sole objective of 
estimating gradient mean with low variance.

We compare methods on \textbf{a single trajectory of mini-batch SGD} to 
decouple the optimization from gradient estimation.  That is, we do not train 
with any of the estimators (hence no `D' in \SG and \SGdB).  This allows us to 
continue analyzing a method even after it fails in reducing the variance.  For 
training results using \SG, \SGdB and, SVRG, we refer the reader to 
\citet{shallue2018measuring, defazio2019ineffectiveness}.  For training with 
\GC, it suffices to say that behaviours observed in this section are directly 
related to the performance of \GC used for optimization.

As all estimators in this work are unbiased, the estimator with lowest variance 
is better estimating the gradient mean.
We define {\it Average Variance}\/ (variance in short) as the average over all 
coordinates of the variance of the gradient mean estimate for a fixed model 
snapshot.  Average variance is the normalized trace of the covariance matrix 
and of particular interest in random matrix theory~\citep{tao2012topics}.

We also measure {\it Normalized Variance}, defined as $\V[g]/E[g^2]$ where the 
variance of a $1$-dimensional random variable is divided by its second 
non-central moment.
In signal processing, the inverse of this quantity is the signal to noise ratio 
(SNR). If SNR is less than one (normalized variance larger than one), the power 
of the noise is greater than the signal.
Normalized variance also appears in standard convergence analysis of stochastic 
gradient descent~\citep{friedlander2012hybrid}. As we show in 
\cref{sec:normalized_var_analysis}, it can be shown that for Lipschitz 
continuous functions, variance reduction is only useful if the normalized 
variance is larger than constant $1$. As such, it better correlates with the 
convergence speed compared with variance.

Additional details of the experimental setup can be found in \cref{app:exp}.

\subsection{MNIST: Low Variance, CIFAR-10: Noisy Estimates, ImageNet: No 
Structure}\label{sec:exp_image}

\begin{figure}[t]
    \centering
    \begin{subfigure}[b]{\threecolfigwidth}
        \includegraphics[width=\textwidth]{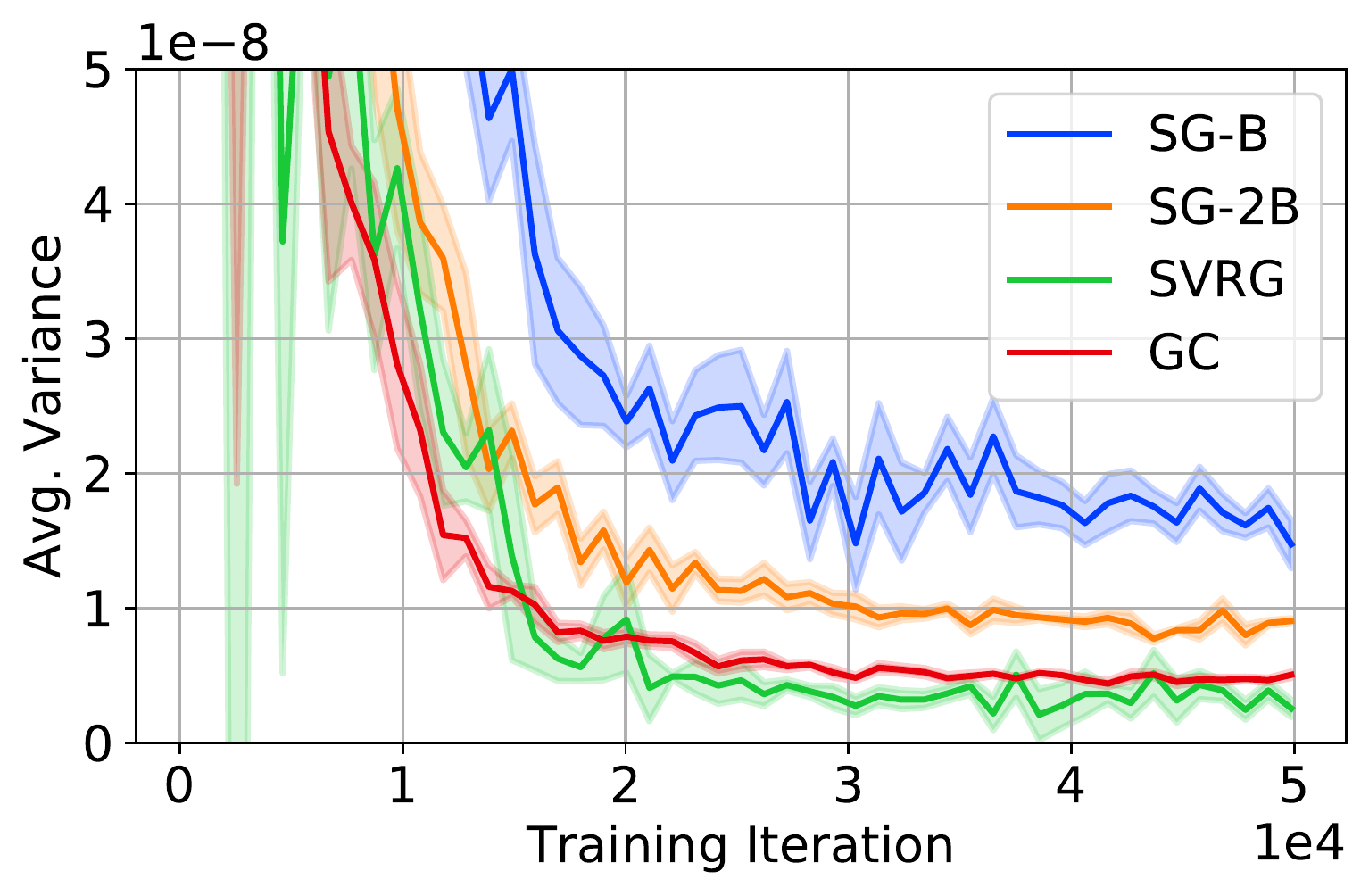}
        \caption{MLP on MNIST}
        \label{fig:mnist_var}
    \end{subfigure}
    \begin{subfigure}[b]{\threecolfigwidth}
        \includegraphics[width=\textwidth]{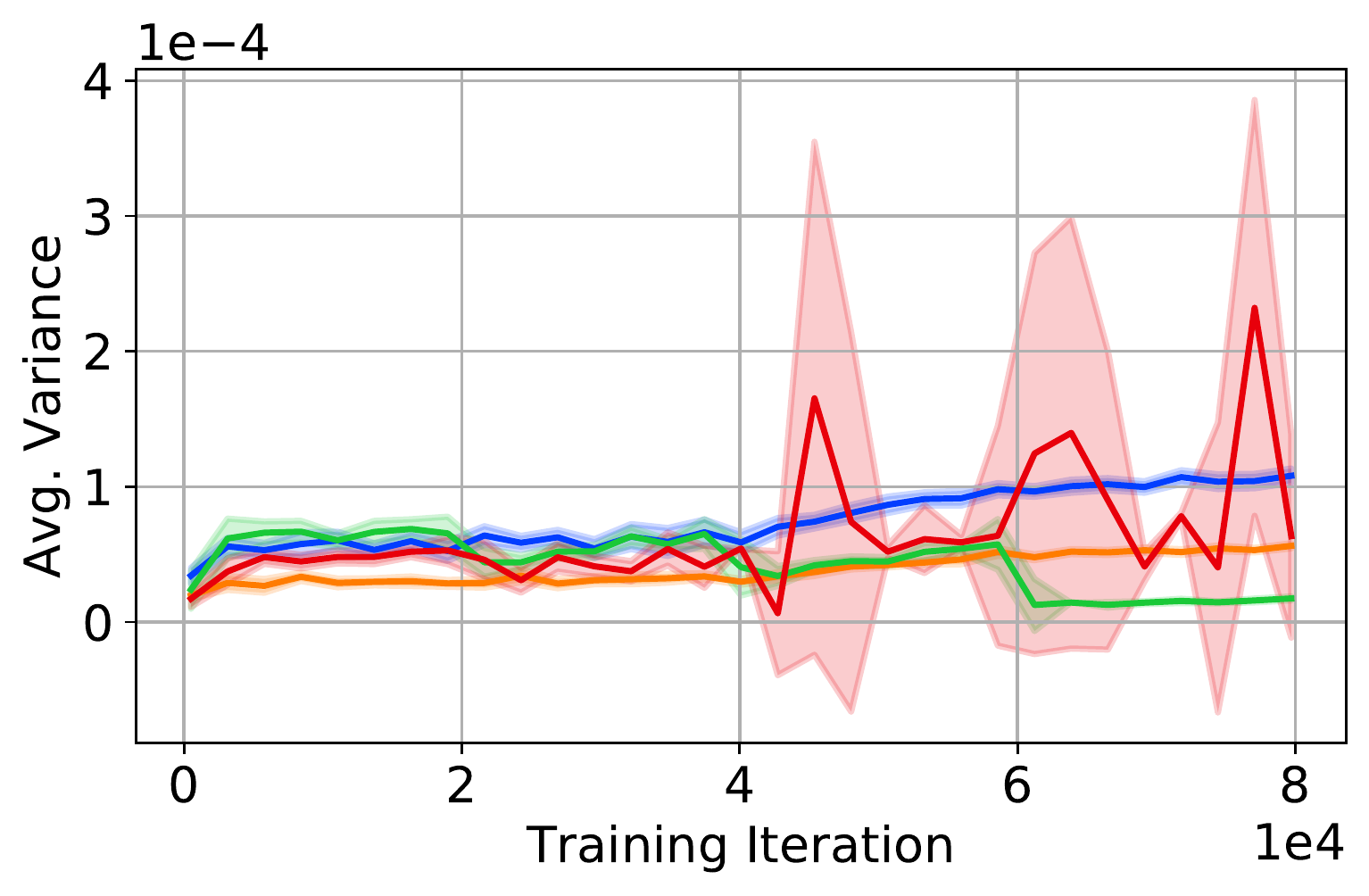}
        \caption{ResNet8 on CIFAR-10}
        \label{fig:cifar10_var}
    \end{subfigure}
    \begin{subfigure}[b]{\threecolfigwidth}
        \includegraphics[width=\textwidth]{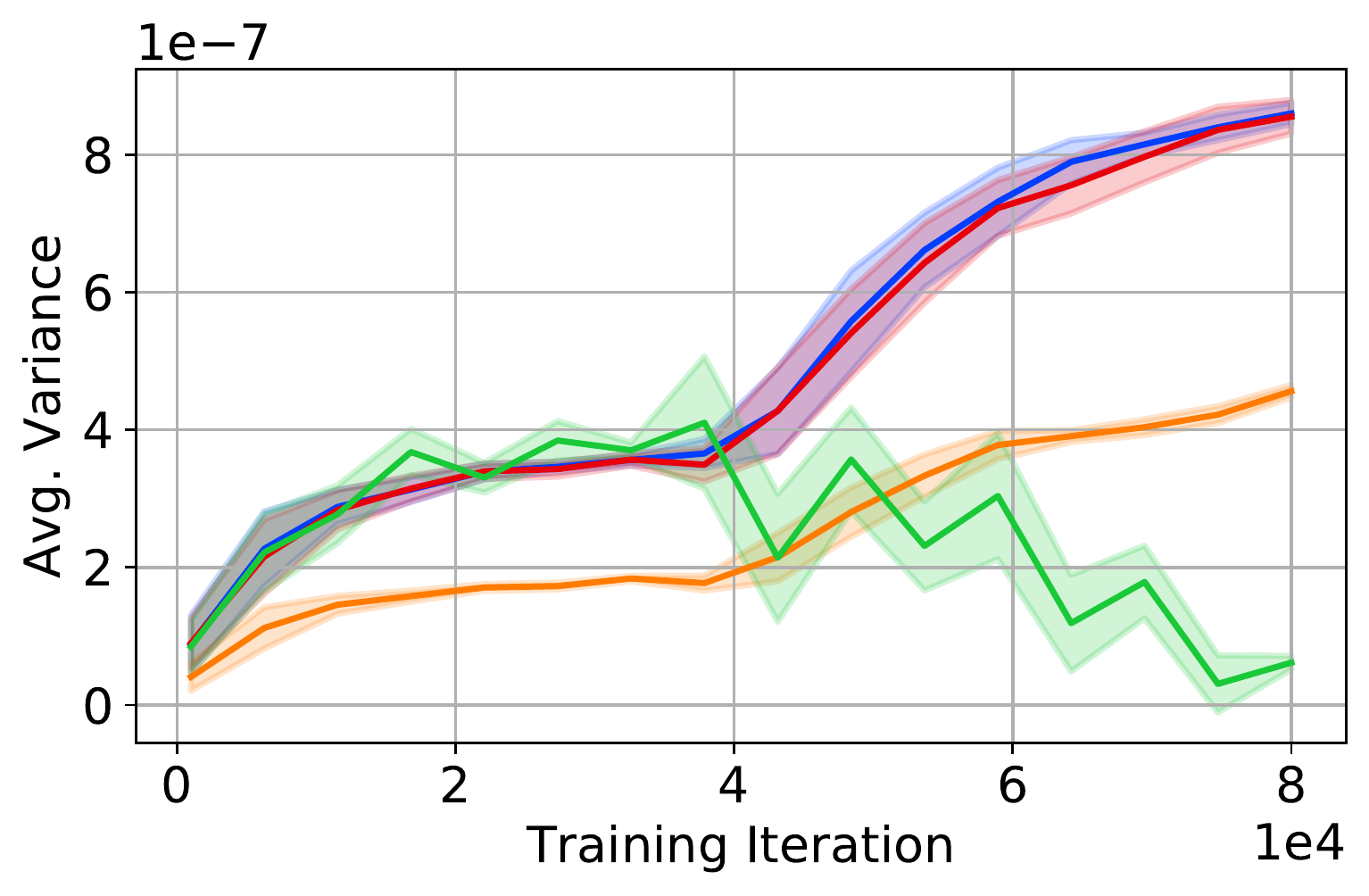}
        \caption{ResNet18 on ImageNet}
        \label{fig:imagenet_var}
    \end{subfigure}\\
    \begin{subfigure}[b]{\threecolfigwidth}
        \includegraphics[width=\textwidth]{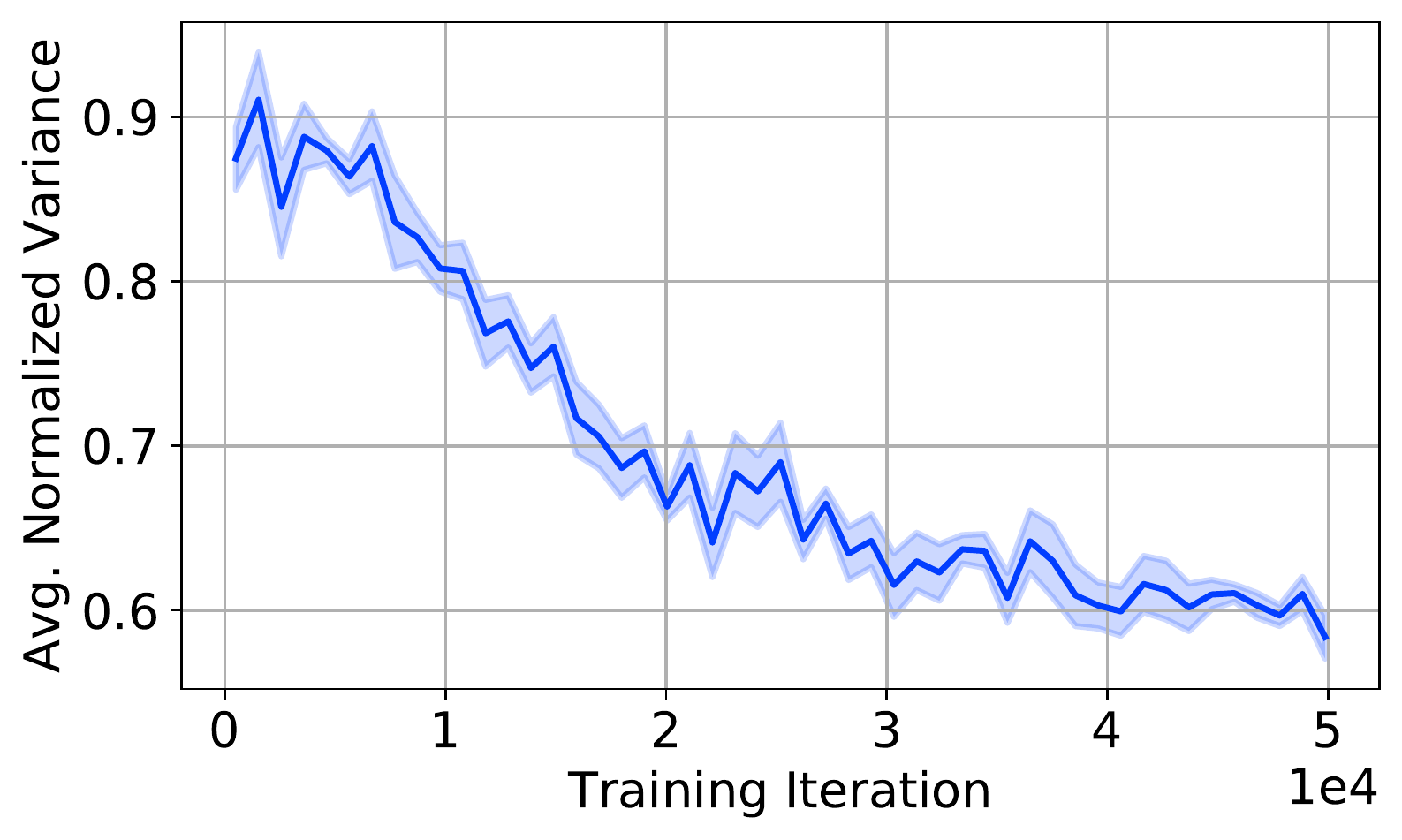}
        \caption{MLP on MNIST}
        \label{fig:mnist_nvar}
    \end{subfigure}
    \begin{subfigure}[b]{\threecolfigwidth}
        \includegraphics[width=\textwidth]{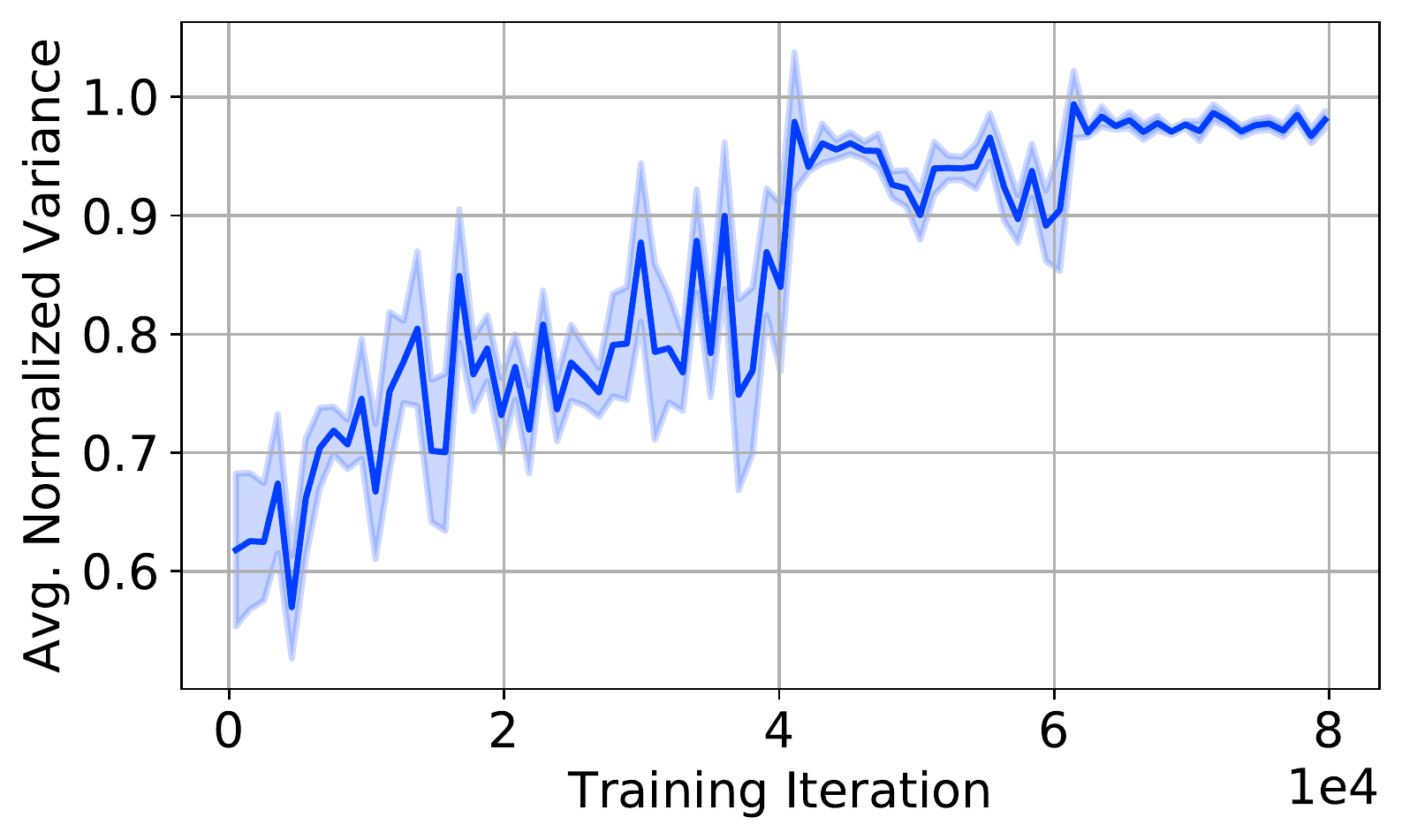}
        \caption{ResNet8 on CIFAR-10}
        \label{fig:cifar10_nvar}
    \end{subfigure}
    \begin{subfigure}[b]{\threecolfigwidth}
        \includegraphics[width=\textwidth]{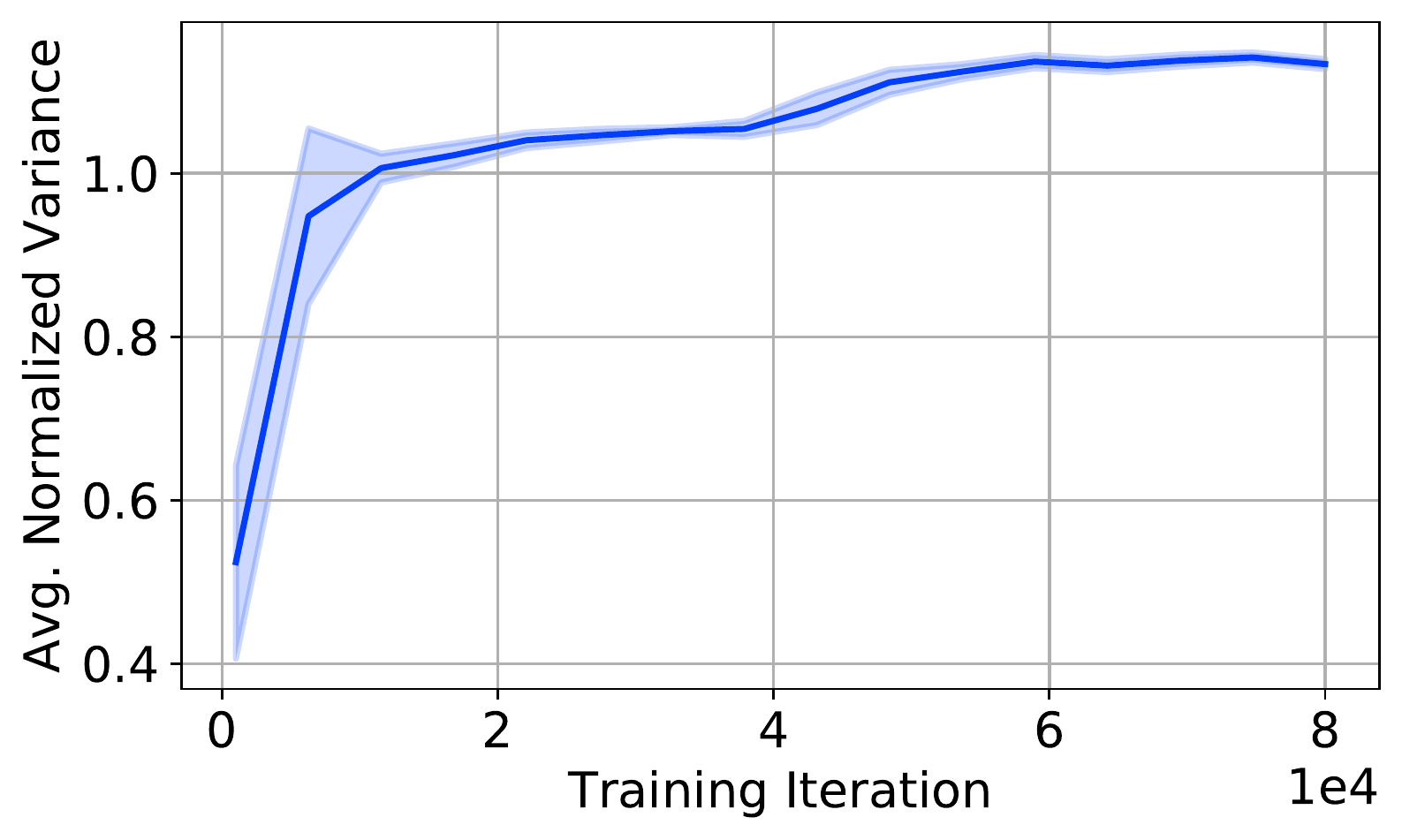}
        \caption{ResNet18 on ImageNet}
        \label{fig:imagenet_nvar}
    \end{subfigure}\\
    \begin{subfigure}[b]{\threecolfigwidth}
        \includegraphics[width=\textwidth]{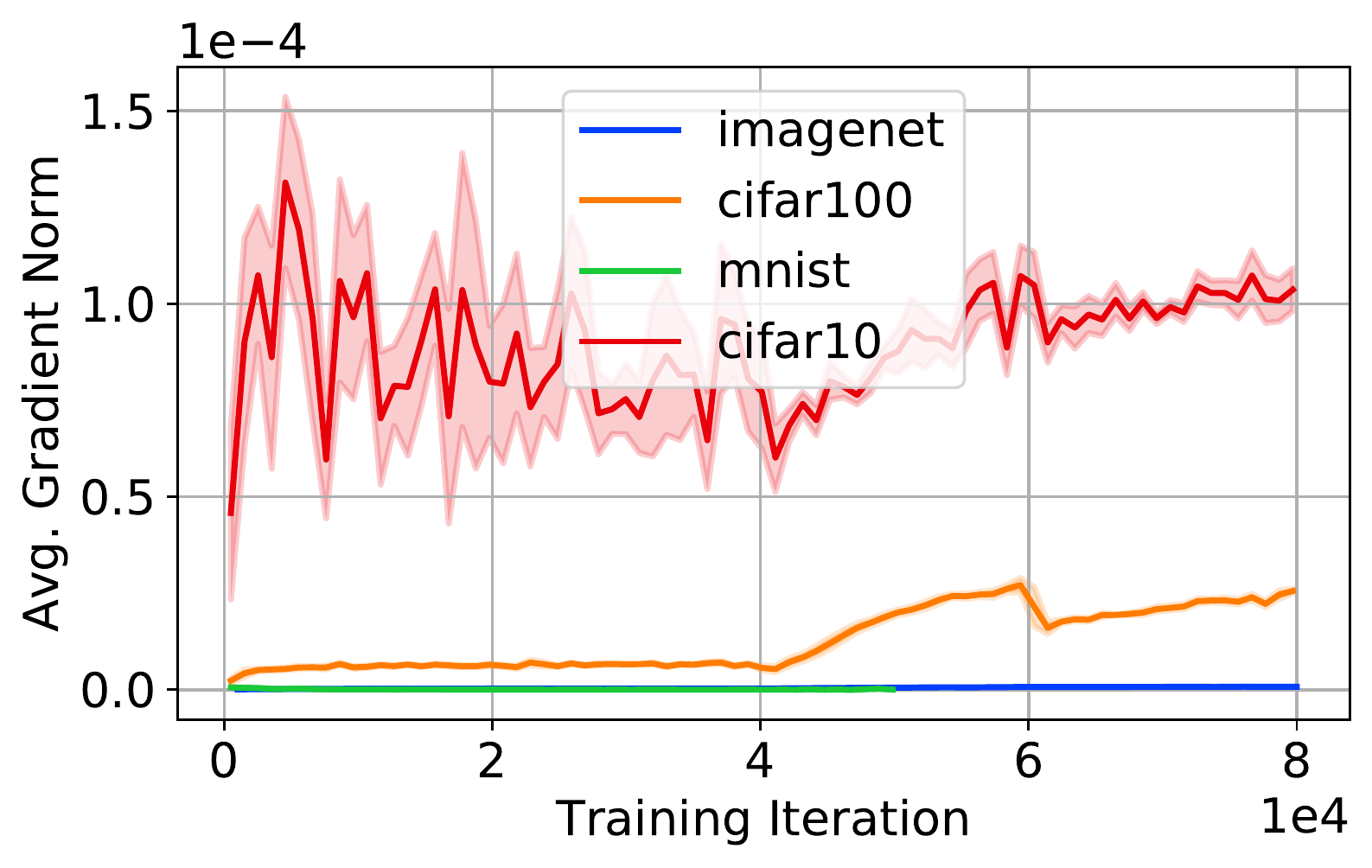}
        \caption{Gradient Norm}
        \label{fig:four_gnorm}
    \end{subfigure}
    \begin{subfigure}[b]{\threecolfigwidth}
        \includegraphics[width=\textwidth]{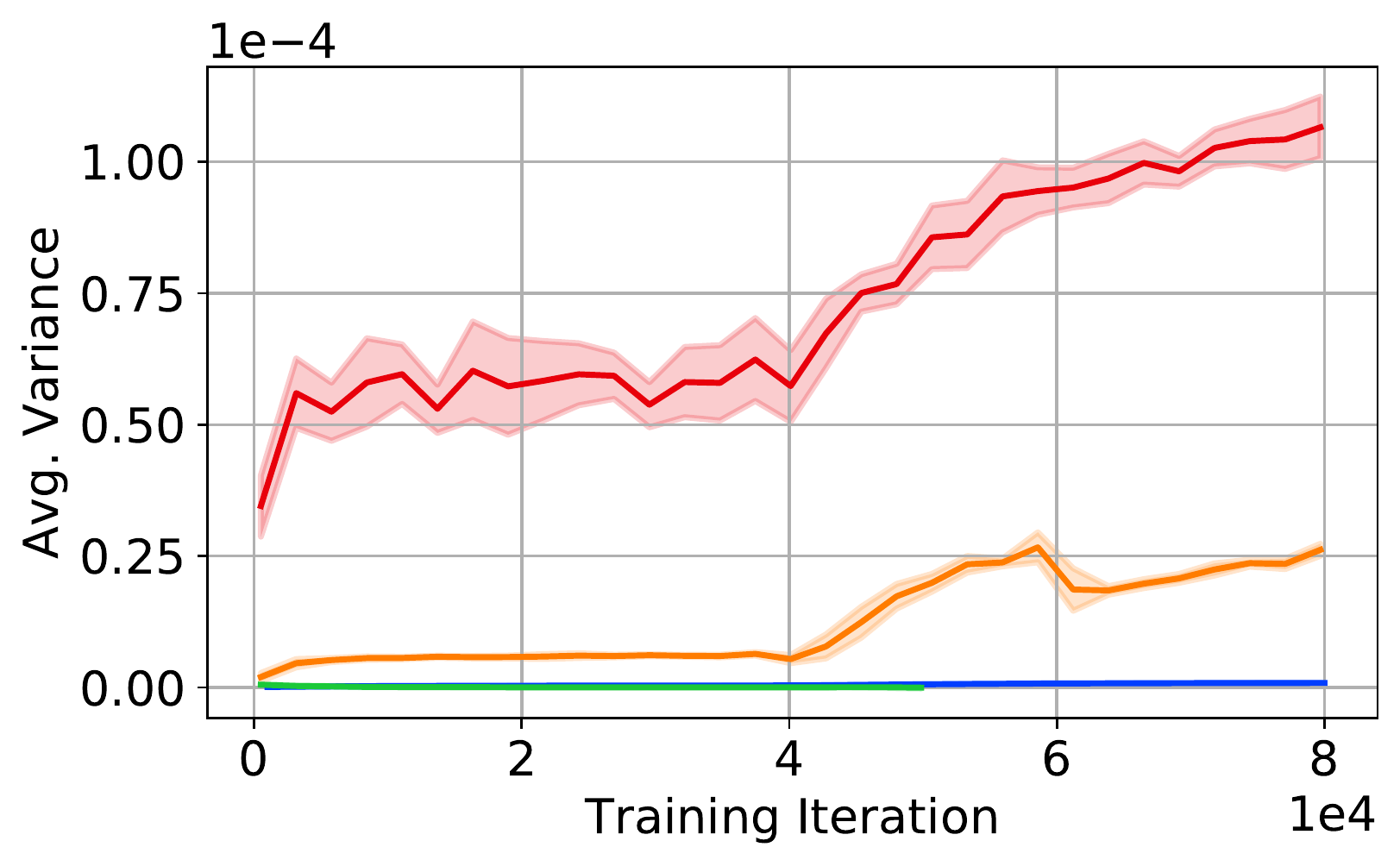}
        \caption{Gradient Variance}
        \label{fig:four_var}
    \end{subfigure}
    \begin{subfigure}[b]{\threecolfigwidth}
        \includegraphics[width=\textwidth]{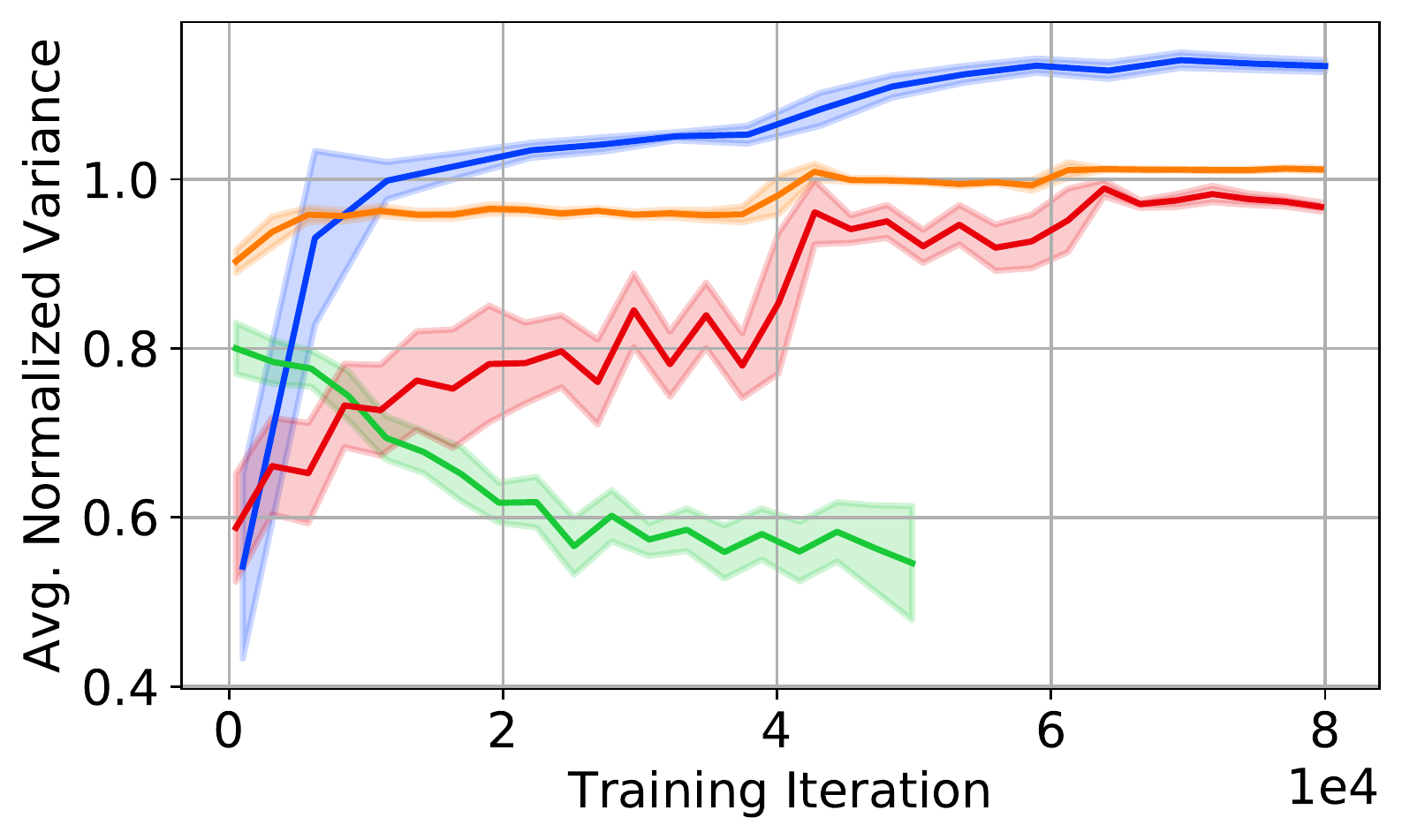}
        \caption{Gradient Normalized Variance}
        \label{fig:four_nvar}
    \end{subfigure}
    \caption{{\bf Image classification models.}\/ Variance (top) and normalized 
    variance plots (middle). Bottom plots compare gradient statistics for SGD 
    across datasets.
    We observe:
    normalized variance correlates with optimization difficulty,
    variance is decreasing on MNIST but increasing on CIFAR-10 and ImageNet,
    and variance fluctuates with \GC on CIFAR-10.
    }
    \label{fig:image}
\end{figure}

In this section, we study the evolution of gradient variance during training of an 
MLP on MNIST~\citep{lecun1998gradient}, ResNet8~\citep{he2016deep} on 
CIFAR-10~\citep{krizhevsky2009learning}, and ResNet18 on 
ImageNet~\citep{deng2009imagenet}. Curves shown are from a single run and 
statistics are smoothed out over a rolling window.  The standard deviation 
within the window is shown as a shaded area.

\textbf{Normalized variance correlates with the time required to improve 
accuracy.}
In \cref{fig:mnist_var,fig:cifar10_var,fig:imagenet_var}, the variance of \SGdB 
is always half the variance of {\SG}. A drawback of the variance is that it is 
not comparable across different problems.  For example, on CIFAR-10 the 
variance of all methods reaches $10^{-4}$ while on ImageNet where usually 
$10\times$ more iterations are needed, the variance is below $10^{-6}$. In 
contrast, as we show by convergence analysis in 
\cref{sec:normalized_var_analysis}, for Lipschitz continuous functions, 
variance reduction is only useful if the normalized variance is larger than 
constant $1$. In \cref{fig:mnist_nvar,fig:cifar10_nvar,fig:imagenet_nvar}, 
normalized variance better correlates with the convergence speed.  Normalized 
variance on both MNIST and CIFAR-10 is always below $1$ while on ImageNet it 
quickly goes above $1$ (noise stronger than gradient).
Notice that the denominator in the normalized variance is shared between all 
methods on the same trajectory of mini-batch SGD\@. As such, the normalized 
variance retains the relation of curves and is a scaled version of variance 
where the scaling varies during training as the norm of the gradient changes.  
For clarity, we only show the curve for {\SG}.

{\bf How does the difficulty of optimization change during training?}\/
The variance on MNIST for all methods is constantly decreasing 
(\cref{fig:mnist_var}), i.e., the strength of noise decreases as we get closer 
to a local optima. These plots suggest that training an MLP on MNIST satisfies 
the Strong Growth Condition (SGC)~\citep{schmidt2013fast} as the variance is 
numerically zero (below $10^{-8}$).
Normalized variance (\cref{fig:mnist_nvar}) decreases over time and is well 
below $1$ (gradient mean has larger magnitude than the variance). SVRG performs 
particularly well by the end of the training because the training loss has 
converged to near zero (cross-entropy less than $0.005$).  Promising published 
results with SVRG are usually on datasets similar to MNIST where the loss 
reaches relatively small values.
In contrast, on both CIFAR-10 (\cref{fig:cifar10_var,fig:cifar10_nvar}) and 
ImageNet (\cref{fig:imagenet_var,fig:imagenet_nvar}), the variance and 
normalized variance of all methods increase from the beginning for almost the 
entire training and especially after the learning rate drops.  This means 
gradient variance depends on the distance to local optima.  We hypothesize that 
the gradient of each training point becomes more unique as training progresses.
Models we considered do not reach zero training loss within the given training 
time. If we increase the model size and train long enough and decrease the 
learning rate, eventually the gradient variance decreases to zero.

\textbf{Variance can widely change during training but it happens only on 
particularly noisy data.}
On CIFAR-10, the variance of \Gluster suddenly goes up but comes back down 
before any updates to the cluster centers (\cref{fig:cifar10_var}) while the 
variance of SVRG monotonically increases between updates.  To explain these 
behaviours, notice that immediately after cluster updates, \Gluster and SVRG 
should always have at most the same average variance as {\SG}.   We observed 
this behaviour consistently across different architectures such as other 
variations of ResNet and {VGG} on CIFAR-10.  \cref{fig:cifar10_noise} shows the 
effect of adding noise on CIFAR-10.  Label 
smoothing~\citep{szegedy2016rethinking} reduces fluctuations but not 
completely. On the other hand, label corruption, where we randomly change the 
labels for $10\%$ of the training data eliminates the fluctuations. We 
hypothesize that the model is oscillating between different states with 
significantly different gradient distributions.  The experiments with corrupt 
labels suggest that mislabeled data might be the cause of fluctuations such 
that having more randomness in the labels forces the model to ignore originally 
mislabeled data.

\textbf{Is the gradient distribution clustered in any dataset?}
The variance of \GC on MNIST (\cref{fig:mnist_var}) is consistently lower than 
\SGdB which means it is exploiting clustering in the gradient space. On 
CIFAR-10 (\cref{fig:cifar10_var}) the variance of \GC is lower than \SG but not 
lower than \SGdB except when fluctuating. The improved variance is more 
noticeable when training with corrupt labels.
On ImageNet (\cref{fig:imagenet_var,fig:imagenet_nvar}), the variance of \GC is 
overlapping with {\SG}. An example of a gradient distribution where \GC is 
overlapping with \SG is a uniform distribution.

\textbf{Can \GC speed-up training?}
Although on MNIST the variance is reduced using \GC, there improvements on 
CIFAR-10 and ImageNet are inconsistent. As such, \GC does not improve the 
convergence speed on CIFAR-10 and ImageNet. On CIFAR-10, we considered 
increasing the frequency of updates to the sampling until the fluctuations 
disappear. To remove all fluctuations, the frequency of updates after the first 
learning rate drop has to be less than every $100$ optimization steps which 
increases the overall wall-clock time of the method unless the assignment step 
is performed fast in parallel with a distributed system.

\subsection{Random Features Models: How Does Overparametrization Affect the 
Variance?}\label{sec:exp_rf}

The Random Features (RF) model~\citep{rahimi2007random} provides an effective 
way to explore the behaviour of optimization methods across a family of 
learning problems.
The RF model facilitates the discovery of optimization behaviours including 
the double-descent shape of the risk curve~\citep{hastie2019surprises, 
mei2019generalization}.
We train a student RF model with hidden dimensions $h_s$ on a fixed training 
set,
$(\bm{x}_i, \bm{y}_i)\in\R^I\times\{\pm1\}$,
$i=1,\ldots,N$,
sampled from a model,
${\bm{x}_i\sim \N(0, \I)}$,
${\bm{y}_i 
= \sign(\sigma(\bm{x}_i^\top\bm{\hat{\theta_1}})^\top\bm{\hat{\theta_2}}+b)}$
where $\sigma$ is the ReLU activation function, and the teacher hidden features 
$\bm{\hat{\theta_1}}\in\R^{I\times h_t}$, and second layer weights and bias, 
$\bm{\hat{\theta_2}}\in\R^{h_t\times 1}, b\in\R$, are sampled from the standard 
normal distribution.  Each $I$ dimensional random feature of the teacher is 
scaled to $\ell_2$ norm 1.  We train a student RF model with random features 
$\bm{\theta_1}\in\R^{I\times h_s}$ and second layer weights 
$\bm{\theta_2}\in\R^{h_s\times 1}$ by minimizing the cross-entropy loss.
In \cref{fig:rf}, we train hundreds of Random Features models and plot the 
average variance and normalized variance of gradient estimators.  We show both 
maximum and mean of the statistics during training.  The maximum better 
captures fluctuations of a gradient estimator and allows us to link our 
observations of variance to generalization using standard convergence bounds 
that rely on bounded noise~\citep{bottou2018optimization}.

\begin{figure}[t]
    \centering
    \begin{subfigure}[b]{\threecolfigwidth}
        \includegraphics[width=\textwidth]{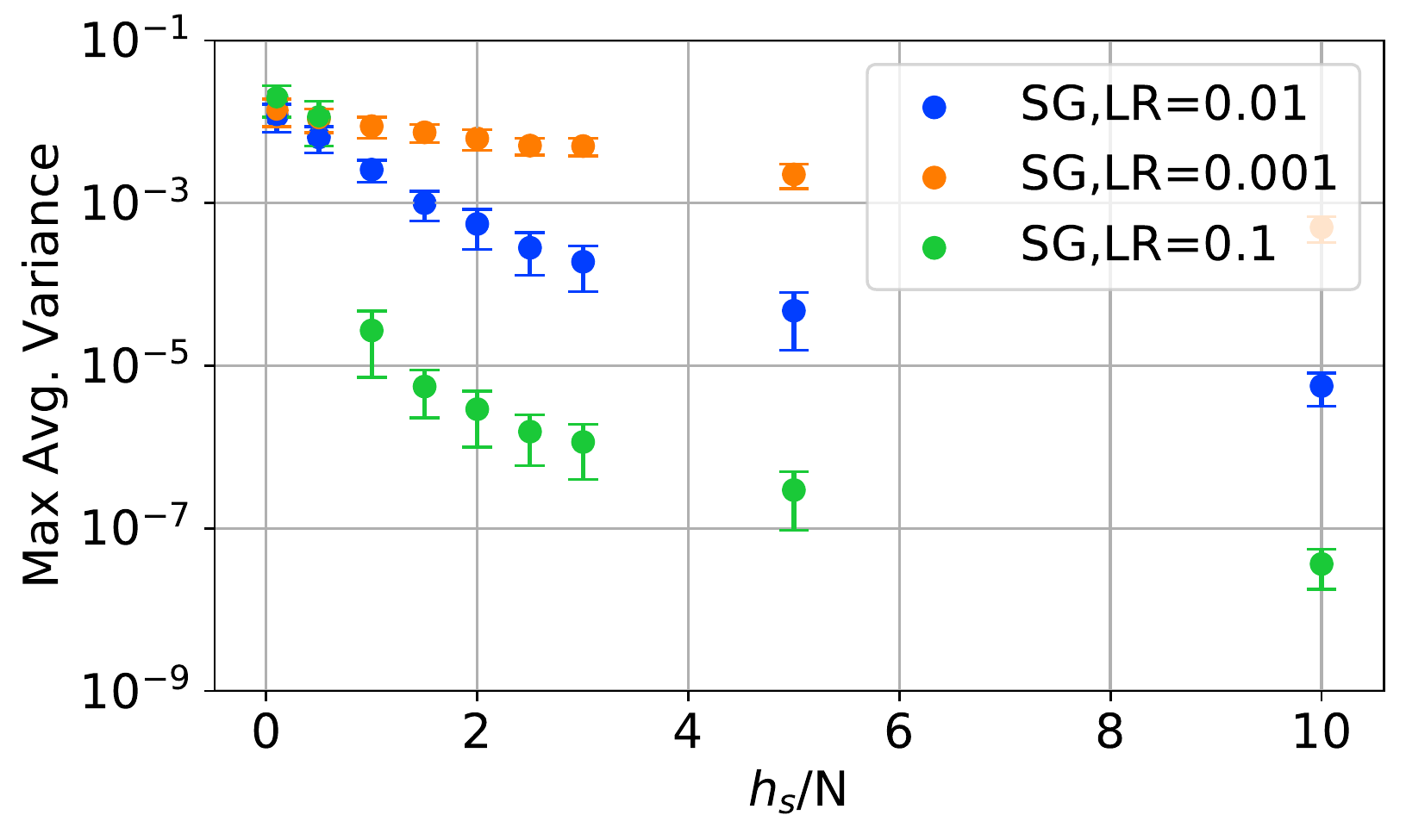}
        \caption{$3$ SGD trajectories}
        \label{fig:rf_sgd_var}
    \end{subfigure}
    \begin{subfigure}[b]{\threecolfigwidth}
        \includegraphics[width=\textwidth]{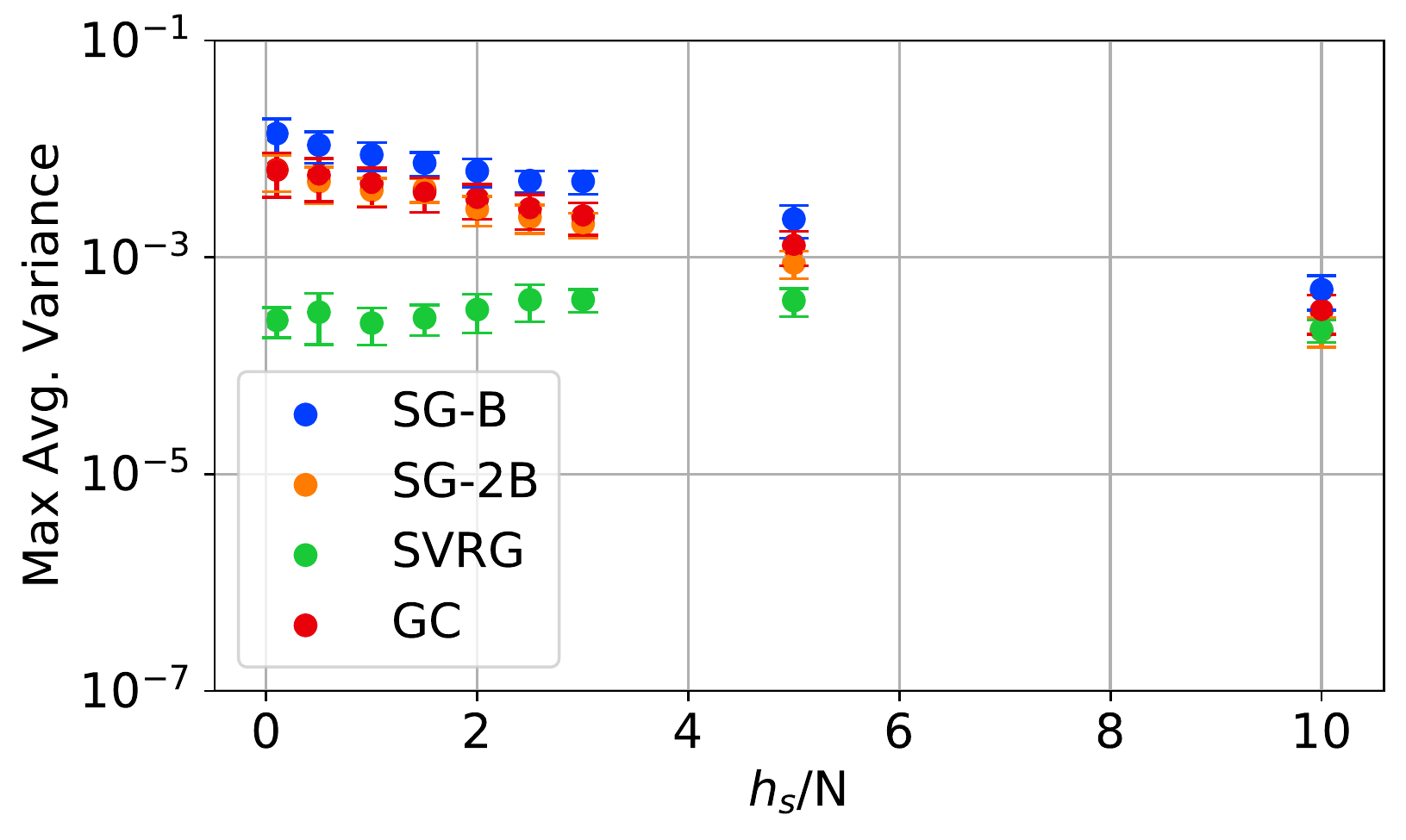}
        \caption{Trajectory of LR=$0.001$}
        \label{fig:rf_lr0.001_var}
    \end{subfigure}
    \begin{subfigure}[b]{\threecolfigwidth}
        \includegraphics[width=\textwidth]{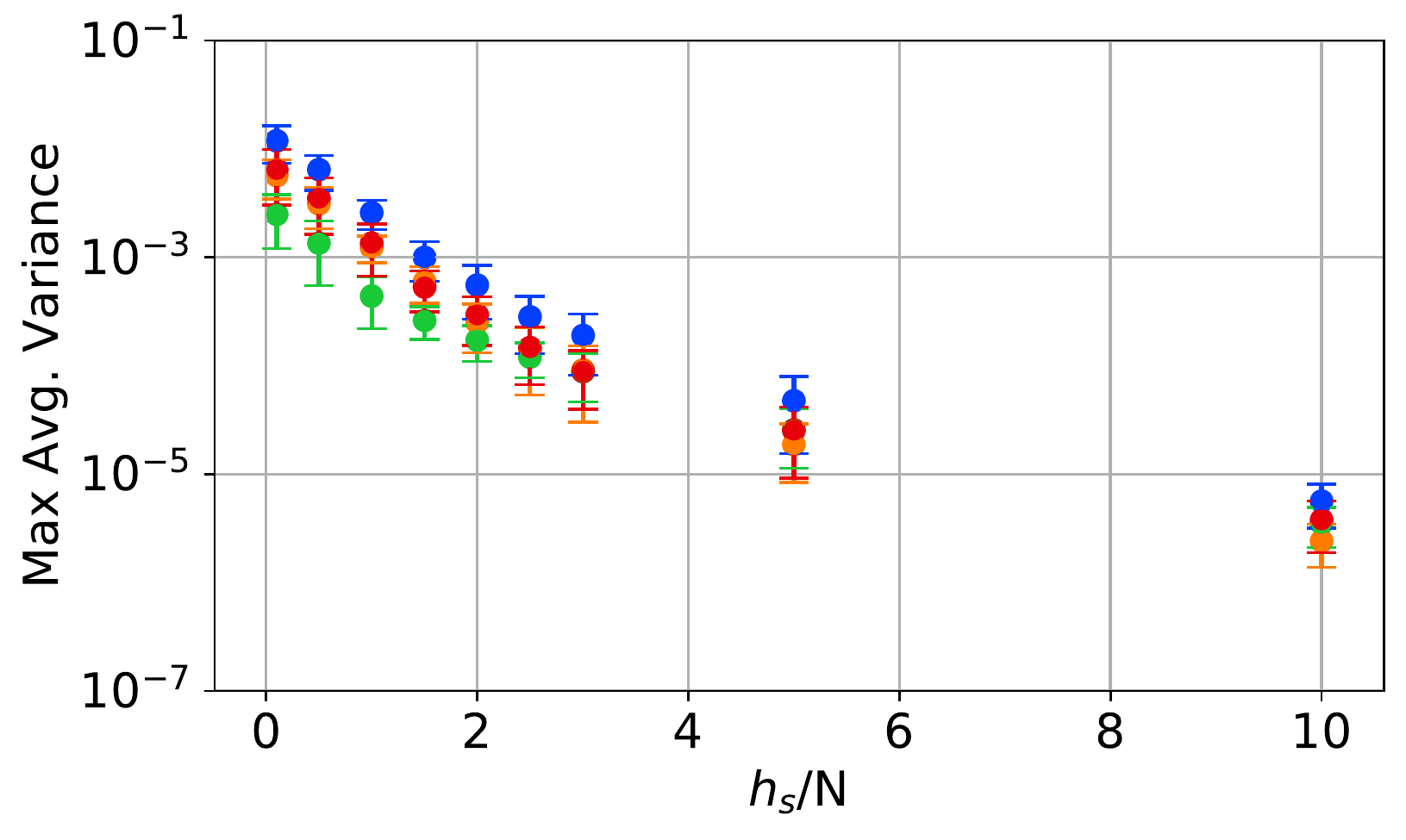}
        \caption{Trajectory of LR=$0.01$}
        \label{fig:rf_lr0.01_var}
    \end{subfigure}
    \caption{{\bf Random Features models.}\/
    Variance (log-scale) versus the over-parametrization coefficient (student's 
    hidden divided by the training set size).
    We observe: teacher's hidden is not influential,
    variance is low in overparametrized regime,
    and with larger learning rates.
    We aggregate results from hyper-parameters not shown.
    }
    \label{fig:rf}
\end{figure}

\begin{figure}[t]
    \centering
    \begin{minipage}{\threecolfigwidth}
    \centering
   \begin{subfigure}[b]{\textwidth}
       \includegraphics[width=\textwidth]{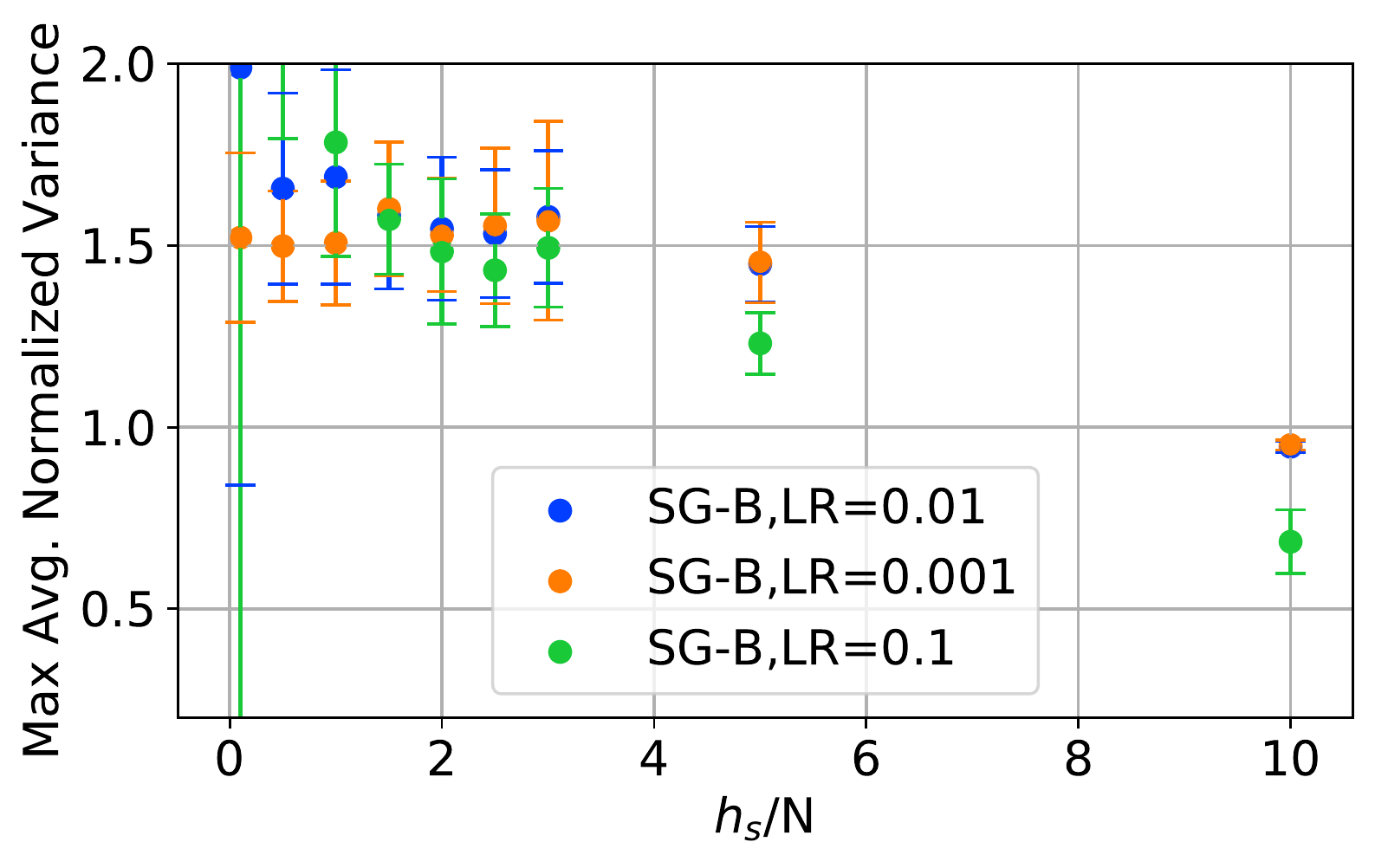}
       \caption{SGD max norm.\ var.}
       \label{fig:rf_sgd_nvar_max}
   \end{subfigure}\\
   \begin{subfigure}[b]{\textwidth}
       \includegraphics[width=\textwidth]{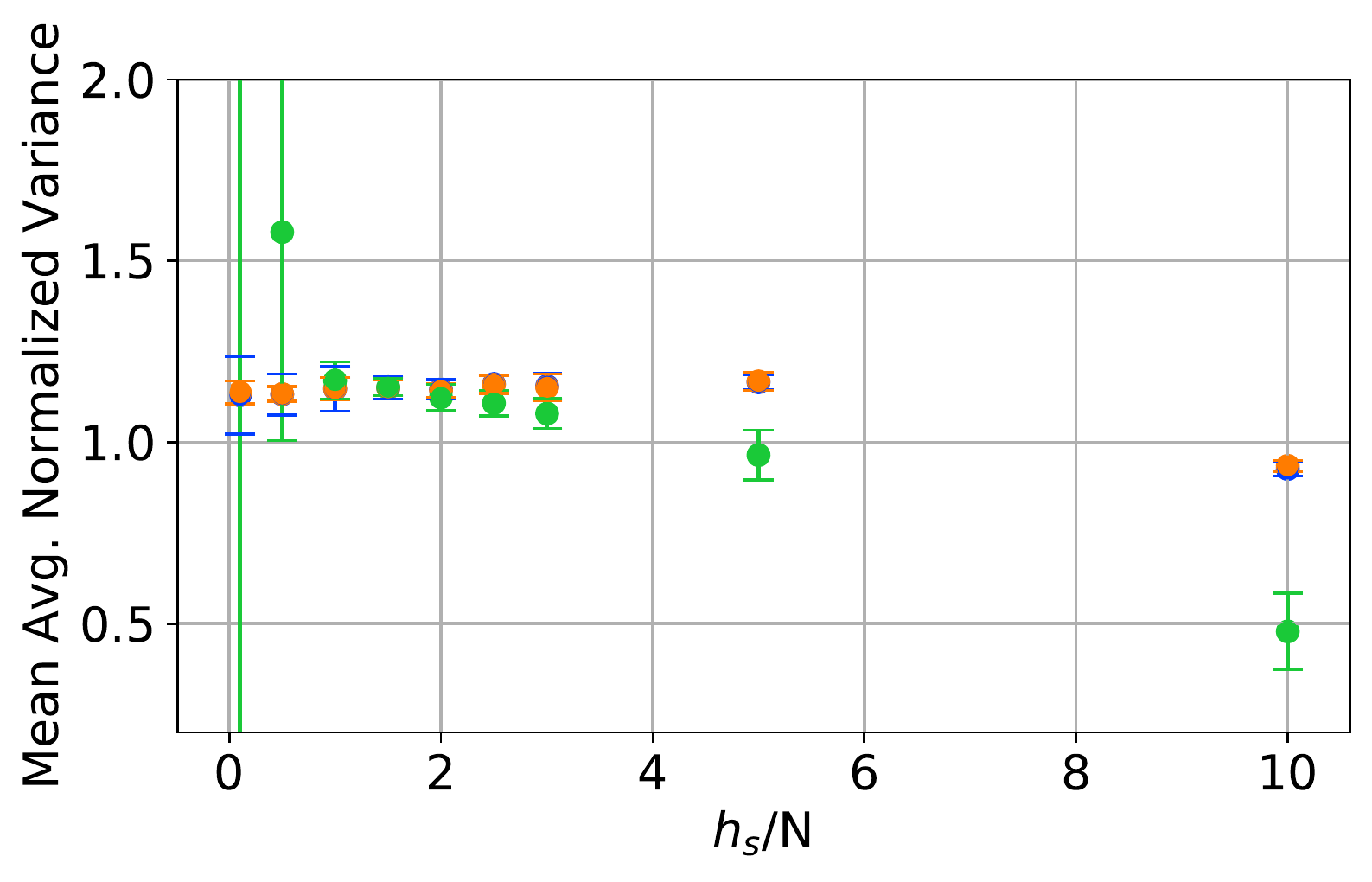}
       \caption{SGD mean norm.\ var.}
       \label{fig:rf_sgd_nvar_mean}
   \end{subfigure}
    \caption{{\bf Normalized variance on overparam.\ RF} is less than $1$.}
    \label{fig:rf_sgd_nvar}
    \end{minipage}
    \hfill
    \begin{minipage}{0.32\textwidth}
    \centering
   \begin{subfigure}[b]{\textwidth}
       \includegraphics[width=\textwidth]{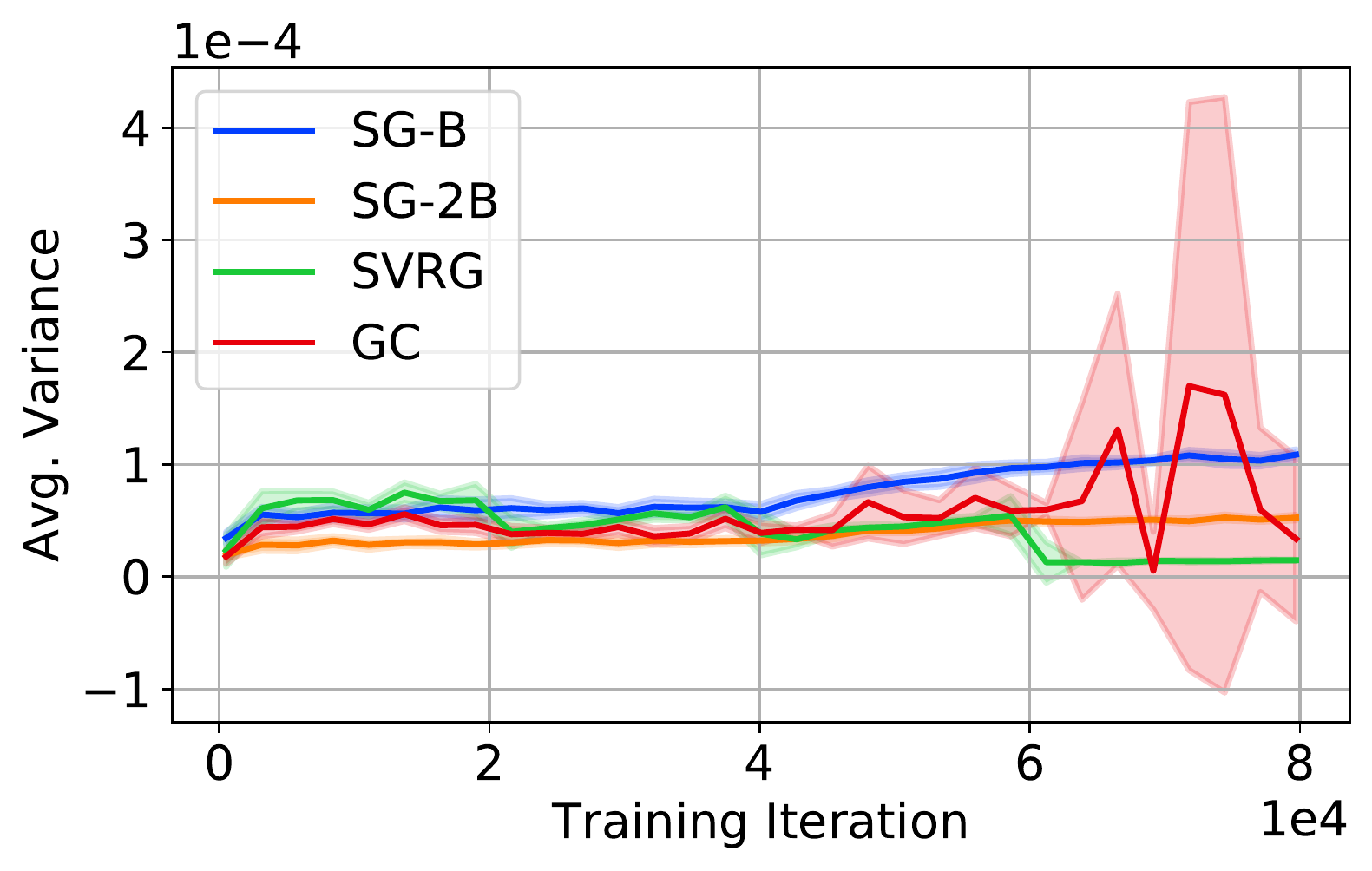}
       \caption{Label smoothing}
       \label{fig:cifar10_label_smooth}
   \end{subfigure}\\
   \begin{subfigure}[b]{\textwidth}
       \includegraphics[width=\textwidth]{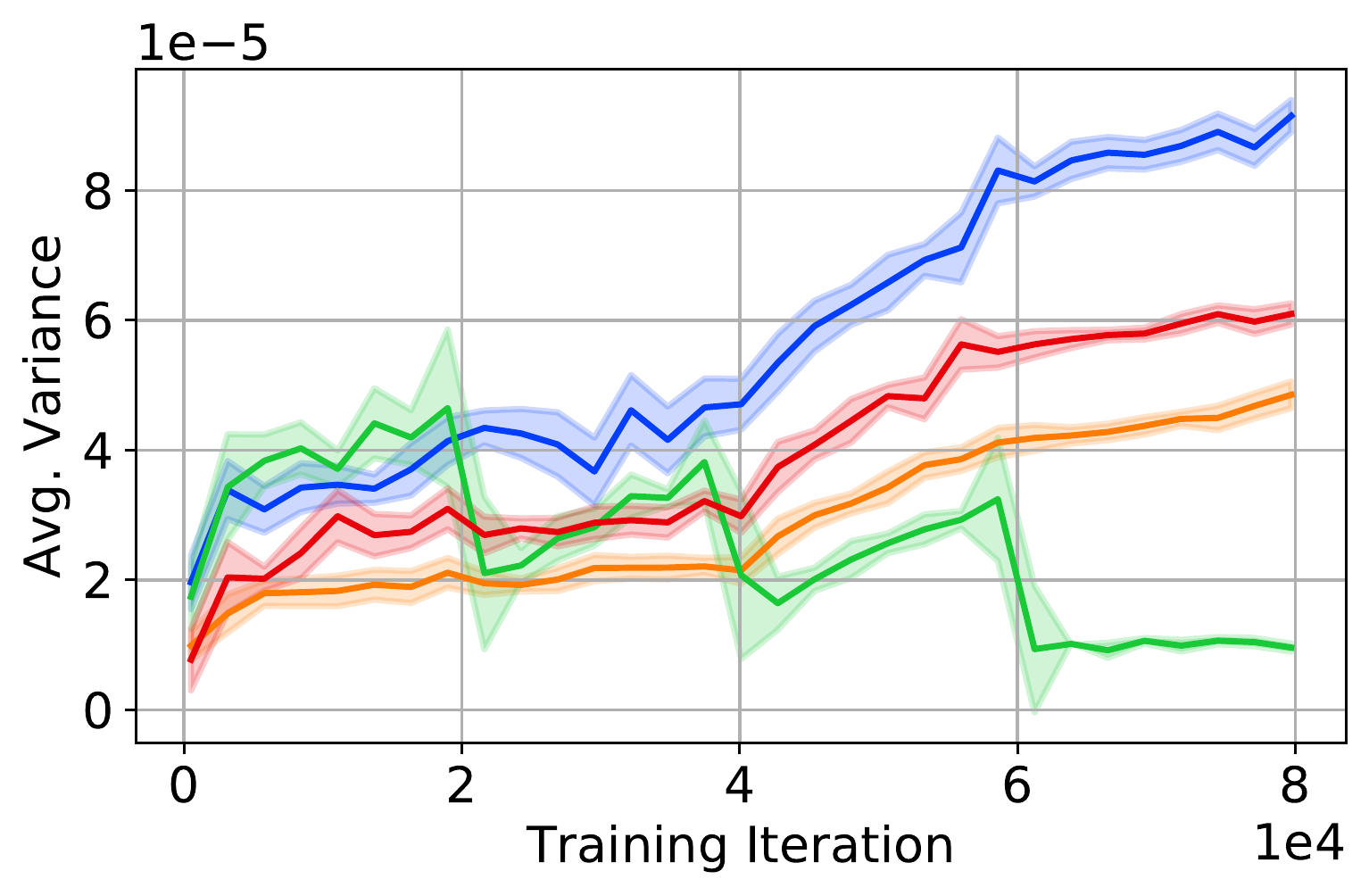}
       \caption{Corrupt labels}
       \label{fig:cifar10_corrupt}
   \end{subfigure}
        \caption{{\bf CIFAR-10 Fluctuations} disappear with corrupt labels.}
    \label{fig:cifar10_noise}
    \end{minipage}
    \hfill
    \begin{minipage}{\threecolfigwidth}
        \centering
   \begin{subfigure}[b]{\textwidth}
        \includegraphics[width=\textwidth]{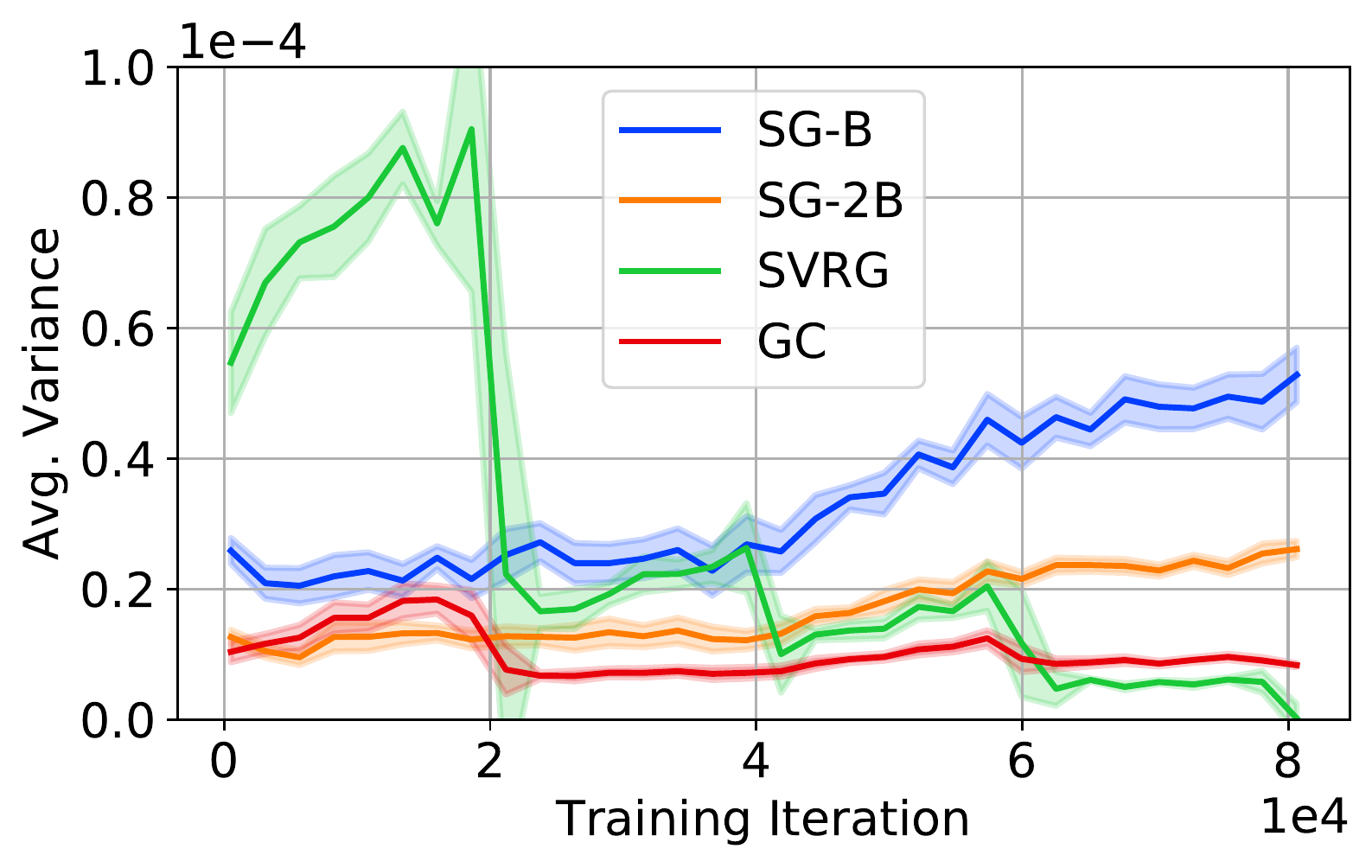}\\
       \caption{CIFAR-10}
       \label{fig:cifar10_dup}
   \end{subfigure}
   \begin{subfigure}[b]{\textwidth}
        \includegraphics[width=\textwidth]{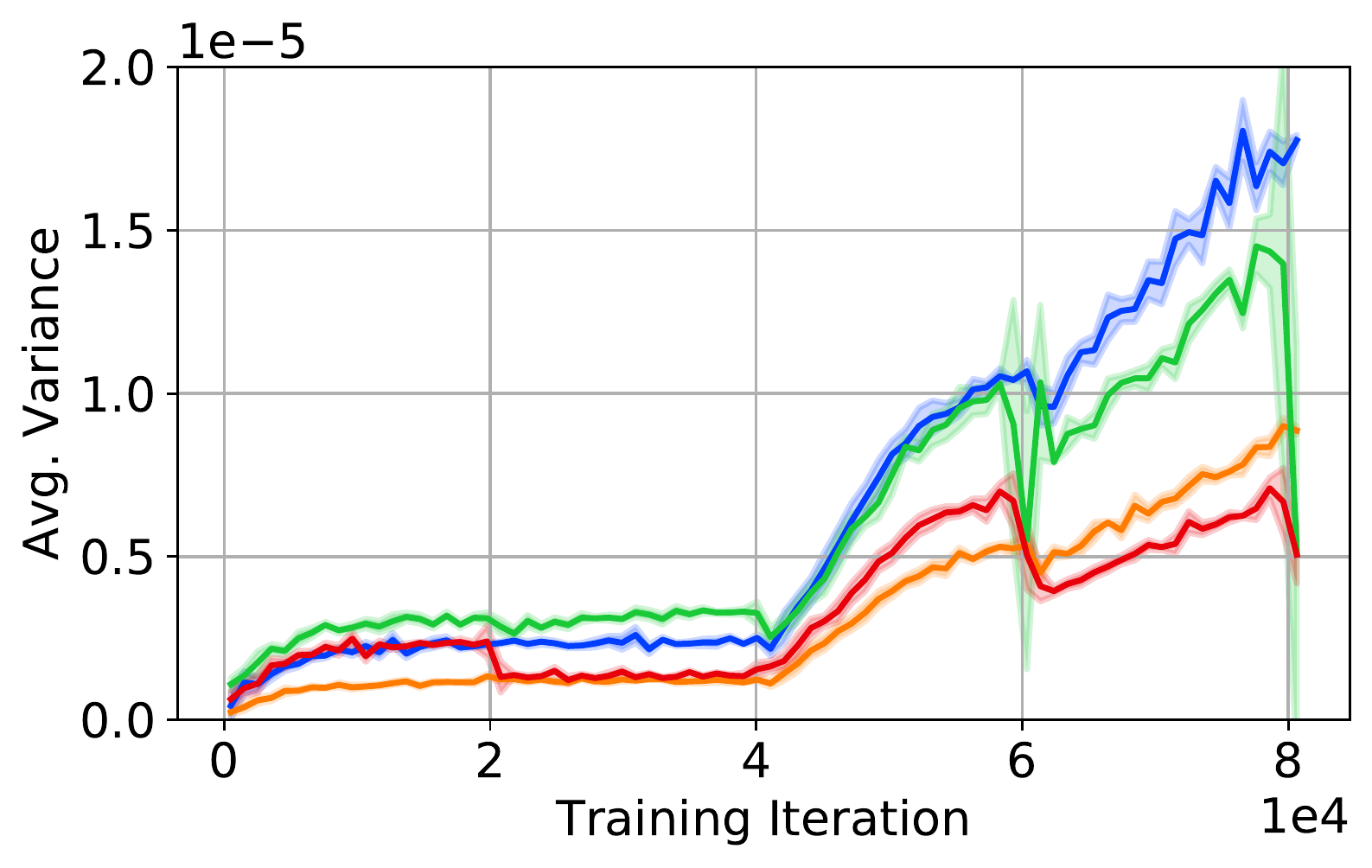}
       \caption{CIFAR-100}
       \label{fig:cifar100_dup}
   \end{subfigure}
        \caption{{\bf Image classification with duplicates} exploited by \GC.}
        \label{fig:image_dup}
    \end{minipage}
\end{figure}

{\bf Do models with small generalization gap converge faster?}\/
Based on small error bars, the only hyper-parameters that affect the variance 
are learning rate and the ratio of the size of the student hidden layer over 
the training set size.  In contrast, in analysis of risk and the double descent 
phenomena, we usually observe a dependence on ratio of the student hidden layer 
size to the teacher hidden layer size~\citep{mei2019generalization}. This 
suggests that models that generalize better are not necessarily ones that train 
faster.

{\bf Does ``diminishing returns'' happen because of  overparametrization?}\/
\cref{fig:rf_lr0.001_var,fig:rf_lr0.01_var} show that with the same learning 
rate, all methods achieve similar variance in the overparametrized regime.  
Note that due to the normalization of random features, the gradients in each 
coordinate are expected to decrease as overparametrization increases. We 
conjecture that the diminishing returns in increasing the mini-batch size 
should also be observed in overparametrized random features models similar to 
linear and deeper models~\citep{zhang2019algorithmic, shallue2018measuring}.

{\bf How does the variance change as learning rate varies?}
\cref{fig:rf_sgd_var} shows that the variance is smaller for trajectories with 
larger learning rates and that the gap grows as overparametrization grows.  
This is a direct consequence of the dependence of the noise in the gradient on 
current parameters. In \cref{sec:exp_image} we observe the opposite of this 
behaviour in deep models.   In contrast, \cref{fig:rf_sgd_nvar} shows that for 
overparametrization less than $5$, all trajectories have similar normalized 
variance that is larger than one (noise is more powerful than the gradient).  

\subsection{Duplicates: Back to the Motivation for Gradient Clustering}
\label{sec:exp_dup}
\begin{figure}[t]
    \centering
    \begin{subfigure}[b]{\threecolfigwidth}
        \includegraphics[width=\textwidth]{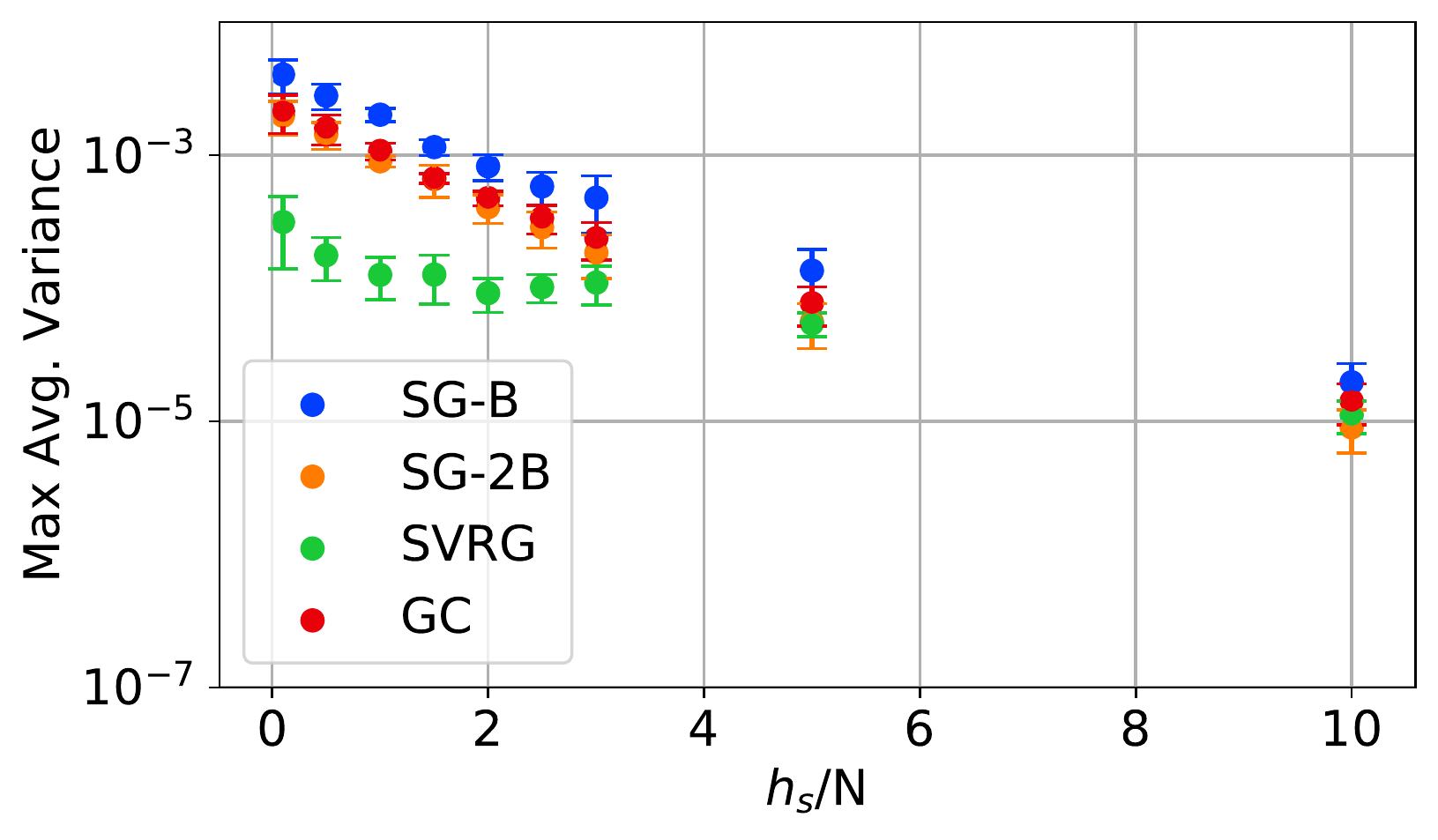}
        \caption{$10\%$ duplicates}
        \label{fig:dup_10}
    \end{subfigure}
    \begin{subfigure}[b]{\threecolfigwidth}
        \includegraphics[width=\textwidth]{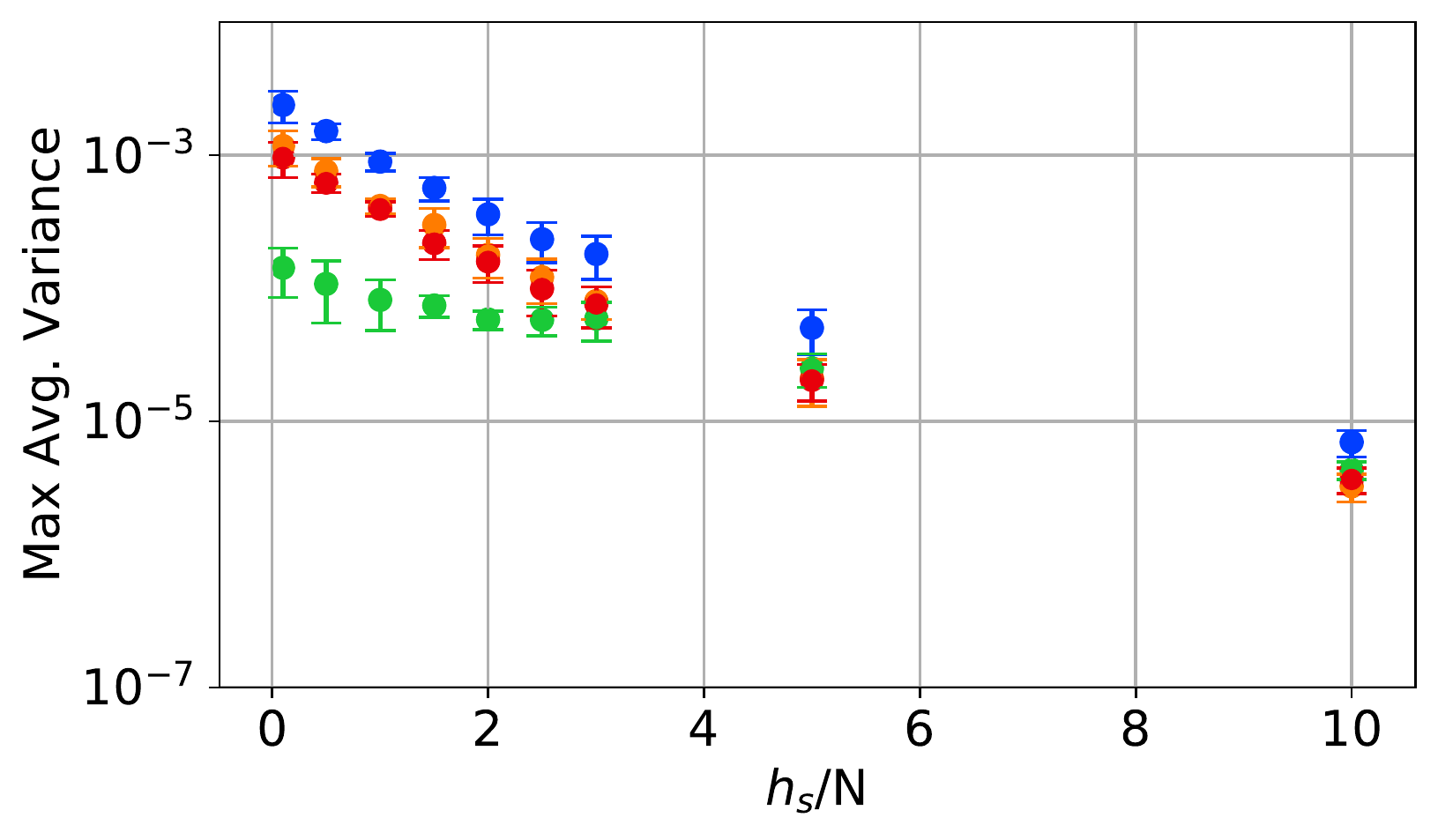}
        \caption{$50\%$ duplicates}
        \label{fig:dup_50}
    \end{subfigure}
    \begin{subfigure}[b]{\threecolfigwidth}
        \includegraphics[width=\textwidth]{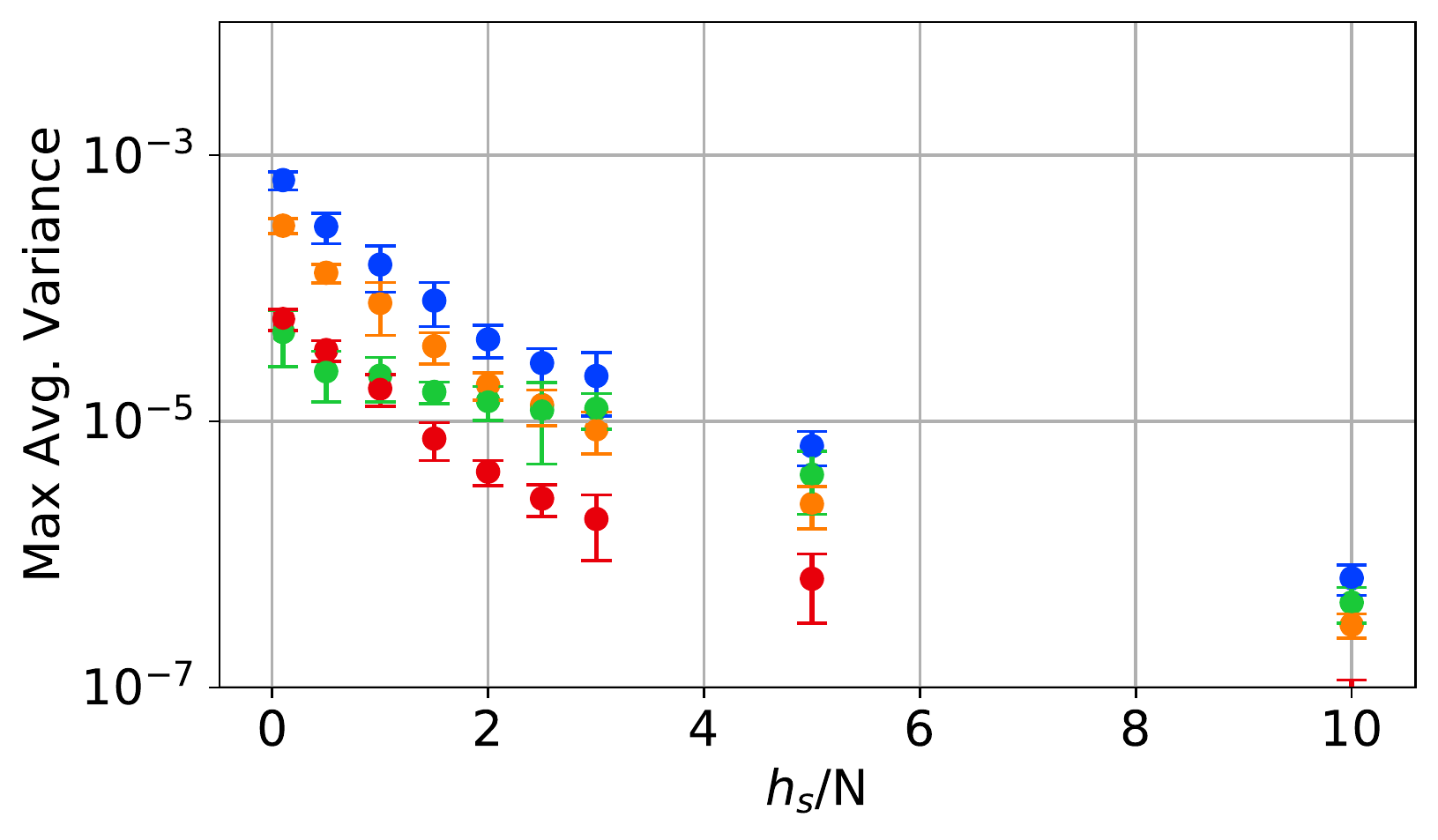}
        \caption{$90\%$ duplicates}
        \label{fig:dup_90}
    \end{subfigure}
    \caption{{\bf Training RF Models with Duplicates.}\/ \GC identifies and 
    exploits duplicates.  Plots are similar to \cref{fig:rf}.  Learning rate in 
    all three is $0.01$.  In each training, there are $5$ data points that are 
    repeated equally to make up $10\%$ (left), $50\%$ (middle), and $90\%$ 
    (right) of the training set.
    }
    \label{fig:rf_dup}
\end{figure}

In \cref{fig:rf_dup}, we trained random features models with additional 
duplicated data points.  We observe that as the ratio of duplicates to 
non-duplicates increases, the gap between the variance of \Gluster and other 
methods improve.  Without duplicate data, \Gluster is always between \SG and 
\SGdB.  It is almost never worse than \SG and never better than \SGdB.  
\Gluster is as good as \SGdB at mild overparametrization ($1-4$). We need 
a degree of overparametrization for \Gluster to reduce the variance but too 
much overparametrization leaves no room for improvement. When duplicates exist, 
\Gluster performs well with a gap that does not decrease by 
overparametrization.

Similarly, experiments on CIFAR-10 and CIFAR-100 (\cref{fig:image_dup}) show 
that \Gluster significantly reduces the variance when duplicate data points 
exist. In this experiment, $10$ training points are selected randomly and 
duplicated $10000$ times that dominate the training. As the training set size 
is $50000$, the ratio of redundant points to the original data is $20\times$.  
Note that because of common data augmentations, duplicate data points are not 
exactly duplicate in the input space and there is no guarantee that their 
gradients would be similar.

\section{Conclusion}
In this chapter, we introduced an efficient gradient clustering method and 
showed that stratified sampling based on the gradient clusters induces a low 
variance gradient estimator. We observed challenges in employing gradient 
clustering for optimization. To investigate, we studied the variance of the 
gradients for various optimization trajectories on standard benchmarks.

In the following, we summarize the hypotheses, results, and future work:
\begin{itemize}
    \item We proved that stratified sampling based on a weighted clustering in 
        the gradient space minimizes the variance of a mini-batch gradient 
        estimator for a constant distribution of Gradients (\cref{sec:gvar}).  
        We designed an efficient gradient clustering method with 
        a per-iteration computational cost comparable to a single back-prop.
    \item We demonstrated the success of gradient clustering in reducing the 
        gradient variance on MNIST as well as in the presence of redundancy 
        such as duplicate data (\cref{sec:exp_image,sec:exp_dup}).  Future work 
        could consider detecting and exploiting other types of redundancy and 
        within-class imbalances using gradient clustering.
    \item We provided preliminary evidence and justification for the 
        correlation between the normalized gradient variance and convergence 
        speed of SGD (\cref{sec:exp,sec:normalized_var_analysis}). On MNIST, 
        the normalized gradient variance decreases during training while it is 
        increasing on CIFAR-10 and ImageNet for the models we tested.  Future 
        work could verify this correlation empirically for more model 
        architectures and datasets. If verified, future work can study 
        optimization methods such as stratified sampling with gradient 
        clustering to minimize the normalized variance instead of the variance 
        of gradients.
    \item We observed that the distribution of gradient changes significantly 
        during the training of CIFAR-10 models studied in this chapter. This 
        observation does not challenge standard optimization methods but 
        consistently impedes our proposed stratified sampling method.
        We hypothesized that a small subset of the training set such as a few 
        mislabeled data might be responsible for the sudden changes in gradient 
        clusters. In our follow-up work (not included in this thesis), we 
        observed that the mean and variance of the normalized gradient do not 
        change significantly during the training~\citep{faghri2020adaptive}.  
        This observation aligns with our hypothesis.  Future work can test our 
        hypothesis by comparing the gradient cluster centers which can provide 
        a method for detecting mislabeled or ambiguous data in the training 
        set.
    \item We sought an answer for the question "Is there structure in the 
        gradient space?" by applying weighted gradient clustering and measuring 
        the objective, i.e., gradient variance. On MNIST and CIFAR-10, we 
        successfully reduced the objective by performing the efficient but 
        approximate clustering algorithm. This suggests that a clustered 
        structure exists in the gradient space for the studied models on these 
        datasets. In contrast, on ImageNet the variance is not reduced which 
        implies a lack of clustered structure in the gradient space.  This 
        hypothesis is not fully tested and requires future work.  For example, 
        clustering with no weighting can be tried as well as whether the impact 
        of the approximation error in efficient gradient clustering operations.  
        It is also possible that alternative distance metrics to Euclidean 
        distance would induce clusters that could be found by a modified 
        gradient clustering method.
\end{itemize}

%% file: ch-robust.tex
\chapter{Bridging the Gap Between Adversarial Robustness and Optimization Bias}
\label{ch:robust}

The third and last problem studied in this thesis concerns adversarial 
robustness in deep learning. We hypothesize that robustness depends on the 
explicit training mechanisms and their implicit biases. The implication is that 
optimization choices not only impact the speed of training, computational 
resource requirements, and generalization performance but also indirectly 
affect the security and robustness of the models.

We demonstrate that the choice of optimizer,
neural network architecture,
and  regularizer significantly affect the adversarial 
robustness of linear neural networks, providing guarantees without
the need for adversarial training.
To this end, we revisit a known result linking maximally robust classifiers and minimum norm solutions, and combine it with recent results on the implicit bias of optimizers.
First, we show that, under certain conditions, it is possible to achieve both perfect standard accuracy and a certain degree of robustness,
simply by training an overparametrized model using the implicit bias of the optimization. In that regime, there is a direct relationship between the type of the optimizer and the attack to which the model is robust.
To the best of our knowledge, this work is the first to
study the impact of optimization methods such as sign gradient descent and proximal methods on adversarial robustness.
Second, we characterize the robustness of
linear convolutional models, showing that they resist attacks subject to a constraint on the Fourier-$\ell_\infty$ norm.
To illustrate these findings we design a novel
Fourier-$\ell_\infty$ attack that finds
adversarial examples with controllable frequencies.
We evaluate Fourier-$\ell_\infty$ robustness of adversarially-trained
deep CIFAR-10 models from the standard RobustBench
benchmark and visualize adversarial perturbations.

The content of this chapter have appeared in the following publication:

\begin{itemize}
    \item {\bf Faghri, Fartash} and Gowal, Sven and Vasconcelos, Cristina and 
        Fleet, David J.\ and Pedregosa, Fabian and Le Roux, Nicolas, {\sl 
        ``Bridging the Gap Between Adversarial Robustness and Optimization 
        Bias"}, Workshop on Security and Safety in Machine Learning Systems, 
        International Conference on Learning Representations (ICLR), 
        2021.
\end{itemize}

To ensure reproducibility, our code is publicly 
available~\footnote{\url{https://github.com/fartashf/robust_bias}}.

\input{ch-robust-tex/body}

%% file: ch-robust-tex/body.tex
\def\figdir{ch-robust-tex/figures}

\section{Introduction}
\label{sec:intro}

Deep neural networks achieve high accuracy on standard test
sets, yet \citet{szegedy2013intriguing} showed that
any natural input correctly classified by a neural network
can be modified with adversarial perturbations.
Such perturbations fool the network into misclassification,
even when they are constrained to be imperceptible to humans.
Adversarial training improves model robustness by augmenting the training set 
with adversarial perturbations~\citep{goodfellow2014explaining}
and can be interpreted as approximately
solving a saddle-point problem~\citep{madry2017towards}.
Adversarial training is the state-of-the-art approach to adversarial
robustness~\citep{gowal2020uncovering, croce2020robustbench} and
alternative approaches are more likely to exhibit
spurious robustness~\citep{tramer2020adaptive}.
Nevertheless, adversarial training is computationally expensive compared to 
standard training, as it involves an alternating optimization.
Adversarial training also exhibits a trade-off between standard generalization 
and adversarial robustness. 
That is, it achieves improved {\it robust accuracy}, on
adversarially perturbed data, at the expense of {\it standard accuracy}, the
probability of correct predictions on natural data \citep{tsipras2018robustness}.
This adversarial robustness trade-off has been shown to be intrinsic in a
number of toy examples~\citep{fawzi2018analysis}, independent of the
learning algorithm in some cases~\citep{schmidt2018adversarially}.
Alternatives to adversarial training have been proposed to
reduce this trade-off, but a gap 
remains in practice~\citep{zhang2019theoretically}.

Here we consider connections between the adversarial robustness trade-off 
and optimization biases in training overparametrized models.
Deep learning models can often achieve interpolation, i.e., they have 
the capacity to exactly fit the training data~\citep{zhang2016understanding}. 
Their ability to generalize well in such cases has been attributed 
to an implicit bias toward simple
solutions~\citep{gunasekar2018characterizing,hastie2019surprises}.

Our main contribution is 
to
connect two large bodies of work on adversarial robustness
and optimization bias.
Focusing on models that achieve interpolation, we use the formulation of a {\it 
Maximally Robust Classifier}\/
from robust optimization~\citep{ben2009robust}.
We theoretically demonstrate that the choice of optimizer
(\cref{cor:implicit_robust}), neural network architecture
(\cref{cor:max_robust_to_min_norm_conv}), and regularizer
(\cref{cor:reg_is_max_margin}), significantly affect
the adversarial robustness of linear neural networks.
Even for linear models, the impact of these choices
had not been characterized precisely prior to our work.
We observe that, in contrast to adversarial training,
under certain
conditions we can find maximally robust classifiers at no additional
computational cost.

Based on our theoretical results on the robustness of linear convolutional
models to Fourier attacks,
we introduce a new class of attacks in the Fourier
domain. In particular, we design the Fourier-$\ell_\infty$
attack and illustrate our theoretical results. Extending to
non-linear models, we attack adversarially-trained deep models on
CIFAR-10 from the RobustBench benchmark~\citep{croce2020robustbench}
and find low and high frequency adversarial perturbations
by directly controlling
spectral properties through Fourier constraints.
This example demonstrates how
understanding maximal robustness of linear models is a
stepping stone to understanding and guaranteeing robustness
of non-linear models.

\section{No Trade-offs with Maximally Robust Classifiers}
\label{sec:def_robust}

We start by defining adversarial robustness and
the robustness trade-off in adversarial training.
Then, \cref{sec:max_robust} provides an alternative formulation to
adversarial robustness that avoids the robustness trade-off.
Let
$\mathcal{D}=\{(\xx_i, y_i)\}_{i=1}^n$
denote a training set sampled i.i.d.\ from a distribution,
where
${\xx_i\in \R^d}$ are features and ${y_i\in \{-1,+1\}}$
are binary labels.~\footnote{We restrict our theoretical analysis to binary 
classification but we expect direct extensions to multi-class classification.}
A binary classifier is a function
 ${\varphi : \R^d \rightarrow \R}$, and its
prediction on an input $\xx$ is given by
$\sign(\varphi(\xx))\in\{-1,+1\}$.
The aim in supervised
learning is to find a classifier that accurately classifies
the training data and generalizes to unseen test data.
One standard framework for training a classifier is
Empirical Risk Minimization (ERM),
$\argmin_{\varphi\in\Phi}
\mathcal{L}(\varphi)$,
where
$\mathcal{L}(\varphi)
\coloneqq \E_{(\xx,y) \sim \mathcal{D}} \,
\zeta(y\varphi(\xx))$,
$\Phi$ is a family of classifiers,
and $\zeta:\R\rightarrow\R^+$
is a loss function that we assume to be
strictly monotonically decreasing to $0$,
i.e., $\zeta' < 0$.
Examples are the exponential loss, $\exp{(-\hat{y}y)}$,
and the logistic loss, $\log{(1+\exp{(-\hat{y}y)})}$, where $\hat{y}, y$ are 
the model prediction and the ground-truth label.

Given a classifier, an adversarial perturbation
$\ddelta \in \R^d$ is any small perturbation that
changes the model prediction, i.e.,
$\sign(\varphi(\xx+\ddelta))\neq
\sign(\varphi(\xx)),\,
\|\ddelta\|\leq\varepsilon$,
where $\|\cdot\|$ is a norm on $\R^d$,
and $\varepsilon$ is an arbitrarily chosen constant.
It is common to use norm-ball constraints to
ensure perturbations are small (e.g., imperceptible in images)
but other constraints exist~\citep{brown2017adversarial}.
Commonly used are the $\ell_p$ norms, 
$\|\vv\|_p=\left(\sum_{i=0}^{d-1} [\vv]_i^p\right)^{1/p}$,
where $[\vv]_i$ denotes
the $i$-th element of a vector $\vv$,
for $i=0,\ldots,d-1$.
In practice, an adversarial perturbation, $\ddelta$,
is found as an approximate solution
to the following optimization problem,
\begin{equation}
\label{eq:find_adv}
\max_{\ddelta : \|\ddelta\|\leq \varepsilon}
\zeta(y\varphi(\xx+\ddelta))\,.
\end{equation}
Under certain conditions, closed form solutions exist
to the optimization problem in \eqref{eq:find_adv}.
For example, \citet{goodfellow2014explaining} observed that
the maximal {$\ell_\infty$-bounded} adversarial perturbation
against a linear model
(i.e., one causing the maximum change in the output)
is the sign gradient direction scaled by $\varepsilon$.

\citet{madry2017towards} defined an adversarially robust classifier
as the solution to the saddle-point optimization problem,
\begin{equation}
\label{eq:saddle_point}
\argmin_{\varphi\in\Phi}\, \E_{(\xx,y) \sim \mathcal{D}}
\max_{\ddelta : \|\ddelta\|\leq \varepsilon}
\zeta(y\varphi(\xx+\ddelta))\,.
\end{equation}
The saddle-point adversarial robustness problem
is the robust counter-part to empirical risk minimization where the expected loss is minimized on
worst-case adversarial samples defined as solutions
to \eqref{eq:find_adv}.
Adversarial Training~\citep{goodfellow2014explaining}
refers to solving \eqref{eq:saddle_point}
using an alternated optimization. It is
computationally expensive because
it often requires solving
\eqref{eq:find_adv} many times.

The main drawback of defining the adversarially robust classifier 
using \eqref{eq:saddle_point}, and a drawback of
adversarial training,
is that the parameter $\varepsilon$
needs to be known or tuned.
The choice of $\varepsilon$ controls a trade-off
between standard accuracy on samples
of the dataset $\mathcal{D}$
versus the robust accuracy, i.e., the accuracy on adversarial samples.
At one extreme $\varepsilon=0$, where
\eqref{eq:saddle_point} reduces to {ERM}.
At the other, 
as $\varepsilon\rightarrow\infty$,
all inputs in $\R^d$ are within the $\varepsilon$-ball
of every training point and can be an adversarial input.
The value of the inner max in \eqref{eq:saddle_point}
for a training point $\xx,y$ is the loss of the most confident prediction over $\R^d$ that is predicted as
$-y$. For large enough $\varepsilon$, the solution to
\eqref{eq:saddle_point} is a classifier predicting
the most frequent label, i.e., 
$\varphi(\cdot)=p^\ast$, where $p^\ast$ is the solution to
$\argmin_p n_{-1}\zeta(-p)+n_{+1}\zeta(p)$, and
$n_{-1}, n_{+1}$ are the number of negative and positive
training labels.

Robust accuracy is often a complementary generalization
metric to standard test accuracy.
In practice, we prefer a classifier that is accurate on the test set,
and that additionally, achieves maximal robustness.
The saddle-point formulation makes this challenging
without the knowledge of the maximal $\varepsilon$.
This trade-off has been studied in various works~\citep{%
tsipras2018robustness, zhang2019theoretically, fawzi2018adversarial,
fawzi2018analysis, schmidt2018adversarially}.
Regardless of the trade-off imposed by $\varepsilon$,
adversarial training
is considered to be the state-of-the-art for adversarial
robustness. The evaluation is
based on the robust accuracy achieved at
fixed $\varepsilon$'s
even though the standard accuracy is usually lower
than a comparable non-robust model~\citep{croce2020robustbench,
gowal2020uncovering}.

\subsection{Maximally Robust Classifier}
\label{sec:max_robust}
In order to avoid the trade-off imposed by $\varepsilon$
in adversarial robustness, we revisit a definition from
robust optimization.
\begin{defn}%
\label{def:max_robust}
A \textbf{Maximally Robust Classifier} (\citet{ben2009robust}) is a solution to
\begin{align}
\label{eq:max_robust0}
    \argmax_{\varphi\in\Phi} \{\varepsilon\,|\,y_i \varphi(\xx_i + \ddelta) > 0 
    ,\,~\forall i,\|\ddelta\| \leq \varepsilon\}\,.
\end{align}
\end{defn}
Compared with the saddle-point formulation~\eqref{eq:saddle_point},
$\varepsilon$ in \eqref{eq:max_robust0}
is not an arbitrary constant. Rather, it is
maximized as part of the optimization problem.
Moreover, the maximal $\varepsilon$ in this definition does
not depend on a particular loss function.
Note, a maximally robust classifier
is not necessarily unique.

The downside of \eqref{eq:max_robust0}
is that the formulation requires
the training data to be separable
so that \eqref{eq:max_robust0} is non-empty,
i.e., there exists $\varphi\in\Phi$
such that $\forall i, y_i\varphi(\xx_i)>0$.
In most deep learning settings,
this is not a concern as models are large
enough that they
can interpolate the training data, i.e.,
for any dataset there exists $\varphi$ such that $\varphi(\xx_i)=y_i$.
An alternative formulation is to modify
the saddle-point problem and
include an outer maximization on $\varepsilon$ by allowing
a non-zero slack loss. However, the new slack loss
reimposes a trade-off between standard and robust accuracy
(See \cref{sec:max_robust_gen}).

One can also show that adversarial training, i.e.,
solving the saddle-point problem \eqref{eq:saddle_point}, 
does not necessarily find a maximally robust classifier.
To see this, suppose we are given the maximal
$\varepsilon$ in \eqref{eq:max_robust0}.
Further assume the minimum of \eqref{eq:saddle_point} is non-zero.
Then the cost in the saddle-point problem does not distinguish
between the following two models: 1) a model that makes no
misclassification errors but has low confidence,
i.e.,
$\forall i,\, 0<\max_{\ddelta} y_i \varphi(\xx_i+\ddelta)\leq c_1$
for some small $c_1$
2) a model that classifies a training point, $\xx_j$, incorrectly
but is highly confident on all other training data and adversarially
perturbed ones, i.e.,
$\forall i\neq j,\,
0<c_2 < \max_{\ddelta} y_i\varphi(\xx_i+\ddelta)$.
The second model can incur a loss $n\zeta(c_1)-(n-1)\zeta(c_2)$
on $\xx_j$ while being no worse than the first model according
to the cost of the saddle-point problem.
The reason is another trade-off between
standard and robust accuracy caused by
taking the expectation over data points.

\subsection{Linear Models: Maximally Robust is the Minimum Norm Classifier}
Given a dataset and a norm, what is the maximally robust
linear classifier with respect to that norm? In this section, we revisit a result from~\citet{ben2009robust}
for classification.

\begin{defn}[Dual norm]
Let $\|\cdot\|$ be a norm on $\R^{n}$. The associated \emph{dual norm}, denoted 
    $\| \cdot \|_\ast$, is defined as
$\|\ddelta\|_\ast = \sup_{\xx}\{ |\langle \ddelta, \xx \rangle|\; |\; \|\xx\| \leq 1 \}\;$.
\end{defn}

\begin{defn}[Linear Separability]
    \label{def:sep}
    We say a dataset is linearly separable if there exists $\ww,b$ such that
    ${y_i (\ww^\top \xx_i +b)> 0}$ for all $i$.
\end{defn}

\begin{restatable}[Maximally Robust Linear Classifier (\citet{ben2009robust}, \S12)]{lem}{lemmaxrobust}
\label{lem:max_robust_to_min_norm_linear}
For linear models and linearly separable data,
the following problems are equivalent; i.e., 
from a solution of one, a solution of the other is readily found.
\begin{align}
\label{eq:max_robust}
   \text{Maximally robust classifier:}\quad&
   \!\! \argmax_{\ww,b} \{\varepsilon\,|\,
    y_i (\ww^\top (\xx_i + \ddelta) +b) > 0 
    ,\,~\forall i,\|\ddelta\| \leq \varepsilon\}\, ,\\
\label{eq:max_margin}
   \text{Maximum margin classifier:}\quad&
    \argmax_{\ww,b : \|\ww\|_\ast\leq 1} \{\varepsilon\,|\,
    y_i (\ww^\top  \xx_i + b)
    \geq \varepsilon,\,~\forall i\}\,,\\
\label{eq:min_norm}
   \text{Minimum norm classifier:}\quad&
    \argmin_{\ww,b} \{\|\ww\|_\ast\,|\,
    y_i (\ww^\top  \xx_i+b) \geq 1,\,~\forall i\}\,.
\end{align}
The expression $\min_i y_i (\ww^\top \xx_i+b)/\|\ww\|$
is the margin
of a classifier $\ww$ that is the distance of the nearest training
point to the classification boundary,
i.e., the line ${\{\vv:\ww^\top \vv=-b\}}$.
\end{restatable}

We provide a proof for
general norms based on \citet{ben2009robust} in \cref{proof:max_robust_to_min_norm_linear}.
Each formulation in \cref{lem:max_robust_to_min_norm_linear}
is connected to a wide array of results that can be transferred
to other formulations. Maximally robust classification is
one example of a problem in robust optimization that can
be reduced and solved efficiently. Other problems
such as robust regression as well as
robustness to correlated input perturbations have been studied
prior to deep learning~\citep{ben2009robust}.

On the other hand, maximum margin and minimum norm
classification have long been popular because of their
generalization guarantees.
Recent theories for overparametrized models
link the margin and the norm of a model
to generalization~\citep{hastie2019surprises}.
Although the tools are different, connecting the margin and
the norm of a model has also been the basis of generalization
theories for Support Vector Machines 
and AdaBoost~\citep{shawe1998structural,telgarsky2013margins}.
Maximum margin classification does not require linear separability,
because there can exist a classifier with $\varepsilon<0$ that
satisfies the margin constraints.
Minimum norm classification is the easiest formulation
to work with in practice
as it does not rely on $\varepsilon$ nor $\ddelta$ and
minimizes a function of the weights subject to a set of
constraints.

In what follows, we use
\cref{lem:max_robust_to_min_norm_linear}
to transfer recent results about minimum norm classification
to maximally robust classification. These results 
have been the basis for explaining generalization properties
of deep learning models~\citep{hastie2019surprises,nakkiran2019deep}.

\section{Implicit Robustness of Optimizers}
\label{sec:implicit_bias}

The most common approach to empirical risk minimization (ERM) is through 
gradient-based optimization. 
As we will review shortly, \citet{gunasekar2018characterizing} showed
that gradient  descent, and more generally steepest descent methods, 
have an implicit bias towards minimum norm solutions. 
From the infinitely many solutions that minimize 
the empirical risk,
we can characterize the one found by steepest descent.
Using \cref{lem:max_robust_to_min_norm_linear}, we show that
such a classifier is also maximally robust
w.r.t.\ a specific norm.

Recall that ERM is defined as
$\argmin_{\varphi\in\Phi}\mathcal{L}(\varphi)$, where
$\mathcal{L}(\varphi)
= \E_{(\xx,y) \sim \mathcal{D}} \zeta(y\varphi(\xx))$.
Here we assume $\mathcal{D}$ is a finite dataset of size $n$.
For the linear family of functions, we write $\mathcal{L}(\ww,b)$.
Hereafter, we rewrite the loss as
$\mathcal{L}(\ww)$ and use an augmented representation with
a constant $1$ dimension.
For linearly separable data and overparametrized models
($d>n$), there exist infinitely many linear
classifiers that minimize the empirical
risk~\citep{gunasekar2018characterizing}.
We will find it convenient to ignore the scaling
and focus on the normalized vector $\ww/\|\ww\|$, i.e.,
the direction of $\ww$.
We will say that the sequence $\ww_1, \ww_2, \ldots$ 
converges in direction to a vector $\vv$
if $\lim_{t\rightarrow\infty} \ww_t/\|\ww_t\|=\vv$.

\subsection{Steepest Descent on Fully-Connected Networks}

\begin{defn}[Steepest Descent]
    Let $\langle\cdot\rangle$ denote an inner product and $\|\cdot\|$ its 
    associated norm, $f$ a function to be minimized, and
$\gamma$ a step size. The steepest descent method associated with this norm finds
\begin{align}\label{eq:steepest_descent}
\ww_{t+1} &= \ww_t + \gamma \Delta\ww_t,\nonumber\\
\text{where}~~\Delta \ww_{t} &\in \argmin_\vv \langle\nabla f(\ww_t), \vv \rangle
+ \frac{1}{2}\|\vv\|^2\,.
\end{align}
The steepest descent step, $\Delta \ww_{t}$, can be equivalently written as
${-\|\nabla f(\ww_t)\|_\ast\, g_{\text{nst}}}$,
where
$g_{\text{nst}} \in \argmin\left\{ \langle\nabla f(\ww_t), \vv\rangle\; |\; \|\vv\| = 1\right\}$.
A proof can be found in \citep[\S9.4]{boyd2004convex}.
\end{defn}

\paragraph{Remark.} For some $p$-norms,
steepest descent steps have closed form expressions.
Gradient Descent (GD) is steepest descent w.r.t.\ $\ell_2$ norm 
where ${- \nabla f(\ww_t)}$ is a steepest descent step.
Sign gradient descent is steepest descent w.r.t.\
$\ell_\infty$ norm where  
${- \|\nabla f(\ww_t)\|_1 \sign(\nabla f(\ww_t))}$ is a steepest descent step.
Coordinate Descent (CD) is steepest descent w.r.t.\ $\ell_1$ norm where
${- \nabla f(\ww_t)_i {\boldsymbol{e}}_i }$ is a steepest descent step
($i$ is the coordinate for which the gradient has the largest absolute magnitude).

\begin{thm}[Implicit Bias of Steepest Descent (\citet{gunasekar2018characterizing} (Theorem 5))]
\label{thm:gunasekar_linear}
For any separable dataset $\{\xx_i, y_i\}$
and any norm $\|\cdot\|$,
consider the steepest descent updates
from \eqref{eq:steepest_descent} for minimizing
the empirical risk $\mathcal{L}(\ww)$ (defined in \cref{sec:def_robust})
with the exponential loss, $\zeta(z)=\exp{(-z)}$. For all
initializations $\ww_0$, and
all bounded step-sizes satisfying a known upper bound,
the iterates $\ww_t$ satisfy
\begin{equation}
\label{eq:gunasekar_linear}
    \lim_{t \to \infty}
    \min_i \frac{y_i \ww_t^\top \xx_i}{\|\ww_t\|}
    = \max_{\ww : \|\ww\| \leq 1}
    \min_i y_i \ww^\top \xx_i\,.
\end{equation}
In particular, if a unique maximum margin classifier
$\ww_{\|\cdot\|}^\ast=\argmax_{\ww : \|\ww\| \leq 1} \min_i y_i \ww^\top \xx_i$
exists, the limit direction converges to it, i.e.,
$\lim_{t \to \infty}\frac{\ww_t}{\|\ww_t\|}
=\ww^\ast_{\|\cdot\|}$.
\end{thm}

In other words, the margin converges to the maximum margin
and if the maximum margin classifier is unique, the
iterates converge in direction to $\ww^\ast_{\|\cdot\|}$.
We use this result to derive our \cref{cor:implicit_robust}.

\begin{cor}[Implicit Robustness of Steepest Descent]
\label{cor:implicit_robust}
For any linearly separable dataset and any norm $\|\cdot\|$,
steepest descent iterates
minimizing the empirical risk, $\mathcal{L}(\ww)\,$,
satisfying the conditions of \cref{thm:gunasekar_linear},
converge in direction to a maximally robust classifier,
\begin{align*}
    \argmax_{\ww}
    \{\varepsilon\,|\,y_i \ww^\top (\xx_i + \ddelta) > 0 
    ,\,~\forall i,\,\|\ddelta\|_\ast \leq \varepsilon\}\,.
\end{align*}
In particular, a maximally robust classifier against
$\ell_1$, $\ell_2$, and $\ell_\infty$ is reached, respectively, by
sign gradient descent, gradient descent, and coordinate descent.
\end{cor}

\begin{proof}
By \cref{thm:gunasekar_linear},
the margin of the steepest descent iterates,
$\min_i \frac{y_i \ww_t^\top \xx_i}{\|\ww_t\|}$ ,
converges as $t \to \infty$ to the maximum margin,
$\max_{\ww : \|\ww\| \leq 1} \min_i y_i \ww^\top \xx_i$.
By  \cref{lem:max_robust_to_min_norm_linear},
any maximum margin classifier w.r.t. $\|\cdot\|$
gives a maximally robust classifier w.r.t. $\|\cdot\|_\ast$.
\end{proof}

\cref{cor:implicit_robust} implies that for
overparametrized linear models,
we obtain guaranteed robustness by an appropriate choice of
optimizer without the additional cost
and trade-off of adversarial training.
We note that \cref{thm:gunasekar_linear} and
\cref{cor:implicit_robust},
characterize linear models, but do not account for the bias $b$.
We can close the gap with an augmented input representation,
to include the bias explicitly.
Or one could preprocess the data, 
removing the mean before training.

To extend \cref{cor:implicit_robust} to deep learning models
one can use generalizations of \cref{thm:gunasekar_linear}.
For the special case of gradient descent,
\cref{thm:gunasekar_linear} has been generalized to multi-layer
fully-connected linear networks and a larger family
of  strictly monotonically decreasing loss functions
including the logistic loss~\citep[Theorem 2]{nacson2019convergence}.

\subsection{Gradient Descent on Linear Convolutional Networks}
\label{sec:implicit_conv}

In this section, we show that even for linear models,
the choice of the architecture affects implicit robustness,
which gives another alternative for achieving maximal robustness.
We use a generalization of
\cref{thm:gunasekar_linear} to linear convolutional models.

\begin{defn}[Linear convolutional network]
An $L$-layer convolutional network with
$1$-D circular convolution is parameterized using
weights of $L-1$ convolution layers,
$\ww_1,\ldots,\ww_{L-1}\in\R^d$,
and weights of a final linear layer, $\ww_L\in\R^d$,
such that the linear mapping of the network is
\begin{equation*}
\varphi_{\text{conv}}(\xx;\ww_1,\ldots,\ww_L)
\coloneqq
\ww_L^\top (\ww_{L-1} \star \cdots (\ww_1 \star \xx))\,.
\end{equation*}
Here, circular convolution is defined as
$ [\ww\star\xx]_i
\coloneqq
\frac{1}{\sqrt{d}} \sum_{k=0}^{d-1}
[\ww]_{\overline{-k}} [\xx]_{\overline{i+k}}$, where
$[\vv]_i$ denotes the $i$-th element of a vector $\vv$
for $i=0,\ldots,d-1$,
and $\overline{\phantom{i}i \phantom{i}}=i\bmod d$.~\footnote{%
We use the usual definition of circular
convolution in signal processing, rather than cross-correlation,
$\ww^\downarrow \star \xx$ with $[\vv^\downarrow]_i=[\vv]_{\overline{-i}}$, which is used in deep learning literature, but not associative.}
\end{defn}

A linear convolutional network is equivalent
to a linear model with weights
$ \ww= \ww_L\star(\cdots\star (\ww_{2}\star\ww_1)) $ because
of the associative property of convolution.
In particular, for two-layer linear convolutional networks
$\ww=\ww_2\star\ww_1$.

\begin{defn}[Discrete Fourier Transform]
$\mathcal{F}(\ww)\in\C^d$ denotes the Fourier coefficients of $\ww$
where
$[\mathcal{F}(\ww)]_d = \frac{1}{\sqrt{d}}
\sum_{k=0}^{d-1} [\ww]_k \exp(-\frac{2\pi j}{d} k d)$
and $j^2=-1$.
\end{defn}

\begin{thm}[Implicit Bias towards Fourier Sparsity
(\citet{gunasekar2018implicit}, Theorem 2, 2.a)]
\label{thm:gunasekar_conv}
Consider the family of
$L$-layer linear convolutional networks and
the sequence of gradient descent iterates, $\ww_t$,
minimizing the empirical risk, $\mathcal{L}(\ww)$, with
the exponential loss, $\exp{(-z)}$.
For almost all linearly separable datasets
under known conditions on the step size and
convergence of iterates,
$\ww_t$  converges in direction
to the classifier minimizing
the norm of the Fourier coefficients given by
\begin{align}
\argmin_{\ww_1,\ldots,\ww_L} \{ \|\mathcal{F}(\ww)\|_{2/L}\, |\,
y_i \langle\ww, \xx_i\rangle\geq 1,\,\forall i\}.
\end{align}
In particular, for two-layer linear convolutional networks
the implicit bias is towards the solution with minimum $\ell_1$ norm of the Fourier coefficients,
$\|\mathcal{F}(\ww)\|_1$.
For $L>2$, the convergence is to a first-order stationary
point.
\end{thm}

We use this result to derive our \cref{cor:max_robust_to_min_norm_conv}.

\begin{restatable}[Maximally Robust to Perturbations with Bounded Fourier Coefficients]{cor}{cormaxrobustconv}
\label{cor:max_robust_to_min_norm_conv}
Consider the family of two-layer linear convolutional networks
and the gradient descent iterates, $\ww_t$,
minimizing the empirical risk.
For almost all linearly separable datasets under conditions
of \cref{thm:gunasekar_conv}, $\ww_t$ converges in direction
to a maximally robust classifier,
\begin{align*}
    \argmax_{\ww_1,\ldots,\ww_L}
    \{\varepsilon\,|\,y_i
    \varphi_{\text{conv}}(\xx_i+\ddelta;\{\ww_l\}_{l=1}^L) > 0
    ,\,~\forall i,\,\|\mathcal{F}(\ddelta)\|_\infty \leq \varepsilon\}\,.
\end{align*}
\end{restatable}

Proof in \cref{proof:max_robust_to_min_norm_conv}.
\cref{cor:max_robust_to_min_norm_conv} implies
that, at no additional cost,
linear convolutional models are already maximally robust,
but w.r.t.\ perturbations in the Fourier domain.
We call attacks with $\ell_p$ constraints in the
Fourier domain \textit{Fourier-$\ell_p$} attacks.
\cref{sec:unit_norm_balls} depicts various norm-balls
in $3$D to illustrate the significant geometrical difference
between the Fourier-$\ell_\infty$ and other commonly used norm-balls for adversarial robustness.
One way to understand \cref{cor:max_robust_to_min_norm_conv}
is to think of perturbations that succeed in fooling a
linear convolutional network.
Any such
adversarial perturbation must have at least one frequency beyond the maximal robustness of the model.
This condition is satisfied for perturbations
with small $\ell_0$ norm in the spatial domain,
i.e., only a few pixels are perturbed and similarly by $\ell_1$ norm 
perturbations as they are constrained to be more sparse than other $\ell_p$ 
norm perturbations. Sparse perturbations in the spatial domain can be dense in 
the Fourier domain.

\subsection{Fourier Attacks}
    \begin{figure*}
\begin{algorithm}[H]
   \caption{Fourier-$\ell_\infty$ Attack
   (see \cref{sec:fourier_ops})}
   \label{alg:fourier_linf_attack}
\begin{algorithmic}
   \STATE {\bfseries Input:} data $\xx$, label $y$,
   loss function $\zeta$,
   classifier $\varphi$,
   perturbation size $\varepsilon$,
   number of attack steps $m$,
   dimensions $d$,
   Fourier transform $\mathcal{F}$
   \FOR{$k=1$ {\bfseries to} $m$}
       \STATE $\hat{\gg} = \mathcal{F}(\nabla_\xx \zeta(y\varphi(\xx)))$
       \STATE $[\ddelta]_i =
       \varepsilon\frac{[\hat{\gg}]_i}{|[\hat{\gg}]_i|},\,
       \forall i\in \{0,\ldots,d-1\}$
    \STATE $\xx = \xx + \mathcal{F}^{-1}(\ddelta)$
   \ENDFOR
\end{algorithmic}
\end{algorithm}
    \end{figure*}

The predominant motivation for designing new attacks is to
\emph{fool} existing models. 
In contrast, our results characterize the attacks 
that existing models \emph{perform best} 
against, as measured by maximal robustness.
Based on \cref{cor:max_robust_to_min_norm_conv} we
design the Fourier-$\ell_p$ attack to verify our results.
Some adversarial attacks exist with
Fourier constraints~\citep{tsuzuku2019structural,guo2019low}.
\citet{sharma2019effectiveness} proposed a Fourier-$\ell_p$ 
attack that includes Fourier constraints
in addition to $\ell_p$ constraints in the spatial domain.
Our theoretical results suggest a more general class of
attacks with only Fourier constraints.

The maximal {$\ell_p$-bounded} adversarial perturbation
against a linear model in \eqref{eq:find_adv} consists of real-valued
constraints with a closed form solution.
In contrast, maximal Fourier-$\ell_p$ has complex-valued constraints.
In \cref{sec:fourier_ops} we derive the Fourier-$\ell_\infty$
attack in closed form for linear models and provide the pseudo-code
in \cref{alg:fourier_linf_attack}.
To find perturbations as close as possible to
natural corruptions such as
blur, $\varepsilon$ can be a matrix of constraints
that is multiplied elementwise by $\ddelta$.
As our visualizations in \cref{fig:image_attack} show,
adversarial perturbations under bounded
Fourier-$\ell_\infty$ 
can be controlled to be high frequency and
concentrated on subtle details of the image, or low frequency
and global.
We observe that high frequency Fourier-$\ell_\infty$ attacks
succeed more easily with smaller perturbations compared
with low frequency attacks.
The relative success of our
band-limited Fourier attacks matches
the empirical observation
that the amplitude spectra
of common $\ell_p$ attacks are largely 
band-limited as such attacks succeed more easily~\citep{yin2019fourier}.

\begin{figure*}[t]
    \centering
    \begin{subfigure}[b]{.31\textwidth}
    \begin{tabular}{c}
    $\xx$\hfill
    $\xx+\ddelta$\hfill
    \,$\ddelta$\hfill
    \\
    \includegraphics[width=.95\textwidth]{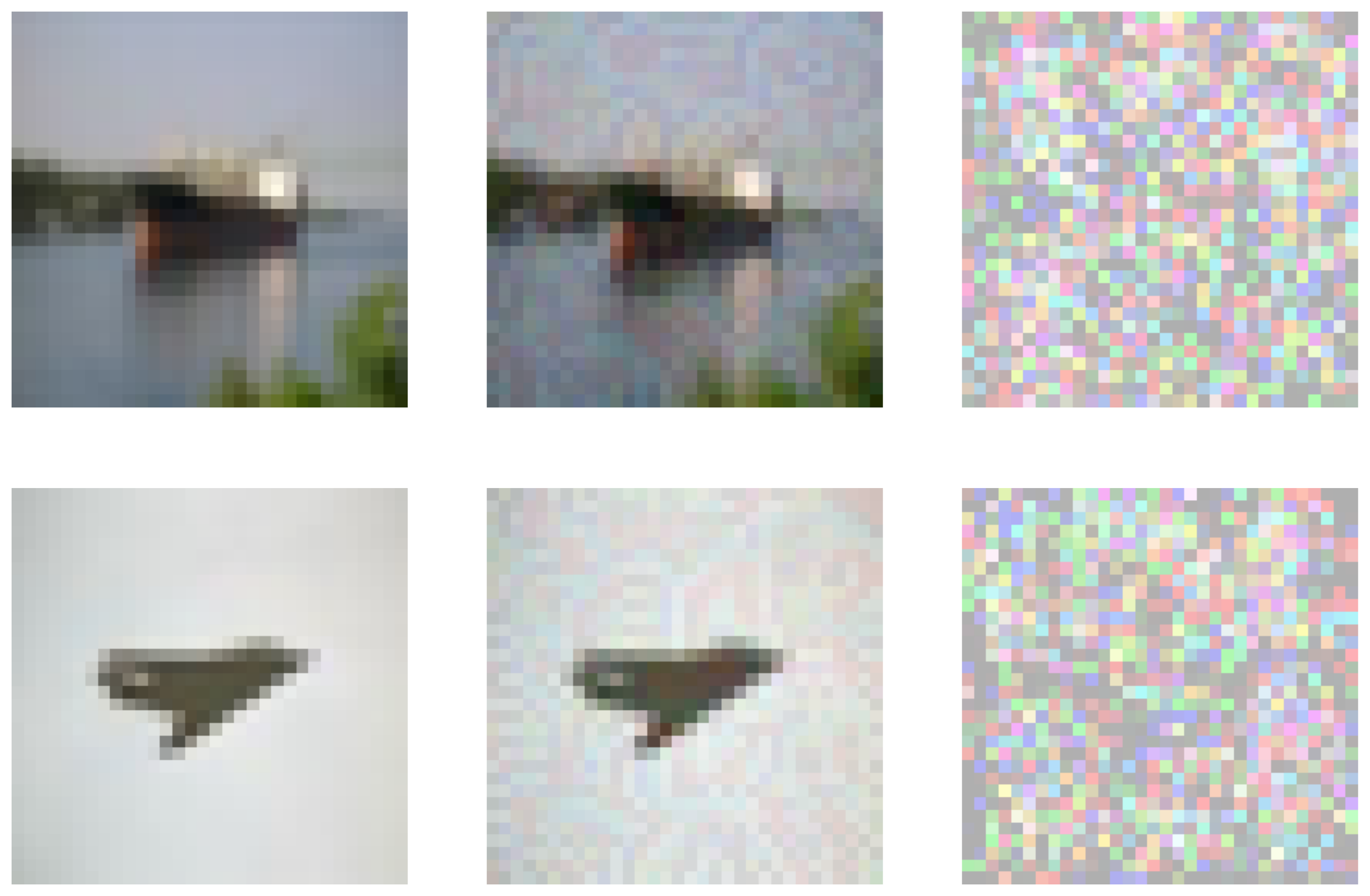}
    \end{tabular}
    \caption{$\ell_\infty$ attack}
    \label{fig:image_attack_linf}
    \end{subfigure}
    \hfill
    \begin{subfigure}[b]{.21\textwidth}
    \begin{tabular}{c}
    $\xx+\ddelta$\hfill
    \,$\ddelta$\hfill
    \\
    \includegraphics[width=.95\textwidth]{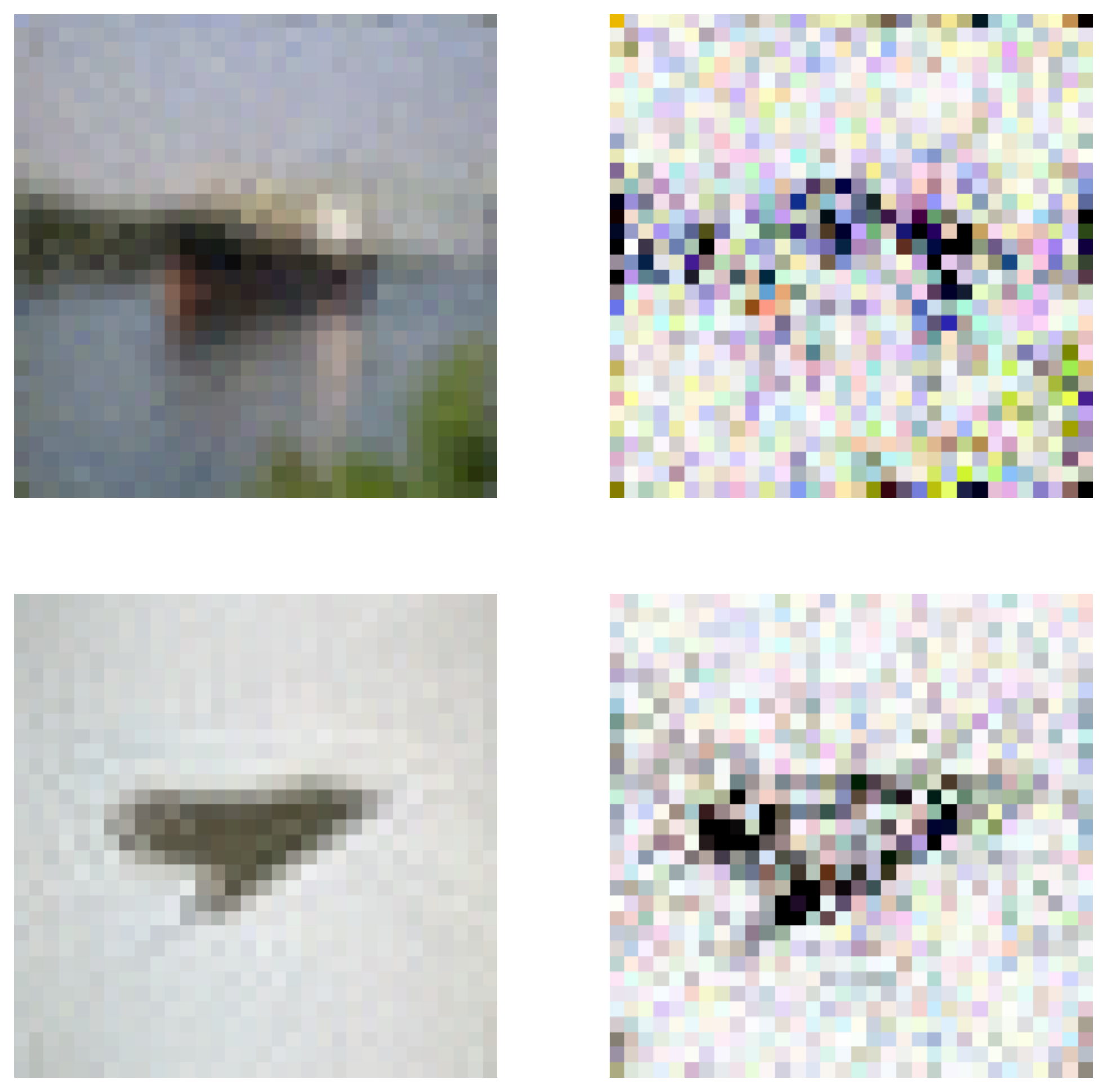}
    \end{tabular}
    \caption{Fourier-$\ell_\infty$ attack}
    \label{fig:image_attack_dftinf_any}
    \end{subfigure}
    \hfill
    \begin{subfigure}[b]{.21\textwidth}
    \begin{tabular}{c}
    $\xx+\ddelta$\hfill
    \,$\ddelta$\hfill
    \\
    \includegraphics[width=.95\textwidth]{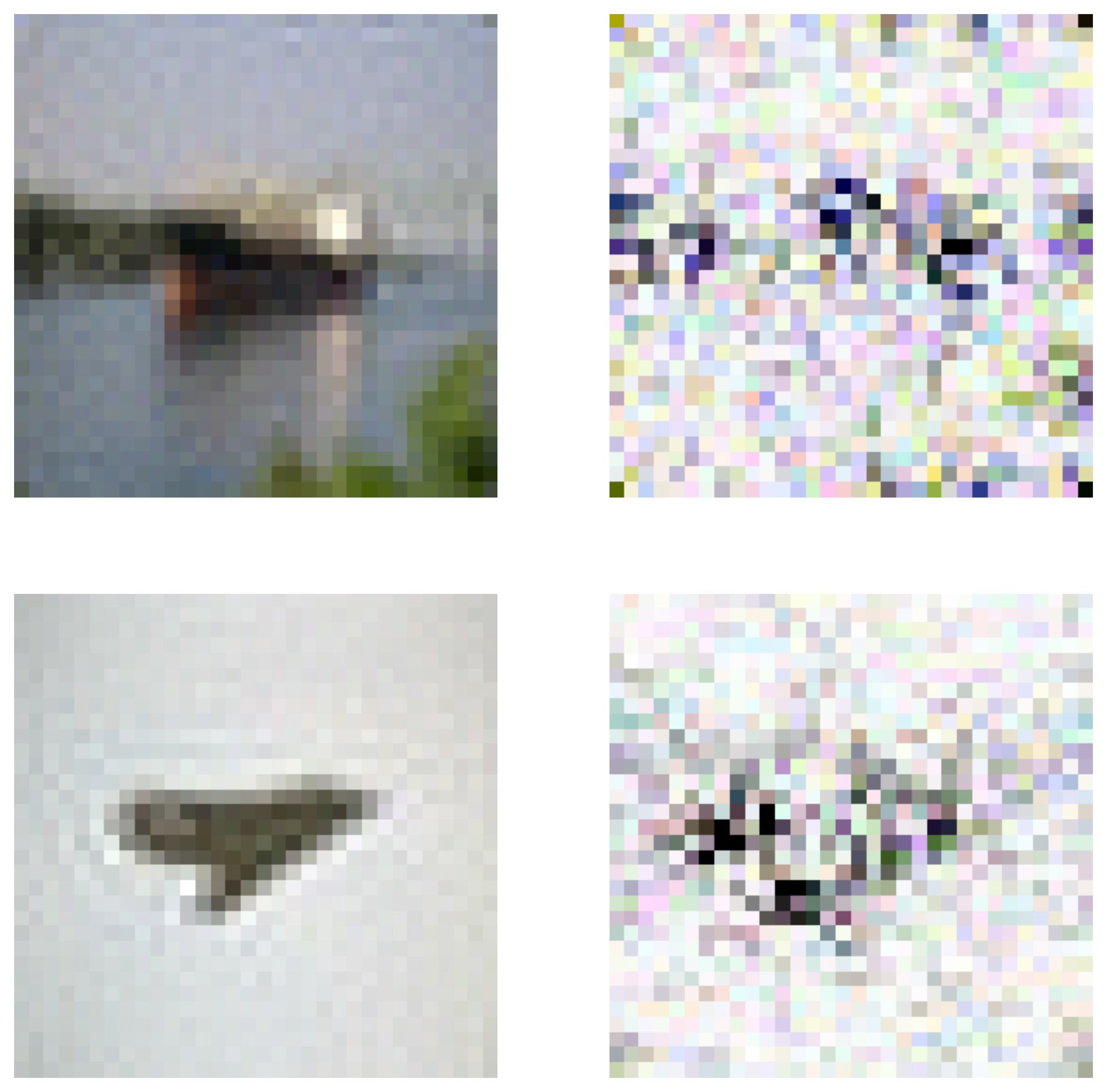}
    \end{tabular}
    \caption{High freq. F-$\ell_\infty$}
    \label{fig:image_attack_dftinf_highf}
    \end{subfigure}
    \hfill
    \begin{subfigure}[b]{.21\textwidth}
    \begin{tabular}{c}
    $\xx+\ddelta$\hfill
    \,$\ddelta$\hfill
    \\
    \includegraphics[width=.95\textwidth]{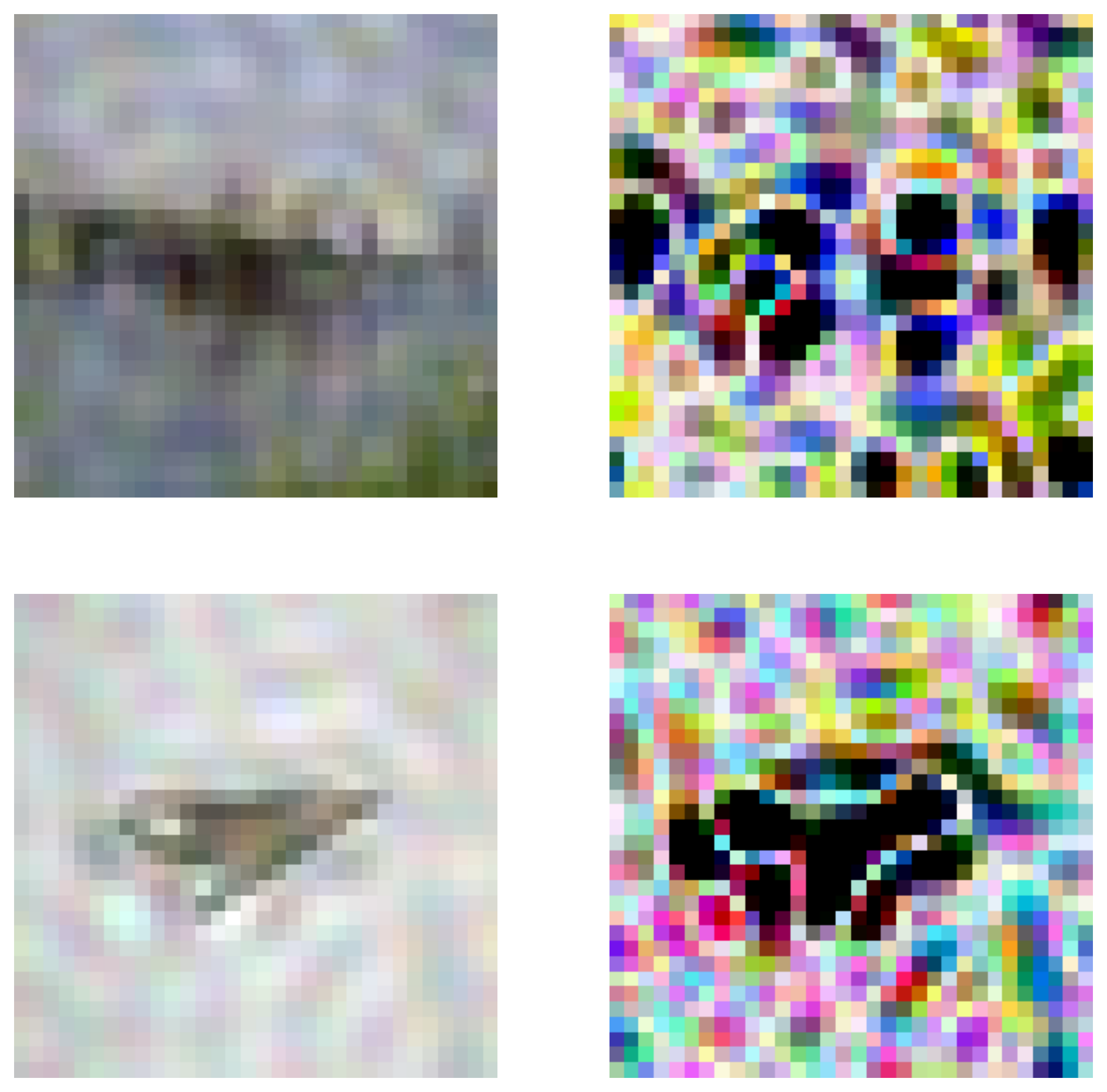}
    \end{tabular}
    \caption{Low Freq. F-$\ell_\infty$}
    \label{fig:image_attack_dftinf_lowf}
    \end{subfigure}
    \hfill
    \caption{%
    \textbf{Adversarial attacks ($\ell_\infty$ and Fourier-$\ell_\infty$)
    against CIFAR-10 classification models.}
    Fourier-$\ell_\infty$ perturbations
    (\subref{fig:image_attack_dftinf_any}) in the spatial
    domain are concentrated around subtle details of the object
    (darker means stronger perturbation).
    In contrast, $\ell_\infty$ perturbations
    (\subref{fig:image_attack_linf})
    are perceived by people as random noise.
    Fourier-$\ell_\infty$ can also be controlled to be
    high or low frequency
    (\subref{fig:image_attack_dftinf_highf},
    \subref{fig:image_attack_dftinf_lowf}). It is more difficult to attack a standard model
    with only low frequency perturbations
    (for all attacks $\varepsilon=8/255$ but for low frequency Fourier-$\ell_\infty$ $\varepsilon=50/255$, otherwise
    attack fails).
    \cref{sec:fourier_vis} shows 
    visualizations for variety of models in RobustBench.
    }
    \label{fig:image_attack}
\end{figure*}

\section{Explicit Regularization}
\label{sec:explicit_reg}

Above we discussed the impact of optimization method
and model architecture on robustness. Here, 
we discuss explicit regularization as
another choice that affects robustness.

\begin{defn}[Regularized Classification]
\label{def:reg_classifier}
The regularized empirical risk minimization problem
for linear classification
is defined as
    \begin{align*}
        \hat{\ww}(\lambda) &=
\argmin_\ww \E_{(\xx,y) \sim \mathcal{D}}
\zeta(y\ww^\top \xx) + \lambda \|\ww\|,
    \end{align*}
where $\lambda$ denotes a regularization constant,
$\zeta$ is a monotone loss function,
and $\mathcal{D}$ is a dataset. For simplicity we assume
this problem has a unique solution while the original ERM can
have multiple solutions.
\end{defn}

\begin{thm}[Maximum Margin Classifier using Regularization
(\citet{rosset2004margin}, Theorem 2.1)]
\label{thm:rosset2004_reg}
Consider
linearly separable finite
datasets and monotonically non-increasing
loss functions. Then as $\lambda\rightarrow 0$,
the sequence of solutions, $\hat{\ww}(\lambda)$,
to the regularized problem  in
\cref{def:reg_classifier}, 
converges in direction to a maximum margin classifier
as defined in \eqref{eq:max_margin}. Moreover, if the maximum margin
classifier is unique,
\begin{align}
\lim_{\lambda\rightarrow0}
\frac{\hat{\ww}(\lambda)}{\|\hat{\ww}(\lambda)\|}
&= \argmax_{\ww : \|\ww\| \leq 1}
\min_i y_i \ww^\top \xx_i\,.
\end{align}
\end{thm}
The original proof in \citep{rosset2004margin}
was given specifically for $\ell_p$
norms, however we observe that their proof only requires
convexity of the norm, so we state it more generally.
Quasi-norms such as $\ell_p$
for $p<1$ are not covered by this theorem. In addition,
the condition on the loss function is weaker than
our strict monotonic decreasing condition as shown
in \citep[Appendix A]{nacson2019convergence}.

We use this result to derive our \cref{cor:reg_is_max_margin}.

\begin{cor}[Maximally Robust Classifier via
Infinitesimal Regularization]
\label{cor:reg_is_max_margin}
For linearly separable data, under conditions of
\cref{thm:rosset2004_reg},
the sequence of solutions to regularized classification problems
converges in direction to a maximally robust classifier. That is,
$\lim_{\lambda\rightarrow 0}\hat{\ww}(\lambda)/\|\hat{\ww}(\lambda)\|$
converges to a solution of
$\argmax_{\ww}
\{\varepsilon\,|\,y_i \ww^\top (\xx_i + \ddelta) > 0 
,\,~\forall i,\,\|\ddelta\|_\ast \leq \varepsilon\}$.
\end{cor}
\begin{proof}
By \cref{thm:rosset2004_reg},
the margin of the sequence of regularized classifiers,
$\min_i y_i \frac{\hat{\ww}(\lambda)^\top}{\|\hat{\ww}(\lambda)\|} \xx_i$,
converges to the maximum margin,
$\max_{\ww : \|\ww\| \leq 1} \min_i y_i \ww^\top \xx_i$.
By \cref{lem:max_robust_to_min_norm_linear},
any maximum margin classifier w.r.t. $\|\cdot\|$
gives a maximally robust classifier w.r.t. $\|\cdot\|_\ast$.
\end{proof}
Assuming the solution to the
regularized problem is unique,
the regularization term replaces other implicit
biases in minimizing the empirical risk.
The regularization coefficient
controls the trade-off between robustness
and standard accuracy.
The advantage of this formulation compared with
adversarial training is that we do not need the knowledge
of the maximally robust $\varepsilon$ to find
a maximally robust classifier. It suffices to choose
an infinitesimal regularization coefficient.
\citet[Theorem 4.1]{wei2019regularization}
generalized \cref{thm:rosset2004_reg} for a family of
classifiers that includes fully-connected networks with ReLU
non-linearities, which allows for potential extension of
\cref{cor:reg_is_max_margin} to non-linear models.
There remain gaps in this extension
(see \cref{sec:nonlinear}).

Explicit regularization has been explored
as an alternative approach to
adversarial training~\citep{hein2017formal,
sokolic2017robust,
zhang2019theoretically,
qin2019adversarial,
guo2020connections}.
To be clear, we do not propose a new regularization method but rather,
we provide a framework for deriving and guaranteeing
the robustness of existing and future regularization methods.

\section{Experiments}
\label{sec:experiments}

\begin{figure*}[t]
\centering
\begin{subfigure}[c]{0.26\textwidth}
\includegraphics[width=\textwidth]{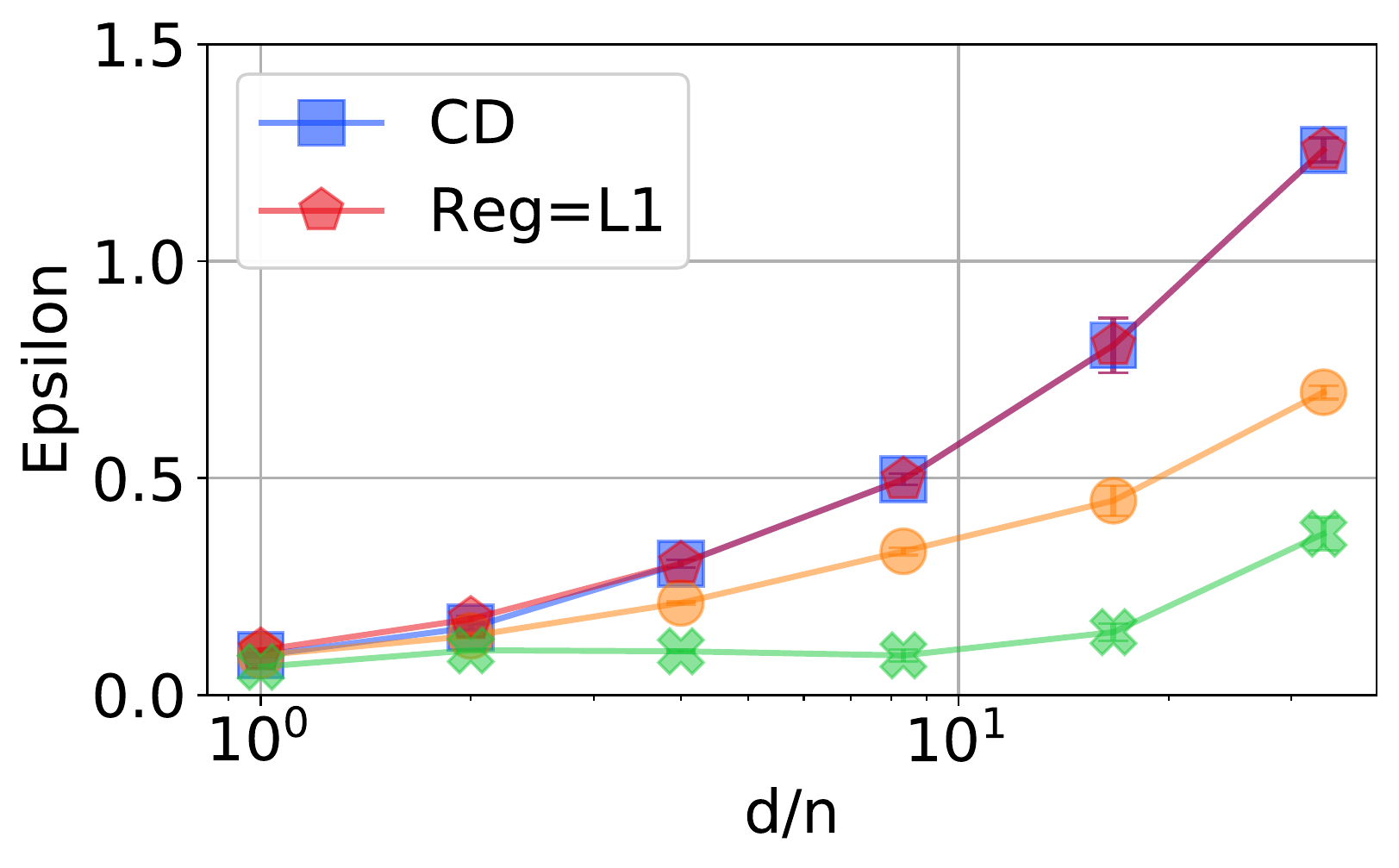}
\caption{$\ell_\infty$ attack}
\label{fig:linear_max_eps_linf}
\end{subfigure}
\hspace{1pt}
\begin{subfigure}[c]{0.26\textwidth}
\includegraphics[width=\textwidth]{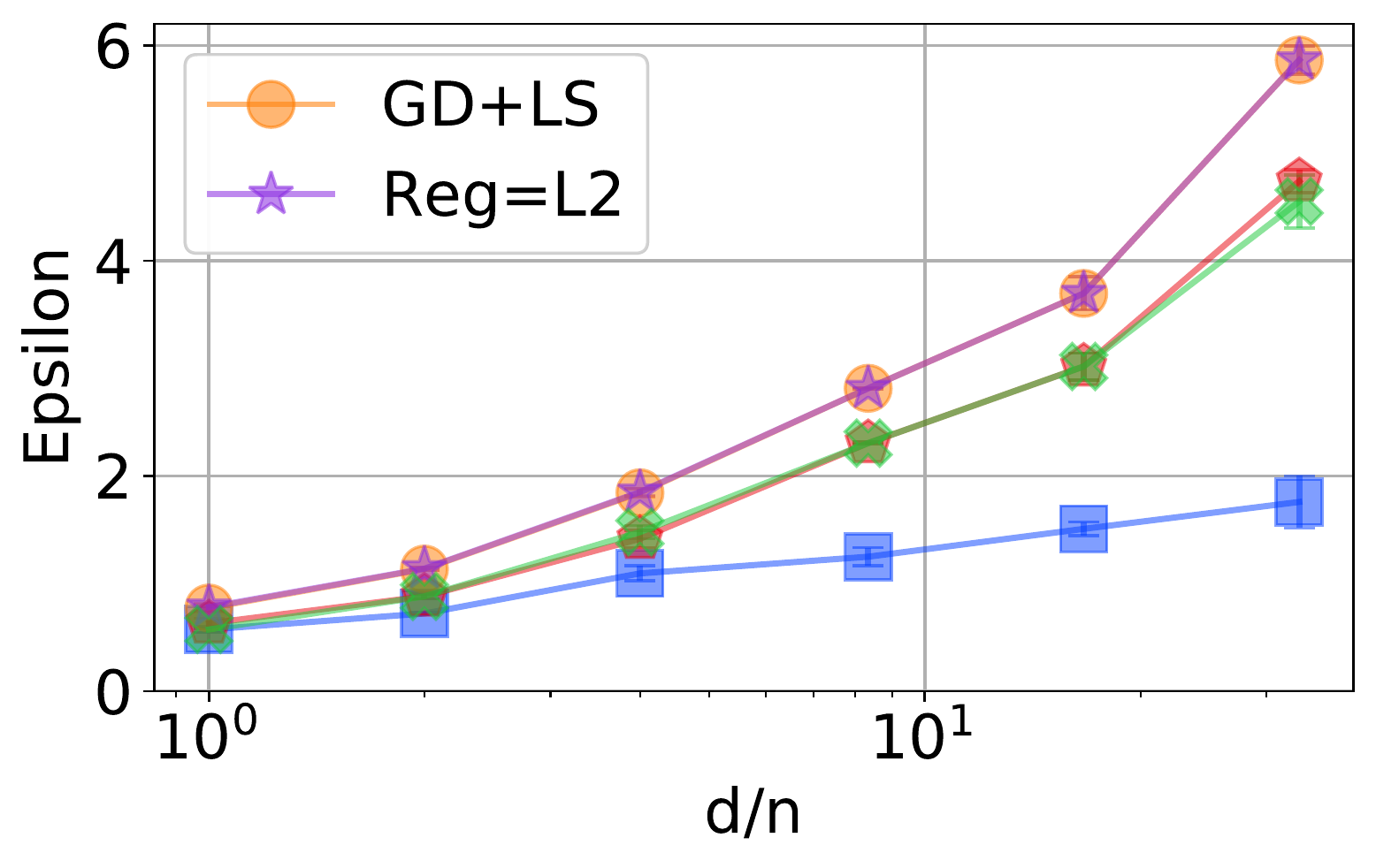}
\caption{$\ell_2$ attack}
\label{fig:linear_max_eps_l2}
\end{subfigure}
\hspace{1pt}
\begin{subfigure}[c]{0.26\textwidth}
\includegraphics[width=\textwidth]{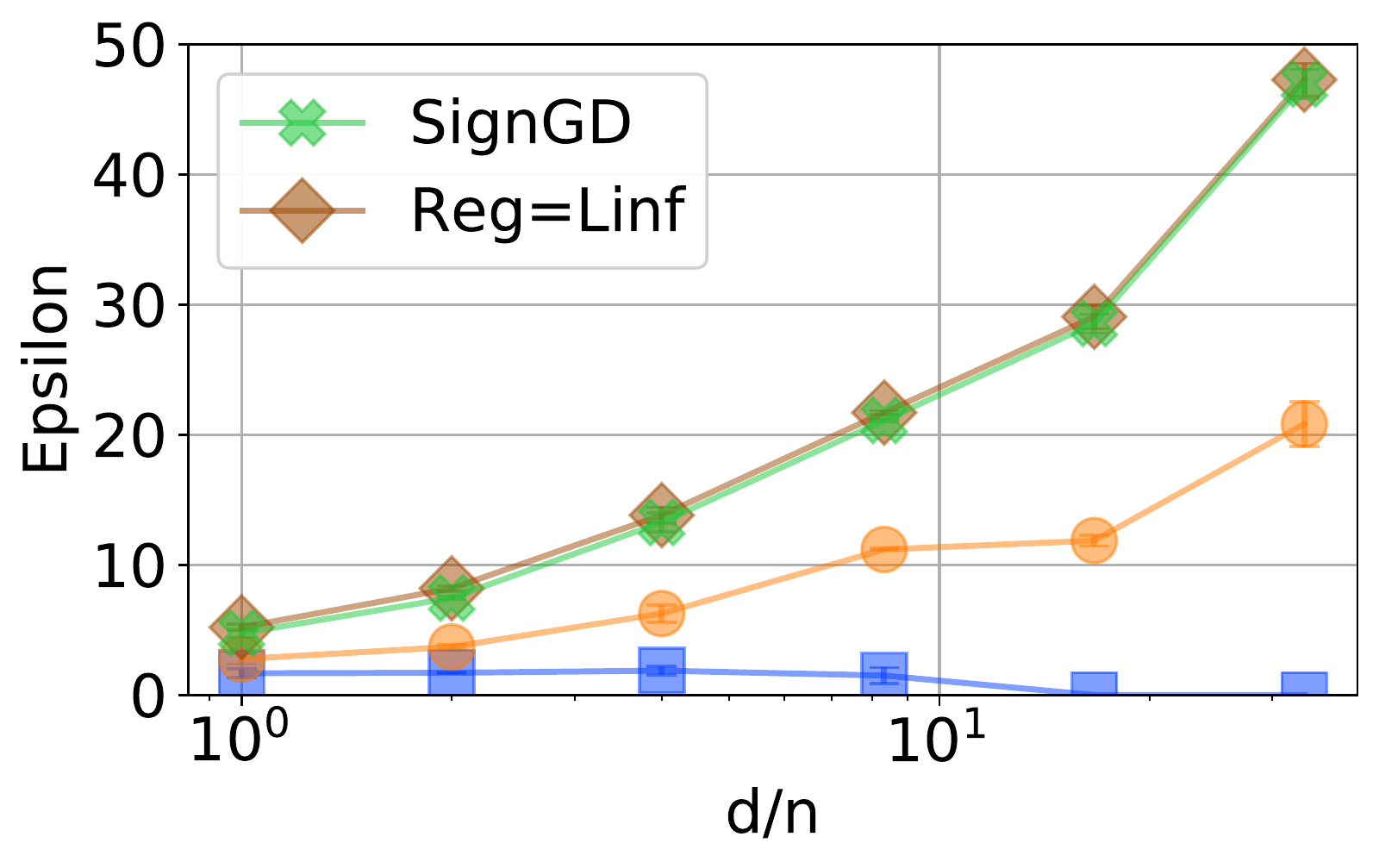}
\caption{$\ell_1$ attack}
\label{fig:linear_max_eps_l1}
\end{subfigure}
\hspace{1pt}
\begin{subfigure}[c]{0.12\textwidth}
\vspace*{-25pt}
\fbox{%
\includegraphics[width=\textwidth]{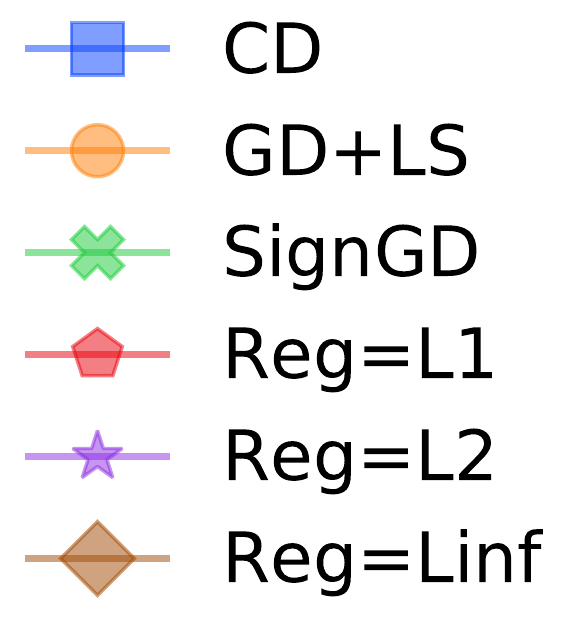}}
\end{subfigure}
\caption{\textbf{Maximally robust perturbation size ($\varepsilon$)
for linear models against $\ell_\infty$, $\ell_2$, and $\ell_1$ attacks.}
For each attack, there exists one optimizer and one regularization method
that finds a maximally robust classifier
(inner legends). 
We compare Coordinate Descent (CD), Gradient Descent with Line Search
(GD+LS), Sign Gradient Descent (SignGD), and explicit $\ell_1$, $\ell_2$,
and $\ell_\infty$ regularization.
The gap between methods grows with the overparametrization ratio ($d/n$).
(More figures in \cref{sec:margin})}
\label{fig:linear_max_eps}
\end{figure*}

This section empirically compares approaches to finding
maximally robust classifiers.
\cref{sec:cifar10} evaluates the robustness of
CIFAR-10~\citep{krizhevsky2009learning}
image classifiers against our Fourier-$\ell_\infty$ attack. We
implement our attack in AutoAttack~\citep{croce2020reliable}
and  evaluate the robustness
of recent defenses available in
RobustBench~\citep{croce2020robustbench}.
Details of the experiments and additional
visualizations are in \cref{sec:exp_ext}.

\subsection{Maximally Robust to
\texorpdfstring{$\ell_\infty,\ell_2,\ell_1$}{Linf, L2, L1}, and Fourier-\texorpdfstring{$\ell_\infty$}{Linf} Bounded Attacks}
    \begin{figure*}
\centering
\begin{subfigure}[r]{.5\linewidth}
\includegraphics[width=\textwidth]{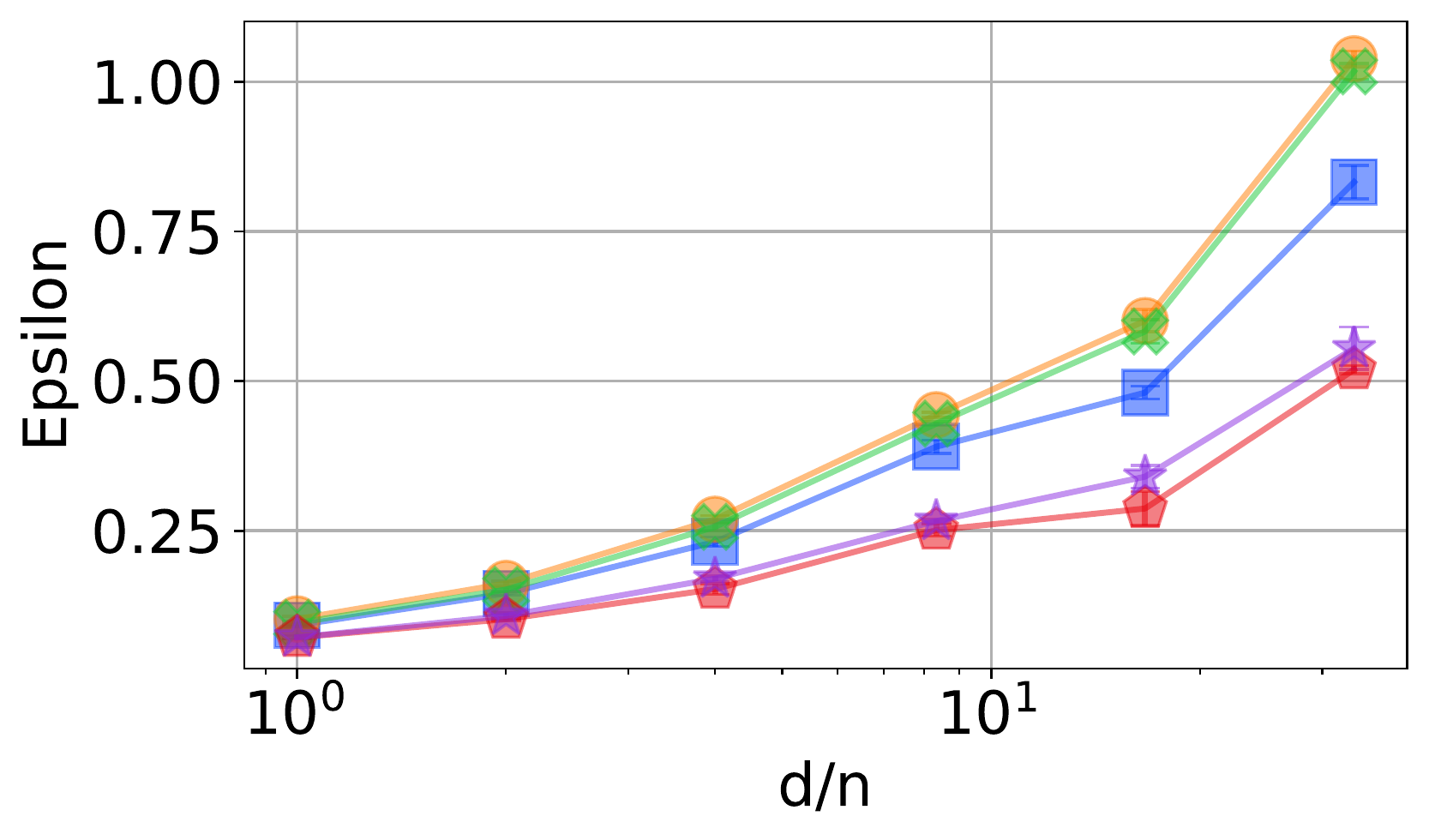}
\end{subfigure}
        \hspace*{20pt}
\begin{subfigure}[l]{.2\linewidth}
\fbox{%
\includegraphics[width=\textwidth]{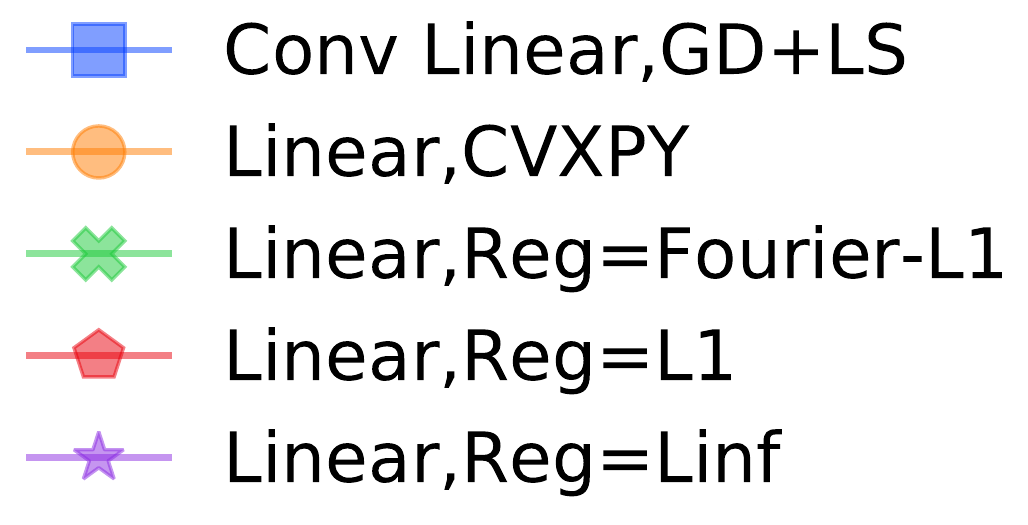}}
\end{subfigure}
\caption{\textbf{Maximally robust $\varepsilon$ against
Fourier-$\ell_\infty$ attack.}
Explicit Fourier-$\ell_1$ regularization finds
a maximally robust classifier as it achieves similar robustness
as CVXPY's solution.
A linear convolutional model converges to a solution
but a small gap exists.}
\label{fig:conv}
    \end{figure*}

\cref{fig:linear_max_eps} and \ref{fig:conv} plot the maximally robust 
$\epsilon$ as a function of the overparametrization ratio $d/n$, where $d$ is 
the model dimension and $n$ is the number of data points.
\cref{fig:linear_max_eps} shows robustness against $\ell_\infty$, $\ell_2$, and $\ell_1$ attacks for linear models.
Coordinate descent and explicit $\ell_1$ regularization
find a maximally robust $\ell_\infty$ classifier.
Gradient descent and $\ell_2$ regularization
find a maximally robust $\ell_2$ classifier.
Sign gradient descent and $\ell_\infty$ regularization
find a maximally robust $\ell_1$ classifier.
The gap between margins grows as $d/n$ increases.
\cref{fig:conv} shows robustness against
Fourier-$\ell_\infty$ attack; training a 2-layer
linear convolutional network with gradient descent converges to a maximally 
robust classifier.  A gap exists between the linear convolutional network and 
the maximally robust classifier that should theoretically disappear with 
training budget greater than what we have used.

For these plots we synthesized linearly separable data focusing on overparametrized classification problems (i.e., $ d > n $).
Plotting the overparametrization ratio shows how robustness
changes as models become more complex.
We compare models by computing the maximal $\varepsilon$
against which they are robust, or equivalently,
the margin for linear models, $\min_i y_i\ww^\top\xx_i/\|\ww\|$.
As an alternative to the margin, we estimate the maximal $\varepsilon$
for a model by choosing a range of potential values,
generating adversarial samples, and finding the largest value against which the classification error is zero.
Generating adversarial samples involves optimization, and requires
more implementation detail compared with computing the margin.
Plots in this section are based on generating adversarial
samples to match common practice in the evaluation of non-linear models.
Matching margin plots are presented in \cref{sec:margin},
which compare against the solution found using CVXPY~\citep{diamond2016cvxpy}
and adversarial training given the maximal $\varepsilon$.
Our plots depict mean and error bars for $3$ random seeds.

\subsection{Plotting the Trade-offs}
\label{sec:tradeoffs}
\begin{figure*}[t]
\centering
\begin{subfigure}[b]{0.24\textwidth}
\includegraphics[width=\textwidth]{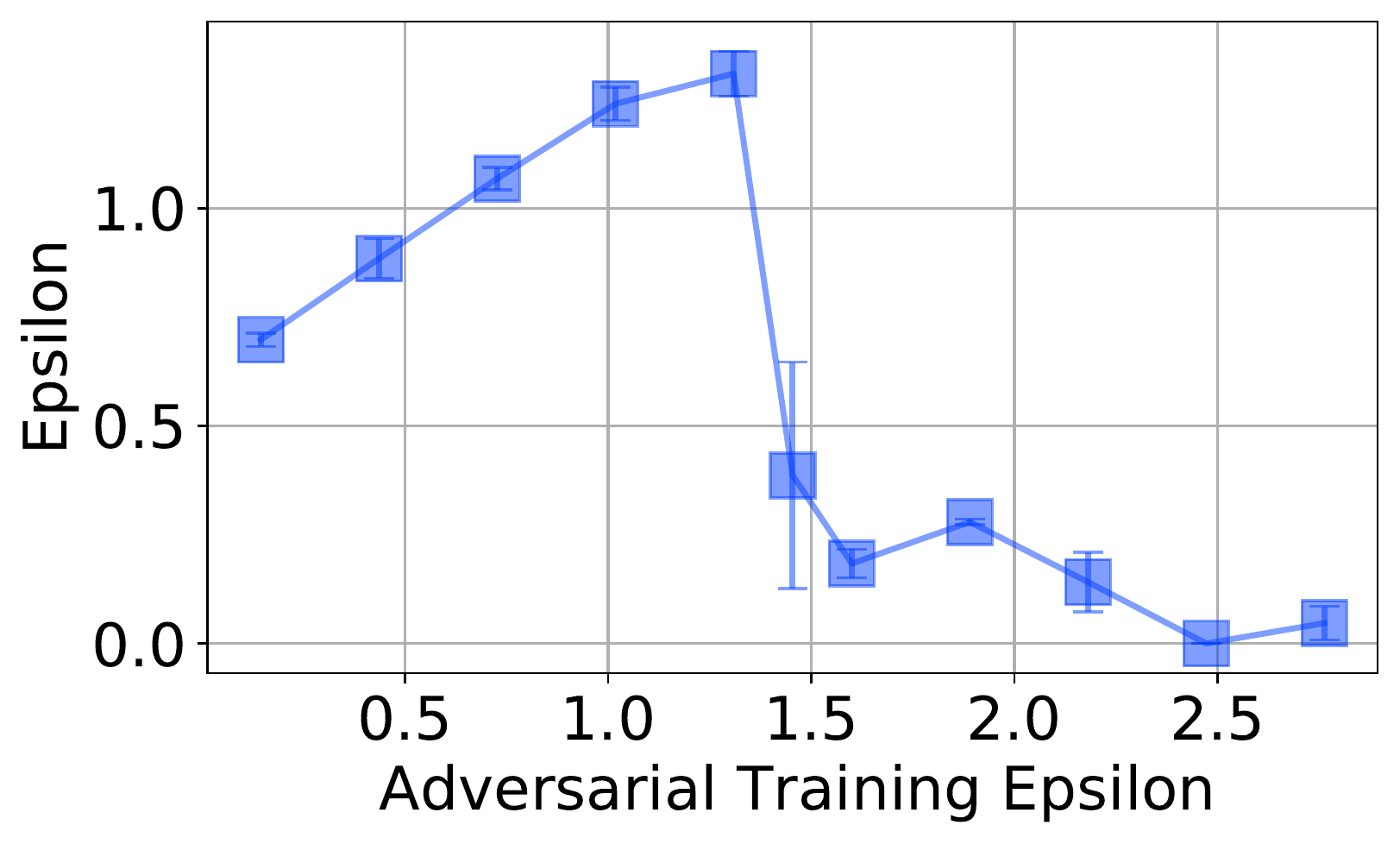}
\caption{Adversarial training}
\label{fig:tradeoffs_at}
\end{subfigure}
\begin{subfigure}[b]{0.24\textwidth}
\includegraphics[width=\textwidth]{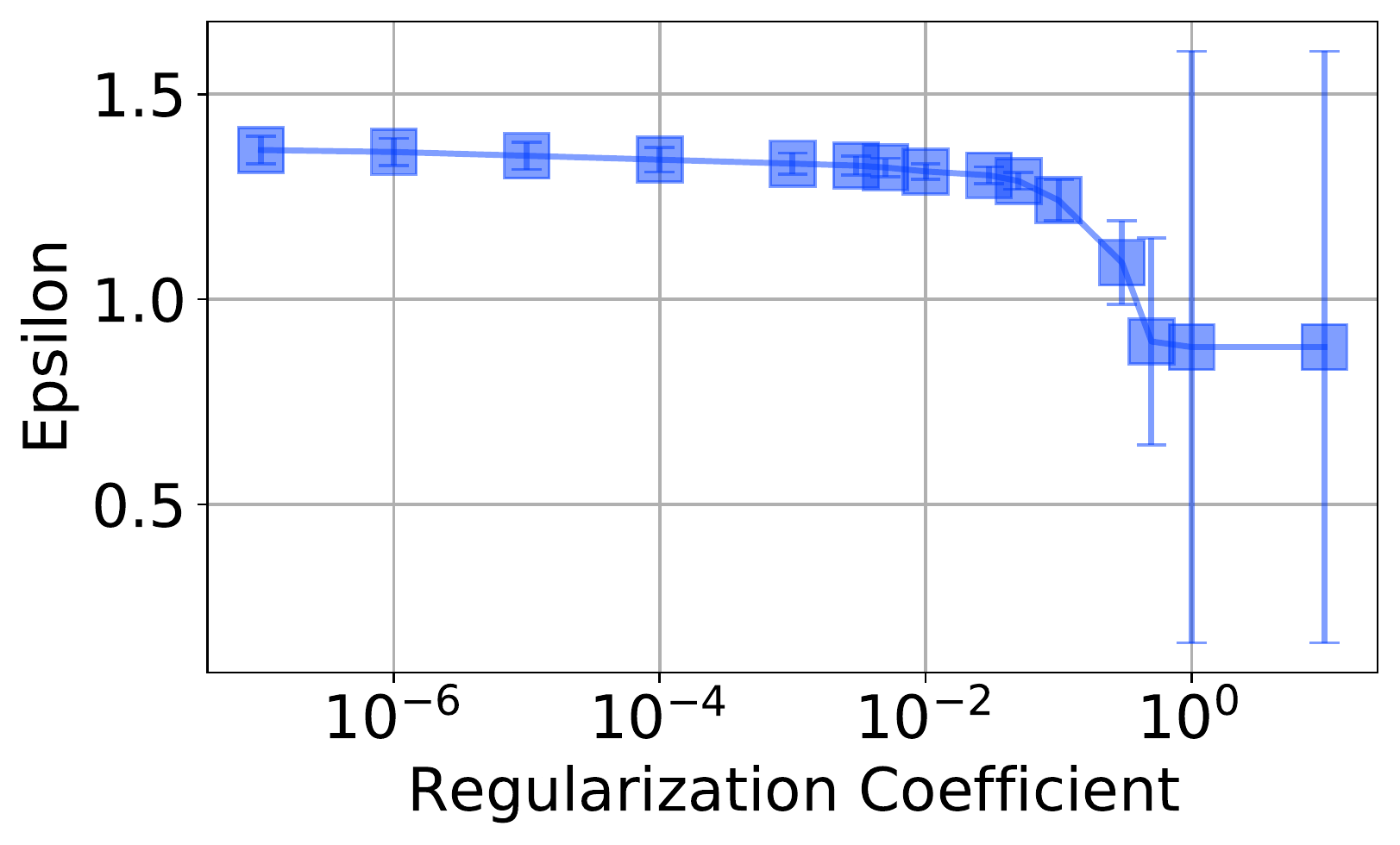}
\caption{$\ell_1$ regularization}
\label{fig:tradeoffs_l1}
\end{subfigure}
\begin{subfigure}[b]{0.24\textwidth}
\includegraphics[width=\textwidth]{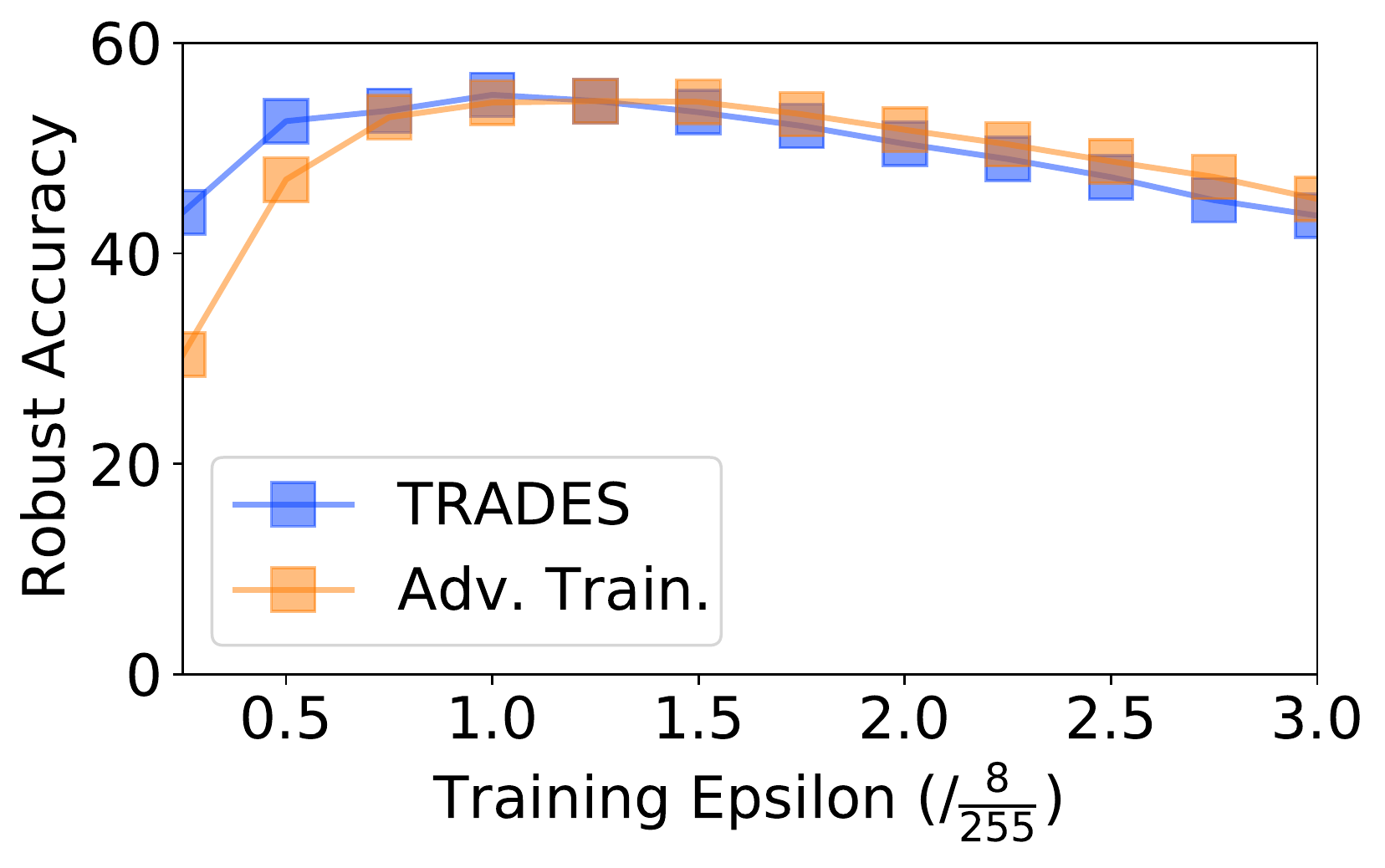}
\caption{CIFAR-10 Train. Eps.}
\label{fig:tradeoffs_cifar10_eps}
\end{subfigure}
\begin{subfigure}[b]{0.24\textwidth}
\includegraphics[width=\textwidth]{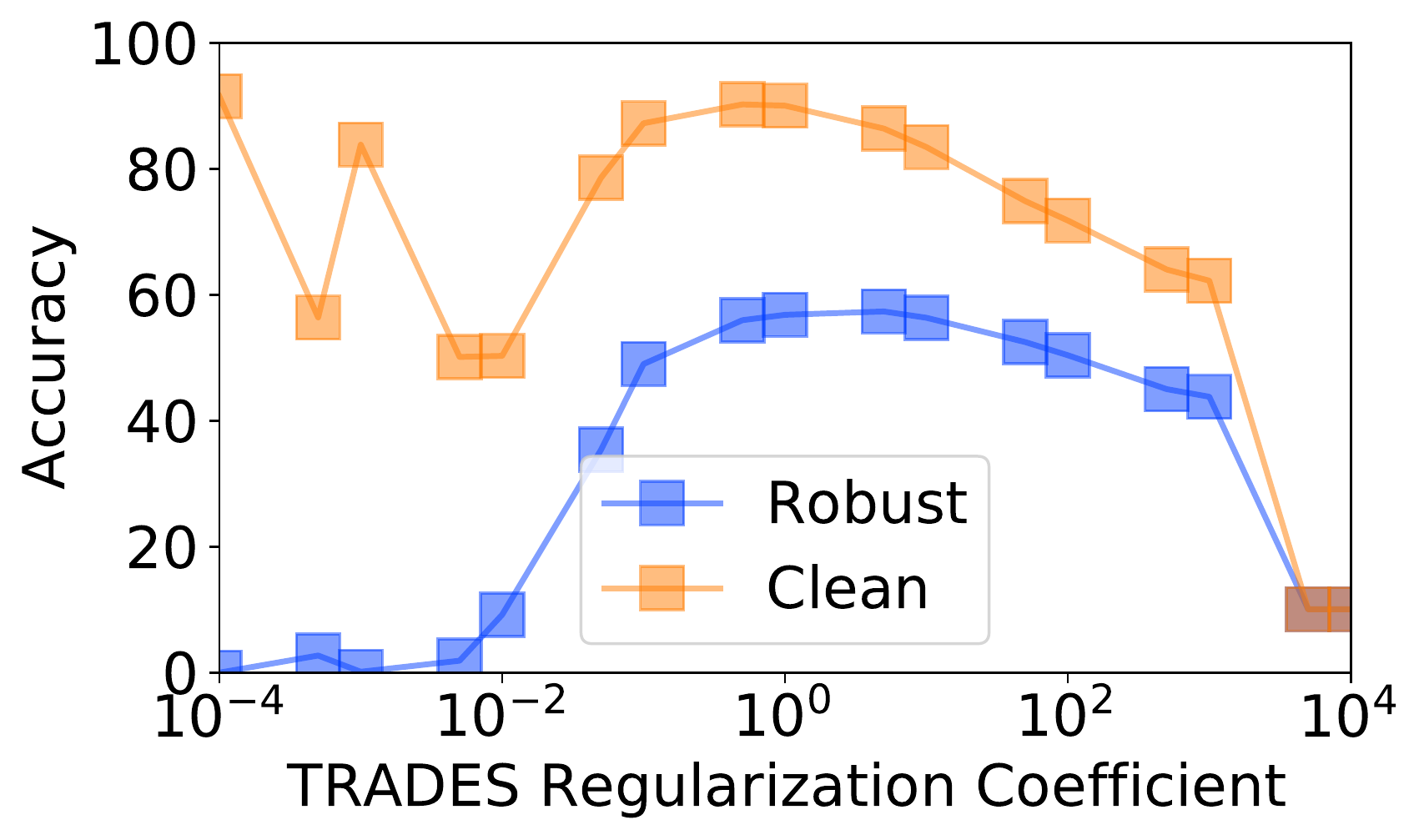}
\caption{CIFAR-10 TRADES}
\label{fig:tradeoffs_cifar10_trades}
\end{subfigure}
\caption{\textbf{Trade-off
in robustness against $\ell_\infty$ attack
in linear models and CIFAR-10.}
We plot the maximally robust $\varepsilon$ for adversarial
training and explicit regularization.
Robustness is controlled using $\varepsilon$ in adversarial training (\subref{fig:tradeoffs_at})
and regularization coefficient in explicit regularization (\subref{fig:tradeoffs_l1}).
Using adversarial training we have to search for the maximal
$\varepsilon$ but for explicit regularization it suffices
to choose a small regularization coefficient.
Similarly, on CIFAR-10,
the highest robustness at a fixed test $\varepsilon$
is achieved for $\varepsilon$ used during
training (\subref{fig:tradeoffs_cifar10_eps}).
In contrast to optimal linear regularizations,
TRADES shows degradation as the
regularization coefficient decreases
(\subref{fig:tradeoffs_cifar10_trades}).
Discussion in \cref{sec:tradeoffs}
}
\label{fig:tradeoffs_linf}
\end{figure*}

\cref{fig:tradeoffs_linf} illustrates the trade-off between
standard accuracy and adversarial robustness.
Adversarial training finds the maximally robust classifier
only if it is trained with the
knowledge of the maximally robust $\varepsilon$ (\cref{fig:tradeoffs_at}).
Without this knowledge, we have to search for the maximal $\varepsilon$
by training multiple models.
This adds further computational complexity to adversarial
training which performs an alternated optimization.
In contrast, explicit regularization converges to a
maximally robust classifier for a small enough regularization constant (\cref{fig:tradeoffs_l1}).

On CIFAR-10, we compare adversarial training with
the regularization method TRADEs~\citep{zhang2019theoretically}
following the state-of-the-art
best practices~\citep{gowal2020uncovering}.
Both methods depend on a constant $\varepsilon$
during training.
\cref{fig:tradeoffs_cifar10_eps} shows optimal robustness
is achieved for a model trained and tested with
the same $\varepsilon$. When the test
$\varepsilon$ is unknown, both methods need to search for
the optimal $\varepsilon$.
Given the optimal training $\varepsilon$,
\cref{fig:tradeoffs_cifar10_trades} investigates whether
TRADES performs similar to an optimal linear regularization
(observed in \cref{fig:tradeoffs_l1}),
that is the optimal robustness is achieved with infinitesimal
regularization.
In contrast to the linear regime, the robustness degrades with
smaller regularization.
We hypothesize that with enough model capacity,
using the optimal $\varepsilon$,
and sufficient training iterations, smaller
regularization should improve robustness.
That suggests that there is potential for improvement
in TRADES and better understanding of robustness in non-linear
models.

\subsection{CIFAR-10 Fourier-\texorpdfstring{$\ell_\infty$}{Linf} Robustness}
\label{sec:cifar10}
    \begin{figure*}
    \centering
    \includegraphics[width=0.5\linewidth]{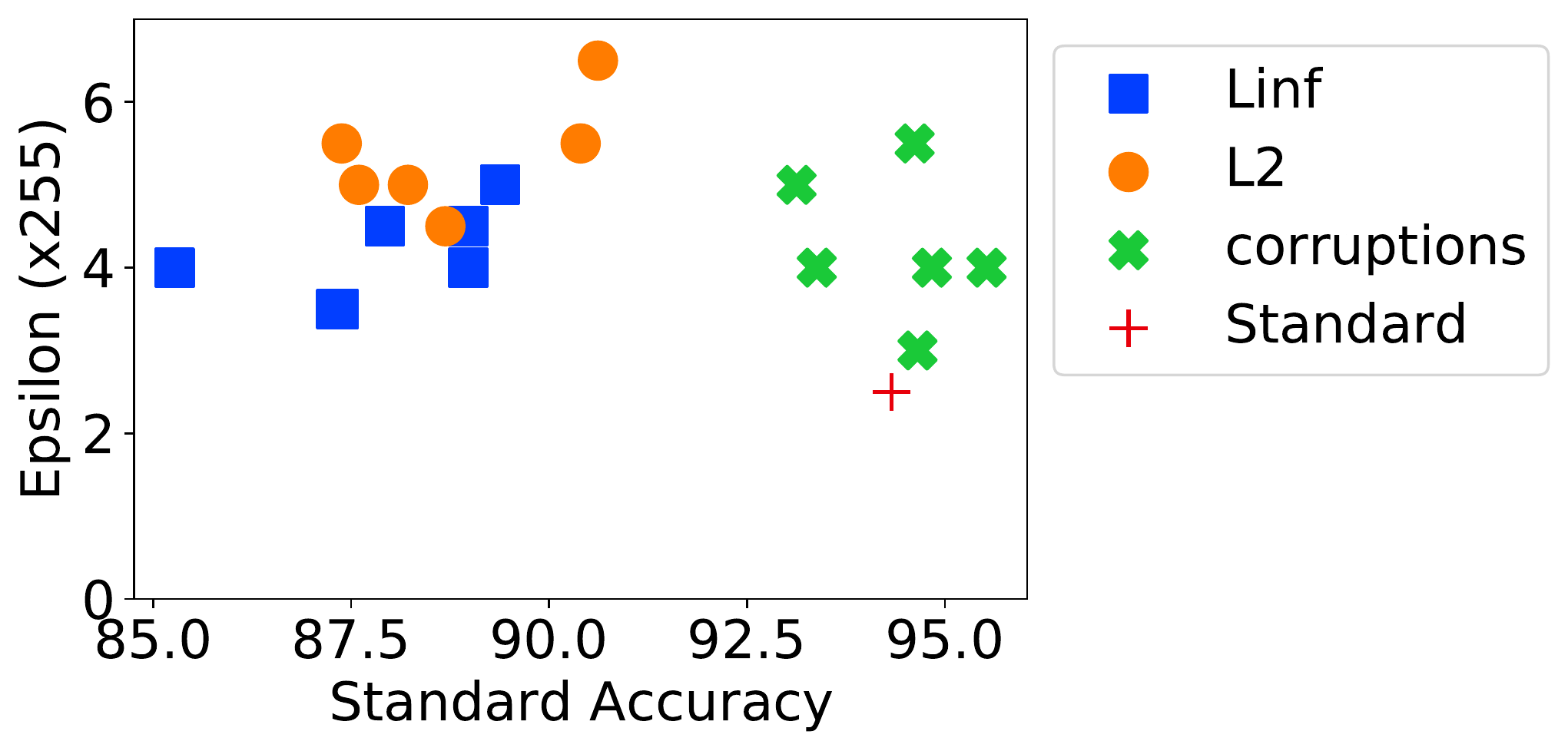}
    \caption{\textbf{Maximally robust
$\varepsilon$ against Fourier-$\ell_\infty$ for recent defenses.}
Color and shape denote the type of
robust training: adversarial ($\ell_2$  or $\ell_\infty$),
corruptions,
or standard training.
Fourier-$\ell_\infty$ is a strong attack against
current robust models.
}
    \label{fig:cifar10_acc}
    \end{figure*}

\cref{fig:cifar10_acc} reports the maximally robust $\varepsilon$
of image classification models on CIFAR-10.
We evaluate top defenses on the leaderboard of
RobustBench~\citep{croce2020robustbench}.
The attack methods are APGD-CE and APGD-DLR with default
hyper-parameters in RobustBench and $\varepsilon=8/255$.
Theoretical results do not provide guarantees beyond the maximally
robust $\varepsilon$.
Even robust models
against corruptions with no adversarial training achieve
similar robustness to $\ell_2$/$\ell_\infty$ models.
The maximal $\varepsilon$ is the largest
at which adversarial accuracy is no more than $1\%$ worse
than the standard accuracy.
All models have almost zero accuracy
against larger, but still perceptually small, perturbations
($\varepsilon=20/255$).
\cref{sec:fourier_vis} gives more examples of
Fourier-$\ell_\infty$ attacks and band-limited variations
similar to \cref{fig:image_attack} for robustly trained models, showing that
perturbations are qualitatively different from those against
the standard model.

\section{Related work}
\label{sec:related}

This paper bridges two bodies of work on adversarial robustness and
optimization bias. 
As such there are many related works,
the most relevant of which we discuss here.
Prior works either did not connect optimization bias
to adversarial robustness beyond
margin-maximization~\citep{ma2020increasing,ding2018max,elsayed2018large}
or only considered adversarial
training with a given perturbation size~\citep{li2019implicit}.

\paragraph{Robustness Trade-offs}
Most prior work defines the metric for robustness and generalization
using an expectation over the loss. Instead, we define robustness
as a set of classification constraints.
Our approach better matches the security perspective that
even a single inaccurate prediction is a vulnerability.
The limitation is explicit constraints
only ensure perfect accuracy near the training set.
Standard generalization remains to be studied using
other approaches such as those with assumptions on the data distribution.
Existing work has used assumptions about the data distribution to
achieve explicit trade-offs between robustness and standard generalization
~\citep{dobriban2020provable,
javanmard2020preciseb,%
javanmard2020precise,
raghunathan2020understanding,
tsipras2018robustness,
zhang2019theoretically,
schmidt2018adversarially,
fawzi2018adversarial,
fawzi2018analysis}.

\paragraph{Fourier Analysis of Robustness.}
Various observations have been made about Fourier properties
of adversarial perturbations against deep non-linear models~\citep{ilyas2019adversarial,
tsuzuku2019structural,
sharma2019effectiveness}.
\citet{yin2019fourier} showed that adversarial
training increases robustness to perturbations concentrated
at high frequencies and reduces robustness to perturbations
concentrated at low frequencies. \citet{ortiz2020hold}
also observed that the measured margin of classifiers
at high frequencies is larger than the margin at
low frequencies.
Our \cref{cor:max_robust_to_min_norm_conv}
does not distinguish between
low and high frequencies but we establish an exact characterization of
robustness.
\citet{caro2020using} hypothesized about
the implicit robustness to $\ell_1$ perturbations in the Fourier
domain while we prove maximal robustness to
Fourier-$\ell_\infty$ perturbations.
\paragraph{Architectural Robustness.}
An implication of our results is that robustness can be achieved
at a lower computational cost compared with adversarial training by
various architectural choices as recently explored~\citep{xie2020smooth,
galloway2019batch,
awais2020towards}.
Moreover, for architectural choices that align with human biases, standard
generalization can also improve~\citep{vasconcelos2020effective}.
Another potential future direction is to rethink $\ell_p$ robustness as an architectural
bias and find inspiration in the human visual system for
appropriate architectural choices.

\paragraph{Robust Optimization}
A robust counterpart to an optimization
problem considers uncertainty in the data and optimizes for the worst-case.
\citet{ben2009robust}
provided extensive formulations and discussions
on robust counterparts
to various convex optimization problems. Adversarial robustness
is one such robust counterpart and many other robust counterparts
could also be considered in deep learning. An example is
adversarial perturbations with different norm-ball constraints
at different training inputs.
\citet{madry2017towards}
observed the link between robust optimization
and adversarial robustness where the objective
is a min-max problem
that minimizes the worst-case loss. However, they did not
consider the more challenging problem of maximally robust
optimization that we revisit.

\paragraph{Implicit bias of optimization methods.}
In \cref{sec:background:bias} we discussed prior work on the implicit bias of 
optimization methods. Here we revisit prominent works and explain the 
connection to this chapter.

Minimizing the empirical risk for
an overparametrized model with more parameters than the training
data has multiple solutions.
\citet{zhang2017understanding} observed that overparametrized
deep models can even fit to randomly labeled training data, yet
given correct labels they consistently generalize to test data.
This behavior has been explained using the implicit bias of
optimization methods towards particular solutions.
\citet{gunasekar2018characterizing} proved that
minimizing the empirical risk using  steepest descent
and mirror descent have an implicit bias towards minimum norm
solutions in overparametrized linear classification.
\citet{gunasekar2018implicit} proved the implicit bias of
gradient descent in training linear convolutional classifiers
is towards minimum norm solutions in the Fourier domain that
depends on the number of layers.

Recent theory of generalization in deep learning, in particular
the double descent phenomenon, studies the generalization
properties of minimum norm solutions
for finite and noisy training sets~\citep{hastie2019surprises}.
Characterization of the double descent phenomenon relies on the
implicit bias of optimization methods
while using additional assumptions
about the data distribution.
In contrast, our results only rely on the implicit bias of optimization
and hence are independent of the data distribution.

\paragraph{Hypotheses.}
\citet{goodfellow2014explaining}  proposed the linearity
hypothesis that informally suggests
$\ell_p$ adversarial samples exist because deep learning
models converge to functions similar to linear models.
To improve robustness, they argued models have to be more non-linear.
Based on our framework, linear models are not inherently weak.
When trained, regularized, and parametrized
appropriately they can be robust to some degree, the extent of
which depends on the dataset.
\citet{gilmer2018adversarial} proposed adversarial spheres
as a toy example where a two layer neural network exists with perfect
standard and robust accuracy for non-zero perturbations. Yet,
training a randomly initialized model with gradient descent
and finite data does not converge to a robust model.
Based on our framework, we interpret this as an example
where the implicit bias of gradient descent is not towards
the ground-truth model, even though
there is no misalignment in the architecture.
It would be interesting to understand this implicit bias in future work.

\paragraph{Robustness to {$\ell_p$-bounded} attacks.}
Robustness is achieved when any perturbation to natural
inputs that changes a classifier's prediction also confuses a human.
\mbox{$\ell_p$-bounded} attacks are the first step in achieving adversarial
robustness.
\citet{tramer2020adaptive} have recently shown many recent robust
models only achieve spurious robustness against $\ell_\infty$ and
$\ell_1$ attacks.
\citet{croce2020reliable} showed that on image classification datasets
there is still a large gap in adversarial robustness to {$\ell_p$-bounded}
attacks and standard accuracy.
Robustness to multiple \mbox{$\ell_p$-bounded} perturbations through
adversarial training and its trade-offs has also been
analyzed~\citep{tramer2019adversarial,maini2020adversarial}.
\citet{sharif2018suitability, sen2019should} argue that none of
$\ell_0$, $\ell_1$, $\ell_\infty$, or SSIM are a perfect match for human
perception of similarity. That is for any such norm, for any
$\varepsilon$, there exists a perturbation such that humans classify
it differently.
Attacks based on other perceptual similarity metrics
exist~\citep{zhao2020adversarial,liu2019beyond}.
This shows that the quest for adversarial robustness should also
be seen as a quest for understanding human perception.

\paragraph{Robustness through Regularization}
Various regularization
methods have been proposed for adversarial robustness that
penalize the gradient norm and can be studied using the framework
of maximally robust classification.
\citet{lyu2015unified} proposed general $\ell_p$ norm regularization of
gradients.
\citet{hein2017formal} proposed the Cross-Lipschitz penalty
by regularizing the norm
of the difference between two gradient vectors of the function.
\citet{ross2018improving} proposed $\ell_2$ regularization
of the norm of the gradients.
\citet{sokolic2017robust} performed regularization of
Frobenius norm of the per-layer Jacobian.
\citet{moosavi2019robustness} proposed penalizing the curvature
of the loss function.
\citet{qin2019adversarial} proposed encouraging local linearity
by penalizing the error of local linearity.
\citet{simon2019first}
proposed regularization of the gradient norm where
the dual norm of the attack norm is used.
\citet{avery2020adversarial}
proposed Hessian regularization. 
\citet{guo2020connections} showed that some regularization methods
are equivalent or perform similarly in practice.
Strong gradient or curvature regularization methods can
suffer from gradient masking~\citep{avery2020adversarial}.

\paragraph{Certified Robustness.}
Adversarially trained models are empirically harder to attack than
standard models. 
But their robustness is not often provable.
Certifiably robust models seek to close this gap~\citep{hein2017formal,
wong2018provable,
cohen2019certified,%
gowal2018effectiveness,
salman2019convex}.
A model is certifiably
robust if for any input, it also provides an 
$\varepsilon$-certificate that guarantees robustness to any
perturbation within the $\varepsilon$-ball of the input.
In contrast, a maximally robust classifier finds a classifier
that is guaranteed to be robust to maximal $\varepsilon$ while
classifying all training data correctly. That allows
for data dependent robustness guarantees at test time.
In this work, we have not explored standard generalization guarantees.

\subsection{Investigating the Gap in the Convergence of Linear Convolutional 
Networks}
\label{sec:conv_gap}

We further investigated the gap between theory and experiment in 
\cref{fig:conv}, and our results are consistent with the observations of 
\citet[Section 
6]{yun2020unifying}.  The gap appears to be due to finite training time and it 
depends on the scale of the initialization and the learning rate.  We were able 
to eliminate the gap for low dimensional problems (e.g., $d=2$) by tuning the 
initialization scale.  For high-dimensional problems (e.g., $d=100$), it is 
more challenging to find the optimal initialization scale.  As suggested by 
\citet{yun2020unifying}, this is an important problem for future work of 
implicit optimization bias literature.

For completeness, we did test different aspects of our implementation to rule 
out other causes. Here is a list of potential causes we eliminated:
\begin{itemize}
    \item  The scaling of the DFT matrix is accurate for comparison with the 
        minimum-norm Fourier-$\ell_1$ solution.
    \item Numerical error does not seem to contribute to the gap as we do not 
        see an improvement by switching from Float32 to Float64.
    \item We thoroughly tested our implementation of the convolution operation 
        and linearization operation.
\end{itemize}

\section{Conclusion}
\label{sec:conclusion}

We demonstrated that the choice of optimizer,
neural network architecture,
and  regularizer,  significantly affect the adversarial 
robustness of linear neural networks. These results lead us to
a novel Fourier-$\ell_\infty$ attack with controllable spectral
properties applied against deep non-linear CIFAR-10 models.
Our results provide a framework, insights, and directions for improving robustness
of non-linear models through approaches other than adversarial training.

In the following, we summarize the hypotheses, results, and future work:
\begin{itemize}
    \item We revisited the definition of Maximally Robust Classification and 
        justified as a formulation of the problem of finding adversarially 
        robust models.  We argue that Maximally Robust Classification is an 
        alternative problem formulation possibly more challenging to solve than 
        standard min-max formulation of adversarial robustness. There is also 
        no ambiguity in the choice of hyper-parameters in the formulation of the 
        Maximally Robust Classifier.
    \item For linear models, we prove guaranteed maximal robustness achieved 
        only by the choice of the optimizer, regularization, or architecture.  
        We have achieved this through our Corollaries. For example, because of 
        our results, it can no longer be claimed that ``A regularization method 
        alone cannot achieve maximal robustness''. At least in the linear case 
        we showed that such a regularizer exists.
    \item We rigorously established the adversarial robustness of linear 
        convolutional neural networks in the Fourier domain. In that direction 
        our novel Fourier attack creates one potential direction for 
        understanding the robustness of non-linear models.
    \item We prove that under certain conditions no additional per-iteration 
        cost of solving an optimization problem is needed, yet such a solution 
        might take many iterations to find. We have not done computational 
        analysis of the methods studied in this chapter which can be done with 
        additional assumptions on the data distribution.  Our results only 
        require linear separability of the data.
    \item We have not proposed a novel defense for non-linear models. We 
        discuss future directions and challenges in 
        \cref{sec:related,sec:nonlinear}.  In particular, there is a growing 
        literature on the implicit bias of non-linear networks that can be used 
        to extend our
        results~\citep{chizat2020implicit,lyu2020gradient,ongie2020function}.   
        Non-linear solutions might require significantly different mathematical 
        tools and efforts from a wider community and take years to find.
    \item We do not claim that robustness guarantees achieved through the choice of 
        the optimizer, architecture, or regularization are sufficient for any given 
        model family, task or domain of application. For example, if the data 
        is not linearly separable, additional sacrifices need to be made (See 
        \cref{sec:max_robust_gen}).  Having said that, efficiently finding the 
        maximally robust classifier for deep non-linear models might be 
        achievable by appropriate choices of all these in addition to 
        adversarial training with an appropriate epsilon.
    \item There is a small gap in \cref{fig:conv} between theory and experiment 
        that might be due to a missing condition in the original theory of 
        \citet{gunasekar2018implicit}. We have provided additional discussion 
        in \cref{sec:conv_gap}.
\end{itemize}

%% file: ch-conclusion.tex
\chapter{Conclusion and Future Work}
\label{ch:conclusion}

In this thesis we discussed ideas for improving training efficiency and 
connections to robustness in deep learning.  \cref{ch:vsepp} showed that hard 
negatives are better than uniformly sampled negatives; they result in faster 
training and better generalization.  \cref{ch:gvar} proposed gradient 
clustering to reduce the gradient variance by automatically discovering and 
exploiting data diversity. The mixed results revealed gaps in our understanding 
of training deep learning models.  Finally, \cref{ch:robust} establishes 
a connection between robustness and optimization choices and shows that the two 
challenges are closely related.

What is the most efficient and robust training method in deep learning? This 
remains an open question for us. Various approaches focus on improving training 
efficiency and standard generalization but lack attention to adversarial 
robustness and biased data distributions. Interesting and vibrant approaches 
include few-shot learning, transfer learning, and meta 
learning~\citep{hochreiter2001learning,vinyals2016matching,finn2017model}.  On 
the other hand, contrastive learning and adversarial 
learning~\citep{hadsell2006dimensionality,goodfellow2014gan} are designed with 
only the objective of robustness to adversarial inputs or preset data 
transformations.

Our observations in \cref{ch:vsepp} are in favor of adversarial and robust 
representation learning approaches. Hard negatives in triplet losses are 
examples are adversarial inputs where only inputs in the training set define 
the natural data distribution. Adversarial learning and contrastive learning 
methods might also benefit from a similar observation that semi-hard negatives 
provide superior generalization than absolute hard negatives.  Curriculum 
learning~\citep{bengio2009curriculum} might provide more adjustable contrastive 
and adversarial samples. Learning multi-modal embeddings remain particularly 
propitious for efficient hard negative-based methods with no need for 
curriculum learning or generative adversarial 
networks~\citep{chen2021learning,wang2020consensus,awad2020trecvid}.  

Our results in \cref{ch:robust} suggest an alternative or complementary 
approach by design.  Instead of adversarial learning, under certain conditions, 
similar solutions can be found by appropriate architecture, optimizer and 
regularizer. The right architecture, e.g., convolution versus fully-connected 
neural networks, would significantly reduce the amount of adversarial training 
required. In some cases, adversarial training may become unnecessary.

In conclusion, this thesis raises more questions than it answers.
What follows are some possible directions for further research on related 
problems.

In \cref{ch:robust} we discuss the problem of maximally robust classification 
through which we transferred results on implicit and explicit bias of 
optimization methods to adversarial robustness. There is potential in extending 
this connection to transfer results in both directions. One example is to 
transfer provable guarantees for deep non-linear models to infer the implicit 
bias of various optimization methods that have not yet been derived.  
Exploiting the connection in the opposite direction, there have been recent 
results on the implicit bias of ReLU networks and the directional convergence 
of deep linear networks that imply further maximal robustness results.

In \cref{ch:gvar} we designed an efficient gradient clustering method. An 
application of gradient clustering is in automatically detecting imbalanced 
data according to gradient information and modifying the sampling to adjust and 
calibrate training. For example, by changing the learning paradigm from 
Empirical Risk Minimization to Distributionally Robust Optimization we should 
be able to formalize the design of distributionally robust sampling methods.  
Worst-group error is an evaluation metric commonly used in distributionally 
robust optimization that is a robust counterpart to the average error in 
expected risk minimzation.  In SGD with gradient clustering we used 
a reweighting to achieve unbiased gradient estimates that would guarantee 
convergence to the same optimum as SGD with uniform sampling.  Recently it has 
been shown that upweighting the minority class negatively impacts the minority 
error while subsampling the majority class improves 
it~\citep{sagawa2020investigation}.  Inspired by this observation one might 
consider gradient clustering for distributionally robust optimization without 
reweighting samples.  The gradient clustering sampler would then intrinsically 
find majority groups of data and sample the same number for all majority and 
minority groups.

An extension to both \cref{ch:robust,ch:gvar} is to understand the implications 
of implicit bias in continual and lifelong learning.  Continual learning is 
a challenging task in machine learning restricted by catastrophic forgetting in 
training and lack of backward and forward transfer. In continual learning and 
for memoryless methods, we can describe the impact of prior learning tasks as 
a change in the initialization of the future learnings. Based on the prior work 
discussed in \cref{ch:robust}, we know in certain problems, the initialization 
changes the solution found. It is possible to connect the implicit bias and 
initialization for sequential linear regression tasks and design an optimal 
method of initialization depending on the similarity of tasks. Extending this 
idea to continual learning benchmarks is an interesting future direction.

Another extension to \cref{ch:robust} is to characterize the implicit 
robustness of preconditioned optimization methods~\citep{amari2020does}.  The 
implicit bias of preconditioned methods is different from gradient descent and 
depends on how the preconditioning is estimated. The alternative implicit bias 
results in an alternative implicit robustness that may be allow robustness to 
a new family of practical adversarial perturbations.

Finally, in adversarial learning and contrastive learning, the performance 
depends on the strength of adversarial inputs and data 
augmentations~\citep{chen2020improved}. One possible direction for future work 
is the exact characterization of the dependence on adversarial strength in 
simple learning settings that would allow better justifications and recipes for 
curriculum learning. Gradient-aware methods such as variations of gradient 
clustering proposed in \cref{ch:gvar} should be useful in the direction of 
adversarial and contrastive learning as recent work 
corroborates~\citet{ma2021active}.

%% file: ch-gvar-tex/appendix.tex
\def\figdir{ch-gvar-tex/figures}

\section{Additional Details of Gradient Clustering}
\subsection{Proof of \texorpdfstring{\cref{thm:gvar}}{Proposition 
4.3.1}}\label{sec:gvar_proof}
The gradient estimator, $\gC$, is unbiased for any partitioning of data, i.e., 
equal to the average gradient of the training set,
\begin{align}
    \E[\gC]
    &
    = \frac{1}{N} \SK \Nk \E[\gpk]
    = \frac{1}{N} \SK \Nk \left(\frac{1}{\Nk} \SJ \bm{g}^{(k)}_j\right)
    = \underbrace{\frac{1}{N} \SK \SJ \bm{g}^{(k)}_j}_{(*)}
    = \frac{1}{N} \SI \gi = \bm{g}~,\nonumber
\end{align}
where we use the fact that the expectation of a random sample drawn uniformly 
from a subset is equal to the expectation of the average of samples from that 
subset.  Also note that the gradient of every training example appears once in 
$(*)$. %

Although partitioning does not affect the bias of $\gC$, it does affect the 
variance,
\begin{align}
    \V[\gC]
    &= \frac{1}{N^2} \left(\SI \Nk^2 \V[\gpk]
    + 2\SK\sum_{\hat{k}=1}^{N_{\hat{k}}}
    \Nk N_{\hat{k}} \C[\gpk, \bm{g}^{(\hat{k})}] \right) %
    = \frac{1}{N^2} \SK \Nk^2 \V[\gpk]%
\end{align}
where the variance is defined as the trace of the covariance matrix. Since we 
assume the training set is sampled i.i.d., the covariance between gradients of 
any two samples is zero.  In a dataset with duplicate samples, the gradients of 
duplicates will be clustered into one cluster with zero variance if mingled 
with no other data points.

\section{Additional Details of Efficient \GC}
\subsection{Convolutional Layers}\label{sec:conv}

In neural networks, the convolution operation is performed as an inner product 
between a set of weights
$\lparam\in \R^{h\times w\times \hat{\Di}\times \Do}$,
namely kernels, by patches of size
$h\times w$
in the input.  Assuming that we have preprocessed the input by extracting 
patches, the gradient w.r.t.\ $\lparam$ is
    $ \gb = {\sum_t \bm{g}_{b,t} }$,
$\bm{g}_{b,t}\in \R^{\Di\times \Do}$ is the gradient at the spatial location 
$t\in T$ and
$\Di=h \times w \times \hat{\Di}$
is the flattened dimension of a patch.
The gradient at spatial location $t$ is computed as $\bm{g}_{b,t}= \AAA_{b,t} 
\DDD_{b,t}^\top$.

Like the fully-connected case, we use a rank-$1$ approximation to the cluster 
centers in a convolution layer, defining $\Ck=\ck \dk^\top$.  As such, \bAU steps 
are performed efficiently. For the \bA step we rewrite
${\vvv{\Ck}\odot \vvv{\gb}}$,
\begin{align}
    \vvv{\Ck}\odot \vvv{\sum_t \AAA_{b,t} \DDD_{b,t}^\top}
    &= \sum_{u,v} (\bm{c}_{ku} ~ \bm{d}_{kv} ~~ (\sum_t \AAA_{btu} ~ \DDD_{btv}))
    \label{eq:conv1}\\
    &= \sum_t (\sum_u \bm{c}_{ku} ~ \AAA_{btu}) (\sum_{v} \bm{d}_{kv} 
    ~ \DDD_{btv})
    ,
    \label{eq:conv2}
\end{align}
where the input dimension is indexed by $u$ and the output dimension is indexed 
by $v$. Eqs.~\ref{eq:conv1} and~\ref{eq:conv2} provide two ways of computing 
the inner-product, where we first compute the inner sums, then the outer sum. The 
efficiency of each formulation depends on the size of the kernel and layer's 
input and output dimensions.

\subsection{Complexity Analysis}\label{sec:comp}
{\renewcommand{\arraystretch}{1.5}
\begin{table}[t]
    \centering
    \begin{tabular}{ccc}
        \toprule
        {\bf Operation} & {\bf FC Complexity} & {\bf Conv Complexity}\\
        \midrule
        $\bm{C}\odot \bm{g}$ & $K B (\Di+\Do)$
        & \makecell{%
            Eq.~\ref{eq:conv1}: $B(T+K)\Di \Do$\\
         Eq.~\ref{eq:conv2}: $ B T K (\Di + \Do)$}\\
        $\bm{C}\odot \bm{C}$ & $K (\Di+\Do)$ & $K (\Di+\Do)$\\ %
        $\bm{g}\odot \bm{g}$ & $B (\Di+\Do)$
        & \makecell{%
            Eq.~\ref{eq:conv1}: $B T \Di \Do,$ \\
        Eq.~\ref{eq:conv2}: $B T^2 (\Di + \Do)$}\\
        \midrule
        Back-prop & $B \Di \Do$ & $B T \Di \Do$\\
        \bA step & $K B (\Di+\Do)$ & See Sec.~\ref{sec:comp}\\
        \bU step & $B (\Di + \Do)$ & $B(\Di +\Do)$\\
        \bottomrule
    \end{tabular}
    \caption{Complexity of \Gluster compared to the cost of back-prop.}
    \label{tab:comp}
\end{table}
}

\Gluster, described in \cref{alg:full}, performs two sets of operations, 
namely, the cluster center updates (\bU step), and the assignment update of 
data to clusters (\bA step). \bA steps instantly affect the optimization by 
changing the sampling process.  As such, we perform an \bA step every few 
epochs and change the sampling right after. In contrast, the \bU step can be 
done in parallel and more frequently than the \bA step, or online using 
mini-batch updates. The cost of both steps is amortized over optimization 
steps.

Table~\ref{tab:comp} summarizes the run-time complexity of \Gluster compared to
the cost of single SGD step.
The \bU step is always cheaper than a single back-prop step.
The \bA step is cheaper for fully-connected layers if ${K<\min(\Di,\Do)}$.

For convolutional layers, we have two ways to compute the terms in the \bA step 
(Eqs.~\ref{eq:conv1} and~\ref{eq:conv2}).  For ${\bm{C}\odot \bm{g}}$, if
${\min(T, K)}<{K<\min(\Di, \Do)}$, Eq.~\ref{eq:conv2} is more efficient. For 
${\bm{g}\odot \bm{g}}$, Eq.~\ref{eq:conv2} is more efficient if $T<\min(I, O)$.  
If $K<T$, both methods have lower complexity than a single back-prop step. If 
we did not have the $\Nk$ multiplier in the \bA step, we could ignore the 
computation of the norm of the gradients, and hence further reduce the cost.

In common neural network architectures, the condition $K<T$ is easily satisfied 
as $T$ in all layers is almost always more than $10$ and usually greater than $100$, 
while $10$-$100$ clusters provides significant variance reduction. As such, the 
total overhead cost with an efficient implementation is at most $2\times$ the 
cost of a normal back-prop step. We can further reduce this cost by performing 
\Gluster on a subset of the layers, e.g., one might exclude the lowest 
convolutional layers.

The total memory overhead is equivalent to increasing the mini-batch size by 
$K$ samples as we only need to store rank-$1$ approximations to the cluster 
centers.

\section{Additional Details for Experiments}\label{app:exp}

The mini-batch size in \Gluster and SVRG and the number of clusters in \Gluster 
are the same as the mini-batch size in \SG and the same as the mini-batch size 
used for training using SGD\@. To measure the gradient variance, we take 
snapshots of the model during training, sample tens of mini-batches from the 
training set (in case of \Gluster, with stratified sampling), and measure the 
average variance of the gradients.

We measure the performance metrics (e.g., loss, accuracy and variance) as  
functions of the number of training iterations rather than wall-clock time. In 
other words, we do not consider computational overhead of different methods.  
In practice, such analysis is valid as long as the additional operations could 
be parallelized with negligible cost.

\subsection{Experimental Details for Image Classification 
Models}\label{app:exp_image}

On MNIST, our MLP model consists of there fully connected layers:
layer1: $28*28\times 1024$,
layer2: $1024\times 1024$,
layer3: $1024, 10$. We use ReLU activations and no dropout in this MLP\@. We 
train all methods with learning rate $0.02$, weight decay $5 \times 10^{-4}$, 
and momentum $0.5$.
On CIFAR-10, we train ResNet8 with no batch normalization layer and learning 
rate $0.01$, weight decay $5 \times 10^{-4}$, and momentum $0.9$ for $80000$ 
iterations.  We decay the learning rate at $40000$ and $60000$ iterations by 
a factor of $0.1$.
On CIFAR-100, we train ResNet32 starting with learning rate $0.1$. Other 
hyper-parameters are the same as in CIFAR-10.
On ImageNet, we train ResNet18 starting with learning rate $0.1$, weight decay 
$1\times 10^{-4}$, and momentum $0.9$. We use a similar learning rate schedule 
to CIFAR-10.

{\renewcommand{\arraystretch}{1.5}
\begin{table}[t]
    \centering
    \begin{tabular}{ccccccccc}
        \toprule
        Dataset & Model & $B$ & T & Log T
        & Estim T & U & \GC T\\
        \midrule
        MNIST & MLP & $128$ & $50000$ & $500$ & $50$ & $2000$ & $10$\\
        CIFAR-10 & ResNet8 & $128$ & $80000$ & $500$ & $50$ & $20000$ & $3$\\
        CIFAR-100 & ResNet32 & $128$ & $80000$ & $500$ & $50$ & $20000$ & $3$\\
        ImageNet & ResNet18 & $128$ & $80000$ & $1000$ & $10$ & $10000$ & $3$\\
        \bottomrule
    \end{tabular}
    \label{tbl:exp_setup}
    \caption{Hyperparameters.}
\end{table}
}

In \cref{tbl:exp_setup} we list the following hyper-parameters: the interval of 
measuring gradient variance and normalized variance (Log T), number of gradient 
estimates used on measuring variance (Estim T), the interval of updating the 
control variate in SVRG and the clustering in \GC (U), and the number of \GC 
update iterations (\GC T).

In plots for random features models, each point is generated by keeping $h_s$ 
fixed at $1000$ and varying $N$ in the range $[0.1, 10]$. We average over $3$ 
random seeds, $2$ teacher hidden dimensions and $2$ input dimensions (both 
$\times0.1$ and $\times10$ student hidden). We use mini-batch size $10$ for 
\SG, SVRG, and \Gluster.

A rough estimate of the overparametrization coefficient (discussed in 
\cref{sec:exp_rf}) for deep models is to divide the total number of parameters 
by the number of training data. On MNIST the coefficient is approximately $37$ 
for CNN and $31$ for MLP\@. On CIFAR-10 it is approximately $3$ for ResNet8 and 
$9$ for ResNet32\@.  Common data augmentations increase the effective training 
set size by $10\times$. On the other hand, the depth potentially increases the 
capacity of models exponentially (cite the paper that theoretically says how 
many data points a model can memorize). As such, it is difficult to directly 
relate these numbers to the behaviours observed in RF models.

\subsection{Experimental Details for Random Features Models}\label{app:exp_rf}

The number of training iterations is chosen such that the training loss has 
flattened.  The maximum is taken over the last $70\%$ of iterations (the 
variance is usually high for all methods in the first $30\%$).  
Mean variance plots for random features models are similar to max variance 
plots presented in \cref{sec:exp_rf}.
\begin{figure}[t]
    \centering
    \includegraphics[width=.32\textwidth]{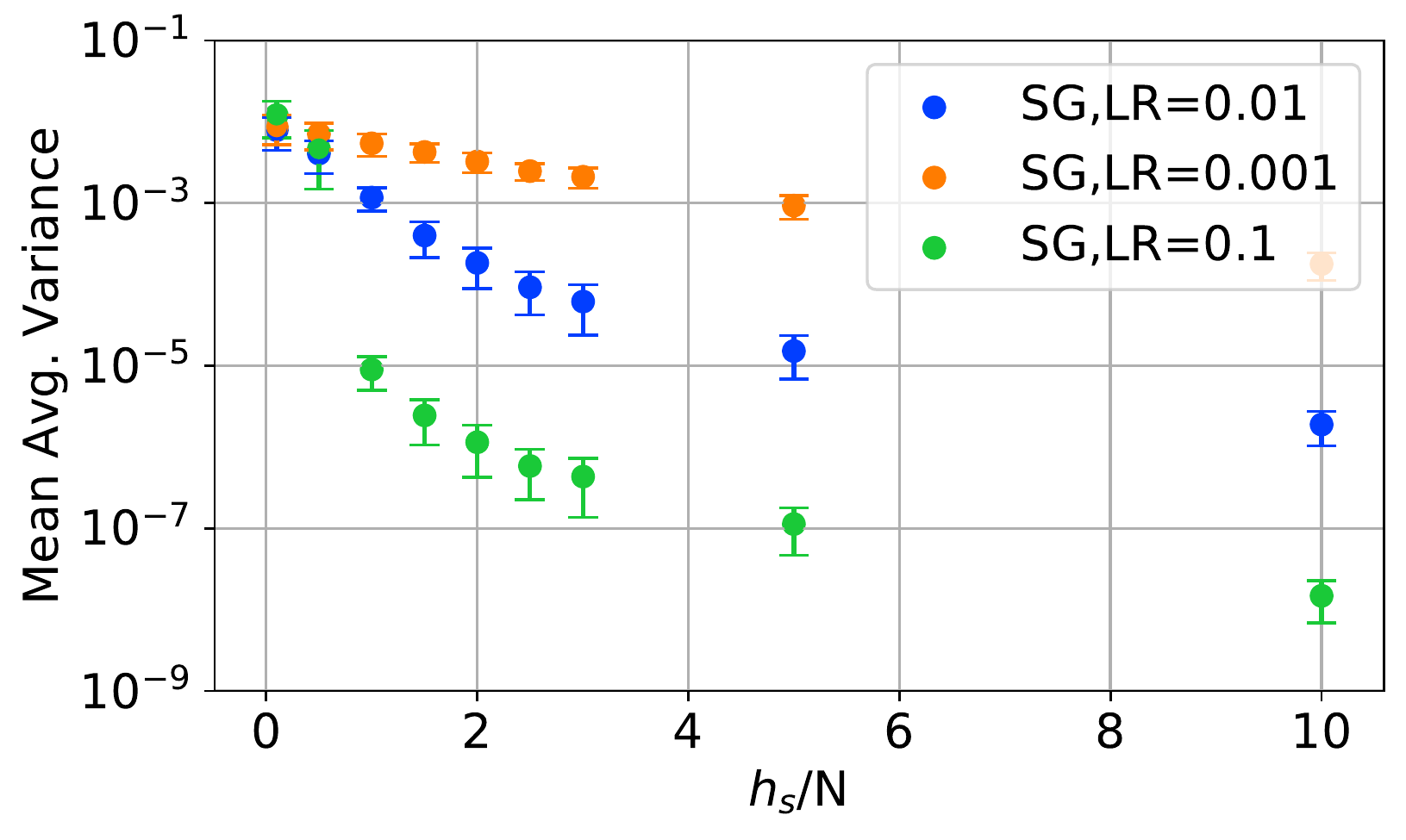}
    \hfill
    \includegraphics[width=.32\textwidth]{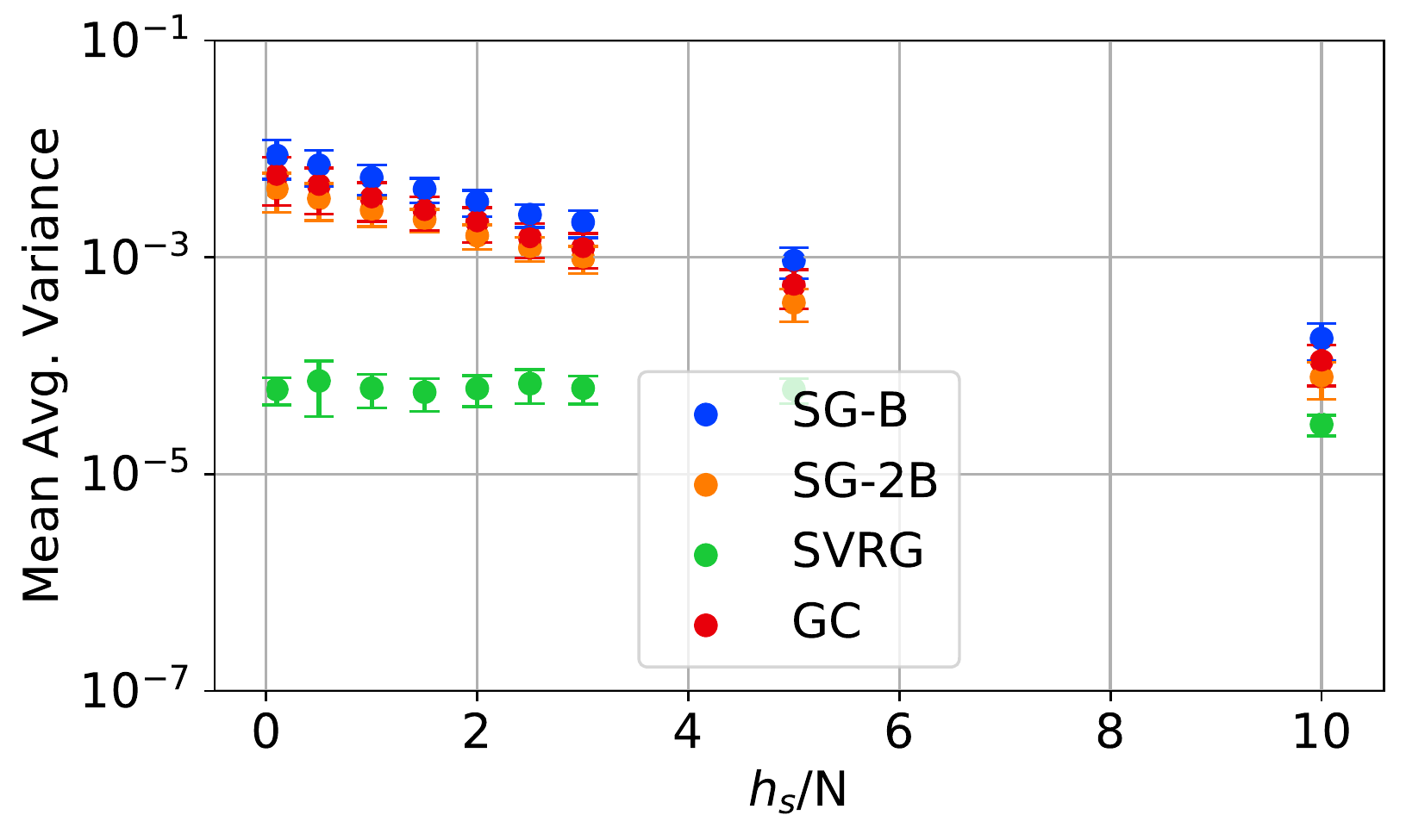}
    \hfill
    \includegraphics[width=.32\textwidth]{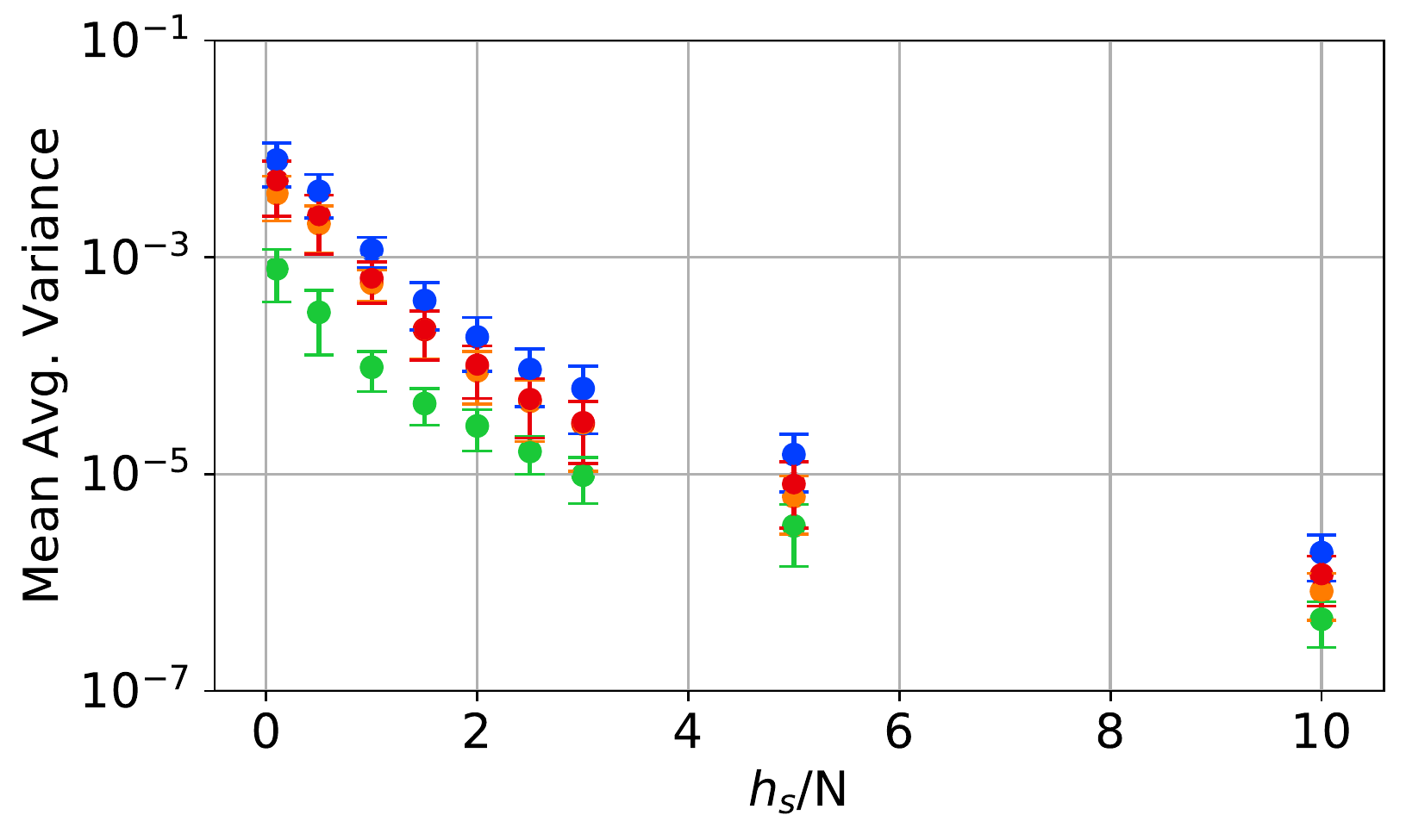}
    \caption{Mean variance plots for Fig.~\ref{fig:rf}}
    \label{fig:rf_mean}
\end{figure}

We aggregate results from multiple experiments with the following range of 
hyper-parameters. Each point is generated by keeping $h_s$ fixed at $1000$ and 
varying $N$ in the range $[0.1, 10]$. We average over $3$ random seeds, $2$ 
teacher hidden dimensions and $2$ input dimensions (both $\times0.1$ and 
$\times10$ student hidden).

\section{Normalized Variance and Convergence Analysis}
\label{sec:normalized_var_analysis}

Normalized variance appears in standard convergence analysis of stochastic 
gradient descent~\citep{friedlander2012hybrid}. To understand the connection, 
we briefly review a standard result. Let $f : \R^n \rightarrow \R$ be 
a $L$-Lipschitz continuous function, i.e.,
\begin{align}
    \|\nabla f(\xx)-\nabla f(\yy)\| &\leq L\|\xx-\yy\|, \forall \xx,y\in 
    \R^n\,.
\end{align}
Consider the stochastic gradient descent method with updates,
\begin{align}
    \xx_{k+1} &= \xx_k - \alpha_k \gg_k
\end{align}
where $\alpha_k$ is the step size and
\begin{align}
    \gg_k &\coloneqq \nabla f(\xx_k) + \ee_k
\end{align}
is an estimate of the gradient with the residual $\ee_k$.

For $\alpha_k=1/L$, the iterates of stochastic gradient descent satisfy the 
inequality~\citep[Eq. 2.3]{friedlander2012hybrid}
\begin{align}
    f(\xx_{k+1}) &\leq
    f(\xx_k) -\frac{1}{2L} \| \nabla f(\xx_k) \|^2
    + \frac{1}{2L} \|\ee_k\|^2\,,
\end{align}

Let $E[\ee_k]=0$, by subtracting $f(x_\ast)$ from both sides and taking
expectations we have
\begin{align}
    \E[f(\xx_{k+1})-f(x_\ast)] &\leq
    \E[f(\xx_k)-f(x_\ast)]
    -\frac{1}{2L} \E[\| \nabla f(\xx_k) \|^2]
    + \frac{1}{2L} \E[\|\ee_k\|^2]\\
    \E[f(\xx_{k+1})-f(x_\ast)] &\leq
    \E[f(\xx_k)-f(x_\ast)]
    -\frac{1}{2L} \E[\| \nabla f(\xx_k) \|^2] (1- \zeta)\,,
\end{align}
where $\zeta=\E[\|\ee_k\|^2]/\E[\| \nabla f(\xx_k) \|^2]$ is the normalized 
variance of the residual and the gradient as defined in \cref{sec:exp}.

If the normalized variance is small, i.e.\, $\zeta \ll 1$, the noise error is 
dominated by the non-stochastic error. As such, reducing the gradient variance 
would not speed up the training while for $1 \ll \zeta$, reducing the variance 
improves the convergence speed. This matches the diminishing returns 
observations in practice as we discuss in \cref{sec:exp}.

%% file: ch-robust-tex/appendix.tex
\def\figdir{ch-robust-tex/figures}

\section{Generalization of the Maximally Robust Classifier}
\label{sec:max_robust_gen}

\begin{defn}[Maximally Robust Classifier with Slack Loss] 
\label{def:eta_max_robust}
Let $\xi \geq 0$ denote a given slack variable.
A maximally robust classifier with slack loss is the solution to
\begin{align}
\label{eq:mrc_slack}
\argmax_{\varphi \in \Phi}
\left\{ \vphantom{i_2^\top} \varepsilon \,|\,
\EE_{(\xx, y)} \max_{\|\ddelta\| \leq \varepsilon}
\zeta(y\varphi(\xx + \ddelta)) \leq \xi\right\}\,.
\end{align}
\end{defn}
This formulation is similar to the saddle-point problem
in that we seek to minimize the expectation of the worst case loss.
The difference is that we also seek to maximize $\varepsilon$.
However, we have introduced another arbitrary variable $\xi$
that is not optimized as part of the problem.
For linear classifiers and the hinge loss,
${\zeta(z)=[1-z]_+}$,
\cref{eq:mrc_slack} can be written as,
\begin{align}
&\argmax_{\ww}\left\{
\vphantom{i_2^\top} \varepsilon \,|\,
\EE_{(\xx, y)} [1-y \ww^\top\xx + \varepsilon \|\ww\|_\ast]_+
\leq \xi\right\}\,,
\end{align}
where $[\cdot]_+$ is the hinge loss,
and the weight penalty term $\|\ww\|_\ast$ is inside the hinge loss.
This subtle difference makes solving the problem more challenging
than weight penalty outside the loss.

Because of the two challenges we noted, we do not study the
maximal robustness with slack loss.

\section{Proofs}
\label{sec:proofs}
\subsection{Proof of \texorpdfstring{\cref{lem:max_robust_to_min_norm_linear}}{Lemma 1}}
\begin{proof}
\label{proof:max_robust_to_min_norm_linear}
We first show that the maximally robust classifier is equivalent to a robust counterpart
by removing $\ddelta$ from the problem,
\begin{align*}
& \argmax_{\ww,b}\,\{ \varepsilon \mid 
 y_i (\ww^\top (\xx_i + \ddelta) + b) > 0
,~\forall i,\, \|\ddelta\| \leq \varepsilon\}\\
&\qquad \text{ (homogeneity of $p$-norm)} \\
&= \argmax_{\ww,b}\,\{ \varepsilon \mid 
 y_i (\ww^\top (\xx_i + \varepsilon\ddelta) + b) > 0
,~\forall i,\, \|\ddelta\| \leq 1\}\\
&\qquad \text{ (if it is true for all $\ddelta$ it is true for the worst of them)}\\ %
& = \argmax_{\ww,b}\,\{ \varepsilon \mid 
 \inf_{\|\ddelta\| \leq 1} y_i (\ww^\top (\xx_i + \varepsilon\ddelta) + b) > 0,~\forall i \}\\
& = \argmax_{\ww,b}\,\{ \varepsilon \mid 
 y_i (\ww^\top \xx_i  + b) + \varepsilon\inf_{\|\ddelta\| \leq 1} \ww^\top \ddelta > 0,~\forall i \}\\
 &\qquad \text{ (definition of dual norm)}\\
& = \argmax_{\ww,b}\,\{ \varepsilon \mid 
 y_i (\ww^\top \xx_i + b) > \varepsilon \|\ww\|_\ast,~\forall i \}\\
\end{align*}

Assuming $\ww\neq 0$, which is a result of linear separability
assumption, we can divide both sides by $\|\ww\|_\ast$
and change variables,
\begin{align*} 
& = \argmax_{\ww,b}\,\{ \varepsilon \mid
 y_i (\ww^\top \xx_i + b) \geq \varepsilon ,~\forall i, \|\ww\|_\ast \leq 1 \}\,,
\end{align*}
where we are also allowed to change $>$ to $\geq$ because any solution
to one problem gives an equivalent solution to the other given
$\ww\neq 0$.

Now we show that the robust counterpart is equivalent to the
minimum  norm classification problem by removing $\varepsilon$.
When the data is linearly separable there exists a solution with $\varepsilon> 0$,
\begin{align*}
& \argmax_{\ww,b}\,\{ \varepsilon \mid 
 y_i (\ww^\top \xx_i + b) > \varepsilon
 \|\ww\|_\ast,~\forall i \}\\
&=\argmax_{\ww,b}\,\left\{ \varepsilon \mid 
 y_i \left(\frac{\ww^\top}{\varepsilon\|\ww\|_\ast} \xx_i + \frac{b}{\varepsilon\|\ww\|_\ast}\right) \geq 1,~\forall i \right\}
\end{align*}

This problem is invariant to any non-zero scaling of $(\ww,b)$,
so with no loss of generality we set $\|\ww\|_\ast=1$.

\begin{align*}
&=\argmax_{\ww,b} \left\{ \vphantom{i_2^\top} \varepsilon \mid
y_i \left(\frac{\ww^\top}{\varepsilon} \xx_i +b\right)\geq 1,\, \forall i,
\|\ww\|_\ast= 1\right\}
\end{align*}

Let $\ww^\prime=\ww/\epsilon$, then the solution to the following problem gives a solution for $\ww$,
\begin{align*}
&\argmax_{\ww^\prime,b} \left\{ \vphantom{i_2^\top} \frac{1}{\|\ww^\prime\|_\ast} \mid y_i (\ww^{\prime\top} \xx_i +b)\geq 1,\, \forall i \right\}\\
&=\argmin_{\ww^\prime,b} \left\{ \vphantom{i_2^\top} \|\ww^\prime\|_\ast \mid y_i (\ww^{\prime\top} \xx_i +b)\geq 1,\, \forall i  \right\}.
\end{align*}

\end{proof}

\subsection{Proof of Maximally Robust to Perturbations Bounded in Fourier Domain (\texorpdfstring{\cref{cor:max_robust_to_min_norm_conv}}{Corollary 2})}
\label{proof:max_robust_to_min_norm_conv}
The proof mostly follows from the equivalence for linear models
in \cref{proof:max_robust_to_min_norm_linear}
by substituting
the dual norm of Fourier-$\ell_1$.
Here, $\AA^\ast$ denotes the complex conjugate transpose,
$\langle \uu,\vv \rangle=\uu^\top\vv^\ast$
is the complex inner product,
$[\bF]_{ik}=\frac{1}{\sqrt{D}}\omega_D^{i k}$
the DFT matrix
where $\omega_D=e^{-j 2\pi/D}$, $j=\sqrt{-1}$.

Let $\|\cdot\|$ be a norm on $\C^{n}$
and $\langle\cdot,\cdot\rangle$ be the complex inner product.
Similar to $\R^{n}$, the associated dual norm is defined as
$\|\ddelta\|_\ast = \sup_{\xx}\{ |\langle \ddelta, \xx \rangle|\; |\; \|\xx\| \leq 1 \}\;$.

\begin{align*}
&\|\mathcal{F}(\ww)\|_1\\
&= \sup_{\|\ddelta\|_\infty\leq 1}
|\langle\mathcal{F}(\ww), \ddelta\rangle|\\
&\qquad \text{ (Expressing DFT as a linear transformation.)}\\
&= \sup_{\|\ddelta\|_\infty\leq 1}
|\langle\bF\ww, \ddelta\rangle|\\
&= \sup_{\|\ddelta\|_\infty\leq 1}
|\langle\ww, \bF^\ast\ddelta\rangle|\\
&\qquad \text{ (Change of variables and $\bF^{-1}=\bF^\ast$.)}\\
&= \sup_{\|\bF\ddelta\|_\infty\leq 1}
|\langle\ww, \ddelta\rangle|\\
&= \sup_{\|\mathcal{F}(\ddelta)\|_\infty\leq 1}
|\langle\ww, \ddelta\rangle|
\,.
\end{align*}

\section{Linear Operations in Discrete Fourier Domain}
\label{sec:fourier_ops}

Finding an adversarial sample with bounded
Fourier-$\ell_\infty$
involves $\ell_\infty$ complex projection to ensure adversarial
samples are bounded, as well as the steepest ascent direction
w.r.t the Fourier-$\ell_\infty$ norm. We also use the complex
projection onto $\ell_\infty$ simplex for proximal gradient method
that minimizes the regularized empirical risk.

\subsection{\texorpdfstring{$\ell_\infty$}{Linf} Complex Projection}

Let $\vv$ denote the $\ell_2$ projection of $\xx\in\C^d$
onto the $\ell_\infty$ unit ball. It can be computed as,
\begin{align}
&\argmin_{\|\vv\|_\infty\leq 1} \frac{1}{2} \|\vv-\xx\|_2^2 \\
&=\{\vv : \forall i,\, \vv_i=
\argmin_{|\vv_i| \leq 1}  \frac{1}{2} |\vv_i-\xx_i|^2\}\,,
\end{align}
that is independent projection per coordinate which can be
solved by 2D projections onto $\ell_2$ the unit ball
in the complex plane.

\subsection{Steepest Ascent Direction w.r.t. Fourier-\texorpdfstring{$\ell_\infty$}{Linf}}

Consider the following optimization problem,
\begin{align}
&\argmax_{\vv : \|\bF\vv\|_\infty\leq 1} f(\vv)\,,
\end{align}
where $\bF\in\C^{d\times d}$ is the
Discrete Fourier Transform
(DFT) matrix and $\bF^\ast=\bF^{-1}$
and $\bF^\ast$ is the conjugate transpose.

Normalized steepest descent direction is defined as (See \citet[Section 9.4]{boyd2004convex}),
\begin{align}
\argmin_\vv\{\nabla \langle f(\ww), \vv\rangle : \|\vv\| = 1 \}\,.
\end{align}

Similarly, we can define the steepest ascent direction,
\begin{align}
&\argmax_{\vv\in\R^d}\{|\langle\nabla f(\ww), \vv\rangle|
: \|\bF\vv\|_\infty = 1\}\\
&\qquad \text{ (Assuming $f$ is linear.)}\\
&\argmax_{\vv\in\R^d}\{|\langle\gg, \bF^\ast\bF\vv\rangle|
: \|\bF\vv\|_\infty = 1\}\\
&\argmax_{\vv\in\R^d}\{|\langle\bF\gg, \bF\vv\rangle|
: \|\bF\vv\|_\infty = 1\}
\end{align}
where $\gg=\nabla f(\ww)$.

Consider the change of variable $\uu=\bF\vv\in\C^{d\times d}$.
Since $\vv$ is a real vector its DFT is  Hermitian,
i.e., $\uu_i^\ast=[\uu]_{\overline{-i}}$ for all coordinates $i$ where
$\overline{j}=j \mod d$. Similarly, $\bF\gg$ is Hermitian.
\begin{align}
&\argmax_{\uu\in\C^d : \|\uu\|_\infty = 1}\{|\langle \bF\gg, \uu\rangle|
: \uu_i^\ast=[\uu]_{\overline{-i}}\}\\
&\argmax_{\uu\in\C^d : \forall i, |\uu_i| = 1}\{|[\bF\gg]_i \uu_i|
+ |[\bF\gg]_{\overline{-i}}\, \uu_i^\ast|
: \uu_i^\ast=[\uu]_{\overline{-i}}\}\\
&\argmax_{\uu\in\C^d : \forall i, |\uu_i| = 1}\{|[\bF\gg]_i \uu_i|
+ |[\bF\gg]_i^\ast\uu_i^\ast|
: \uu_i^\ast=[\uu]_{\overline{-i}}\}\\
&\argmax_{\uu\in\C^d : \forall i, |\uu_i| = 1}\{|[\bF\gg]_i \uu_i|
: \uu_i^\ast=[\uu]_{\overline{-i}}\}\\
&\uu_i=[\bF\gg]_i/|[\bF\gg]_i|\,.
\end{align}
and the steepest ascent direction is $\vv_i=\bF^{-1}\uu_i$ which
is a real vector. In practice, there can be non-zero small imaginary
parts as numerical errors which we remove.

\section{Non-linear Maximally Robust Classifiers}
\label{sec:nonlinear}

Recall that the definition of a maximally robust classifier
(\cref{def:max_robust})
handles non-linear families of functions, $\Phi$:
\begin{align*}
    \argmax_{\varphi\in\Phi} \{\varepsilon\,|\,y_i \varphi(\xx_i + \ddelta) > 0 
    ,~\forall i,\,\|\ddelta\| \leq \varepsilon\}\,.
\end{align*}

Here we extend the proof in
\cref{lem:max_robust_to_min_norm_linear}
that made the maximally robust classification tractable
by removing $\ddelta$ and $\varepsilon$ from the problem.
In linking a maximally robust classifier to a minimum norm
classifier when there exists a non-linear transformation, the first
step that requires attention is the following,
\begin{align*}
& \argmax_{\varphi\in\Phi}\,\{ \varepsilon : 
 \inf_{\|\ddelta\| \leq 1} y_i \varphi(\xx_i + \varepsilon\ddelta)
 > 0,~\forall i \} \quad\\
& \neq \argmax_{\varphi\in\Phi}\,\{ \varepsilon : 
 y_i \varphi(\xx_i) + \varepsilon\inf_{\|\ddelta\| \leq 1} \varphi(\ddelta)
 > 0,~\forall i \} \quad
\end{align*}

\begin{lem}[Gradient Norm Weighted Maximum  Margin]
\label{lem:nonlinear}
Let $\Phi$ be a family of 
locally linear classifiers near training data, i.e.,
\begin{align*}
\Phi&=\{\varphi : 
\exists \xi>0,
\forall i,{\|\ddelta\|\leq1},\,
{\varepsilon\in[0,\xi)},\,\\
&{\varphi(\xx_i+\varepsilon\ddelta)}=
{\varphi(\xx_i)}+
{\varepsilon\ddelta^\top\dxy{}{\xx}\varphi(\xx_i)}
\}.
\end{align*}
Then a maximally robust classifier is a solution to the following
problem,
\begin{align*}
& \argmax_{\varphi\in\Phi, \varepsilon \leq \xi}\,\{ \varepsilon : 
y_i \varphi(\xx_i)
> \varepsilon \|\dxy{}{\xx}\varphi(\xx_i)\|_\ast
,~\forall i \}\,.
\end{align*}
\end{lem}

\begin{proof}

\begin{align*}
& \argmax_\varphi\,\{ \varepsilon : 
 \inf_{\|\ddelta\| \leq 1} y_i \varphi(\xx_i + \varepsilon\ddelta)
 \geq 0,~\forall i \} \quad\\
&\qquad \text{ (Taylor approx.)}\\
& = \argmax_\varphi\,\{ \varepsilon : 
\inf_{\|\ddelta\| \leq 1} y_i \varphi(\xx_i)
+ y_i \varepsilon\ddelta \dxy{}{\xx}\varphi(\xx_i)
\geq 0,~\forall i \} \quad\\
& = \argmax_\varphi\,\{ \varepsilon : 
y_i \varphi(\xx_i)
+ \varepsilon\inf_{\|\ddelta\| \leq 1} \ddelta \dxy{}{\xx}\varphi(\xx_i)
\geq 0,~\forall i \} \quad\\
&\qquad\text{(Dual to the local derivative.)}\\
& = \argmax_\varphi\,\{ \varepsilon : 
y_i \varphi(\xx_i)
\geq \varepsilon \|\dxy{}{\xx}\varphi(\xx_i)\|_\ast
,~\forall i \} \quad\\
&\qquad\text{(Assuming constant gradient norm near data.)}\\
& = \argmax_{\varphi:\|\dxy{}{\xx}\varphi(\xx)\|_\ast\leq 1}\,
\{ \varepsilon :  y_i \varphi(\xx_i) \geq \varepsilon  ,~\forall i \}\,.
\end{align*}
\end{proof}

The equivalence in \cref{lem:nonlinear}
fails when $\Phi$ includes functions with non-zero
higher order derivatives within the $\varepsilon$
of the maximally robust classifier.
In practice, this failure manifests itself as various forms of
gradient masking or gradient obfuscation where the model
has almost zero gradient near the data but large
higher-order derivatives~\citep{athalye2018obfuscated}.

Various regularization
methods have been proposed for adversarial robustness that
penalize the gradient norm and can be studied using the framework
of maximally robust classification~\citep{ross2018improving,
simon2019first,
avery2020adversarial,
moosavi2019robustness}
Strong gradient or curvature regularization methods can
suffer from gradient masking~\citep{avery2020adversarial}.

For general family of non-linear functions,
the interplay with implicit bias
of optimization and regularization methods
remains to be characterized.
The solution to the regularized problem in
\cref{def:reg_classifier} is not necessarily unique.
In such cases, the implicit bias of the optimizer biases the
robustness.

\section{Extended Experiments}
\label{sec:exp_ext}
\subsection{Details of Linear Classification Experiments}
\label{sec:linear_ext}
For experiments with linear classifiers,
we sample $n$ training data points from
the $\N(0,\I_d)$, $d$-dimensional standard normal
distribution centered at zero.
We label data points $y=\sign(\ww^\top\xx)$,
using a ground-truth linear separator sampled from  $\N(0,\I_d)$.
For $n<d$, the generated training data is linearly
separable. This setting is similar to a number of recent
theoretical works on the implicit bias of optimization methods in
deep learning and specifically the double descent phenomenon in
generalization~\citep{montanari2019generalization, deng2019doubledescent}.
We focus on robustness against norm-bounded
attacks centered at the training data, in particular,
$\ell_2$, $\ell_\infty$,
$\ell_1$ and Fourier-$\ell_\infty$ bounded attacks.

Because the constraints and the objective
in the minimum norm linear classification problem are convex,
we can use off-the-shelf convex optimization toolbox
to find the solution for small enough $d$ and $n$. We use
the CVXPY library~\citep{diamond2016cvxpy}.
We evaluate the following approaches based on
the implicit bias of optimization:
Gradient Descent (GD), Coordinate Descent (CD),
and Sign Gradient Descent (SignGD)
on fully-connected networks as well as GD on
linear two-layer convolutional networks
(discussed in \cref{sec:implicit_bias}).
We also compare with explicit regularization methods
(discussed in \cref{sec:explicit_reg})
trained using proximal gradient methods~\citep{parikh2013proximal}.
We do not use gradient descent because $\ell_p$ norms
can be non-differentiable at some points
(e.g., $\ell_1$ and $\ell_\infty$) and we seek
a global minima of the regularized empirical risk.
We also compare with adversarial training.
As we discussed in \cref{sec:max_robust}
we need to provide the value of maximally robust $\varepsilon$
to adversarial training for finding a maximally robust
classifier. In our experiments,
we give an advantage to adversarial training by providing it with
the maximally robust $\varepsilon$. We also use the steepest
descent direction corresponding to the attack norm
to solve the inner maximization.

For regularization methods a sufficiently
small regularization coefficient achieves
maximal robustness. Adversarial training given
the maximal $\varepsilon$ also converges to the same solution.
We tune all hyper-parameters for all methods including learning rate
regularization coefficient and maximum step size in line search.
We provide a list of values in \cref{tab:hparams}.

\begin{table}[t]
    \centering
    \begin{tabular}{m{0.4\linewidth}|m{0.5\linewidth}}
    \toprule
    Hyperparameter
    & Values
    \\
    \midrule
Random seed & 0,1,2 \\
$d$ & 100 \\
$d/n$ & $1, 2, 4, 8, 16, 32$\\
Training steps & $10000$\\
Learning rate & $1\mathrm{e}{-5}$, $3\mathrm{e}{-5}$,
$1\mathrm{e}{-4}$, $3\mathrm{e}{-4}$, $1\mathrm{e}{-3}$, $3\mathrm{e}{-3}$,
$1\mathrm{e}{-2}$, $3\mathrm{e}{-2}$, $1\mathrm{e}{-1}$, $3\mathrm{e}{-1}$,
$1$, $2$, $3$, $6$, $9$, $10$, $20$, $30$, $50$\\
Reg. coefficient &
$1\mathrm{e}{-7}$, $1\mathrm{e}{-6}$, $1\mathrm{e}{-5}$, $1\mathrm{e}{-4}$,
$1\mathrm{e}{-3}$, $1\mathrm{e}{-2}$, $1\mathrm{e}{-1}$, $1$, $10$,
$3\mathrm{e}{-3}$, $5\mathrm{e}{-3}$, $3\mathrm{e}{-2}$, $5\mathrm{e}{-2}$,
$3\mathrm{e}{-1}$, $5\mathrm{e}{-1}$\\
Line search max step & $1$, $10$, $100$, $1000$\\
Adv. Train steps & 10\\
Adv. Train learning rate & 0.1\\
Runtime (line search/prox. method) & $<20$ minutes\\
Runtime (others) & $<2$ minutes\\
    \bottomrule
    \end{tabular}
    \caption{\textbf{Range of Hyperparameters.} Each run uses 2 CPU cores.}
    \label{tab:hparams}
\end{table}

\subsection{Details of CIFAR-10 experiments}

For \cref{fig:tradeoffs_cifar10_eps,fig:tradeoffs_cifar10_trades},
the model is a WRN-28-10.
We use SGD momentum (momentum set to $0.9$)
with a learning rate schedule that warms up from $0$ to LR for $10$ epochs,
then decays slowly using a cosine schedule back to zero over $200$ epochs.
LR is set to 0.1 * BS / 256, where batch size, BS, is set to $1024$.
Experiments runs on Google Cloud TPUv3 over $32$ cores.
All models are trained from scratch and uses the default
initialization from JAX/Haiku.
We use the KL loss and typical adversarial loss for adversarial training.
The inner optimization either maximizes the KL divergence (for TRADES)
or the cross-entropy loss (for AT)
and we use Adam with a step-size of $0.1$.

For the evaluation, we use $40$ PGD steps
(with Adam as the underlying optimizer and step-size $0.1$).
Instead of optimizing the cross-entropy loss, we used the margin-loss~\citep{carlini2017towards}.

For \cref{fig:cifar10_acc}, we evaluate the models in \cref{tab:cifar10_models}
against our Fourier-$\ell_p$ attack with varying $\varepsilon$ in the range
$[0, 8]\times 255$ with step size $0.5$.
We report the largest $\varepsilon$ at which
the robust test accuracy is at most $1\%$ lower than standard test accuracy
of the model. We run the attack for $20$ iterations
with no restarts and use apgd-ce, and apgd-dlr methods from AutoAttack.

\begin{table}[t]
    \centering
    \begin{tabular}{m{0.6\linewidth}|m{0.3\linewidth}}
    \toprule
    Model name & Robust training type
    \\
    \midrule
   \verb|Standard| & - \\
    \midrule
   \verb|Gowal2020Uncovering_70_16_extra| & \texttt{Linf} \\ 
   \verb|Gowal2020Uncovering_28_10_extra| & \texttt{Linf} \\
   \verb|Wu2020Adversarial_extra| & \texttt{Linf} \\
   \verb|Carmon2019Unlabeled| & \texttt{Linf} \\
   \verb|Sehwag2020Hydra| & \texttt{Linf} \\
   \verb|Gowal2020Uncovering_70_16| & \texttt{Linf} \\
   \verb|Gowal2020Uncovering_34_20| & \texttt{Linf} \\
   \verb|Wang2020Improving| & \texttt{Linf} \\
   \verb|Wu2020Adversarial| & \texttt{Linf} \\
   \verb|Hendrycks2019Using| & \texttt{Linf} \\
    \midrule
   \verb|Gowal2020Uncovering_extra| & \texttt{L2} \\ 
   \verb|Gowal2020Uncovering| & \texttt{L2} \\
   \verb|Wu2020Adversarial| & \texttt{L2} \\
   \verb|Augustin2020Adversarial| & \texttt{L2} \\
   \verb|Engstrom2019Robustness| & \texttt{L2} \\
   \verb|Rice2020Overfitting| & \texttt{L2} \\
   \verb|Rice2020Overfitting| & \texttt{L2} \\
   \verb|Rony2019Decoupling| & \texttt{L2} \\
   \verb|Ding2020MMA| & \texttt{L2} \\
    \midrule
   \verb|Hendrycks2020AugMix_ResNeXt|              & \texttt{corruptions} \\ 
   \verb|Hendrycks2020AugMix_WRN| & \texttt{corruptions} \\
   \verb|Kireev2021Effectiveness_RLATAugMixNoJSD| & \texttt{corruptions} \\
   \verb|Kireev2021Effectiveness_AugMixNoJSD| & \texttt{corruptions} \\
   \verb|Kireev2021Effectiveness_Gauss50percent| & \texttt{corruptions} \\
   \verb|Kireev2021Effectiveness_RLAT| & \texttt{corruptions} \\
   \bottomrule
    \end{tabular}
    \caption{List of models evaluated in \cref{fig:cifar10_acc}.}
    \label{tab:cifar10_models}
\end{table}
    
\subsection{Margin Figures}
\label{sec:margin}
A small gap exists between the solution found using CVXPY
compared with coordinate descent.
That is because of limited number of training iterations.
The convergence of coordinate descent
to minimum  $\ell_1$ norm solution is slower than
the convergence of gradient descent to minimum  $\ell_2$ norm
solution. 
There is also a small gap between the solution
of $\ell_1$ regularization and CVXPY. The reason is the
regularization coefficient has to be infinitesimal but
in practice numerical errors prevent us from training
using very small regularization coefficients.

\begin{figure*}[t]
\centering
\begin{subfigure}[c]{0.25\textwidth}
\includegraphics[width=\textwidth]{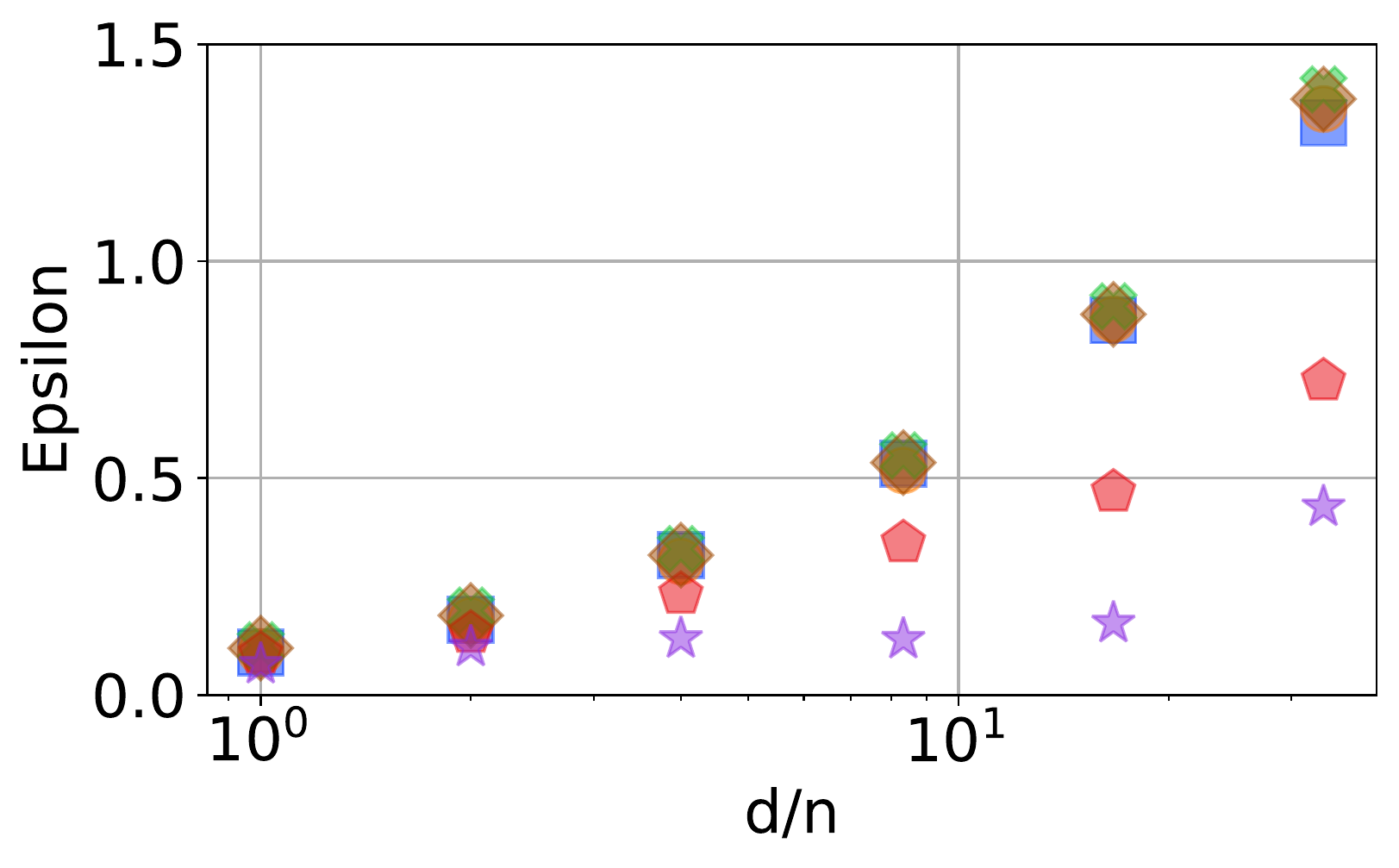}
\caption{$\ell_\infty$ attack}
\end{subfigure}
\hspace{5pt}
\begin{subfigure}[c]{0.25\textwidth}
\includegraphics[width=\textwidth]{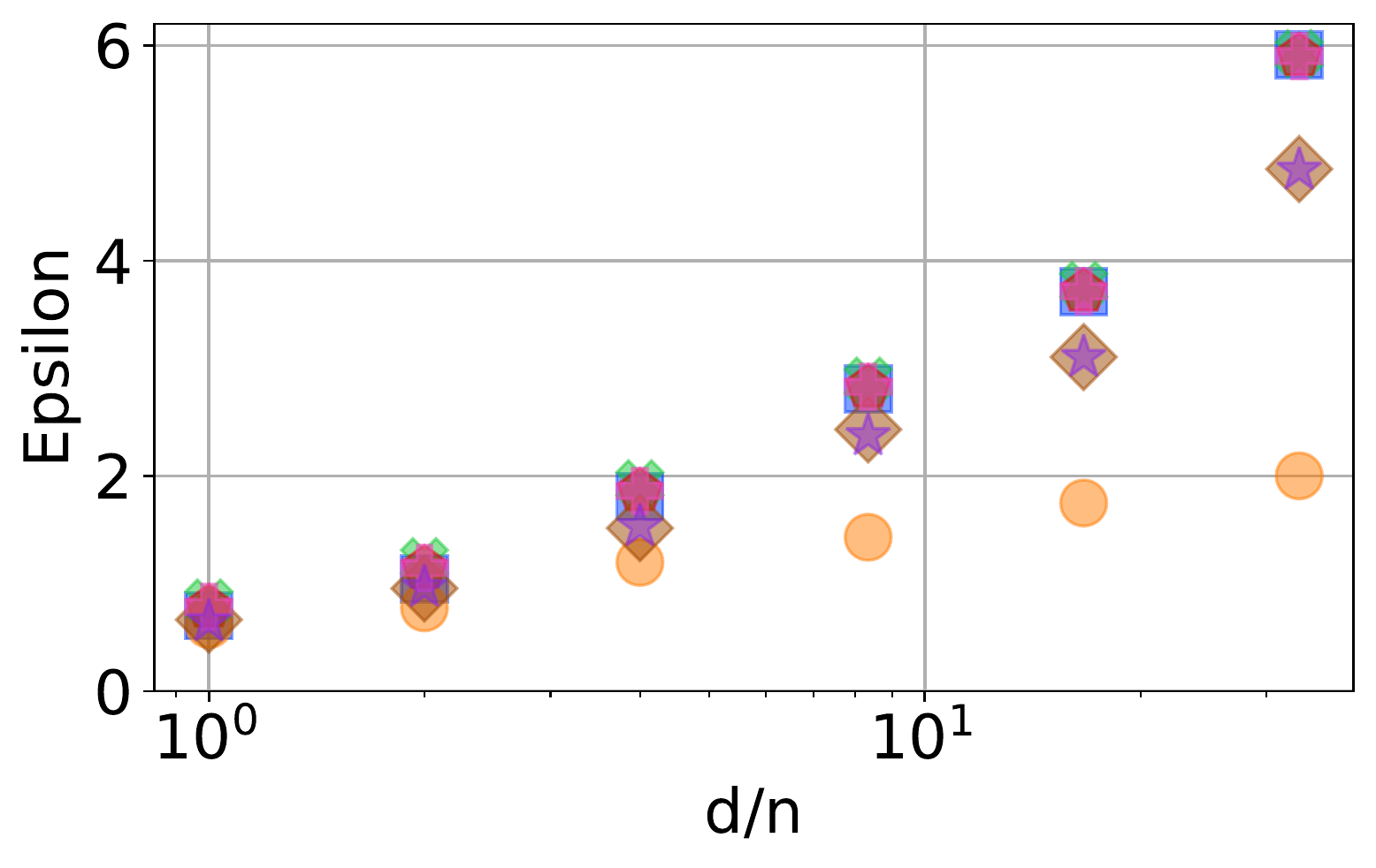}
\caption{$\ell_2$ attack}
\end{subfigure}
\hspace{5pt}
\begin{subfigure}[c]{0.25\textwidth}
\includegraphics[width=\textwidth]{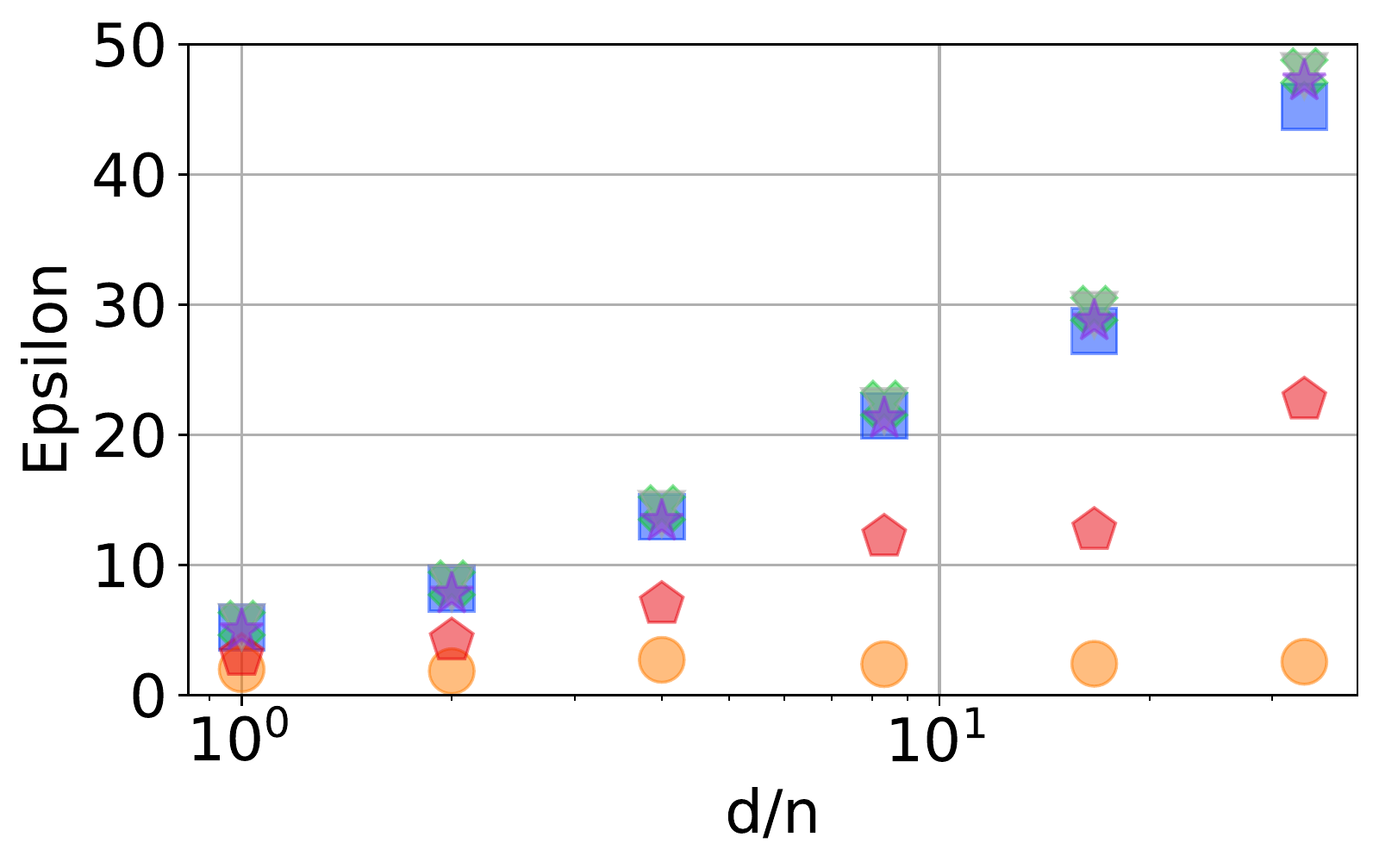}
\caption{$\ell_1$ attack}
\end{subfigure}
\hspace{5pt}
\begin{subfigure}[c]{0.1\textwidth}
\fbox{
\includegraphics[width=\textwidth]{\figdir/runs_linear_postnorm_randinit_l2/dim_num_train_risk_train_adv_l2_legend.pdf}}
\end{subfigure}

\caption{\textbf{Margin of models in \cref{fig:linear_max_eps}.}
Models are trained to be robust
against $\ell_\infty$, $\ell_2$, $\ell_1$ attacks.
For each attack, there exists one optimizer and one regularization
method that finds the maximally robust classifier. Adversarial
training also finds the solution given the maximal $\varepsilon$.}
\label{fig:linear_margin}
\end{figure*}

\begin{figure}[t]
\begin{subfigure}[b]{\linewidth}
\includegraphics[width=\textwidth]{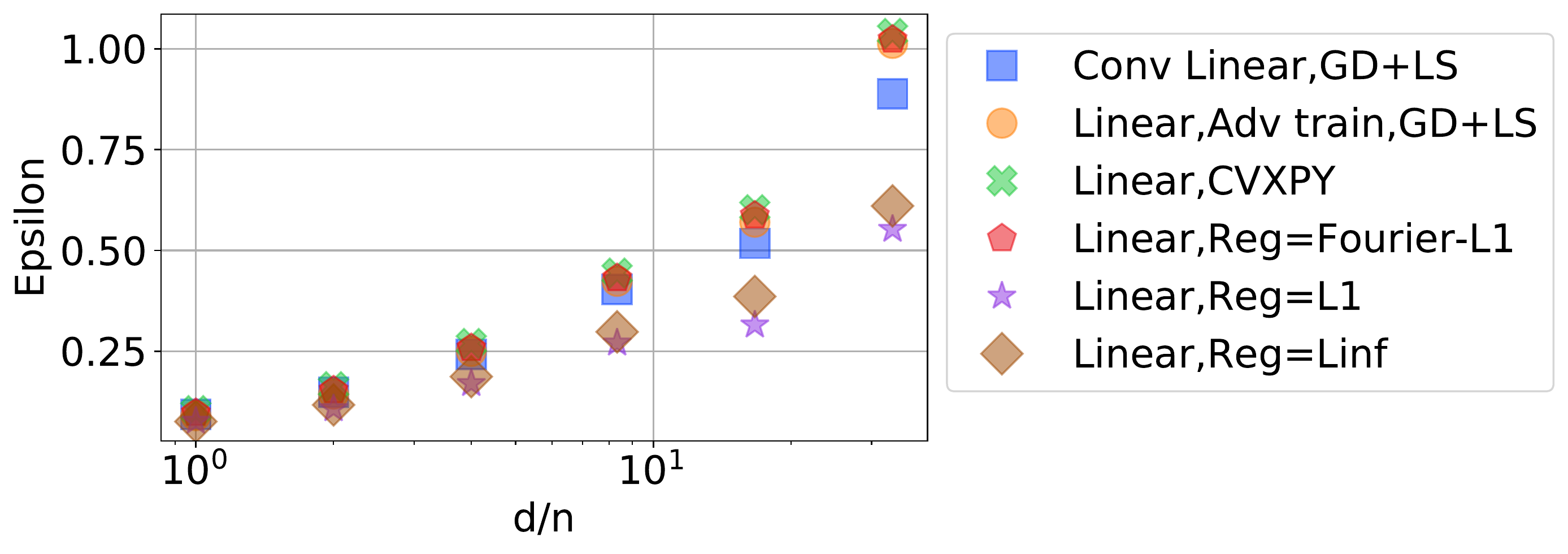}
\end{subfigure}
\caption{\textbf{Fourier-$\ell_1$ margin of
Linear Convolutional Models.}}
\label{fig:conv_margin}
\end{figure}

\subsection{Visualization of Fourier Adversarial Attacks}
\label{sec:fourier_vis}
In
\cref{fig:image_attack_standard,fig:image_attack_carmon,fig:image_attack_augustin}
we visualize adversarial samples for
models available in RobustBench~\citep{croce2020robustbench}.
Fourier-$\ell_\infty$ adversarial samples are qualitatively different
from $\ell_\infty$ adversarial samples as they concentrate on the object.

\begin{figure*}[t]
    \centering
    \begin{subfigure}[b]{\twocolfigwidth}
    \begin{tabular}{c}
    \,\,\,$\xx$\hfill
    $\xx+\ddelta$\hfill
    \,\,\,$\ddelta$ \hfill
    \,\,\,$|\mathcal{F}(\ddelta)|$\hfill\\
    \includegraphics[width=.95\textwidth]{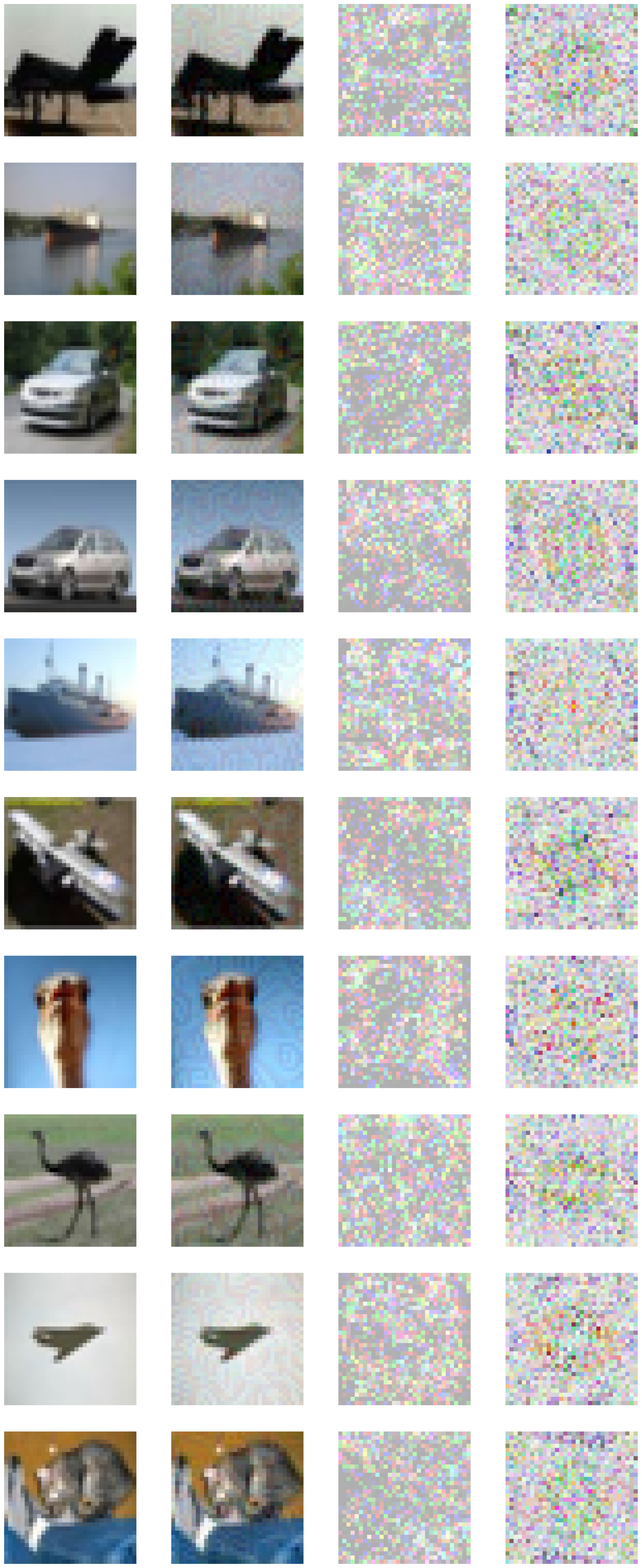}
    \end{tabular}
    \caption{$\ell_\infty$ attack}
    \end{subfigure}
    \hfill
    \begin{subfigure}[b]{\twocolfigwidth}
    \begin{tabular}{c}
    \,\,\,$\xx$\hfill
    $\xx+\ddelta$\hfill
    \,\,\,$\ddelta$ \hfill
    \,\,\,$|\mathcal{F}(\ddelta)|$\hfill\\
    \includegraphics[width=.95\textwidth]{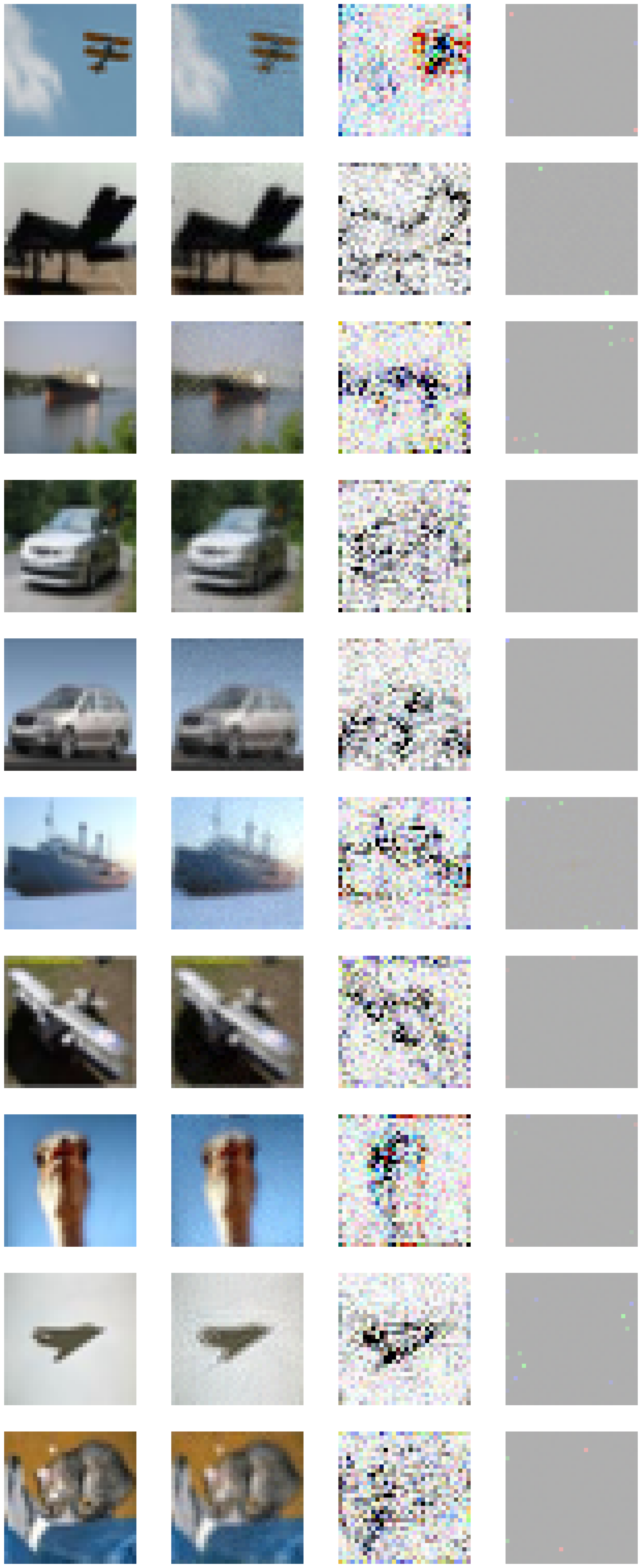}
    \end{tabular}
    \caption{Fourier-$\ell_\infty$ attack}
    \end{subfigure}
    \caption{
    \textbf{Adversarial attacks ($\ell_\infty$ and Fourier-$\ell_\infty$)
    against CIFAR-10 with standard training.}
    WideResNet-28-10 model with standard training.
    The attack methods are APGD-CE and APGD-DLR with default
    hyper-parameters in RobustBench. We use $\varepsilon=8/255$ for both attacks.
    Fourier-$\ell_\infty$ perturbations are more concentrated on
    the object.
    Darker color in perturbations means larger magnitude.
    The optimal Fourier attack step is achieved when the
    magnitude in the Fourier domain is equal to the constraints.
    }
    \label{fig:image_attack_standard}
\end{figure*}

\begin{figure*}[t]
    \centering
    \begin{subfigure}[b]{\twocolfigwidth}
    \begin{tabular}{c}
    \,\,\,$\xx$\hfill
    $\xx+\ddelta$\hfill
    \,\,\,$\ddelta$ \hfill
    \,\,\,$|\mathcal{F}(\ddelta)|$\hfill\\
    \includegraphics[width=.95\textwidth]{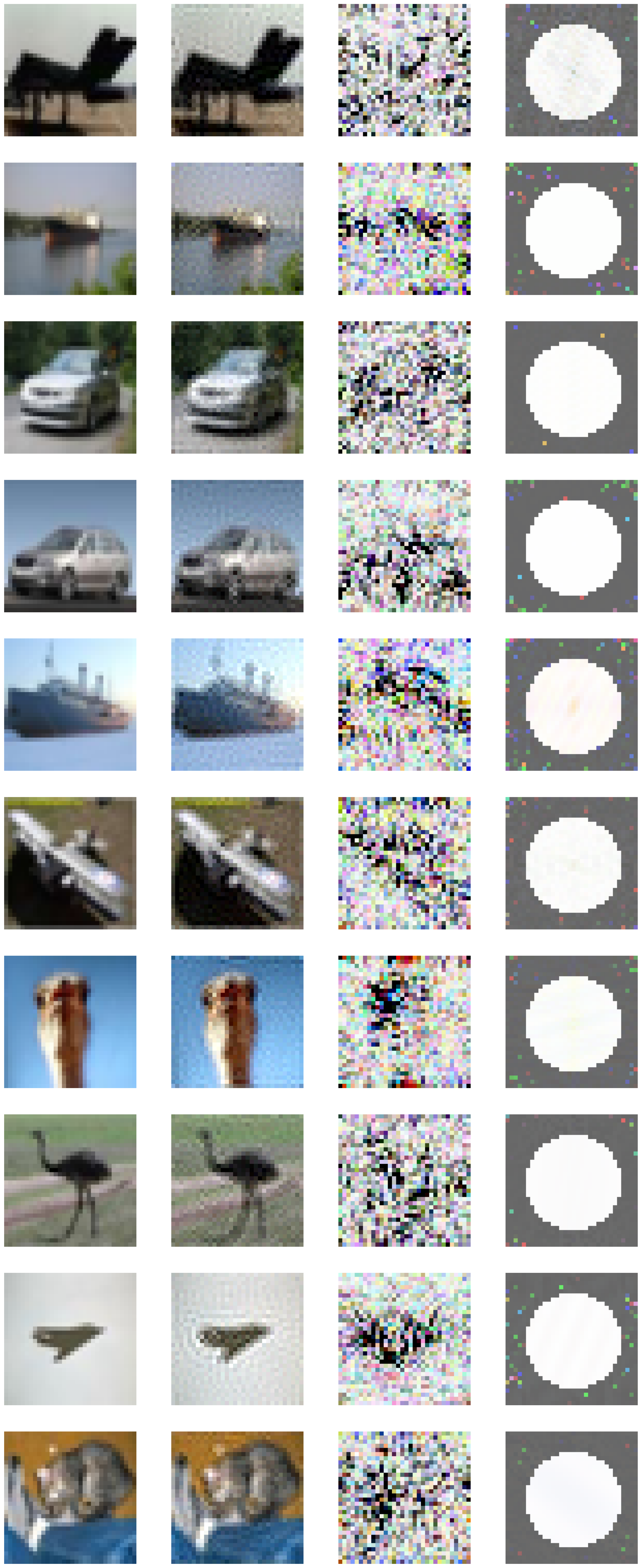}
    \end{tabular}
    \caption{$\ell_\infty$ attack}
    \end{subfigure}
    \hfill
    \begin{subfigure}[b]{\twocolfigwidth}
    \begin{tabular}{c}
    \,\,\,$\xx$\hfill
    $\xx+\ddelta$\hfill
    \,\,\,$\ddelta$ \hfill
    \,\,\,$|\mathcal{F}(\ddelta)|$\hfill\\
    \includegraphics[width=.95\textwidth]{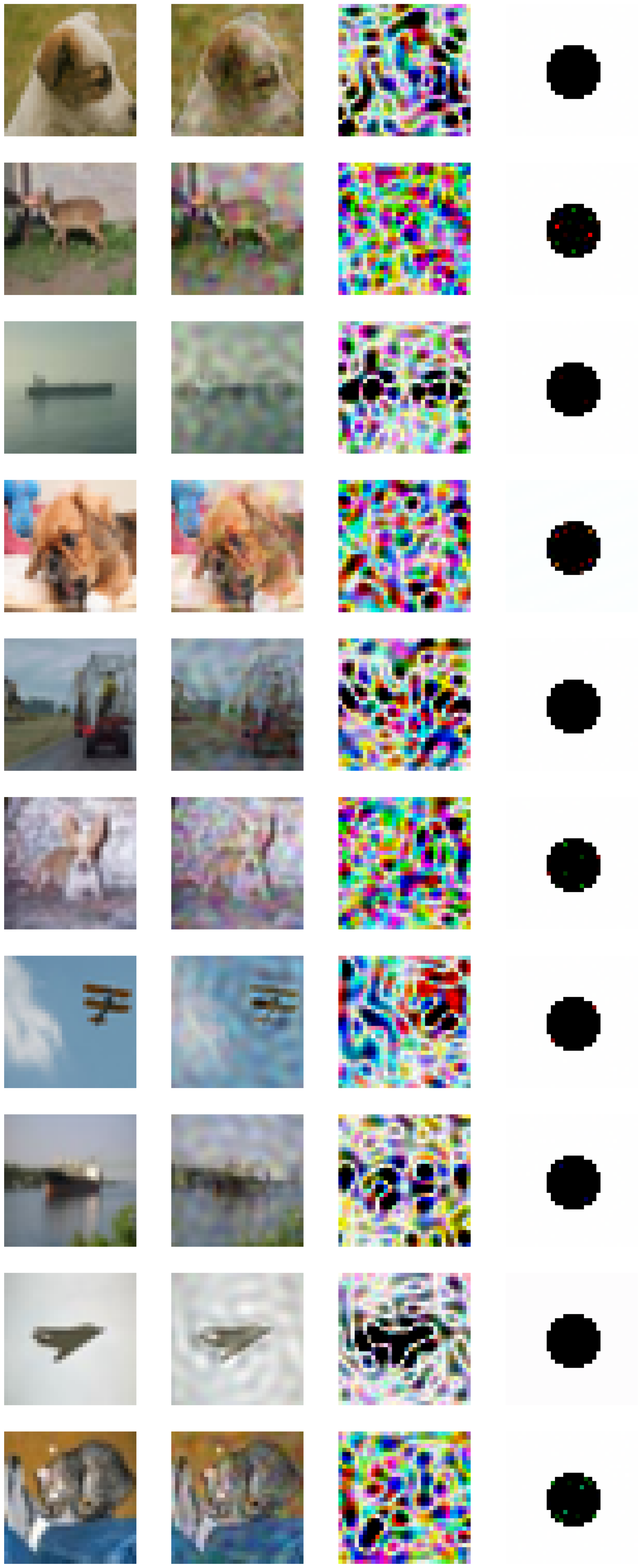}
    \end{tabular}
    \caption{Fourier-$\ell_\infty$ attack}
    \end{subfigure}
    \caption{
    \textbf{Adversarial attacks (High and low frequency Fourier-$\ell_\infty$)
    against CIFAR-10 with standard training.}
    WideResNet-28-10 model with standard training.
    The attack methods are APGD-CE and APGD-DLR with default
    hyper-parameters in RobustBench. We use $\varepsilon=15/255,45/255$ respectively for high and low frequency.
    Darker color in perturbations means larger magnitude.
    The optimal Fourier attack step is achieved when the
    magnitude in the Fourier domain is equal to the constraints.
    }
    \label{fig:image_attack_standard_band}
\end{figure*}

\begin{figure*}[t]
    \centering
    \begin{subfigure}[b]{\twocolfigwidth}
    \begin{tabular}{c}
    \,\,\,$\xx$\hfill
    $\xx+\ddelta$\hfill
    \,\,\,$\ddelta$ \hfill
    \,\,\,$|\mathcal{F}(\ddelta)|$\hfill\\
    \includegraphics[width=.95\textwidth]{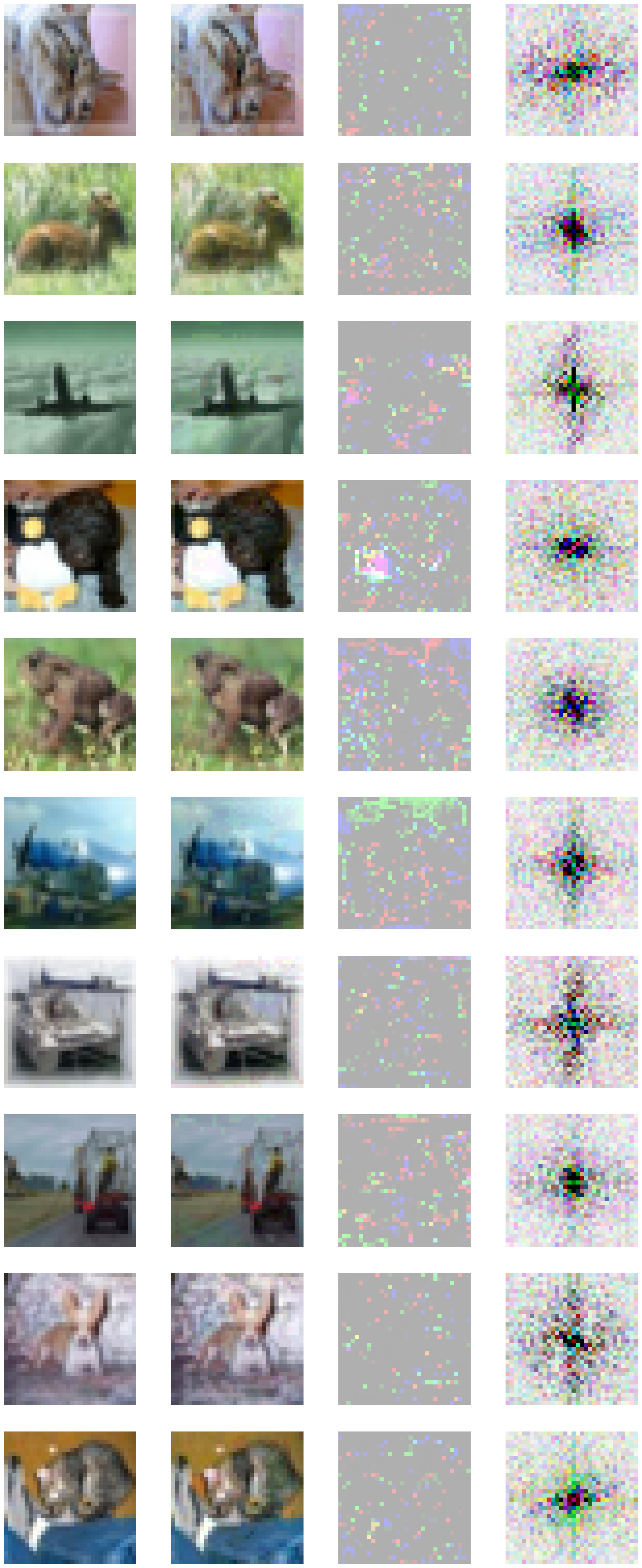}
    \end{tabular}
    \caption{$\ell_\infty$ attack}
    \end{subfigure}
    \hfill
    \begin{subfigure}[b]{\twocolfigwidth}
    \begin{tabular}{c}
    \,\,\,$\xx$\hfill
    $\xx+\ddelta$\hfill
    \,\,\,$\ddelta$ \hfill
    \,\,\,$|\mathcal{F}(\ddelta)|$\hfill\\
    \includegraphics[width=.95\textwidth]{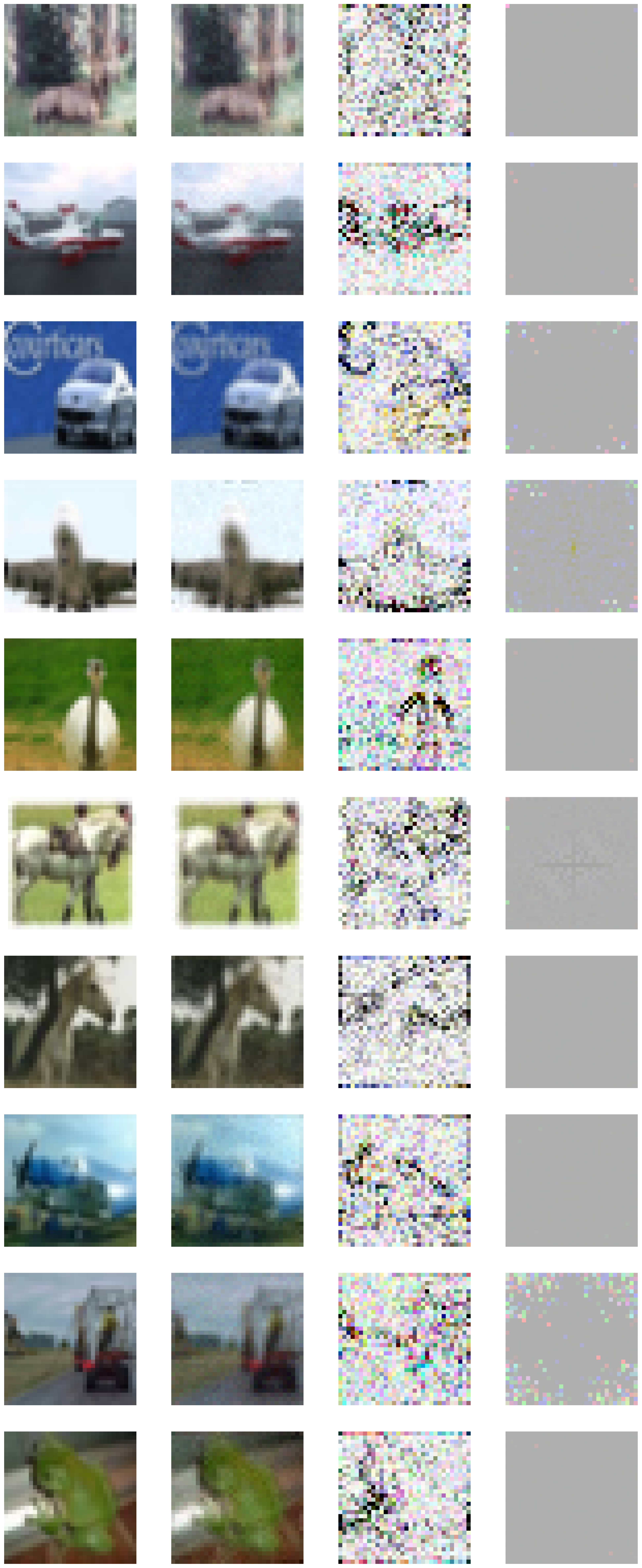}
    \end{tabular}
    \caption{Fourier-$\ell_\infty$ attack}
    \end{subfigure}
    \caption{
    \textbf{Adversarial attacks ($\ell_\infty$ and Fourier-$\ell_\infty$)
    against CIFAR-10 $\ell_\infty$ model of \citep{carmon2019unlabeled}.}
    Adversarially trained model against $\ell_\infty$ attacks.
    The attack methods are APGD-CE and APGD-DLR with default
    hyper-parameters in RobustBench. We use $\varepsilon=8/255$ for both attacks.
    Fourier-$\ell_\infty$ perturbations are more concentrated on
    the object.
    Darker color in perturbations means larger magnitude.
    The optimal Fourier attack step is achieved when the
    magnitude in the Fourier domain is equal to the constraints.
    }
    \label{fig:image_attack_carmon}
\end{figure*}

\begin{figure*}[t]
    \centering
    \begin{subfigure}[b]{\twocolfigwidth}
    \begin{tabular}{c}
    \,\,\,$\xx$\hfill
    $\xx+\ddelta$\hfill
    \,\,\,$\ddelta$ \hfill
    \,\,\,$|\mathcal{F}(\ddelta)|$\hfill\\
    \includegraphics[width=.95\textwidth]{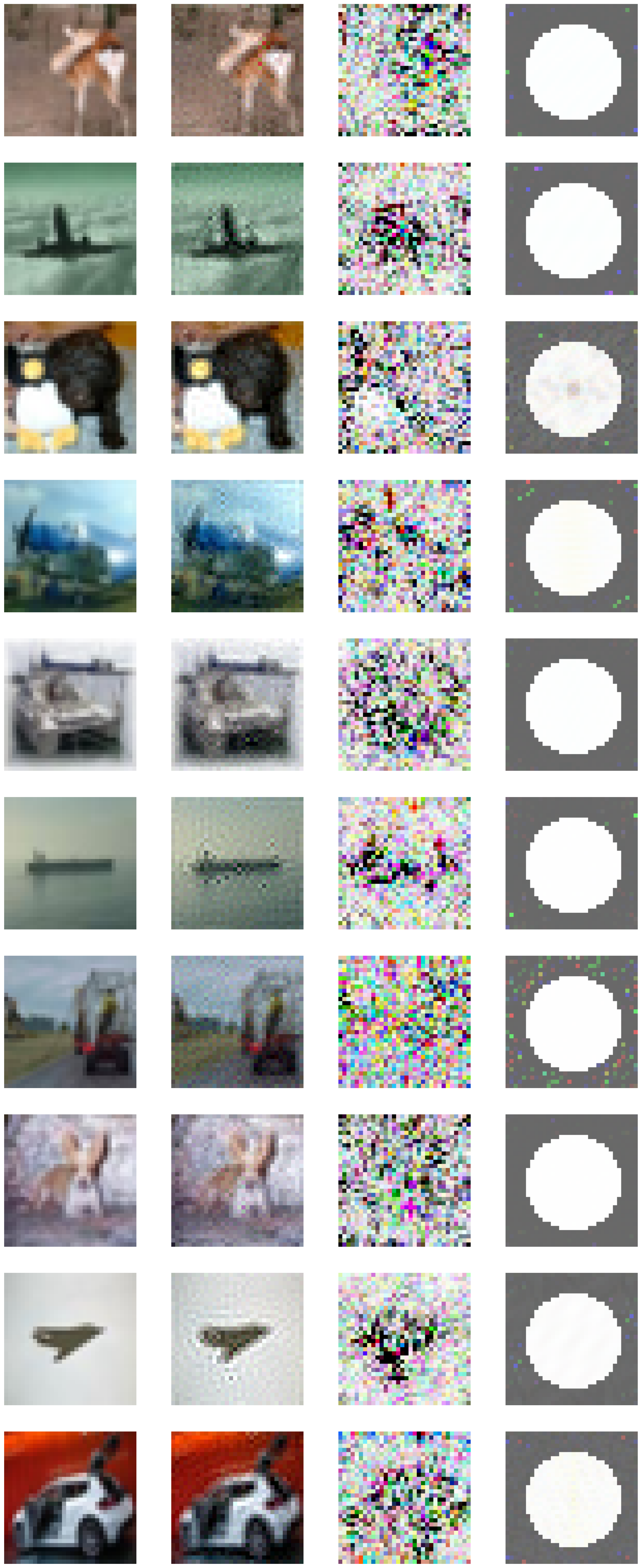}
    \end{tabular}
    \caption{$\ell_\infty$ attack}
    \end{subfigure}
    \hfill
    \begin{subfigure}[b]{\twocolfigwidth}
    \begin{tabular}{c}
    \,\,\,$\xx$\hfill
    $\xx+\ddelta$\hfill
    \,\,\,$\ddelta$ \hfill
    \,\,\,$|\mathcal{F}(\ddelta)|$\hfill\\
    \includegraphics[width=.95\textwidth]{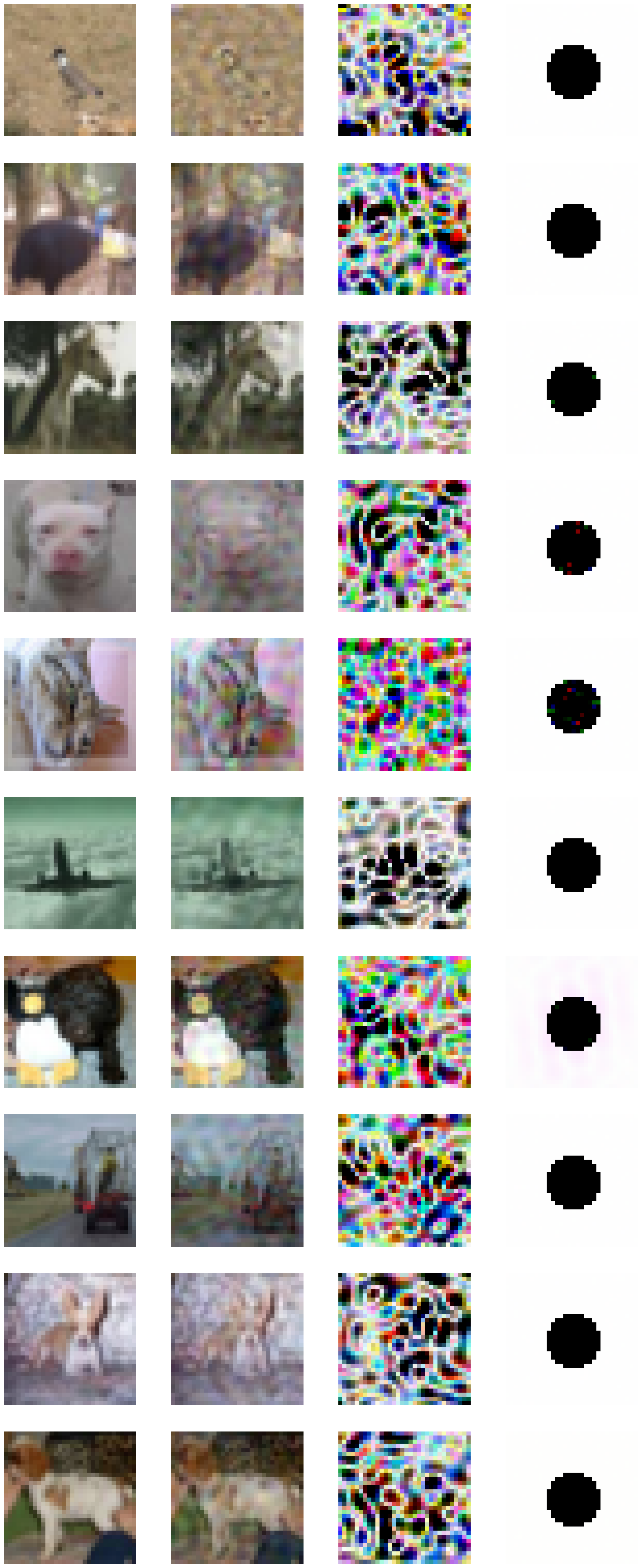}
    \end{tabular}
    \caption{Fourier-$\ell_\infty$ attack}
    \end{subfigure}
    \caption{
    \textbf{Adversarial attacks (High and low frequency Fourier-$\ell_\infty$)
    against CIFAR-10 $\ell_\infty$ model of \citep{carmon2019unlabeled}.}
    WideResNet-28-10 model with standard training.
    The attack methods are APGD-CE and APGD-DLR with default
    hyper-parameters in RobustBench. We use $\varepsilon=15/255,45/255$ respectively for high and low frequency.
    Darker color in perturbations means larger magnitude.
    The optimal Fourier attack step is achieved when the
    magnitude in the Fourier domain is equal to the constraints.
    }
    \label{fig:image_attack_carmon_band}
\end{figure*}

\begin{figure*}[t]
    \centering
    \begin{subfigure}[b]{\twocolfigwidth}
    \begin{tabular}{c}
    \,\,\,$\xx$\hfill
    $\xx+\ddelta$\hfill
    \,\,\,$\ddelta$ \hfill
    \,\,\,$|\mathcal{F}(\ddelta)|$\hfill\\
    \includegraphics[width=.95\textwidth]{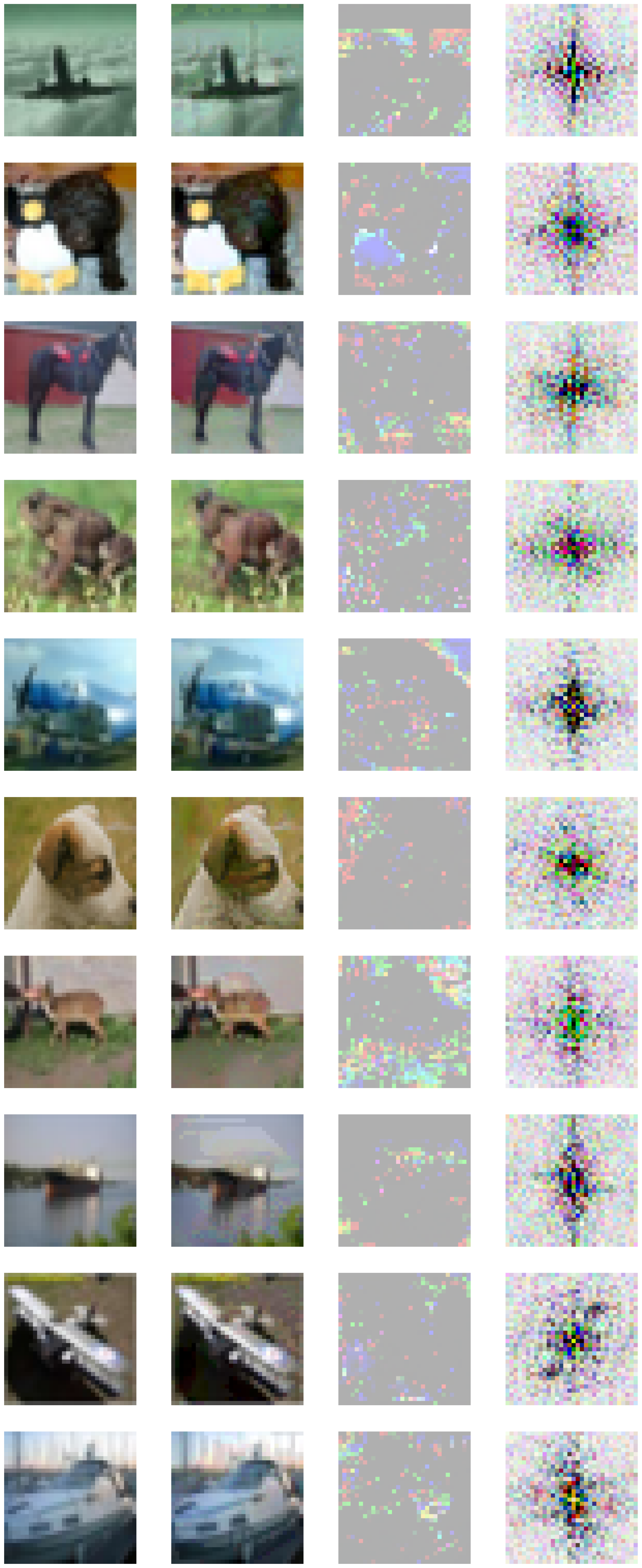}
    \end{tabular}
    \caption{$\ell_\infty$ attack}
    \end{subfigure}
    \hfill
    \begin{subfigure}[b]{\twocolfigwidth}
    \begin{tabular}{c}
    \,\,\,$\xx$\hfill
    $\xx+\ddelta$\hfill
    \,\,\,$\ddelta$ \hfill
    \,\,\,$|\mathcal{F}(\ddelta)|$\hfill\\
    \includegraphics[width=.95\textwidth]{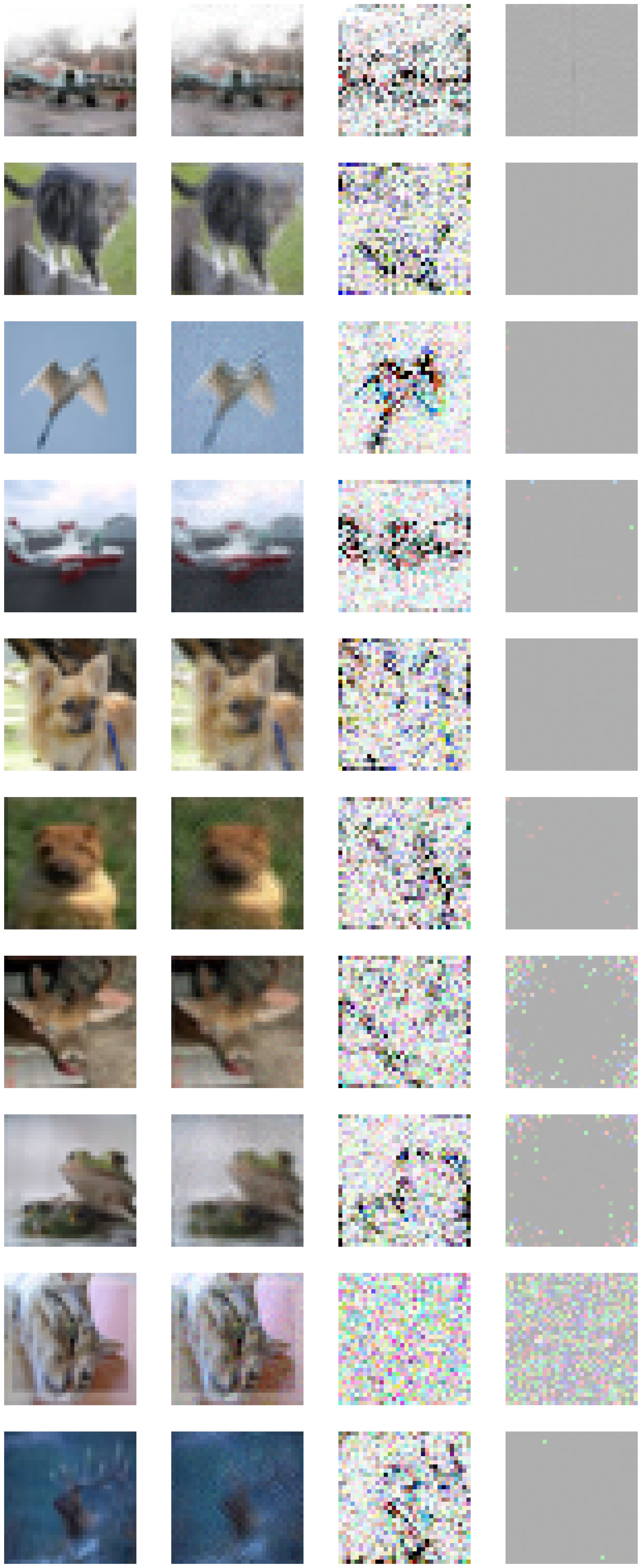}
    \end{tabular}
    \caption{Fourier-$\ell_\infty$ attack}
    \end{subfigure}
    \caption{
    \textbf{Adversarial attacks ($\ell_\infty$ and Fourier-$\ell_\infty$)
    against CIFAR-10 $\ell_2$ model of \citep{augustin2020adversarial}.}
    Adversarially trained model against $\ell_2$ attacks.
    The attack methods are APGD-CE and APGD-DLR with default
    hyper-parameters in RobustBench. We use $\varepsilon=8/255$ for both attacks.
    Fourier-$\ell_\infty$ perturbations are more concentrated on
    the object.
    Darker color in perturbations means larger magnitude.
    The optimal Fourier attack step is achieved when the
    magnitude in the Fourier domain is equal to the constraints.
    }
    \label{fig:image_attack_augustin}
\end{figure*}

\begin{figure*}[t]
    \centering
    \begin{subfigure}[b]{\twocolfigwidth}
    \begin{tabular}{c}
    \,\,\,$\xx$\hfill
    $\xx+\ddelta$\hfill
    \,\,\,$\ddelta$ \hfill
    \,\,\,$|\mathcal{F}(\ddelta)|$\hfill\\
    \includegraphics[width=.95\textwidth]{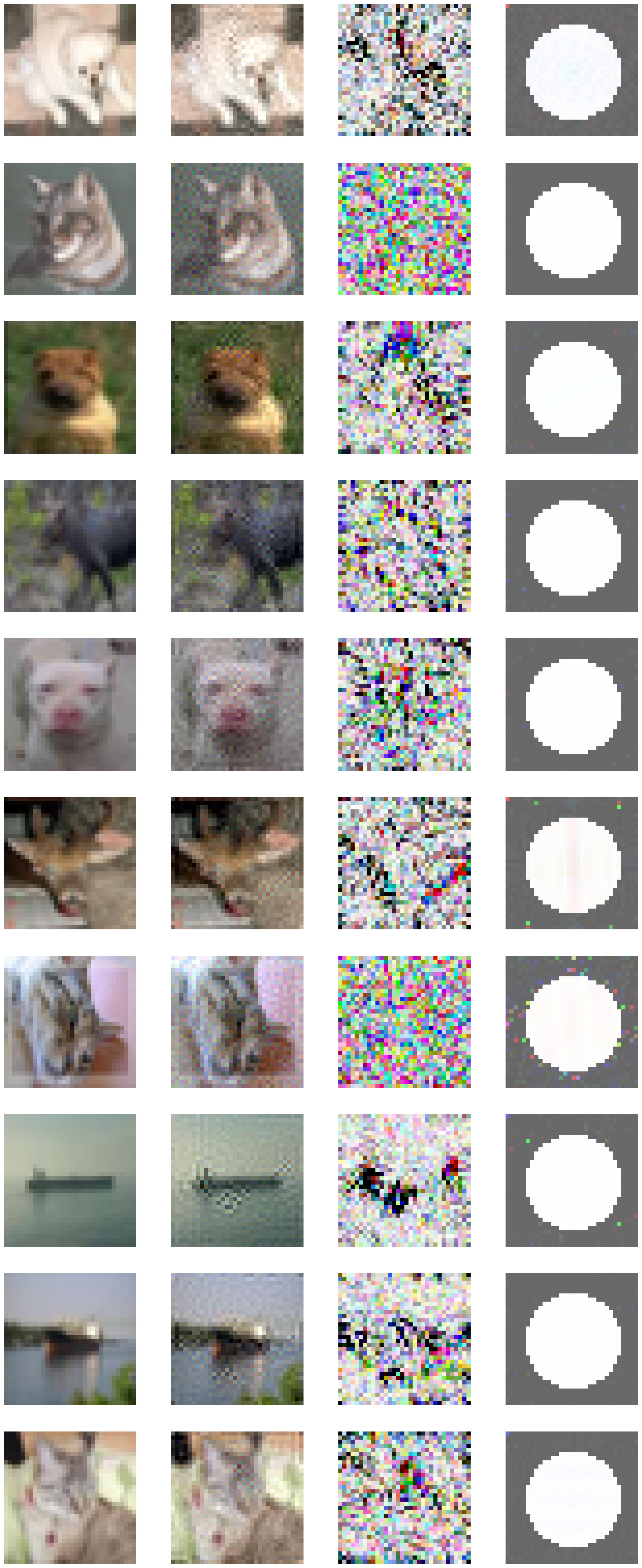}
    \end{tabular}
    \caption{$\ell_\infty$ attack}
    \end{subfigure}
    \hfill
    \begin{subfigure}[b]{\twocolfigwidth}
    \begin{tabular}{c}
    \,\,\,$\xx$\hfill
    $\xx+\ddelta$\hfill
    \,\,\,$\ddelta$ \hfill
    \,\,\,$|\mathcal{F}(\ddelta)|$\hfill\\
    \includegraphics[width=.95\textwidth]{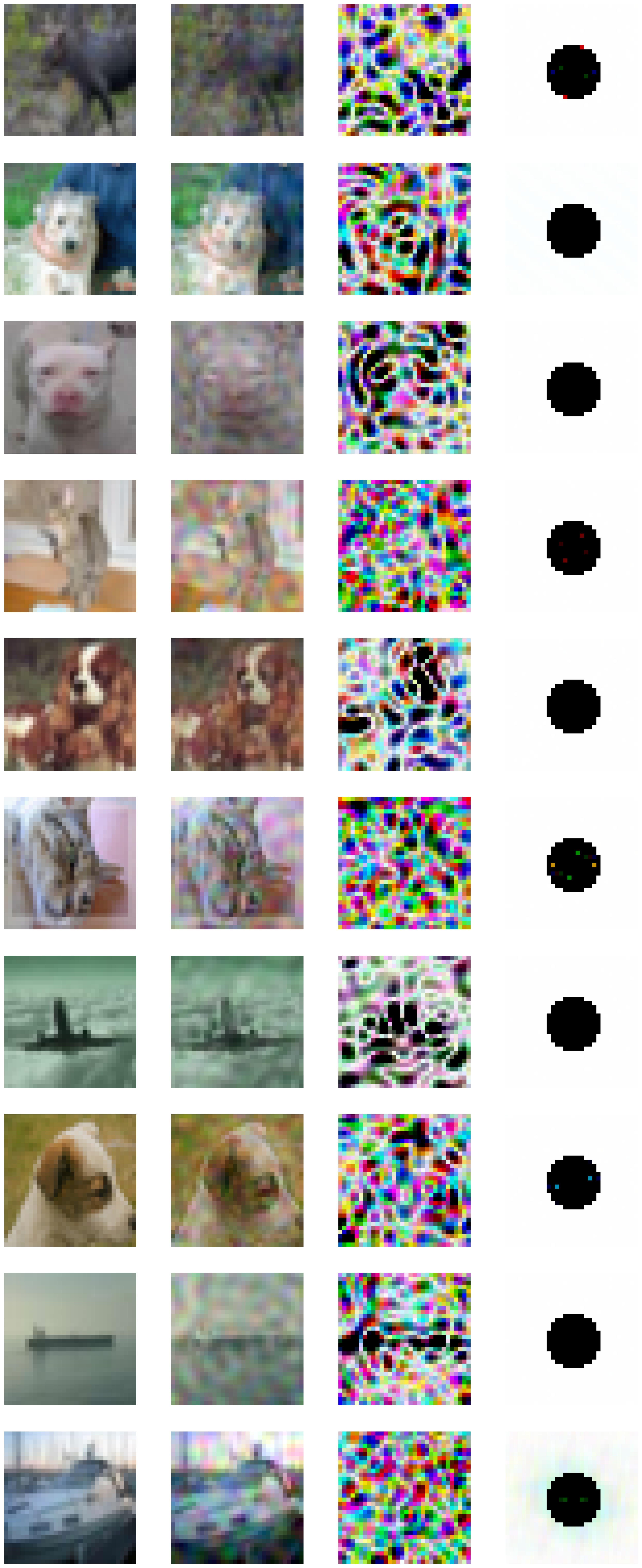}
    \end{tabular}
    \caption{Fourier-$\ell_\infty$ attack}
    \end{subfigure}
    \caption{
    \textbf{Adversarial attacks (High and low frequency Fourier-$\ell_\infty$)
    against CIFAR-10 $\ell_2$ model of \citep{augustin2020adversarial}.}
    WideResNet-28-10 model with standard training.
    The attack methods are APGD-CE and APGD-DLR with default
    hyper-parameters in RobustBench. We use $\varepsilon=15/255,45/255$ respectively for high and low frequency.
    Darker color in perturbations means larger magnitude.
    The optimal Fourier attack step is achieved when the
    magnitude in the Fourier domain is equal to the constraints.
    }
    \label{fig:image_attack_augustin_band}
\end{figure*}

\section{Visualization of Norm-balls}
\label{sec:unit_norm_balls}

To reach an intuition of the
norm-ball for Fourier $\ell_\infty$ norm, we visualize
a number of common norm-balls in $3$D in \cref{fig:unit_norm_balls}.
Norm-balls have been visualized in prior
work~\citep{bach2012structured}
but we are not aware of any visualization
of Fourier-$\ell_\infty$. 

\begin{figure}[t]
\centering
\begin{subfigure}[b]{\twocolfigwidth}
\includegraphics[width=\textwidth]{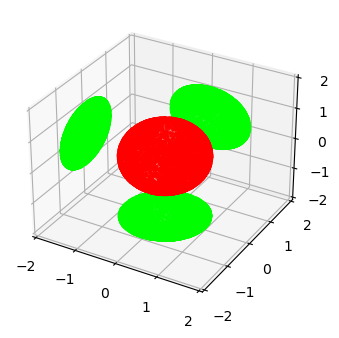}
\caption{$\|\ddelta\|_2=1$}
\end{subfigure}
\begin{subfigure}[b]{\twocolfigwidth}
\includegraphics[width=\textwidth]{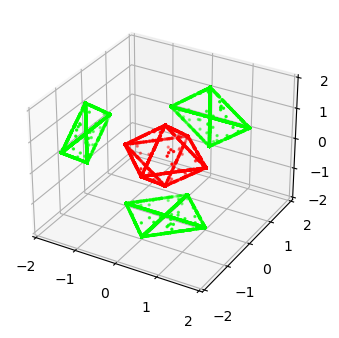}
\caption{$\|\ddelta\|_1=1$}
\end{subfigure}
\begin{subfigure}[b]{\twocolfigwidth}
\includegraphics[width=\textwidth]{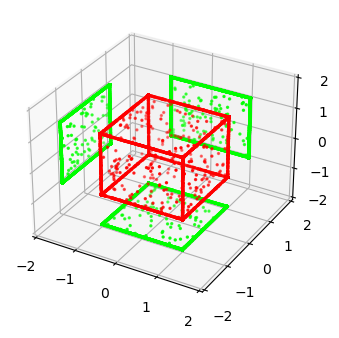}
\caption{$\|\ddelta\|_\infty=1$}
\end{subfigure}
\begin{subfigure}[b]{\twocolfigwidth}
\includegraphics[width=\textwidth]{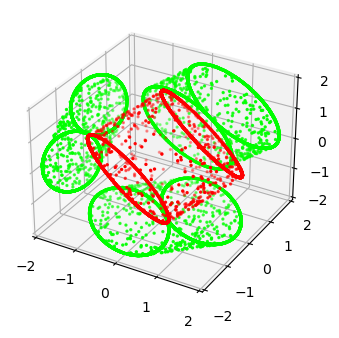}
\caption{$\|\mathcal{F}(\ddelta)\|_\infty=1$}
\end{subfigure}
\caption{Unit norm balls in $3$-D (red) and their $2$-D projections (green). Linear models trained with gradient descent
are maximally robust to $\ell_2$ perturbations.
Two-layer linear convolutional networks trained
with gradient descent are maximally robust to
perturbations with bounded Fourier-$\ell_\infty$.}
\label{fig:unit_norm_balls}
\end{figure}